\documentclass[11pt]{article}

\usepackage{amssymb}
\usepackage{ragged2e}
\usepackage{hhline}
\usepackage{lmodern}
\usepackage[utf8]{inputenc} 
\usepackage[T1]{fontenc}    
\usepackage{amsfonts}       
\usepackage{nicefrac}       
\usepackage{microtype}      
\usepackage{fullpage}
\usepackage{xcolor,colortbl}
\usepackage{appendix}
\usepackage{setspace}

\usepackage{url}            
\usepackage{booktabs}       
\usepackage{xcolor}         

\usepackage{mystyle}

\title{Traversing Pareto Optimal Policies: Provably Efficient Multi-Objective Reinforcement Learning}

\newcounter{subroutine}
\makeatletter
\newenvironment{subroutine}[1][htb]{%
	\let\c@algorithm\c@subroutine
	\renewcommand{\ALG@name}{Subroutine}
	\begin{algorithm}[#1]%
	}{\end{algorithm}
}
\makeatother

\begin{document}

\author{Shuang Qiu\thanks{Hong Kong University of Science and Technology. 
Email: \texttt{masqiu@ust.hk}} 
       \quad
        Dake Zhang\thanks{University of Chicago.
Email: \texttt{dkzhang@uchicago.edu}}    	
        \quad
		Rui Yang\thanks{Hong Kong University of Science and Technology. 
    Email: \texttt{yangrui.thu2015@gmail.com}}
	\quad    
       	Boxiang Lyu\thanks{University of Chicago.
    Email: \texttt{blyu@chicagobooth.edu}}
    \quad 
    Tong Zhang\thanks{
   University of Illinois Urbana-Champaign. 
	Email: \texttt{tongzhang@tongzhang-ml.org}}
}

\maketitle
		
	\addtocontents{toc}{\protect\setcounter{tocdepth}{-1}}
	
	\begin{abstract}
	This paper investigates multi-objective reinforcement learning (MORL), which focuses on learning Pareto optimal policies in the presence of multiple reward functions. Despite MORL's significant empirical success, there is still a lack of satisfactory understanding of various MORL optimization targets and efficient learning algorithms. Our work offers a systematic analysis of several optimization targets to assess their abilities to find all Pareto optimal policies and controllability over learned policies by the preferences for different objectives. We then identify Tchebycheff scalarization as a favorable scalarization method for MORL. Considering the non-smoothness of Tchebycheff scalarization, we reformulate its minimization problem into a new min-max-max optimization problem. Then, for the stochastic policy class, we propose efficient algorithms using this reformulation to learn Pareto optimal policies. We first propose an online UCB-based algorithm to achieve an $\varepsilon$ learning error with an $\tilde{\mathcal{O}}(\varepsilon^{-2})$ sample complexity for a single given preference. To further reduce the cost of environment exploration under different preferences, we propose a preference-free framework that first explores the environment without pre-defined preferences and then generates solutions for any number of preferences. We prove that it only requires an $\tilde{\mathcal{O}}(\varepsilon^{-2})$ exploration complexity in the exploration phase and demands no additional exploration afterward. Lastly, we analyze the smooth Tchebycheff scalarization, an extension of Tchebycheff scalarization, which is proved to be more advantageous in distinguishing the Pareto optimal policies from other weakly Pareto optimal policies based on entry values of preference vectors. Furthermore, we extend our algorithms and theoretical analysis to accommodate this optimization target.
	\end{abstract}

\section{Introduction}

Multi-objective reinforcement learning (MORL) \citep{puterman1990markov,ehrgott2005multiobjective,roijers2013survey} focuses on learning a single policy that simultaneously performs well for a collection of diverse reward functions, as opposed to one that performs well under only one reward function. This generalization of reinforcement learning (RL) has been deployed in a diverse range of tasks, including personalized recommendation systems~\citep{stamenkovic2022choosing}, grid scheduling ~\citep{perez2010multi}, 
cancer screening~\citep{yala2022optimizing}, robot control \citep{xu2020prediction,hwang2023promptable}, text generation \citep{chen2020just}, and training personalized large models for diverse human preferences~\citep{zhou2023beyond,chen2023preference,yang2024rewards,zhong2024panacea,wang2024arithmetic,guo2024controllable}. Since the multiple reward functions can be highly diverse or even in conflict, a single optimal policy for all objectives, i.e., value functions defined by different reward functions, may not exist. 
Therefore, in alignment with the general multi-objective learning problems \citep{choo1983proper,steuer1986multiple,ehrgott2005multicriteria}, MORL aims to learn Pareto optimal policies, under which no other policies can improve at least one objective's value without making other objectives worse off.

At a high level, MORL is related to multi-objective optimization~\citep{choo1983proper,steuer1986multiple,ehrgott2005multicriteria,caramia2020multi, gunantara2018review, deb2016multi, giagkiozis2015methods, riquelme2015performance,liu2021profiling, liu2021conflict, chen2023preference, mahapatra2023multi, sener2018multi, fernando2022mitigating,hu2024revisiting,chen2024three,mahapatra2020multi,xiao2024direction,lin2024smooth,jiang2023multiobjective}, which focuses on learning Pareto optimal solutions based on various optimization techniques such as the first-order methods. But these approaches are difficult to apply to MORL due to the special problem structures of RL. In addition, there is a line of research concentrating on the multi-objective bandit problems, including multi-arm bandit,  contextual bandit, or the generalized linear bandit \citep{drugan2013designing,turgay2018multi, lu2019multi}. These works apply the Pareto suboptimality gap as their optimization target. However, we note that its implication on Pareto optimality remains vague in theory. Moreover, the learned solution is uncontrollable and thus can be an arbitrary Pareto optimal one.  
Recently, there has been a line of research studying MORL by proposing provable algorithms with theoretical guarantees. \citep{yu2021provably} studies a competitive MORL setting, which is different from the topic of this paper. \citep{zhou2022anchor,wu2021offline} considers general scalarization functions integrating the multiple value functions in online and offline settings respectively. \citep{wu2021accommodating} studies linear scalarization of objectives but considers time-varying learner preferences. Nevertheless, none of these works investigates MORL from the perspective of Pareto optimal policy learning, and thus, their methods lack guarantees of achieving Pareto optimality. Therefore, it remains elusive how to design provable MORL algorithms that can approximate all Pareto optimal policies. In practice, the empirical MORL works focus on learning all Pareto optimal policies with constructing a mapping from a preference for different objectives to the learned solutions such that the learning process is controllable. The recent work \citet{lu2022multi} theoretically shows that optimizing the linearly scalarized objective via different weights can find all Pareto optimal policies for MORL. We note that this result is only restricted to the stochastic policy class and will not generally hold for the deterministic policy class. Moreover, as discussed later in our work, even within a stochastic policy, the solutions to the maximization problem of the linearly scalarized objective are less controllable in some situations.
Motivated by both theoretical and practical considerations, our work aims to answer the following crucial open question:

\vspace{-0.2cm}
\begin{center}
    \emph{Can we design provably efficient multi-objective reinforcement learning algorithms \\that can traverse all Pareto optimal policies 
    	in a controllable way?}
\end{center}
\vspace{-0.2cm}

The above question poses several critical challenges: First, the demand of learning toward Pareto optimality necessitates the investigation of a suitable optimization target to guarantee the full coverage of Pareto optimal policies; Second, it remains unclear what the practically and theoretically sound approach will be to integrate the learner-specified preference on different reward functions into the model to guide the learning process; Finally, it is even challenging to design a provably efficient algorithm that can learn Pareto optimal policies associated with all learner preferences but via exploring the environment only once. As an initial step toward answering the above question via tackling those challenges, our work conducts a systematic analysis of several primary optimization targets and identifies a favorable scalarization method for MORL by which we can traverse all Pareto optimal policies controlled by learner preferences. We reformulate this optimization target and propose efficient algorithms that can learn all Pareto policies with environment exploration even only once.

\vspace{5pt}
\noindent\textbf{Contribution.} Our major contributions are summarized below:
\begin{itemize}[itemsep=1pt, topsep=1pt,parsep=1pt,leftmargin=*]
	\item Our work first systematically analyzes three major multi-objective optimization targets: linear scalarization, Pareto suboptimality gap, and Tchebycheff scalarization. We rigorously show that: \textbf{(1)} Linear scalarization cannot always find all Pareto optimal policies for the deterministic policy class. In addition, for the deterministic policy class, despite the coverage of all Pareto optimal policies, the maximizers of linear scalarization are less controllable w.r.t. learners' preferences. \textbf{(2)} The zeros of the Pareto suboptimality gap correspond to (weak) Pareto optimality, but the Pareto suboptimality gap metric lacks control of its solutions by learners' preferences. \textbf{(3)} The minimizers of Tchebycheff scalarization can be better controlled by learners' preferences, and those minimizers under different preferences can cover all (weakly) Pareto optimal policies. These findings motivate us to apply Tchebycheff scalarization as a suitable metric for MORL. 
	\item Although Tchebycheff scalarization has such favorable properties, it is formally non-smooth and thus hinders its direct optimization. To address this issue, we reformulate the minimization of Tchebycheff scalarization into a min-max-max problem. For the stochastic policy class, we propose an upper confidence bound (UCB)-based algorithm featuring an alternating update of policies and intermediate weights to solve this min-max-max problem in the online setting. Under a given learner preference, we prove that the proposed algorithm can efficiently find a (weakly) Pareto optimal policy with an $\tilde{\cO}(\varepsilon^{-2})$ sample complexity for achieving $\varepsilon$-minimization error of Tchebycheff scalarization.
	\item Nevertheless, the online algorithm needs to explore the environment whenever a new preference is introduced, which can be significantly costly when there are numerous preferences to consider, as interacting with the environment is typically expensive in real-world situations. To address this issue, we propose a preference-free framework featuring decoupled exploration and planning phases to learn Pareto optimal stochastic policies. The agent first thoroughly explores the environment to gather trajectories guided by both reward and transition estimation uncertainty without relying on any pre-defined preferences. Then, using the pre-collected data, solutions can be generated with any number of preferences, requiring no further exploration. We show that with only $\tilde \cO(\varepsilon^{-2})$ rounds of environment exploration, the $\varepsilon$-learning error can be achieved for any given preferences.
	\item Finally, we analyze an extension of the Tchebycheff scalarization, named smooth Tchebycheff scalarization. 
	We prove that smooth Tchebycheff scalarization exhibits a more advantageous property compared to the original Tchebycheff scalarization for MORL, i.e., the Pareto optimal policies can be differentiated from other weakly Pareto optimal policies based on the entry values of the preference vectors.
	We further reformulate smooth Tchebycheff scalarization into a new form, which better fits the UCB-based algorithmic design. Based on this reformulation, we propose an efficient online algorithm and a preference-free framework for MORL inspired by our algorithms for Tchebycheff scalarization. We prove that the proposed algorithms exhibit faster learning rates in the non-dominating terms compared to those for Tchebycheff scalarization.
\end{itemize}
Overall, our work contributes to an improved understanding of scalarization methods for MORL and offers efficient learning algorithms with theoretical guarantees. Furthermore, some of our theoretical analyses are sufficiently general to be extended to other multi-objective learning problems beyond MORL, such as multi-objective stochastic optimization.



\vspace{5pt}
\noindent\textbf{Related Work.} 
Our work is related to a long line of works on multi-objective optimization, e.g., \citep{choo1983proper,steuer1986multiple,geoffrion1968proper,ehrgott2005multicriteria,bowman1976relationship,miettinen1999nonlinear,caramia2020multi,gunantara2018review, deb2016multi, giagkiozis2015methods, riquelme2015performance,das1997closer,liu2021profiling, liu2021conflict, chen2023preference,mahapatra2023multi, sener2018multi,klamroth2007constrained,kasimbeyli2019comparison,fernando2022mitigating,hu2024revisiting,chen2024three,mahapatra2020multi,xiao2024direction,lin2024smooth}, which have explored various scalarization methods including linear scalarization and Tchebycheff scalarization. 
However, these works on multi-objective optimization cannot be applied to the MORL setting that our work considers. Among these works, the recent work \citep{lin2024smooth} studies a smoothed version of Tchebycheff scalarization, which is different from the original Tchebycheff scalarization formulation, and proposes a gradient-based optimization algorithm. In our work, we adapt the smooth Tchebycheff scalarization to MORL and further propose a reformulation that can better fit the algorithmic design and theoretical analysis of RL. Moreover, A strand of literature extends multi-objective optimization to the online learning setting, including online convex optimization and bandit problems~\citep{drugan2013designing, yahyaa2014scalarized, turgay2018multi, lu2019multi, tekin2018multi, busa2017multi, yahyaa2014annealing, jiang2023multiobjective}. Specifically, \citep{drugan2013designing,turgay2018multi, lu2019multi} consider learning Pareto optimal arms in the multi-objective multi-armed bandit, contextual bandit, and the generalized linear bandit settings, respectively, utilizing the Pareto suboptimality gap as an optimization target.~\citep{jiang2023multiobjective} further generalizes these works to online convex optimization. In spite of the wide application of the Pareto suboptimality gap, its implication on (weak) Pareto optimality remains vague in theory. Our work further provides rigorous proof to justify this implication.

In addition, there have been a rich body of works studying MORL ~\citep{roijers2013survey,ehrgott2005multiobjective,puterman1990markov,agarwal2022multi,van2013hypervolume,natarajan2005dynamic,wang2013hypervolume,barrett2008learning,pirotta2015multi,van2013scalarized,xu2020prediction,hayes2022practical,van2013scalarized, van2014multi,chen2019meta,yang2019generalized, wiering2014model,zhu2023scaling,wu2021offline, yu2021provably, wu2021accommodating,zhou2022anchor,li2020deep,lu2022multi}, which have studied different scalarization methods including linear scalarization and Tchebycheff scalarization. From a theoretical perspective, \citep{yu2021provably} studies multi-objective reinforcement learning in a competitive setup, which is beyond the scope of this paper. In addition, \citep{wu2021offline,zhou2022anchor} considers a general optimization target that scalarizes the multiple value functions together in either online or offline settings, which, nevertheless, are not capable of not covering the study of Tchebycheff scalarization as in our work. Moreover, \citep{wu2021accommodating} studies linear scalarization of objectives but considers a time-varying setting with adversarial learner preferences. However, these theoretical works do not investigate MORL from the perspective of Pareto optimal policy learning. Thus, there is no guarantee that their solutions are approximately Pareto optimal. In addition, the work \citet{lu2022multi} shows that for MORL with a stochastic policy class, linear scalarization is able to find all Pareto optimal policies. However, when the policy class is a deterministic policy class, our work shows by a concrete example (Appendix \ref{sec:proof-linear-comb}) that linear scalarization is not sufficient. In addition, as discussed in Section \ref{sec:metric}, even for a stochastic policy class, the solutions to the maximization of linear scalarization are less controllable w.r.t. learners' preferences on objectives in some situations. Our work steps forward to analyze the application of several common scalarization methods in MORL and identify (smooth) Tchebycheff scalarization as a favorable method that can find all Pareto optimal policies in a more controllable manner for both stochastic and deterministic policy classes.

Our preference-free framework is closely related to the reward-free RL approach \citep{wang2020reward,qiu2021reward,jin2020reward,zhang2023optimal,zhang2021near,qiao2022near,chen2022statistical,miryoosefi2022simple,cheng2023improved,modi2024model}. The reward-free RL studies a framework where the agent conducts the exploration first without any reward function, and then the full reward function is given in the planning phase for policy learning. The MORL work \citep{wu2021accommodating} also proposes a preference-free algorithm. However, the algorithm is very similar to the reward-free method as the full reward function is also directly given in the planning phase. In contrast, reward functions in our preference-free method are estimated through data collected in the exploration phase, which thus generalizes the reward-free framework. Please see Remark \ref{re:reward-free} for a detailed discussion.

\vspace{5pt}
\noindent\textbf{Notation.} Define $[n]:=\{1,2,\ldots,n\}$. Let $(x_i)_{i=1}^n:=(x_1,x_2, \cdots,x_n)_{i=1}^n$ be a vector with its entries indexed from $1$ to $n$ and $\{x_i\}_{i=1}^n:=\{x_1,x_2, \cdots,x_n\}_{i=1}^n$ be a set with its elements. We define $\bx \odot\by:=(x_1y_1, x_2 y_2, \cdots,x_n y_n)$ for any two vectors $\bx=(x_1,\cdots,x_n),\by=(y_1,\cdots,y_n)\in\RR^n$. Define $\{\cdot\}_{[x,y]}:=\max\{\min\{\cdot,y\},x\}$ if $x\leq y$, i.e., casting a value between $x$ and $y$. We let $\Delta_n:=\{\bx=(x_1,\cdots,x_n)\in\RR^n\given \sum_{i=1}^n x_i = 1, 0\leq x_i\leq1\}$ be a probability simplex in $\RR^n$. In addition, we let $\Delta_m^o:=\{\bx=(x_1,\cdots,x_n)\in\RR^n\given \sum_{i=1}^n x_i = 1,  x_i>0\}$, which is the relative interior of the probability simplex $\Delta_n$. Across this paper, we let $\Pi^*_{\mathrm{P}}$ be the set of all Pareto optimal policies and $\Pi^*_{\mathrm{W}}$ be the set of all weakly Pareto optimal policies.

\section{Problem Formulation} \label{sec:formulate}

\textbf{Multi-Objective Markov Decision Process.} We consider an episodic multi-objective Markov decision process (MOMDP) characterized by a tuple $(\cS,\cA, H, m, \PP, \br)$, where $\cS$ is a finite state space, $\cA$ is a finite action space, $H$ is the length of an episode, $m$ is the number of objectives. We define the transition kernel by $\PP:=\{\PP_h\}_{h=1}^H$ with $\PP_h:\cS\times\cA\times\cS\mapsto[0, 1]$ such that $\PP_h(s'|s,a)$ denotes the probability of the agent transitioning to state $s'\in\cS$ from state $s\in\cS$ by taking action $a\in\cA$ at step $h\in[H]$. The reward function $\br$ is comprised of $m$ components, i.e., $\br=(r_1,r_2,\cdots,r_m)$, which are $m$ reward functions associated with $m$ learning objectives. We further define $r_i := \{r_{i,h}\}_{h=1}^H$ where $r_{i,h}:\cS\times\cA\mapsto [0,1]$ such that $r_{i,h}(s,a)$ denotes the reward for the $i$-th objective when the agent takes action $a\in\cA$ at state $s\in\cS$ at step $h\in[H]$. For simplicity, we assume that the interaction with the environment always starts from a fixed initial state $s_1$. When $m=1$, the MOMDP reduces to the single-objective MDP. This work assumes that the true reward function $\br$ and transition $\PP$ are \emph{unknown} and should be learned from observations. The observed reward $r_{i,h}^t\in[0,1]$ at time $t$ is assumed to stochastic and has an expectation of $r_{i,h}$, i.e., $\EE[r_{i,h}^t] = r_{i,h}$.

\vspace{5pt}
\noindent\textbf{Value Function.} We define a policy as $\pi:=\{\pi_h\}_{h=1}^H$ with $\pi_h:\cS\times\cA\mapsto [0,1]$ so that $\pi_h(a|s)$ represents the probability of taking an action $a\in\cA$ given state $s\in\cS$ at step $h\in[H]$. The policy $\pi$ lies in a policy space $\Pi$, which can be either a stochastic policy space or a deterministic policy space.
If $\pi$ is in a deterministic policy space, then the agent at each state takes a certain action with probability $1$ and others with probability $0$.
Next, we define the value function $V_{i,h}^{\pi}:\cS\mapsto[0,H-h+1]$ for the $i$-th objective as $V_{i,h}^{\pi}(s):=\EE [ \sum_{h'=h}^H r_{i,h'}(s_{h'},a_{h'}) \given s_h=s,\pi,\PP ]$. The associated Q-function $Q_{i,h}^{\pi}:\cS\times\cA\mapsto[0,H-h+1]$ is defined as $Q_{i,h}^{\pi}(s,a):=\EE [\sum_{h'=h}^H r_{i,h'}(s_{h'},a_{h'}) \given s_h=s,a_h=a,\pi,\PP ]$. Letting $\bV_h^{\pi}(s) = (V_{1,h}^{\pi}(s), V_{2,h}^{\pi}(s), \cdots, V_{m,h}^{\pi}(s))$ and	$\bQ_h^{\pi}(s,a) = (Q_{1,h}^{\pi}(s,a), Q_{2,h}^{\pi}(s,a), \cdots, Q_{m,h}^{\pi}(s,a))$ be the value function and Q-function vectors for MOMDPs, we have the following Bellman equation:
\begin{align}
	\begin{aligned}\label{eq:bellman}
	\bV_h^{\pi}(s) = \sum_{a\in\cA}\bQ_h^{\pi}(s,a)\pi_h(a|s), \quad   \bQ_h^{\pi}(s,a) = \br_h(s,a) + \sum_{s'\in \cS}\PP_h(s'|s,a) \bV_{h+1}^{\pi}(s') , 		
	\end{aligned} 
\end{align}
where $\br_h(s,a) := (r_{1,h}(s,a), r_{2,h}(s,a), \cdots,r_{m,h}(s,a))$. Hereafter, for abbreviation, we denote $\PP_h V(s,a) := \sum_{s'\in\cS}\PP_h(s'|s,a)V(s')$ for any $V:\cS\mapsto[0,H]$ throughout this paper. In addition, under the MORL setting, we refer to the ($i$-th) objective as the ($i$-th) value function $V_{i,1}^\pi(s_1)$ associated with the reward function $r_i$. 

\section{Learning Goal of Multi-Objective RL} \label{sec:pareto}


In this section, we revisit some fundamental definitions and properties in multi-objective optimization and MORL.
\vspace{5pt}
\noindent\textbf{Pareto Optimality.}
In a single-objective RL problem, we aim to find an optimal policy to maximize the value function $V_1^\pi(s_1)$ defined under a single reward function, i.e.,  $\max_{\pi\in\Pi} V_1^\pi(s_1)$. Following this intuition, a straightforward extension from single-objective RL to MORL could be finding a single optimal policy $\pi$ which is expected to simultaneously maximize all objectives, i.e., 
\begin{align*}
	\max_{\pi\in\Pi} \Big\{\bV_1^\pi(s_1):= \big(V_{1,1}^\pi(s_1), V_{2,1}^\pi(s_1), \cdots, V_{m,1}^\pi(s_1) \big)\Big\}.
\end{align*} 
However, such a single optimal policy in general does not exist since those reward functions are typically diverse and even conflicting. Hence the optimal policy for each objective could be largely different from each other and eventually no single policy will maximize all objectives concurrently. 
In general multi-objective learning, finding the Pareto optimal solutions rather than a (possibly nonexistent) single global optimum solution becomes the learning goal. Thus, we turn to finding the \emph{Pareto optimal policies}, which is regarded as the learning goal for MORL.

Formally, the Pareto optimal policy for MORL based on an MOMDP is defined as follows:
\begin{definition}[Pareto Optimal Policy] \label{def:pareto} For any two policies $\pi\in\Pi$ and $\pi'\in\Pi$ , we say $\pi'$ \emph{dominates} $\pi$ if and only if $V_{i,1}^{\pi'}(s_1) \geq V_{i,1}^{\pi}(s_1)$ for all $i\in[m]$ and there exists at least one $j\in [m]$ such that $V_{j,1}^{\pi'}(s_1) > V_{j,1}^\pi(s_1)$. A policy $\pi$ is a Pareto optimal policy if and only if no other policies dominate $\pi$.
\end{definition}
Intuitively, the domination of $\pi'$ over $\pi$ indicates that $\pi'$ would be a better solution than $\pi$ as it can strictly improve the value of at least one objective without making others worse off. \emph{By this definition, a policy $\pi$ is Pareto optimal when no other policies can improve the value of an objective under $\pi$ without hurting other objectives' values}. The set of all Pareto optimal policies is called the \emph{Pareto set} or \emph{Pareto front}. Across this paper, we denote the Pareto set as $\Pi^*_{\mathrm{P}}$. 

In particular, the Pareto set $\Pi^*_{\mathrm{P}}$ has the following fundamental properties. We revisit these properties as follows and further present their proof under the MORL setting for completeness in the appendix.
\begin{property}\label{pro:property} The Pareto set $\Pi_{\mathrm{P}}^*$ satisfies the following properties:
	\begin{itemize} [topsep=-2pt,itemsep=2pt,parsep=0pt]
		\item[\textbf{(a)}] For any policy $\pi \notin \Pi_{\mathrm{P}}^*$, there always exists a Pareto optimal policy $\pi^*\in\Pi_{\mathrm{P}}^*$ dominating $\pi$.
		\item[\textbf{(b)}] A policy $\pi\in\Pi_{\mathrm{P}}^*$ if and only if $\pi$ is not dominated by any Pareto optimal policy $\pi^*\in\Pi_{\mathrm{P}}^*$.
	\end{itemize}
\end{property}
The above properties show that Pareto optimal policies dominate non-Pareto-optimal ones but cannot dominate each other themselves, characterizing the relation between a Pareto optimal policy and any other policies. Property \ref{pro:property} indicates that Pareto optimal policies are ``mutually independent'' in a sense. These properties pave the way to proving the Pareto optimality via a policy's relation to only other Pareto optimal policies. Based on Property \ref{pro:property}, we are able to study crucial properties of an MORL optimization target named the Pareto suboptimality gap in the next section. On the other hand, when reducing to the single-objective setting where $\Pi_{\mathrm{P}}^*$ becomes an optimal policy set, this proposition indicates all optimal policies lead to the same optimal value that is larger than values under other suboptimal policies, matching the fact in single-objective learning. 


In addition to Pareto optimality, we introduce a relatively weaker notion named weak Pareto optimality as follows:
\begin{definition}[Weakly Pareto Optimal Policy]\label{def:weakpareto} A policy $\pi\in\Pi$ is a weakly Pareto optimal policy if and only if there are no other policies $\pi'\in\Pi$ satisfying $V_{i,1}^\pi(s_1) < V_{i,1}^{\pi'}(s_1)$ for all $i\in[m]$.
\end{definition} 
Comparing Definition \ref{def:pareto} with Definition \ref{def:weakpareto}, the weak Pareto optimality is achieved when no other policies can strictly improve \emph{all} objective functions instead of \emph{at least one} objective function as in the definition of Pareto optimality.  According to this definition, all optimal policies $\nu_i^*\in\argmax_\nu V_{i,1}^\nu(s_1), \forall i\in[m],$ for each objective are weakly Pareto optimal. We consider a multi-objective multi-arm bandit example, a simple and special MOMDP 
whose state space size $|\cS|=1$, episode length $H=1$, with a deterministic policy, to illustrate definitions and propositions in this section.
\begin{example}\label{ex:pareto} We consider a multi-objective multi-arm bandit problem with $m=2$ reward functions $r_1$ and $r_2$ and $|\cA|=5$ actions. We define
	\vspace{-0.2cm}
	\begin{align*}
		&r_1(a_1)=0.1, \quad r_1(a_2)=0.8,\quad r_1(a_3)=0.3,\quad r_1(a_4)=0.8,\quad r_1(a_5)=0.1,\\[-2pt]
		&r_2(a_1)=0.1, \quad r_2(a_2)=0.2,\quad r_2(a_3)=0.5,\quad r_2(a_4)=0.7,\quad r_2(a_5)=0.2.\\[-22pt]
	\end{align*}
	By definitions, the Pareto optimal arm is $a_4$, while the weakly Pareto optimal arms are both $a_4$ and $a_2$ since $r_1(a_2)=r_1(a_4)$ in spite of $r_2(a_2)<r_2(a_4)$. The optimal arms for $r_1$ are $a_2$ and $a_4$, which are weakly Pareto optimal. The optimal arm for $r_2$ is $a_4$, which is Pareto optimal and hence also weakly Pareto optimal.
\end{example}
The example above demonstrates that when values are equal under some policies for one objective, it can result in the presence of weakly Pareto optimal policies that are not Pareto optimal. In this paper, the set of all weakly Pareto optimal policies is denoted as $\Pi^*_{\mathrm{W}}$. By the definitions of (weakly) Pareto optimal policies, the Pareto set is the subset of the weak Pareto set, i.e.,
\begin{align*}
	\Pi_{\mathrm{P}}^* \subseteq \Pi_{\mathrm{W}}^*,
\end{align*}
which can also be verified by Example \ref{ex:pareto}.
Furthermore, according to their definitions, we can show that under certain conditions, all weakly Pareto optimal policies are Pareto optimal. 
\begin{proposition} \label{cond:subopt-iff} If for each $\pi\notin\Pi_{\mathrm{P}}^*$, there always exists a Pareto optimal policy $\pi^*\in\Pi_{\mathrm{P}}^*$ such that $V_{i,1}^{\pi}(s_1) <  V_{i,1}^{\pi^*}(s_1)$ for all $i\in [m]$, then we have $\Pi_{\mathrm{W}}^* = \Pi_{\mathrm{P}}^*$.
\end{proposition}
The condition in Proposition \ref{cond:subopt-iff} explicitly avoids the presence of a policy $\pi$ that satisfies $V_{i,1}^\pi(s_1) \leq  V_{i,1}^{\pi^*}(s_1)$ for all $i \in [m]$ with $V_{j,1}^\pi(s_1) = V_{j,1}^{\pi^*}(s_1)$ for some $j\in[m]$, where $\pi^*\in \Pi_{\mathrm{P}}^*$. Such a case can also be understood through Example \ref{ex:pareto}. Moreover, in the following sections, we will further discuss how  Pareto optimal policies can be identified with no prerequisite of $\Pi_{\mathrm{W}}^* = \Pi_{\mathrm{P}}^*$.

\vspace{5pt}
\noindent \textbf{Learning Goal --- Traversing Pareto Optimal Policies.} 
Based on the above discussions, we can see that it is critical to find all Pareto optimal policies or weakly Pareto optimal policies rather than seeking to find a solution to $\max_{\pi\in\Pi} \bV_1^\pi(s_1)$ which most likely does not exist. Empirically, a common practice is to map a learner-specified preference vector $\blambda\in\Delta_m$ to the Pareto optimal policies, such that all of them can be traversed in a controllable way. Therefore, the learning goal of MORL is formulated as
 \begin{align*}
	\text{Find all } \pi\in\Pi^*_{\mathrm{P}} 	
\end{align*}
controlled by learner-specified preferences $\blambda\in\Delta_m$. In the next section, we systematically discuss how to choose a favorable optimization target to achieve this goal.

\section{Optimization Targets for Multi-Objective RL} \label{sec:metric}
In order to find (weakly) Pareto optimal policies associated with different learner-specified preference $\blambda\in\Delta_m$, a common practice is to design an optimization target that can properly incorporate all objectives and the preference $\blambda$. Then, we expect to optimize such an optimization target to obtain a solution that can be (weakly) Pareto optimal. Therefore, various scalarization methods \citep{kasimbeyli2019comparison,ehrgott2005multicriteria,steuer1986multiple,drugan2013designing} have been proposed that can scalarize multiple objective functions into a single functional as an optimization target. Specifically, a favorable optimization target for MORL should  
\begin{itemize} [itemsep=3pt, topsep=3pt,parsep=3pt]
	\item have controllability of the learned policies under different preferences $\blambda\in\Delta_m$,
	\item have full coverage of all Pareto policies.
\end{itemize}
In this section, we systematically investigate three major scalarization methods, namely linear scalarization, Pareto suboptimality gap, and Tchebycheff scalarization, and identify a suitable scalarization method, i.e., Tchebycheff scalarization, for MORL that meets the above-mentioned requirements. In this section, a proposition will apply to both stochastic and deterministic policy classes if no specification is given.


Before our analysis of different optimization targets for MORL, we first study the geometry of the set of objective values
\begin{align*}
\mathbb{V}(\Pi):=\{\bV_1^{\pi}(s_1) ~|~ \pi\in\Pi \},	
\end{align*}
when $\Pi$ is a stochastic policy class. In particular, \citet{lu2022multi} proves that \emph{for the stochastic policy class $\Pi$, $\mathbb{V}(\Pi)$ is convex}. Our work steps forward and shows the following result.
\begin{proposition} \label{prop:geometry} For the stochastic policy class $\Pi$, $\mathbb{V}(\Pi)$ is a convex polytope.
\end{proposition}
We note that this result only holds for a stochastic policy class. When $\Pi$ is a deterministic policy class, the situation will be largely different (e.g., Example \ref{ex:pareto}). This proposition provides a clear characterization of the set of objective values, which is the key to showing properties of different scalarization methods.

\vspace{5pt}
\noindent\textbf{Linear Scalarization.} The most common scalarization method for multi-objective learning is the linear scalarization of objectives \citep{geoffrion1968proper,das1997closer,klamroth2007constrained,ehrgott2005multicriteria,steuer1986multiple}. For MORL \citep{puterman1990markov,van2013scalarized,ehrgott2005multiobjective,wu2021accommodating,lu2022multi}, it can be thus formulated as 
\begin{align*}
\linl(\pi)=\blambda^\top \bV_1^{\pi}(s_1),
\end{align*} 
where $\blambda\in\Delta_m$ is a vector characterizing learner's preferences on different objectives. For each $\blambda$, we solve $\max_{\pi\in\Pi} \linl(\pi)$ to obtain a solution that could be weakly Pareto optimal. Prior works \citep{geoffrion1968proper,ehrgott2005multicriteria,steuer1986multiple} have proved important properties for linear scalarization for general multi-objective optimization. When it comes to MORL, we provide a more specific characterization of the properties of linear scalarization due to the special structure of $\VV(\Pi)$.
\begin{proposition} \label{prop:linear-comb} For a stochastic policy class $\Pi$, the maximizers of linear scalarization satisfy $\{\pi~|~\pi\in\argmax_{\pi\in\Pi} \linl(\pi), \forall \blambda \in\Delta_m\}= \Pi_{\mathrm{W}}^*$ and $\{\pi~|~\pi\in\argmax_{\pi\in\Pi} \linl(\pi), \forall \blambda \in\Delta_m^o\}= \Pi_{\mathrm{P}}^*$. For a deterministic policy class $\Pi$, we have $\{\pi~|~\pi\in\argmax_{\pi\in\Pi} \linl(\pi), \forall \blambda \in\Delta_m\}\subseteq \Pi_{\mathrm{W}}^*$ and $\{\pi~|~\pi\in\argmax_{\pi\in\Pi} \linl(\pi), \forall \blambda \in\Delta_m^o\} \subseteq \Pi_{\mathrm{P}}^*$.
A (weakly) Pareto optimal policy may not be the solution to $\max_{\pi\in\Pi} \linl(\pi)$ for any $\blambda \in\Delta_m$ when $\Pi$ is a deterministic policy class.
\end{proposition}
This proposition shows that when policies are allowed to be stochastic, we can find all (weakly) Pareto optimal policies by solving $\max_{\pi\in\Pi} \linl(\pi), \forall \blambda \in\Delta_m$ or $\blambda \in\Delta_m^o$. We can determine if the solution is a Pareto optimal policy, rather than merely a weakly Pareto optimal policy, by the setting of $\blambda$. However, Proposition \ref{prop:linear-comb} indicates that for a deterministic policy class $\Pi$, \emph{the solutions to $\max_{\pi\in\Pi} \linl(\pi)$ may not cover all (weakly) Pareto optimal policies}. In Appendix \ref{sec:proof-linear-comb}, we illustrate this claim via a multi-objective multi-arm bandit example, a special case of MORL with a deterministic policy, where the solutions to $\max_{\pi\in\Pi} \blambda^\top \bV_1^{\pi}(s_1)$ fail to identify all Pareto optimal arms. This result motivates us to further explore other scalarization methods.

\vspace{5pt}
\noindent\textbf{Pareto Suboptimality Gap.}
To measure the difference between a policy and the Pareto set $\Pi_{\mathrm{P}}^*$, recent works on multi-objective online learning study efficient algorithms of finding zeros of the Pareto suboptimality gap \citep{drugan2013designing,turgay2018multi,lu2019multi,jiang2023multiobjective}. We adapt the definition of the Pareto suboptimality gap to MORL inspired by these works.
\begin{definition}[Pareto Suboptimality Gap] \label{def:subopt} The Pareto suboptimality gap $\psg(\pi)$ for a policy $\pi\in\Pi$ is defined as the minimal value of $\epsilon \geq 0$ such that there exists one objective whose value for the policy $\pi$ added $\epsilon$ is larger than the value under the Pareto optimal policy of $\bV_1^{\pi}(s_1)$, i.e.,  	
\begin{align}
\begin{aligned}\label{eq:subopt}
&\psg(\pi):= \inf_{\epsilon \geq 0} \epsilon, \\
&\text{s.t.}~~\forall \pi^*\in \Pi_{\mathrm{P}}^*:~ \exists i\in[m], V_{i,1}^{\pi}(s_1) + \epsilon   >  V_{i,1}^{\pi^*}(s_1), ~\textbf{or}~~\forall i\in[m], V_{i,1}^{\pi}(s_1) + \epsilon  =  V_{i,1}^{\pi^*}(s_1).
\end{aligned}
\end{align}

\end{definition} 
Here $\psg(\pi)$ measures the discrepancy between a policy $\pi$ and $\Pi_{\mathrm{P}}^*$ in terms of the value function with $\psg(\pi)\geq 0$. Intuitively, it shows that after shifting the value function via $V_{i,1}^{\pi}(s_1)+ \epsilon$ with a minimal $\epsilon$ for all $i\in[m]$, the policy $\pi$ behaves similarly to the Pareto optimal policy in $\Pi_{\mathrm{P}}^*$ in terms of the notion of non-dominance. To further facilitate the understanding, inspired by \citet{jiang2023multiobjective}, we obtain an equivalent formulation of the Pareto suboptimality gap for MORL:
\begin{proposition}[Equivalent Form of PSG] \label{prop:eq-psg} $\psg(\pi)$ is equivalently formulated as  	
	\begin{align}
		\psg(\pi) = \textstyle\sup_{\pi^*\in \Pi_{\mathrm{P}}^*}\textstyle\inf_{\blambda\in \Delta_m} \blambda{}^\top\big(\bV_1^{\pi^*}(s_1)-\bV_1^{\pi}(s_1)\big). \label{eq:eq-psg}
	\end{align}
\end{proposition}
The above proposition indicates that the Pareto suboptimality gap is inherently in a sup-inf form, where the preference vector $\blambda\in\Delta_m$ behaves \emph{adaptively} to combine each objective such that the gap between the Pareto set $\Pi_{\mathrm{P}}^*$ and a policy $\pi$ is minimized. 

Although the Pareto suboptimality gap has been used in a number of prior works, its properties are not fully investigated, \emph{particularly the sufficiency and necessity of $\psg(\pi) = 0$ for $\pi$ being (weakly) Pareto optimal}. Previous works \citep{drugan2013designing,turgay2018multi,lu2019multi,jiang2023multiobjective} have shown the necessity of it, i.e., $\pi$ being Pareto optimal implies $\psg(\pi) = 0$. However, the critical question of whether $\psg(\pi) = 0$ implies the Pareto optimality of $\pi$ remains elusive. In what follows, by employing Property \ref{pro:property}, we contribute to resolving this question by the proposition below:
\begin{proposition}\label{prop:subopt} The set of all zeros of $\psg(\pi)$ satisfies $\{\pi\in\Pi~|~\psg(\pi) = 0\} = \Pi_{\mathrm{W}}^*$.	
\end{proposition}
This proposition indicates that all weakly Pareto optimal policies can be identified by finding zeros of $\psg(\pi)$, and thus \emph{$\psg(\pi) = 0$ is a sufficient and necessary condition for $\pi$ being (weakly) Pareto optimal}. The claim is also sufficiently general for multi-objective learning problems and not limited to the reinforcement learning setting. Although (weakly) Pareto optimal policies can be achieved via $\psg(\pi)$, Proposition \ref{prop:eq-psg} indicates that we have no controllability of the solutions as the underlying preference $\blambda\in\Delta_m$ is not explicitly determined by learners. Thus, it remains unclear how to traverse the (weak) Pareto optimal policies utilizing the Pareto suboptimality gap, which motivates us to explore other scalarization methods.

\vspace{5pt}
\noindent\textbf{Tchebycheff Scalarization.} We further investigate Tchebycheff scalarization \citep{miettinen1999nonlinear,lin2024smooth,ehrgott2005multicriteria,bowman1976relationship,klamroth2007constrained,choo1983proper,klamroth2007constrained,steuer1986multiple} that is widely studied in multi-objective learning. Following these works, we adapt the definition of Tchebycheff scalarization to MORL as follows:
\begin{definition}[Tchebycheff Scalarization] \label{def:tch} Tchebycheff scalarization converts the original MORL problem into a minimization problem of the following optimization target, i.e.,  	
\begin{align}\label{eq:tch}
	\tch_{\blambda}(\pi):= \max_{i\in[m]} \{ \lambda_i (V_{i,1}^*(s_1) + \iota - V_{i,1}^\pi(s_1)) \},
\end{align}
where $\blambda = [\lambda_i]_{i=1}^m\in \Delta_m$ is a preference vector defined by the learner, $V_{i,1}^*(s_1):=\max_{\pi\in\Pi} V_{i,1}^\pi(s_1)$ is the maximum of the $i$-th value function, and $\iota>0$ is a  sufficiently small pre-defined regularizer.
\end{definition} 
Definition \ref{def:tch} shows that $\tch_{\blambda}(\pi)$ is defined based on a vector $\blambda\in\Delta_m$ that can be viewed as a preference over different objectives. Next, following \citet{choo1983proper,ehrgott2005multicriteria} for multi-objective optimization, we have that for MORL, $\tch_{\blambda}(\pi)$ has the following property: 
\begin{proposition} \label{prop:tch} The solutions to the minimization problem of Tchebycheff scalarization under all $\blambda\in\Delta_m^o$ satisfy $\{\pi~|~\pi \in \min_{\pi\in\Pi} \tchl(\pi), \forall \blambda \in\Delta_m^o\}= \Pi_{\mathrm{W}}^*$.	
Moreover, if $\pi$ is a unique solution to $\min_{\pi\in\Pi} \tchl(\pi)$ for a preference $\blambda$, then $\pi$ is Pareto optimal.
\end{proposition}
Proposition \ref{prop:tch} indicates that by solving $\min_{\pi\in\Pi} \tchl(\pi)$ for all $\lambda_i>0$, we can obtain all Pareto optimal policies. However, it is noted that the solution to $\min_{\pi\in\Pi} \tchl(\pi)$ for a $\blambda$ may not be unique. It is computationally not easy to find all solutions to an optimization problem. Therefore, it is crucial to look into the relation between the distribution of Pareto optimal policies and the problem $\min_{\pi\in\Pi} \tchl(\pi)$. Then, using Proposition \ref{prop:tch}, we prove the following result:
\begin{proposition}\label{prop:tch-unique} There always exists a subset $\Lambda \subseteq \Delta_m$ such that $\{\pi~|~\pi \in \argmin_{\pi\in\Pi} \tchl(\pi), \forall \blambda \in\Lambda\subseteq\Delta_m^o\}= \Pi_{\mathrm{P}}^*$. Supposing that $\pi_{\blambda}^*$ is one arbitrary solution to $\min_{\pi\in\Pi} \tchl(\pi)$, for a $\blambda\in\Lambda$, all the solutions to $\min_{\pi\in\Pi} \tchl(\pi)$ are Pareto optimal with the same value $V_{i,1}^{\pi_{\blambda}^*}(s_1)$ for all $i\in[m]$. 
\end{proposition}
This proposition implies that for $\blambda\in\Lambda$, each solution to $\min_{\pi\in\Pi} \tchl(\pi)$ is Pareto optimal, and all the solutions are equivalent in terms of their values. \emph{Therefore, although it is computationally not easy to estimate all solutions, it will be satisfactory to obtain only one solution for each $\blambda\in\Delta_m^o$. In this way, the subset $\Lambda$ is covered, and we obtain all representative Pareto optimal policies that are equivalent to others in terms of their values on each objective.} Please see more discussion from the algorithmic perspective in Remark \ref{rm:tch-unique-Pareto}. Such favorable properties of both the controllability of the solutions and the full coverage of Pareto optimal policies motivate us to develop efficient algorithms to solve the problem of $\min_{\pi\in\Pi} \tchl(\pi)$. 
In the following sections, we propose such efficient algorithms and prove their sample efficiency.

\begin{figure}[t!] 
	
	\centering
	\begin{subfigure}[b]{0.32\textwidth}
		\centering
		\includegraphics[height=1.8in]{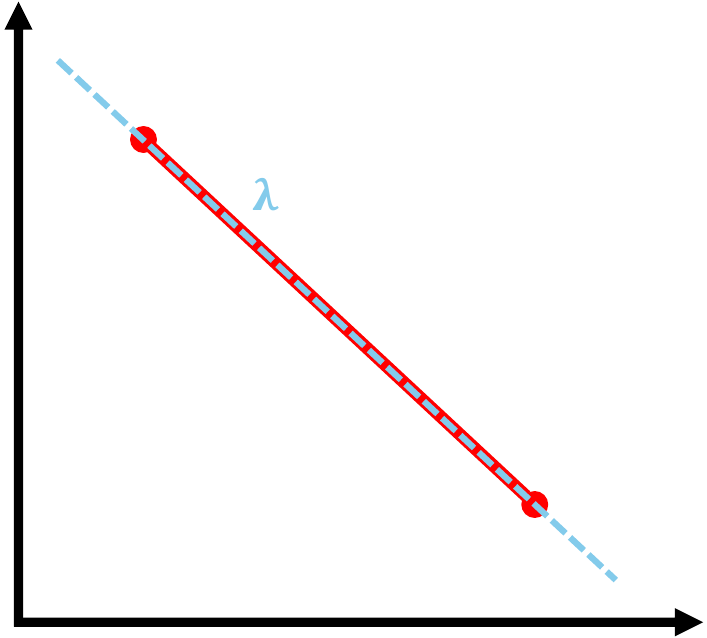}
		\caption{\small Linear}
	\end{subfigure}%
	~ 
	\begin{subfigure}[b]{0.32\textwidth}
		\centering
		\includegraphics[height=1.8in]{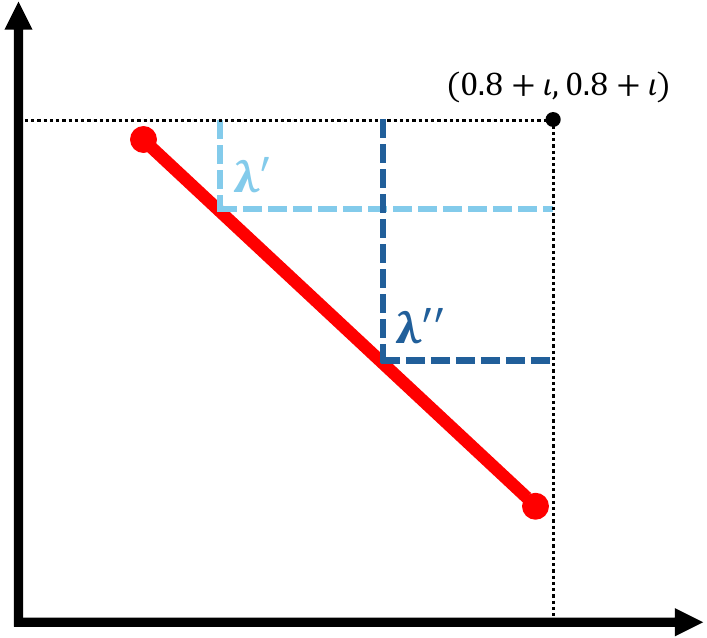}
		\caption{\small Tchebycheff}
	\end{subfigure}
	\begin{subfigure}[b]{0.32\textwidth}
		\centering
		\includegraphics[height=1.8in]{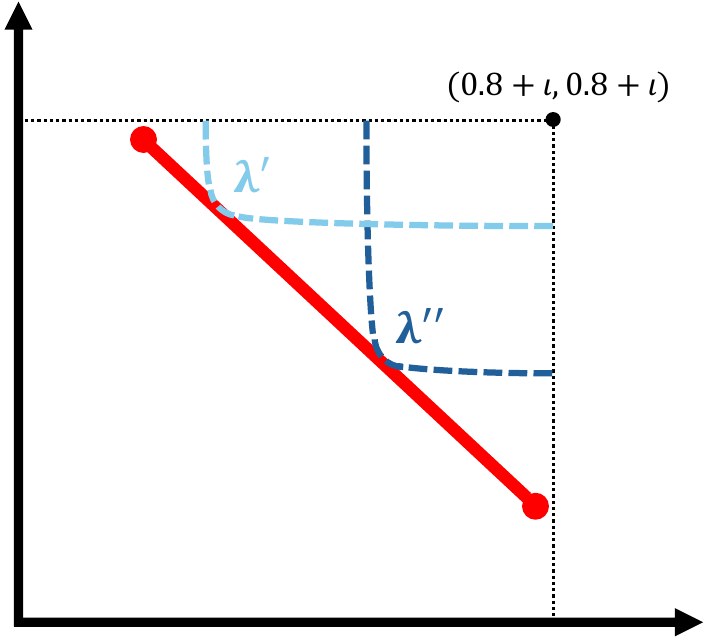}
		\caption{\small Smooth Tchebycheff}
	\end{subfigure}
	\caption{\small Example of Scalarization Methods for Stochastic Policy Class. Consider a bi-objective multi-arm bandit problem with stochastic policies, whose reward functions are $r_1(a_1)=0.2, r_1(a_2)=0.8$ and $r_2(a_1)=0.8, r_2(a_2)=0.2$. The red dots in the figures are $(r_1(a_1),r_2(a_1))$ and $(r_1(a_2),r_2(a_2))$. The red lines represent reward values under all stochastic Pareto optimal policies. We show the level sets of different scalarization functions in terms of $(r_1,r_2)$, i.e., Linear scalarization $\sum_{i=1}^2 \lambda_i r_i$, Tchebycheff scalarization $\max_i \lambda_i (0.8+\iota-r_i)$, and smooth Tchebycheff scalarization $\mu\log\sum_{i=1}^2 e^{\frac{\lambda_i (0.8+\iota-r_i)}{\mu}}$. The blue dotted lines are the level sets for the optimal (maximal or minimal) scalarization function values with $(r_1,r_2)$ defined on the red lines. Then, linear scalarization does not differentiate different Pareto values, but (smooth) Tchebycheff scalarization can identify each point by setting different $\blambda$, showing better solution controllability.}
	\label{fig:comp}
\end{figure}

\vspace{5pt}

\begin{remark}[Linear Scalarization \emph{vs.} Tchebycheff Scalarization.] \label{re:LinvsTch} From the above propositions, we can observe that for the deterministic policy class, Tchebycheff scalarization can find all Pareto optimal policies, whereas linear scalarization fails to do so. Additionally, for the stochastic policy class, we discover that it is computationally difficult to identify all representative Pareto optimal policies utilizing linear scalarization under certain situations. Particularly, when $\bV_1^\pi(s_1)$ for a set of Pareto optimal policies $\pi$ lie on a hyperplane in the form of $\overline{\blambda}^\top \bV_1^\pi(s_1)=C$ for some $C>0$ and $\overline{\blambda}$, then we need to find all the solutions to $\max_\pi \mathrm{LIN}_{\overline{\blambda}}(\pi)$, since all these Pareto optimal policies have the same preference $\overline{\blambda}$ defined based on linear scalarization.
In contrast, Proposition \ref{prop:tch-unique} indicates that for each of $\bV_1^\pi(s_1)$ on $\overline{\blambda}^\top \bV_1^\pi(s_1)=C$ under different $\pi$'s, we can use different $\blambda$'s that are associated with distinct $\pi$'s and solve $\min_\pi \tch_{\blambda}(\pi)$ to obtain one representative policy for each $\blambda$. In Figure \ref{fig:comp}, we illustrate the above discussion via a bi-objective multi-arm bandit example. The discussion above implies that Tchebycheff scalarization has better controllability over its solution and thus is recognized as a favorable scalarization method for MORL. Furthermore, in Section \ref{sec:stch}, we study an extension of Tchebycheff scalarization named smooth Tchebycheff scalarization, which also has a better controllability of its solution compared to the linear scalarization. We refer readers to Section \ref{sec:stch} for more details, where we prove that the smooth Tchebycheff scalarization offers a more advantageous property compared to the original Tchebycheff scalarization in finding Pareto optimal policies. In Table \ref{tab:comparison}, we provide a detailed comparison of different scalarization methods. 	
\end{remark}

\section{MORL via Tchebycheff Scalarization} \label{sec:tchs}

In this section, we propose an efficient algorithm to solve $\min_{\pi\in\Pi} \tchl(\pi)$ for any given $\blambda$. We note that since all our proposed algorithms in this paper allow mixture policies and utilize the minimax theorem \citep{v1928theorie}, according to \citet{wu2021offline,miryoosefi2022simple}, the proposed algorithms only solve the problem within a \emph{stochastic} policy class $\Pi$. However, for the deterministic policy class, one may need different algorithmic design ideas to solve this problem. We leave designing efficient algorithms for the deterministic policy class as an open question.

\vspace{5pt}
\noindent\textbf{Reformulation of Tchebycheff Scalarization.} We note that although \eqref{eq:tch} has favorable properties, it is a non-smooth function, and thus $\min_{\pi\in\Pi} \tchl(\pi)$ is difficult to solve directly. To resolve this challenge, we propose the following equivalent form for Tchebycheff scalarization loss.

\begin{proposition}[Equivalent form of TCH] \label{prop:eq-tch}  Letting $\bw=(w_1,w_2,\cdots, w_m) \in\Delta_m$, the Tchebycheff scalarization is reformulated as $\tchl(\pi) = \max_{\bw \in\Delta_m} \max_{\{\nu_i\in\Pi\}_{i=1}^m} \sum_{i=1}^m w_i\lambda_i (V_{i,1}^{\nu_i}(s_1) + \iota - V_{i,1}^\pi(s_1))$. Then, the minimization problem of Tchebycheff scalarization is reformulated as	
	\begin{align}
		 \min_{\pi\in\Pi} \tchl(\pi) = \min_{\pi\in\Pi} \max_{\bw \in\Delta_m} \max_{\{\nu_i\in\Pi\}_{i=1}^m} \sum_{i=1}^m w_i\lambda_i \big(V_{i,1}^{\nu_i}(s_1) + \iota - V_{i,1}^\pi(s_1)\big). \label{eq:eq-tch}
	\end{align}
\end{proposition}
Proposition \ref{prop:eq-tch} converts the original problem into a min-max-max problem, which can avoid a non-smooth objective function and thus is expected to have a nice convergence performance. This formulation can better fit the UBC-based algorithmic framework, which motivates us to design efficient optimistic MORL algorithms. Furthermore, we remark that the above proposition also offers a feasible reformulation of Tchebycheff scalarization that can be better applied to multi-objective stochastic optimization or multi-objective online learning, which are not restricted to MORL.

\vspace{5pt}
\noindent\textbf{Algorithm.} We now propose an algorithm named \texttt{TchRL} to solve \eqref{eq:eq-tch}, as summarized in Algorithm \ref{alg:online-tch}, that can efficiently learn the (weakly) Pareto optimal solution for a given preference vector $\blambda$. 
 According to \eqref{eq:eq-tch}, to solve such a min-max-max optimization problem, Algorithm \ref{alg:online-tch} features three alternate updating procedures to update $\bw^t$, auxiliary policies $\nu_i^t$ for $i\in[m]$, and the main policy $\pi^t$, along with optimistic estimation of Q-functions to facilitate exploration. 

Specifically, as \eqref{eq:eq-tch} aims to maximize $\bw$, Line \ref{line:tch-w} performs a mirror ascent step, which can keep any of its iterates in $\Delta_m$ without an extra projection step. The update of $\bw^t$ is equivalent to solving a maximization problem as 
\begin{align*}
\max_{\bw\in\Delta_m} \eta\langle \bw, \blambda \odot (\tilde{\bV}_1^{t-1}(s_1)+\biota -\bV_1^{t-1}(s_1)) \rangle -  D_{\mathrm{KL}}( \bw, \bw^{t-1}),	
\end{align*}
where $D_{\mathrm{KL}}(\bx,\by):=\sumi x_i \log(x_i/y_i)$ is the Kullback–Leibler (KL) divergence, $\biota$ is an all-$\iota$ vector, and $\tilde{\bV}_1^{t-1}:=(\tilde{V}_{1,1}^{t-1}, \cdots,\tilde{V}_{m,1}^{t-1})$, and $\bV_1^{t-1}:=(V_{1,1}^{t-1}, \cdots,V_{m,1}^{t-1})$ from the last round. In addition, Algorithm \ref{alg:online-tch} updates two types of policies, i.e., the auxiliary policies $\tilde\nu_i^t, \forall i\in[m],$ and the main policy $\pi^t$, respectively in Line \ref{line:tch-auxpolicy} and Line \ref{line:tch-mainpolicy}. The auxiliary policy $\tilde\nu_i^t$ tracks the individual optimal policy $\nu_i^*:=\argmax_{\nu_i} V_{i,1}^{\nu_i}(s_1)$ for the $i$-th objective associated with the reward function $r_i$, while $\pi^t$ is the estimate of $\pi_{\blambda}^*:=\min_{\pi\in\Pi} \tch_{\blambda}(\pi)$, namely, the (weakly) Pareto optimal policy under the given preference $\blambda$. We further employ two Q-functions, $\tilde Q_{i,h}^t$ and $Q_{i,h}^t$, as estimates of $Q_{i,h}^{\nu_i^*}$ and $Q_{i,h}^{\pi_{\blambda}^*}$. Thus, $\tilde\nu_i^t$ is obtained solely based on $\tilde Q_{i,h}^t$, while $\pi^t$ is updated by a linear combination of $\tilde Q_{i,h}^t$ over all $i\in[m]$ by $\bw^t \odot \blambda$, incorporating the preference $\blambda$ and the current weight $\bw^t$. Due to the greedy updating rules for both policies,  $\tilde\nu_{i,h}^t$ and $\pi_h^t$ are hence deterministic policies based on $\tilde Q_{i,h+1}^t$ and $(\bw^t\odot\blambda)^\top Q_h^t$. Here, the Q-functions $\tilde Q_{i,h}^t$ and $Q_{i,h}^t$ are constructed via ``the principle of optimism in the face of uncertainty'' by calling Subroutine \ref{alg:opt-lstd}, which will be elaborated below. In Lines \ref{line:tch-auxsample} and \ref{line:tch-mainsample}, we further sample trajectories following $\tilde{\nu}_i^t$ and $\pi_i^t$ to update the datasets $\tilde\cD_{i,h}$ and $\cD_{i,h}$ respectively for the construction of $\tilde Q_{i,h}^{t+1}$ and $Q_{i,h}^{t+1}$ in the next round.

The function \texttt{OptQ} in Subroutine \ref{alg:opt-lstd} performs optimistic Q-function construction based on the collected trajectories $\{s_h^\tau,a_h^\tau,s_{h+1}^\tau, r_{i,h}^\tau(s_h^\tau,a_h^\tau)\}_{h,\tau=1}^{H,\ell-1}$, and returns the optimistic Q-function $Q_{i,h}$. It first estimates the reward functions $\hat{r}_i^\ell$ and transition model $\hat{\PP}^\ell$ as well as bonus terms $\Psi_{i,h}^\ell$ and $\Phi_h^\ell$ for such a construction. The bonus terms will then have larger values for less-explored state-action pairs and thus facilitate the exploration of those state-action pairs during the learning process. Our work considers a tabular case such that the reward functions and transition are estimated by 
\small
\begin{align}
	\hat{r}_{i,h}^\ell(s,a)= \frac{\sum_{\tau=1}^{\ell-1}\mathbf{1}_{\{(s,a)=(s_h^\tau,a_h^\tau)\}}r_{i,h}^\tau}{ N_h^{\ell-1}(s,a)\vee 1}, \quad \hat{\PP}_h^\ell(s'|s,a) = \frac{N_h^{\ell-1}(s,a,s')}{ N_h^{\ell-1}(s,a)\vee 1}, \label{eq:estimate}
\end{align}
\normalsize
where $\mathbf{1}_{\{\cdot\}}$ is an indicator function, $N_h^{\ell-1}(s,a,s') = \sum_{\tau=1}^{\ell-1}\mathbf{1}_{\{(s,a,s')=(s_h^\tau,a_h^\tau,s_{h+1}^\tau)\}}$   and  $N_h^{\ell-1}(s,a) = \sum_{\tau=1}^{\ell-1}\mathbf{1}_{\{(s,a)=(s_h^\tau,a_h^\tau)\}}$. Denoting $x\vee y=\max\{x,y\}$ and $x\wedge y=\min\{x,y\}$, the reward bonus and transition bonus terms are constructed by 
\begin{align}
	\Psi_{i,h}^\ell(s,a)=\sqrt{\frac{2\log(8m|\cS||\cA|HT/\delta)}{ N_h^{\ell-1}(s,a)\vee 1}} \wedge 1,	\quad  \Phi_h^\ell(s,a)=\sqrt{\frac{2H^2|\cS|\log(8m|\cS||\cA|HT/\delta)}{N_h^{\ell-1}(s,a)\vee 1}} \wedge H, \label{eq:bonus}
\end{align}
such that the true reward and transition satisfy $|\hat{r}_{i,h}^{\ell-1}(s,a) - r_{i,h}(s,a)| \leq \Psi_{i,h}^\ell(s,a)$ and $|\hat{\PP}_h^\ell V(s,a) - \PP_h V(s,a)| \leq \Phi_h^\ell(s,a)$ for any function $V:\cS\mapsto [0,H]$ with probability at least $1-\delta$. This relation can be verified by applying Hoeffding's inequality, as shown in our appendix.

\begin{algorithm}[tb]
	\small
	\caption{\texttt{TchRL}: Multi-Objective RL via Tchebycheff Scalarization}
	\label{alg:online-tch}
	\setstretch{1.15}
	\begin{algorithmic}[1]
		\State {\bfseries Initialize:} $\blambda$, $\pi_h^0(\cdot|\cdot)=\nu_{i,h}^0(\cdot|\cdot)=\Unif(\cA)$, $\cD_i = \hat{\cD}_i =\emptyset$, $w_i^0 = 1/m$, $V_{i,h}^0(\cdot)=\tilde{V}_{i,h}^0(\cdot)=0, \forall (i,h)$
		
		\For{$t=1,2,\ldots,T$}
		\State\label{line:tch-w}Update $\bw$ via $w_i^t \propto w_i^{t-1} \cdot \exp\{ \eta \cdot \lambda_i  (\tilde{V}_{i,1}^{t-1}(s_1)+\iota -V_{i,1}^{t-1}(s_1))\},~\forall i\in[m]$

		\For{$h=H,H-1,\ldots,1$}  
		\State\label{line:tch-auxQ}$\tilde{Q}_{i,h}^t = \texttt{OptQ}\big(\tilde{\cD}_{i,h}, \tilde{V}_{i,h+1}^t \big)$ where $\tilde{V}_{i,h+1}^t(\cdot)=\langle  \tilde{Q}_{i,h+1}^t(\cdot, \cdot), \tilde\nu_{i,h+1}^t(\cdot|\cdot) \rangle_\cA, ~\forall i\in[m]$
		
		\State\label{line:tch-auxpolicy}Update auxiliary policies $\tilde\nu_{i,h}^t=\argmax_{\nu_h}\langle \tilde{Q}_{i,h}^t(\cdot, \cdot), \nu_h(\cdot|\cdot) \rangle_\cA, ~\forall i\in[m]$ 
		
		\State\label{line:tch-mainQ}$Q_{i,h}^t = \texttt{OptQ}\big(\cD_{i,h}, V_{i,h+1}^t\big)$ where $V_{i,h+1}^t(\cdot)=\langle  Q_{i,h+1}^t(\cdot, \cdot), \pi_{h+1}^t(\cdot|\cdot) \rangle_\cA, \forall i\in[m]$
		
		\State\label{line:tch-mainpolicy}Update main policy $\pi_h^t=\argmax_{\pi_h}\langle (\bw^t\odot\blambda)^\top \bQ_h^t(\cdot, \cdot), \pi_h(\cdot|\cdot) \rangle_\cA$ 
		
		\EndFor
		\State\label{line:tch-auxsample}Sample data $\{\tilde{s}_{i,h}^t,\tilde{a}_{i,h}^t, \tilde{s}_{i,h+1}^t, \tilde{r}_{i,h}^t\}_{h=1}^H\sim (\tilde\nu_i^t, \PP)$ and append it to $\tilde{\cD}_{i,h}, ~\forall (i,h)$
		
		\State\label{line:tch-mainsample}Sample data $\{s_h^t,a_h^t, s_{h+1}^t, (r_{i,h}^t)_{i=1}^m\}_{h=1}^H\sim (\pi^t, \PP)$ and append $(s_h^t,a_h^t, s_{h+1}^t, r_{i,h}^t)$ to $\cD_{i,h}, ~\forall (i,h)$
		
		\EndFor 
		\State {\bfseries Return:} $\{\pi^t\}_{t=1}^T$
	\end{algorithmic}
	\normalsize
\end{algorithm}

\begin{subroutine}[t]
	\small
	\caption{\texttt{OptQ:} Optimistic Q-Function Construction}
	\label{alg:opt-lstd}
	\setstretch{1.15}
	\begin{algorithmic}[1]
		\Statex\hspace{-0.6cm} {\bfseries Function:} \texttt{OptQ}$\big(\{s_h^\tau,a_h^\tau,s_{h+1}^\tau, r_{i,h}^\tau(s_h^\tau,a_h^\tau)\}_{h,\tau=1}^{H,\ell-1}, V_{i,h+1}\big)$ \Comment{{\color{gray}Construct optimistic Q-function}}
		\State Calculate reward and transition estimates $\hat{r}_h^\ell$, $\hat{\PP}_h^\ell$ via \eqref{eq:estimate}
		\State Calculate reward and transition bonuses $\Psi_{i,h}^\ell$, $\Phi_h^\ell$ via \eqref{eq:bonus}
		
		\State $Q_{i,h}(\cdot,\cdot) = \{ (\hat{r}_{i,h}^\ell + \hat{\PP}_h^\ell V_{i,h+1} + \Phi_h^\ell+\Psi_{i,h}^\ell)(\cdot,\cdot)\}_{[0,H-h+1]}$   
		\State {\bfseries Return:} 
		$\{Q_{i,h}\}_{h=1}^H$	
	\end{algorithmic}
	\normalsize
\end{subroutine}

\vspace{5pt}
\noindent\textbf{Theoretical Result.} The following theorem provides a theoretical guarantee for Algorithm \ref{alg:online-tch}.
\begin{theorem}[Theoretical Guarantee for \texttt{TchRL}] \label{thm:tch}
	Setting $\eta = \sqrt{\log m /(H^2T)}$, letting $\hat \pi:= \Mix(\pi^1,\pi^2,\cdots, \pi^T)$ be a mixture of the learned policies, $\pi_{\blambda}^*\in\argmin_{\pi\in\Pi} \tchl (\pi)$ for any preference vector $\blambda\in\Delta_m$, then with probability at least $1-\delta$, after $T$ rounds of Algorithm \ref{alg:online-tch}, we obtain
	\begin{align*}
		&\tchl(\hat{\pi}) - \tchl(\pi_{\blambda}^*) \leq \cO\Big(\sqrt{H^4 |\cS|^2|\cA| \varrho/T} \Big),
	\end{align*}
	where $\varrho:=\log (m|\cS||\cA|HT/\delta)$.
\end{theorem}
When we set $T=\tilde{\cO}(H^4 |\cS|^2|\cA|/\varepsilon^2)$, we are able to achieve $\tchl(\hat{\pi}) - \tchl(\pi_{\blambda}^*)\leq \varepsilon$, where $\tilde\cO$ hides logarithmic dependence on $\varepsilon^{-1}$, $H$, $|\cS|$, $|\cA|$, and $m$. It implies that for one $\blambda$, Algorithm \ref{alg:online-tch} requires $\tilde{\cO}(m H^4 |\cS|^2|\cA|/\varepsilon^2)$ sample complexity to achieve an $\varepsilon$-error, where the extra factor $m$ is from the $m$ rounds of exploration in Lines \ref{line:tch-auxsample} and \ref{line:tch-mainsample}.  

Denoting $\hat\pi$ as a mixture of $\{\pi^1,\pi^2,\cdots, \pi^T\}$, we let $\hat\pi$ be the output of this algorithm, following prior RL works \citet{miryoosefi2022simple,wu2021offline,altman2021constrained}. The policy $\hat{\pi}$ is executed by randomly selecting a policy $\pi^t$ from $\{\pi^1,\pi^2,\cdots, \pi^T\}$ with equal probability beforehand and then exclusively following $\pi^t$ thereafter, which implies $V^{\hat\pi}_{i,1}(s_1)= \frac{1}{T} \sumt V^{\pi^t}_{i,1}(s_1)$ \citep{wu2021offline,miryoosefi2022simple} for theoretical analysis. 

Moreover, we can also use the notion of the occupancy measure to estimate $\hat\pi$. Specifically, denoting by $\theta_h(s,a;\pi)$ the joint distribution of $(s,a)$ at step $h$ induced by $\pi$ and $\PP$, we have $V^{\pi}_{i,1}(s_1)=\sumh\sum_{s\in\cS}\sum_{a\in\cA} \theta_h(s,a;\pi) r_{i,h}(s,a)$. The occupancy measure can be dynamically calculated as $\theta_{h+1}(s',a';\pi) = \sum_{s\in\cS}\sum_{a\in\cA}\theta_h(s,a;\pi)\PP_h(s'|s,a)\pi_{h+1}(a'|s')$. The policy $\pi$ can be recovered as $\pi_h(a|s)=\frac{\theta_h(s,a;\pi)}{\sum_{a\in\cA}\theta_h(s,a;\pi)}$. By the definition of the policy mixture, we have $\theta_h(s,a;\hat\pi) =\frac{1}{T}\sumt \theta_h(s,a;\pi^t)$. If we know the true transition $\PP$, we can recover $\hat\pi$ by $\theta_h(s,a;\hat\pi)$. Practically, we can substitute $\PP$ with an estimate of $\hat{\PP}^{T-1}$. When $T$ is large enough, the estimate of $\hat\pi$ will be sufficiently accurate. For the multi-arm bandit problem, since it has no transition model, we directly have $\hat{\pi}=\frac{1}{T}\sumt \pi^t$.

To avoid the policy mixture, a more intriguing future direction is to further explore the theoretical guarantee of the last-iterate convergence along with using variants of policy optimization methods, which demands a more complex algorithmic design.

\begin{remark}\label{re:off-policy}Algorithm \ref{alg:online-tch} constructs optimistic Q-functions based on trajectories generated by their own corresponding policies $\pi^t$ and $\tilde\nu_i^t, \forall i\in[m],$ as in Lines \ref{line:tch-auxsample} and \ref{line:tch-mainsample}. It is intriguing to investigate off-policy algorithmic designs \citep{andrychowicz2017hindsight,yu2021conservative}, such as hindsight experience replay and importance sampling, to improve sample efficiency by utilizing data across policies. Such a topic is of independent interest and the related theoretical analysis can be left as a future research direction. However, in the next section, our new algorithm can avoid such multiple exploration rounds by an uncertainty-guided preference-free exploration method. 
\end{remark}

\begin{remark}\label{rm:tch-unique-Pareto} Tchebycheff scalarization does not explicitly distinguish between Pareto optimal policies and other weakly Pareto optimal policies based on the setting of $\blambda$. One may consider using $\frac{1}{T} \sumt V^t_{i,1}(s_1)$ to estimate $V^{\hat\pi}_{i,1}(s_1)= \frac{1}{T} \sumt V^{\pi^t}_{i,1}(s_1)$ for a sufficiently large $T$ and comparing the estimated values of $V^{\hat\pi}_{i,1}(s_1)$ under different $\blambda$ to rule out non-Pareto-optimal policies within certain error tolerance, which, however, might be numerically unstable. In practice, weakly Pareto optimal policies are often considered acceptable solutions as they can also capture the trade-off of learners' preferences across different objectives, especially when $\mathbb{V}(\Pi)$ is convex for a stochastic policy $\Pi$.
\end{remark}

\section{Preference-Free MORL via Tchebycheff Scalarization}\label{sec:pre-free}


\textbf{Algorithm.} Although Algorithm \ref{alg:online-tch} has the controllability of the policy learning that can incorporate learner's preference $\blambda$, it always requires exploration of the environment once a new preference vector $\blambda$ is given, which will incur high costs of environment interaction if many different preferences are considered. This motivates us to further design a new preference-free algorithmic framework named \texttt{PF-TchRL}. This new algorithm features decoupled exploration and planning phases as summarized in Algorithms \ref{alg:pure-explore} and \ref{alg:exploit-tch}. Once the exploration stage, i.e.,  Algorithm \ref{alg:pure-explore}, thoroughly explores the environment to collect trajectories $\cD$ without any designated preference $\blambda$, we can then execute the planning phase, Algorithm \ref{alg:exploit-tch}, using $\cD$ for any diverse input preferences $\blambda\in\Delta_m$, requiring no additional environment interaction. 

Specifically, in Algorithm \ref{alg:pure-explore}, Line \ref{line:pf-estimate} estimates the transition model and bonus terms, where $\hat\PP_h^t$, $\Psi_{i,h}^t$, and $\Phi_h^t$ are instantiated according to \eqref{eq:estimate} and \eqref{eq:bonus} with collected trajectories before the $t$-th exploration round, i.e., $\{s_h^\tau,a_h^\tau,s_{h+1}^\tau\}_{h,\tau=1}^{H,t-1}$. Furthermore, in Line \ref{line:pf-reward}, we define the exploration reward $\overline{r}_h^t$ based on the reward and transition bonus terms. Such a reward design can guide the agent to explore the most uncertain state-action pairs, where the uncertainty is characterized by both reward and transition bonus terms. There are two construction options for the exploration reward: option [I] employs a maximum operation on bonuses, and option [II] takes summation over all bonuses, leading to different theoretical results presented below. Then, Line \ref{line:pf-Q} constructs an optimistic Q-function $\overline Q_h^t$ based on the estimated transition, the transition bonus that guarantees optimism, and the exploration reward. Line \ref{line:pf-policy} generates the exploration policy $\overline \pi^t$, via which we further sample data as in Line \ref{line:pf-sample}. Note that although we collect rewards in Line \ref{line:pf-sample}, they are not used to construct $\overline Q_h^t$ as the exploration is only guided by uncertainty-based explore reward $\overline r^t$ rather than the estimated true reward, such that the collected data can have a sufficiently wide coverage of all policies. 

Finally, in the planning phase, i.e., Algorithm \ref{alg:exploit-tch}, we first estimate the reward and transition as 
\begin{align}
\begin{aligned} \label{eq:plan-est}
\hat{r}_{i,h}(s,a)= \frac{\sum_{\tau=1}^T\mathbf{1}_{\{(s,a)=(s_h^\tau,a_h^\tau)\}}r_{i,h}^\tau}{ N_h^T(s,a)\vee 1}, \quad \hat{\PP}_h(s'|s,a) = \frac{N_h^T(s,a,s')}{ N_h^T(s,a)\vee 1},    
\end{aligned}
\end{align}
and calculate their bonus terms as 
\begin{align}
\begin{aligned}\label{eq:plan-bonus}
\Psi_{i,h}(s,a)=\sqrt{\frac{2\log(8m|\cS||\cA|HT/\delta)}{ N_h^T(s,a)\vee 1}} \wedge 1,\quad \Phi_h(s,a)=\sqrt{\frac{2H^2|\cS|\log(8m|\cS||\cA|HT/\delta)}{N_h^T(s,a)\vee 1}} \wedge H,
\end{aligned}
\end{align}
based on the pre-collected data $\cD$ in Algorithm \ref{alg:pure-explore}. Then, for any input preference $\blambda$, we construct optimistic Q-functions, calculate $\tilde \nu_i$ to estimate $\tilde \nu_i^*$, and iteratively update $\bw^k$ and $\pi^k$ for $K$ rounds without further exploration, which significantly saves the exploration cost. The updating steps of $\bw^k$ and $\pi^k$ share a similar spirit as the ones in Algorithm \ref{alg:online-tch}. The updating rule of $\bw^k$ is a mirror ascent step, equivalent to solving the following problem.
\begin{align*}
	\max_{\bw\in\Delta_m} \eta_{k-1}\langle \bw, \blambda \odot (\tilde{\bV}_1^{k-1}(s_1)+\biota -\bV_1^{k-1}(s_1)) \rangle -  D_{\mathrm{KL}}( \bw, \bw^{k-1}).	
\end{align*}

\begin{algorithm}[tb]
	\small
	\caption{Preference-Free Exploration Stage for Multi-Objective RL}
	\label{alg:pure-explore}
	\setstretch{1.15}
	\begin{algorithmic}[1] 			
		\State {\bfseries Initialize:} $\overline{\pi}_h^0(\cdot|\cdot)=\Unif(\cA), \forall h\in [H]$, $\cD =\emptyset$
		\For{$t=1,\ldots, T$}
		\For{$h=H, H-1,\ldots, 1$}	
		
		\State\label{line:pf-estimate}Calculate the transition estimate $\hat\PP_h^t$, and reward and transition bonuses $\Psi_{i,h}^t$, $\Phi_h^t$, $\forall (i,h)$
		\State\label{line:pf-reward}Uncertainty-guided exploration reward $
		\overline{r}_h^t(\cdot,\cdot)=
		\begin{cases}
			\max\{\Phi_h^t(\cdot,\cdot)/H, \Psi_{1,h}^t(\cdot,\cdot),\cdots,\Psi_{m,h}^t(\cdot,\cdot) \} ~~\text{[\textbf{I}]}\\
			\Phi_h^t(\cdot,\cdot)/H + \sumi \Psi_{i,h}^t(\cdot,\cdot)\qquad \qquad\qquad~ \text{[\textbf{II}]}
		\end{cases}
		$
		\State\label{line:pf-Q}$\overline{Q}_h^t(\cdot,\cdot) = \{(\overline{r}_h^t + \hat\PP_h^t \overline{V}_{h+1}^t +\Phi_h^t)(\cdot,\cdot)\}_{[0,H-h+1]}$ where  $\overline{V}_{h+1}^t(\cdot)= \langle \overline{Q}_{h+1}^t(\cdot, \cdot), \overline{\pi}_{h+1}^t(\cdot|\cdot) \rangle_\cA$
		\State\label{line:pf-policy}Generate exploration policy $\overline{\pi}_h^t=\argmax_{\pi_h}\big\langle \overline{Q}_h^t(\cdot, \cdot), \pi_h(\cdot|\cdot) \big\rangle_\cA$ 
		\EndFor   
		
		\State\label{line:pf-sample}Uncertainty-guided sampling $\{s_h^t,a_h^t, s_{h+1}^t, (r_{i,h}^t)_{i=1}^m\}_{h=1}^H\sim (\overline{\pi}^t, \PP)$ and append it to $\cD$ 
		\EndFor 
		\State {\bfseries Return:} $\cD$
	\end{algorithmic}
	\normalsize
\end{algorithm}
\setlength{\floatsep}{3pt}
\begin{algorithm}[tb]
	\small
	\caption{Planning Stage for \texttt{PF-TchRL}}
	\label{alg:exploit-tch}
	\setstretch{1.15}
	\begin{algorithmic}[1]
		\State {\bfseries Input:} Preference $\blambda$ and pre-collected data $\cD:=\{s_h^t,a_h^t, s_{h+1}^t, (r_{i,h}^t)_{i=1}^m\}_{t=1}^T$ by Algorithm \ref{alg:pure-explore} 	
		
		\State {\bfseries Initialize:} $\eta_0 = 0$, $V_{i,h}^0(\cdot)=0$,  $w_i^0 = 1/m$, $\tilde{V}_{i,H+1}(\cdot)=0, ~\forall (i,h)\in[m]\times[H]$
		\State\label{line:plan-est}Calculate reward and transition estimates $\hat{r}_{i,h}$, $\hat{\PP}_h$ and bonus terms $\Psi_{i,h}$, $\Phi_h$ via \eqref{eq:plan-est} \eqref{eq:plan-bonus}, $\forall (i,h)$
		
		\For{$h = H,H-1,\ldots,1$}  
		
		\State $\tilde{Q}_{i,h}(\cdot,\cdot) = \{ (\hat{r}_{i,h} + \hat{\PP}_h \tilde{V}_{i,h+1} + \Phi_h + \Psi_{i,h}) (\cdot,\cdot)\}_{[0,H-h+1]},~ \forall i\in[m]$ 
		
		\State  $\tilde{\nu}_{i,h}=\argmax_{\nu_h}\langle \tilde{Q}_{i,h}(\cdot, \cdot), \nu_h(\cdot|\cdot) \rangle_\cA$ and let $\tilde{V}_{i,h}(\cdot)=\langle  \tilde{Q}_{i,h}(\cdot, \cdot), \tilde{\nu}_{i,h}(\cdot|\cdot) \rangle_\cA, \forall i\in[m]$
		\EndFor

		\For{$k=1,\ldots,K$}
		
		\State Update $\bw$ via $w_i^k \propto w_i^{k-1} \cdot \exp\{ \eta_{k-1} \cdot \lambda_i  (\tilde{V}_{i,1}(s_1)+\iota -V_{i,1}^{k-1}(s_1))\},~\forall i\in[m]$
		
		\For{$h = H,H-1,\ldots,1$} 
		
		\State $Q_{i,h}^k(\cdot,\cdot) = \{ (\hat{r}_{i,h} + \hat{\PP}_h V_{i,h+1}^k + \Phi_h + \Psi_{i,h}) (\cdot,\cdot)\}_{[0,H-h+1]},~ \forall i\in[m]$
		
		\State $\pi_h^k=\argmax_{\pi_h}\langle (\bw^k\odot\blambda)^\top \bQ_h^k(\cdot, \cdot), \pi_h(\cdot|\cdot) \rangle_\cA$ and let  $V_{i,h}^k(\cdot)=\langle  Q_{i,h}^k(\cdot, \cdot), \pi_h^k(\cdot|\cdot) \rangle_\cA, \forall i\in[m]$
		\EndFor
		
		\EndFor 
		\State {\bfseries Return:} $\{\pi^k\}_{k=1}^K$
	\end{algorithmic}
	\normalsize
\end{algorithm}

\begin{remark}[Exploration Reward]
	Line \ref{line:pf-reward} in Algorithm \ref{alg:pure-explore} defines the exploration reward $\overline r^t$ by taking maximum over bonus terms in option [I]. We note that option [I] fits the case where bonus terms share the same structure, e.g., $1/\sqrt{N_h^{t-1}(s,a)\vee 1}$ in the tabular case. Thus, the maximum is indeed taken over the factors other than $1/\sqrt{N_h^{t-1}(s,a)\vee 1}$ in 
	bonus terms. For general cases, when there is no such similar structure, e.g., different feature representations in function approximation settings, we can apply a summation of all bonuses as in option [II], which, however, will introduce extra an $\cO(m)$ term to the exploration complexity as we show in our theoretical result. It remains an intriguing direction to study our proposed algorithms under various function approximation settings.
\end{remark}

\begin{remark}[Comparison with Reward-Free RL]\label{re:reward-free} Algorithm \ref{alg:pure-explore} shares a similar spirit to the prior reward-free exploration methods (e.g., \citet{wang2020reward,qiu2021reward,jin2020reward,zhang2023optimal,zhang2021near}) but has two differences: \textbf{(1)} Algorithm \ref{alg:pure-explore} defines exploration rewards $\overline r^t$ using both reward and transition uncertainties, and \textbf{(2)} it also collects rewards for reward estimation in the planning phase. Reward-free RL defines $\overline r^t$ solely based on the transition bonus, and the reward function is fully given instead of being estimated. Hence, Algorithm \ref{alg:pure-explore} generalizes these prior reward-free exploration methods. The MORL work \citep{wu2021accommodating} also proposes a preference-free algorithm. But its algorithm is very similar to the prior reward-free methods as the reward function is not estimated and is directly given in the planning phase. We remark that Algorithm \ref{alg:pure-explore} is not limited to the Tchebycheff scalarization, but can be applied to any MORL scalarization technique. As shown in the following section, Algorithm \ref{alg:pure-explore} is also utilized for the smooth Tchebycheff scalarization.
\end{remark}

\vspace{5pt}
\noindent\textbf{Theoretical Result.} The following theorem provides theoretical guarantee for Algorithms \ref{alg:pure-explore} and \ref{alg:exploit-tch}.

\begin{theorem}[Theoretical Guarantee for \texttt{PF-TchRL}] \label{thm:pre-free}
		Setting $\eta_k = \sqrt{\log m /(H^2K)}$ if $k>0$ and $0$ otherwise, letting $\hat \pi:= \Mix(\pi^1,\pi^2,\cdots, \pi^K)$ be a mixture of the learned policies and $\pi_{\blambda}^*\in\argmin_{\pi\in\Pi} \tchl (\pi)$ for any preference vector $\blambda\in\Delta_m$, then with probability at least $1-\delta$, after $T$ rounds of exploration via Algorithm \ref{alg:pure-explore} and $K$ rounds of planning via Algorithm \ref{alg:exploit-tch}, we obtain
		\begin{align*}
			\mathrm{[\mathbf{I}]:}&\quad \tchl(\hat{\pi}) - \tchl(\pi_{\blambda}^*) \leq  \cO\Big( \sqrt{H^2\log m/K}+ \sqrt{ H^6 |\cS|^2 |\cA| \varrho/T} \Big),\\
			\mathrm{[\mathbf{II}]:}&\quad \tchl(\hat{\pi}) - \tchl(\pi_{\blambda}^*) \leq  \cO\Big( \sqrt{H^2\log m/K} +\sqrt{ (H^2|\cS|+m) H^4 |\cS| |\cA| \varrho/T} \Big).  
		\end{align*}
		where $\varrho:=\log (m|\cS||\cA|HT/\delta)$, and [I] and [II] correspond to different options in Algorithm \ref{alg:pure-explore}. 
\end{theorem}
Setting the numbers of the planning steps $K=\tilde\cO(H^2 / \varepsilon^2)$ and the exploration rounds $T=\tilde\cO(H^6|\cS|^2|\cA|/\varepsilon^2)$ for option [I] or $T=\tilde\cO((H^2|\cS|+m) H^4 |\cS| |\cA| /\varepsilon^2)$ for option [II], we can achieve $\tchl(\hat{\pi}) - \tchl(\pi_{\blambda}^*) \leq \varepsilon$. We can see option [II] has an extra dependence on $m$, stemming from the sum operation of all bonus terms. Moreover, compared with Theorem \ref{thm:tch}, Theorem \ref{thm:pre-free} has an additional factor $H^2$ which is also observed in the prior reward-free RL works when compared with online algorithms (see, e.g., \citet{jin2020provably,wang2020reward,yang2020provably,qiu2021reward}). However, although such an extra factor $H^2$ exists, when a learner expects to have sufficient coverage of the (weakly) Pareto optimal set under a large number of preferences $\blambda\in\Delta_m$, Algorithm \ref{alg:online-tch} requires exploration for every $\blambda$, while Algorithm \ref{alg:pure-explore} only explore once, significantly saving the exploration cost.

\section{Extension to Smooth Tchebycheff Scalarization}\label{sec:stch}

Most recently, inspired by smoothing techniques for non-smooth optimization problems \citep{nesterov2005smooth,beck2012smoothing,chen2012smoothing}, \citet{lin2024smooth} proposes a smoothed version of Tchebycheff scalarization for multi-objective optimization by utilizing the infimal convolution smoothing method \citet{beck2012smoothing}. In this section, we adapt the definition of the smooth Tchebycheff scalarization \citep{lin2024smooth} to the MORL scenario.
\begin{definition}[Smooth Tchebycheff Scalarization]\label{def:stch} Smooth Tchebycheff scalarization converts the original MORL problem into a minimization problem of the scalarized optimization target,
\begin{align}\label{eq:stch}
	\stchlm(\pi):= \mu \log \left[ \sumi \exp\left(\frac{\lambda_i (V_{i,1}^*(s_1) + \iota - V_{i,1}^\pi(s_1))}{\mu}\right)\right]   ,
\end{align}
where $\blambda = [\lambda_i]_{i=1}^m\in \Delta_m$ is a preference vector defined by the learner, $V_{i,1}^*(s_1):=\max_{\pi\in\Pi} V_{i,1}^\pi(s_1)$ is the maximum of the $i$-th value function, $\iota$ is a  sufficiently small pre-defined regularizer with $\iota > 0$, and $\mu$ is the smoothing parameter.
\end{definition} 
Specifically, according to \citet{lin2024smooth}, we note that $\stchlm(\pi)$ and $\tchl(\pi)$ satisfy the following relation,
\begin{align} \label{eq:stch-err}
\stchlm(\pi) - \mu \log m \leq  \tchl(\pi) \leq  \stchlm(\pi),
\end{align}
which indicates that the approximation error of $\tchl(\pi)$ by $\stchlm(\pi)$ is $\mu \log m$. If $\mu$ is set to be sufficiently small, then we obtain that $\tchl(\pi)\simeq \stchlm(\pi)$.  Moreover, if we set $\mu = \varepsilon/(2\log m)$, then $\tchl(\pi) - \min_{\pi\in\Pi} \tchl(\pi)\leq \stchlm(\pi) - \min_{\pi\in\Pi} \stchlm(\pi) + \varepsilon/2$. This implies that an $\varepsilon/2$-approximate optimal solution to $\min_{\pi\in\Pi} \stchlm(\pi)$ is also an $\varepsilon$-approximate optimal solution to $\min_{\pi\in\Pi} \tchl(\pi)$. 

We further study the properties of the solutions to $\min_{\pi\in\Pi} \stchlm(\pi)$ in the next proposition. 

	\begin{table*}[t] 
		\renewcommand{\arraystretch}{1.3}
		\centering
\begin{small}		
		\begin{tabular}{ | >{\centering\arraybackslash}m{3.98cm} 
				 | >{\centering\arraybackslash}m{1.15cm} |
				>{\centering\arraybackslash}m{1.15cm} |
				>{\centering\arraybackslash}m{1.15cm} | 
				>{\centering\arraybackslash}m{1.15cm} | >{\centering\arraybackslash}m{2.1cm} | >{\centering\arraybackslash}m{2.6cm} | } 
			\hline
			\rowcolor{lightgray}   & \multicolumn{2}{c|}{\textbf{Weak Pareto}}  & \multicolumn{2}{c|}{\textbf{Pareto}} & &  \\\hhline{|>{\arrayrulecolor{lightgray}}->{\arrayrulecolor{black}}|*4{-}|>{\arrayrulecolor{lightgray}}->{\arrayrulecolor{black}}|>{\arrayrulecolor{lightgray}}->{\arrayrulecolor{black}}|}
			\rowcolor{lightgray} \multirow{-2}{*}{\textbf{Scalarization Method}}  & \textbf{Sto.} & \textbf{Det.} & \textbf{Sto.} & \textbf{Det.} & \multirow{-2}{*}{\textbf{Controllable}} & \multirow{-2}{*}{\textbf{Discriminating}} \\ \hline
			Linear  & {\color{green}All} &{\color{red}Subset} & {\color{green}All} &{\color{red}Subset} &  \cmark{\color{red}$(*)$}&\cmark \\\hline
			Pareto Suboptimality Gap   & {\color{green}All}& {\color{green}All}& {\color{green}All}& {\color{green}All} & \xmark & \xmark  \\ \hline
			Tchebycheff& {\color{green}All}&{\color{green}All} & {\color{green}All} &{\color{green}All} &  \cmark & \xmark  \\ \hline
			Smooth Tchebycheff & {\color{green}All}&{\color{red}Subset} & {\color{green}All}&{\color{green}All}  &\cmark & \cmark  \\ \hline
		\end{tabular}
		\caption{\small Comparison of Scalarization Methods.  ``Controllable'' refers to whether a scalarization method has the controllability of its solution by the preference $\blambda$. 
			``Discriminating'' indicates whether the scalarization method is able to determine if its solution is a Pareto optimal policy, rather than merely a weakly Pareto optimal policy, based solely on whether the preference vector $\blambda$ is in $\Delta_m^o$ or not. 
			``Subset'' indicates that the scalarization method is capable of finding a subset of (weakly) Pareto optimal policies. ``All'' indicates that the scalarization method can identify all (weakly) Pareto optimal policies. ``Sto.'' and ``Det.'' denote the stochastic policy class and the deterministic policy class respectively. The notation {\color{red}$(*)$} indicates less controllability compared with (smooth) Tchebycheff scalarization, as discussed in Remark \ref{re:LinvsTch}.}
		\label{tab:comparison}
		\end{small}
	\end{table*}

\begin{proposition}\label{prop:stch}
For a stochastic policy class $\Pi$, the minimizers of smooth Tchebycheff scalarization satisfy $\{\pi~|~\pi\in\argmax_{\pi\in\Pi}\stchlm(\pi), \forall \blambda \in\Delta_m\}= \Pi_{\mathrm{W}}^*$ and $\{\pi~|~\pi\in\argmin_{\pi\in\Pi} \stchlm(\pi), \allowbreak \forall \blambda \in\Delta_m^o\}= \Pi_{\mathrm{P}}^*$ for any $\mu>0$.
For a deterministic policy class $\Pi$, we have $\{\pi~|~\pi\in\argmin_{\pi\in\Pi} \stchlm(\pi), \forall \blambda \in\Delta_m\}\subseteq \Pi_{\mathrm{W}}^*$ and that there exists $\mu^*>0$ such that for any $0<\mu<\mu^*$, $\{\pi~|~\pi\in\argmin_{\pi\in\Pi} \stchlm(\pi), \forall \blambda \in\Delta_m^o\} = \Pi_{\mathrm{P}}^*$. A weakly Pareto optimal policy may not be the solution to $\max_{\pi\in\Pi} \stchlm(\pi)$ for any $\blambda \in\Delta_m$ and $\mu>0$ when $\Pi$ is a deterministic policy class.

\end{proposition}
This proposition indicates that under certain conditions, \emph{we can identify all Pareto optimal policies by minimizing the smooth Tchebycheff scalarization with the preference $\blambda\in\Delta_m^o$ with excluding other weakly Pareto optimal policies that are not Pareto optimal}. This is a more favorable property than the one for the ordinary Tchebycheff scalarization as shown in Proposition \ref{prop:tch}. Minimizing the ordinary Tchebycheff scalarization for a $\blambda\in\Delta_m^o$ may find weakly Pareto optimal policies that are not Pareto optimal, as discussed in Proposition \ref{prop:tch}. We summarized each scalarization method in Table \ref{tab:comparison}. This demonstrates that the smooth Tchebycheff scalarization is more advantageous in finding Pareto optimal policies, which offers better controllability and discriminating capability.

We then have the following proposition to show the distribution of the Pareto optimal policies:
\begin{proposition}\label{prop:stch-unique} 
	For a policy class $\Pi$ (either deterministic or stochastic), there exists $0<\mu^*\leq \infty$ such that for any $0<\mu<\mu^*$, the Pareto optimal policies that are the solutions to $\min_{\pi\in\Pi}\stchlm(\pi)$ for a $\blambda\in \Delta_m^o$ have the same values on all objectives.
\end{proposition}
This proposition implies that all the solutions to $\min_{\pi\in\Pi}\stchlm(\pi)$ for a $\blambda\in \Delta_m^o$ are equivalent in terms of their values. \emph{Therefore, together with Proposition \ref{prop:stch}, the proposition above indicates that obtaining a single solution for each $\blambda\in\Delta_m^o$ can achieve every point in $\VV(\Pi)$ that corresponds to the Pareto front}. By solving for a single solution to $\min_{\pi\in\Pi}\stchlm(\pi)$ for each $\blambda\in\Delta_m^o$, we obtain representative Pareto optimal policies.

We note that \citet{lin2024smooth} proposes a gradient-based method to minimize the smooth Tchebycheff scalarization, which is difficult to generalize to the UCB-based RL methods. Therefore, we contribute to identifying the following equivalent optimization problem of $\min_{\pi\in\Pi} \stchlm(\pi)$, which better fits the MORL scenario. The following proposition also offers a favorable reformulation of smooth Tchebycheff scalarization that can be applied to multi-objective stochastic optimization or multi-objective online learning problems, which are not restricted to MORL.

\begin{proposition}[Equivalent Form of STCH] \label{prop:eq-stch} Letting $\bw=(w_1,w_2,\cdots, w_m) \in\Delta_m$, the smooth Tchebycheff scalarization is reformulated as $\stchlm(\pi) = \max_{\bw \in\Delta_m} \max_{\nu_i\in\Pi} \sumi [ w_i \lambda_i (V_{i,1}^{\nu_i}(s_1) + \iota - V_{i,1}^\pi(s_1)) -\mu  w_i\log w_i ]$. Then, the minimization problem of the smooth Tchebycheff scalarization can be reformulated as 
	\begin{align}
		\min_{\pi\in\Pi} \stchlm(\pi) = \min_{\pi\in\Pi} \max_{\bw \in\Delta_m} \max_{\{\nu_i\in\Pi\}_{i=1}^m} \sumi \left[ w_i \lambda_i \big(V_{i,1}^{\nu_i}(s_1) + \iota - V_{i,1}^\pi(s_1)\big) -\mu  w_i\log w_i\right]. \label{eq:eq-stch}
	\end{align}
\end{proposition} 
Comparing Proposition \ref{prop:eq-stch} to Proposition \ref{prop:eq-tch}, we note that there is only an extra regularization term $\sumi \mu  w_i\log w_i$ in \eqref{eq:eq-stch}, which indicates that we can solve the above problem \eqref{eq:eq-stch} using similar algorithms to the ones for $\min_{\pi\in\Pi} \tchl(\pi)$ with slight modification to the update of $\bw$. Moreover, we can show that such a regularization term can lead to a faster rate $\tilde\cO(1/(\mu T))$ (if $\mu$ is not too small) for learning the optimal $\bw$ rather than an $\tilde\cO(1/\sqrt{T})$ rate. 

\vspace{5pt}
\noindent\textbf{Algorithm.}  We propose algorithms to solve the problem in \eqref{eq:eq-stch}. Similar to the algorithms in Section \ref{sec:tchs} and Section \ref{sec:pre-free}, to learn all stochastic Pareto optimal policies, we also propose an online optimistic MORL algorithm, named \texttt{STchRL}, and a preference-free algorithm, named \texttt{PF-STchRL}. 

The online algorithm \texttt{STchRL} is summarized in Algorithm \ref{alg:online-stch}. In addition, the preference-free algorithm \texttt{PF-STchRL} consists of two stages, i.e., the preference-free exploration stage and the planning stage. Specifically, the exploration stage here uses the same algorithm as the one for the original Tchebycheff scalarization, e.g. Algorithm \ref{alg:online-stch}. We further propose a new algorithm for the planning stage for the smooth Tchebycheff scalarization as in Algorithm \ref{alg:exploit-stch}. The design of Algorithm \ref{alg:online-stch} and Algorithm \ref{alg:exploit-stch} is inspired by Algorithm \ref{alg:online-tch} and Algorithm  \ref{alg:exploit-tch}. 

The difference between Algorithm \ref{alg:online-stch} and Algorithm \ref{alg:online-tch} is their updating rules for $\bw$. Line \ref{line:stch-w} of Algorithm \ref{alg:online-stch} updates $\bw$ via a mirror ascent step, which is equivalent to solving the following maximization problem
\begin{align*}
\max_{\bw\in\Delta_m} \eta_{t-1}\langle \bw, \blambda \odot (\tilde{\bV}_1^{t-1}(s_1)+\biota -\bV_1^{t-1}(s_1)) - \mu \log \tilde \bw^{t-1} \rangle -  D_{\mathrm{KL}}( \bw, \tilde \bw^{t-1}),	
\end{align*}
where we let $\log \bw := (\log w_1, \log w_2, \cdots, \log w_m)$ be element-wise logarithmic operation with a slight abuse of the $\log$ operator. Here $\tilde \bw^{t-1}$ is obtained via a mixing step as in Line \ref{line:stch-mix} of Algorithm \ref{alg:online-stch} such that $\tilde w_i^{t-1}$ is a weighted average of the update from the last step, $w_i^{t-1}$, and $1/m$ with the weights $1-\alpha_{t-1}$ and $\alpha_{t-1}$. This mixing technique, which is employed in some prior works \citep{wei2019online,qiu2023gradient}, can guarantee the boundedness of $\log \tilde w_i^{t-1}$, which is critical to obtain an $\tilde\cO(1/(\mu T))$ rate, instead of $\tilde\cO(1/\sqrt{T})$, for learning the optimal $\bw$.

Moreover, the difference between Algorithm \ref{alg:exploit-stch} and Algorithm \ref{alg:exploit-tch} also lies in their updating rules for $\bw$. We first compute a weighted average of $w_i^{k-1}$ and $1/m$ to obtain $\tilde w_i^{k-1}$ in Line \ref{line:pf-stch-mix} of Algorithm \ref{alg:exploit-stch}. We then run a mirror ascent step in Line \ref{line:pf-stch-w} of this algorithm, which is equivalent to solving 
\begin{align*}
	\max_{\bw\in\Delta_m} \eta_{k-1}\langle \bw, \blambda \odot (\tilde{\bV}_1^{k-1}(s_1)+\biota -\bV_1^{k-1}(s_1)) - \mu \log \tilde \bw^{k-1} \rangle -  D_{\mathrm{KL}}( \bw, \tilde \bw^{k-1}),	
\end{align*}
which leads to an $\tilde\cO(1/(\mu K))$ rate, instead of $\tilde\cO(1/\sqrt{K})$, for learning the optimal $\bw$.

\begin{algorithm}[tb]
	\small
	\caption{\texttt{STchRL}: Multi-Objective RL via Smooth Tchebycheff Scalarization}
	\label{alg:online-stch}
	\setstretch{1.15}
	\begin{algorithmic}[1]
		\State {\bfseries Initialize:} $\blambda$, $\eta_0=\alpha_0=0$, $\pi_h^0(\cdot|\cdot)=\nu_{i,h}^0(\cdot|\cdot)=\Unif(\cA)$, $\cD_i = \hat{\cD}_i =\emptyset$, $w_i^0 = 1/m$, $V_{i,h}^0(\cdot)=\tilde{V}_{i,h}^0(\cdot)=0$
		
		\For{$t=1,2,\ldots,T$}
		\State\label{line:stch-mix}Weighted average: $\tilde w_i^{t-1}=(1-\alpha_{t-1})w_i^{t-1} + \alpha_{t-1}/m,~\forall i\in[m]$
		\State\label{line:stch-w}Update $\bw$ via $w_i^t \propto (\tilde w_i^{t-1})^{1-\mu\eta_{t-1}} \cdot \exp\{ \eta_{t-1} \cdot \lambda_i  (\tilde{V}_{i,1}^{t-1}(s_1)+\iota -V_{i,1}^{t-1}(s_1))\},~\forall i\in[m]$

		\For{$h=H,H-1,\ldots,1$}  
		\State\label{line:stch-auxQ}$\tilde{Q}_{i,h}^t = \texttt{OptQ}\big(\tilde{\cD}_{i,h}, \tilde{V}_{i,h+1}^t \big)$ where $\tilde{V}_{i,h+1}^t(\cdot)=\langle  \tilde{Q}_{i,h+1}^t(\cdot, \cdot), \tilde\nu_{i,h+1}^t(\cdot|\cdot) \rangle_\cA, ~\forall i\in[m]$
		
		\State\label{line:stch-auxpolicy}Update auxiliary policies $\tilde\nu_{i,h}^t=\argmax_{\nu_h}\langle \tilde{Q}_{i,h}^t(\cdot, \cdot), \nu_h(\cdot|\cdot) \rangle_\cA, ~\forall i\in[m]$ 
		
		\State\label{line:stch-mainQ}$Q_{i,h}^t = \texttt{OptQ}\big(\cD_{i,h}, V_{i,h+1}^t\big)$ where $V_{i,h+1}^t(\cdot)=\langle  Q_{i,h+1}^t(\cdot, \cdot), \pi_{h+1}^t(\cdot|\cdot) \rangle_\cA, \forall i\in[m]$
		
		\State\label{line:stch-mainpolicy}Update main policy $\pi_h^t=\argmax_{\pi_h}\langle (\bw^t\odot\blambda)^\top \bQ_h^t(\cdot, \cdot), \pi_h(\cdot|\cdot) \rangle_\cA$ 
		
		\EndFor
		\State\label{line:stch-auxsample}Sample data $\{\tilde{s}_{i,h}^t,\tilde{a}_{i,h}^t, \tilde{s}_{i,h+1}^t, \tilde{r}_{i,h}^t\}_{h=1}^H\sim (\tilde\nu_i^t, \PP)$ and append it to $\tilde{\cD}_{i,h}, ~\forall (i,h)$
		
		\State\label{line:stch-mainsample}Sample data $\{s_h^t,a_h^t, s_{h+1}^t, (r_{i,h}^t)_{i=1}^m\}_{h=1}^H\sim (\pi^t, \PP)$ and append $(s_h^t,a_h^t, s_{h+1}^t, r_{i,h}^t)$ to $\cD_{i,h}, ~\forall (i,h)$
		
		\EndFor 
		\State {\bfseries Return:} $\{\pi^t\}_{t=1}^T$
	\end{algorithmic}
	\normalsize
\end{algorithm}

\vspace{5pt}
\noindent\textbf{Theoretical Result.} Next, we present the theoretical results for \texttt{STchRL} and \texttt{PF-STchRL}.

\begin{theorem}[Theoretical Guarantee for \texttt{STchRL}] \label{thm:stch}
Setting $\eta_t = 1/(\mu t)$ and $\alpha_t = 1/t^2$ for $t\geq 1$ and $\eta_0 = \alpha_0 = 0$, letting $\hat \pi:= \Mix(\pi^1,\pi^2,\cdots, \pi^T)$ be a mixture of the learned policies, $\pi_{\mu,\blambda}^*\in \argmin_{\pi\in\Pi} \stchlm (\pi)$ for any preference vector $\blambda\in\Delta_m$ and smoothing parameter $\mu>0$, then with probability at least $1-\delta$, after $T$ rounds of Algorithm \ref{alg:online-tch}, we obtain	
\begin{align*}
	\stchlm(\hat{\pi}) - \stchlm(\pi_{\mu,\blambda}^*) \leq 	\cO\Big(H^2\log T/(\mu T) +  \mu\log^3 ( mT)/T + \sqrt{H^4 |\cS|^2|\cA| \varrho/T} \Big),
\end{align*}
where $\varrho:=\log (m|\cS||\cA|HT/\delta)$.
\end{theorem}
Compared with Theorem \ref{thm:tch} for the algorithm \texttt{TchRL}, Theorem \ref{thm:stch} shows that when $\mu$ is not too small, we have a faster $\tilde \cO(1/T)$ rate for learning $\bw$, thanks to the extra regularization term $\mu\sumi w_i \log w_i$ in \eqref{eq:eq-stch}. (For the case that $\mu\rightarrow 0^+$, i.e., $\mu$ is too small, we provide a detailed discussion in Remark \ref{re:mu}.) Since the bound in Theorem \ref{thm:stch} is still dominated by the leading term $\tilde\cO(1/\sqrt{T})$, to  achieve $\stchlm(\hat{\pi}) - \stchlm(\pi_{\mu,\blambda}^*) \leq \varepsilon$, we still need $T = \tilde\cO(\varepsilon^{-2})$. According to \eqref{eq:stch-err}, we have that if we set $\mu = \varepsilon/(2\log m)$, then $\tchl(\hat \pi) - \tchl(\pi_{\blambda}^*)\leq \stchlm(\hat \pi) - \stchlm(\pi_{\mu,\blambda}^*) + \varepsilon/2$. This implies that if $T=\tilde\cO(\varepsilon^{-2})$, then we have $\stchlm(\hat{\pi}) - \stchlm(\pi_{\mu,\blambda}^*) \leq \cO(\varepsilon)$, which further yields $\tchl(\hat \pi) - \tchl(\pi_{\blambda}^*)\leq \cO(\varepsilon)$, i.e., $\hat\pi$ is also an $\cO(\varepsilon)$-approximate optimal solution to $\min_{\pi\in\Pi} \tchl(\pi)$.

\begin{algorithm}[tb]
	\small
	\caption{Planning Stage for \texttt{PF-STchRL}}
	\label{alg:exploit-stch}
	\setstretch{1.15}
	\begin{algorithmic}[1]
		\State {\bfseries Input:} Preference $\blambda$ and pre-collected data $\cD:=\{s_h^t,a_h^t, s_{h+1}^t, (r_{i,h}^t)_{i=1}^m\}_{t=1}^T$ by Algorithm \ref{alg:pure-explore} 	
		\State {\bfseries Initialize:} $\eta_0=\alpha_0=0$, $w_i^0 = 1/m$,  $V_{i,h}^0(\cdot)=0$, $\tilde{V}_{i,H+1}(\cdot)=0, ~\forall (i,h)\in[m]\times[H]$
		\State\label{line:plan-est-stch}Calculate reward and transition estimates $\hat{r}_{i,h}$, $\hat{\PP}_h$ and bonus terms $\Psi_{i,h}$, $\Phi_h$ using data $\cD$, $\forall (i,h)$
		
		\For{$h = H,H-1,\ldots,1$}  
		
		\State $\tilde{Q}_{i,h}(\cdot,\cdot) = \{ (\hat{r}_{i,h} + \hat{\PP}_h \tilde{V}_{i,h+1} + \Phi_h + \Psi_{i,h}) (\cdot,\cdot)\}_{[0,H-h+1]},~ \forall i\in[m]$ 
		
		\State  $\tilde{\nu}_{i,h}=\argmax_{\nu_h}\langle \tilde{Q}_{i,h}(\cdot, \cdot), \nu_h(\cdot|\cdot) \rangle_\cA$ and let $\tilde{V}_{i,h}(\cdot)=\langle  \tilde{Q}_{i,h}(\cdot, \cdot), \tilde{\nu}_{i,h}(\cdot|\cdot) \rangle_\cA, \forall i\in[m]$
		\EndFor

		\For{$k=1,\ldots,K$}
		\State\label{line:pf-stch-mix}Weighted average: $\tilde w_i^{k-1}=(1-\alpha_{k-1})w_i^{k-1} + \alpha_{k-1}/m,~\forall i\in[m]$
		\State\label{line:pf-stch-w}Update $\bw$ via $w_i^k \propto (\tilde w_i^{k-1})^{1-\mu\eta_{k-1}} \cdot \exp\{ \eta_{k-1} \cdot \lambda_i  (\tilde{V}_{i,1}(s_1)+\iota -V_{i,1}^{k-1}(s_1))\},~\forall i\in[m]$
		
		\For{$h = H,H-1,\ldots,1$} 
		
		\State $Q_{i,h}^k(\cdot,\cdot) = \{ (\hat{r}_{i,h} + \hat{\PP}_h V_{i,h+1}^k + \Phi_h + \Psi_{i,h}) (\cdot,\cdot)\}_{[0,H-h+1]},~ \forall i\in[m]$
		
		\State $\pi_h^k=\argmax_{\pi_h}\langle (\bw^k\odot\blambda)^\top \bQ_h^k(\cdot, \cdot), \pi_h(\cdot|\cdot) \rangle_\cA$ and let  $V_{i,h}^k(\cdot)=\langle  Q_{i,h}^k(\cdot, \cdot), \pi_h^k(\cdot|\cdot) \rangle_\cA, \forall i\in[m]$
		\EndFor
		
		\EndFor 
		\State {\bfseries Return:} $\{\pi^k\}_{k=1}^K$
	\end{algorithmic}
	\normalsize
\end{algorithm}

\begin{theorem}[Theoretical Guarantee for \texttt{PF-STchRL}]\label{thm:pf-stch} Setting $\eta_k = 1/(\mu k)$ and $\alpha_k = 1/k^2$ for $k\geq 1$ and $\eta_0 = \alpha_0 = 0$, letting $\hat \pi:= \Mix(\pi^1,\pi^2,\cdots, \pi^K)$ be a mixture of the learned policies and $\pi_{\mu,\blambda}^*:=\argmin_{\pi\in\Pi} \stchlm (\pi)$ for any preference vector $\blambda\in\Delta_m$, then with probability at least $1-\delta$, after $T$ rounds of exploration via Algorithm \ref{alg:pure-explore} and $K$ rounds of planning via Algorithm \ref{alg:exploit-stch}, we obtain
\begin{small}
	\begin{align*}
		\mathrm{[\mathbf{I}]:}&\quad \stchlm(\hat{\pi}) - \stchlm(\pi_{\mu,\blambda}^*) \leq  \cO\bigg( H^2\log K/(\mu K) +  \mu\log^3 ( mK)/K +  \sqrt{ H^6 |\cS|^2 |\cA| \varrho/T} \bigg),\\
		\mathrm{[\mathbf{II}]:}&\quad \stchlm(\hat{\pi}) - \stchlm(\pi_{\mu,\blambda}^*) \leq  \cO\Big( H^2\log K/(\mu K) +  \mu\log^3 ( mK)/K + \sqrt{ (H^2|\cS|+m) H^4 |\cS| |\cA| \varrho/T} \Big),
	\end{align*}
\end{small}	
where $\varrho:=\log (m|\cS||\cA|HT/\delta)$, and [I] and [II] correspond to different options in Algorithm \ref{alg:pure-explore}. 
\end{theorem}

Compared with Theorem \ref{thm:pre-free} for \texttt{PF-TchRL}, Theorem \ref{thm:pf-stch} shows that when $\mu$ is not too small, we have a faster $\tilde \cO(1/K)$ rate for the planning stage, thanks to the extra regularization term $\mu\sumi w_i \log w_i$ in \eqref{eq:eq-stch}. Thus, to achieve $\stchlm(\hat{\pi}) - \stchlm(\pi_{\mu,\blambda}^*) \leq \varepsilon$, we need only $K = \tilde\cO((\mu\varepsilon)^{-1})$ for the planning stage but still $T = \tilde\cO(\varepsilon^{-2})$ for the preference-free exploration stage.

\begin{remark}[Discussion on $\mu$] \label{re:mu}
We note that when $\mu$ is too small with $\mu\rightarrow 0^+$, the results in Theorem \ref{thm:stch} and Theorem \ref{thm:pf-stch} can be worse as they have a dependence on $1/\mu$. In Appendix \ref{sec:additional}, we can further show that the term $\tilde\cO(1/(\mu T)+ \mu/T)$ in Theorem \ref{thm:stch} (or $\tilde\cO(1/(\mu K)+ \mu/K)$ in Theorem \ref{thm:pf-stch}) can be replaced by $\tilde\cO(1/\sqrt{T})$ (or $\tilde\cO(1/\sqrt{K})$) without the dependence on $1/\mu$ under different settings of the step sizes $\eta_t$ (or $\eta_k$), which can be better rates for the case of $\mu\rightarrow 0^+$. Such results are associated with the learning guarantees for the update of $\bw$ in both algorithms. We refer readers to Appendix \ref{sec:additional} for detailed analysis. We remark that under the different settings of the step sizes, the weighted average steps for $\bw$ in both algorithms are not necessary. 
\end{remark}


\section{Theoretical Analysis} \label{sec:theory}

\subsection{Proof Sketch of Theorem \ref{thm:tch}} 
	Defining $\pi_{\blambda}^*:=\minimize_{\pi\in\Pi}\tchl(\pi)$, we can decompose $\tchl(\hat{\pi}) - \tchl(\pi_{\blambda}^*)$ into three error terms as 
	\begin{align}
		&\tchl(\hat{\pi}) - \tchl(\pi_{\blambda}^*) \nonumber\\[-3pt]
		&\quad \leq \underbrace{\max_{\bw\in\Delta_m}   \frac{1}{T}\sumt \sumi w_i\lambda_i \Big(V_{i,1}^*(s_1)  -  V_{i,1}^{\tilde\nu_i^t}(s_1)\Big)}_{\text{Err(I)}} + \underbrace{\max_{\bw \in\Delta_m} \frac{1}{T} \sum_{t=1}^T \sum_{i=1}^m (w_i-w_i^t)\lambda_i \Big(V_{i,1}^{\tilde \nu_i^t}(s_1) + \iota - V_{i,1}^{\pi^t}(s_1)\Big)}_{\text{Err(II)}}	\nonumber\\[-3pt]
		&\qquad + \underbrace{\frac{1}{T}\sum_{t=1}^T\sum_{i=1}^m w_i^t\lambda_i \Big(V_{i,1}^{\pi_{\blambda}^*}(s_1) - V_{i,1}^{\pi^t}(s_1)\Big)}_{\text{Err(III)}}, \label{eq:pf-sk-onlinetch}
	\end{align}
	where we use the definition of $\hat\pi$ in the theorem such that $V^{\hat\pi}_{i,1}(s_1)= \frac{1}{T} \sumt V^{\pi^t}_{i,1}(s_1)$.
	Specifically, Err(I) depicts the learning error for $V_{i,1}^*(s_1)$, i.e., the optimal values for each objective. Err(II) is the error of learning $\bw^t$ toward the optimal one. Err(III) is associated with the learning error for achieving the value under the (weakly) Pareto optimal policy corresponding to the preference $\blambda$, which is $V_{i,1}^{\pi_{\blambda}^*}(s_1)$. This decomposition also reflects the fundamental idea of algorithm design for Algorithm \ref{alg:online-tch}, which features three alternating update steps.
	 
	By the construction of optimistic Q-functions and the policy update for $\tilde \nu_{i,h}^t$ and $\pi_h^t$, we have 
	\begin{align*}
	V_{i,1}^*(s_1) - \tilde V_{i,1}^t(s_1)\leq 0,\quad V_{i,1}^{\pi^t}(s_1) - V_{i,1}^t(s_1)\leq 0,  \quad \text{and} \quad  \sum_{i=1}^m w_i^t\lambda_i \Big(V_{i,1}^{\pi_{\blambda}^*}(s_1) - V_{i,1}^t(s_1)\Big) \leq 0,	
	\end{align*}
	which further leads to 
	\begin{align*}
	\text{Err(I)} &\leq  \max_{\bw\in\Delta_m}   \frac{1}{T}\sumt \sumi w_i\lambda_i \Big(V_{i,1}^t(s_1)  -  V_{i,1}^{\tilde \nu_i^t}(s_1)\Big),\\
	\text{Err(II)} &\leq \max_{\bw \in\Delta_m} \frac{1}{T} \sum_{t=1}^T \sum_{i=1}^m (w_i-w_i^t)\lambda_i \Big(\tilde V_{i,1}^t(s_1) + \iota - V_{i,1}^t(s_1)\Big) \\
	&\quad + \frac{1}{T}\sum_{t=1}^T \sum_{i=1}^m \lambda_i  \Big( \tilde V_{i,1}^t(s_1)-V_{i,1}^{\tilde\nu_i^t}(s_1) + V_{i,1}^t(s_1) - V_{i,1}^{\pi^t}(s_1)\Big),\\
	\text{Err(III)} &\leq \frac{1}{T}\sum_{t=1}^T\sum_{i=1}^m w_i^t\lambda_i \Big(V_{i,1}^t(s_1) - V_{i,1}^{\pi^t}(s_1)\Big),
	\end{align*}
	with high probability. Now the upper bounds of Err(I) and Err(III) are associated with the average of on-policy bonus values over $T$ rounds under the policies $\tilde \nu_i^t$ and $\pi^t$ respectively, which can be bounded by $\tilde\cO(T^{-1/2})$ as shown in our proof, i.e.,
	\begin{align*}
		\text{Err(I)} + \text{Err(III)} &\leq  \tilde \cO\bigg(\frac{1}{\sqrt{T}} +  \frac{1}{T}\sum_{t=1}^T\sum_{h=1}^H [\tilde \Psi_{i,h}^t(\tilde s_{i,h}^t, \tilde a_{i,h}^t)  +   \tilde \Phi_{i,h}^t(\tilde s_{i,h}^t, \tilde a_{i,h}^t)]\bigg) \leq \tilde \cO\bigg(\frac{1}{\sqrt{T}}\bigg),
	\end{align*}
	Moreover, the upper bound of Err(II) corresponds to the mirror accent step and the average of on-policy bonus values over $T$ rounds under both policies $\tilde \nu_i^t$ and $\pi^t$,
	which thus can be bounded as
	\begin{align*}
		\text{Err(II)} \leq \tilde \cO\bigg(\frac{1}{\sqrt{T}}\bigg).
	\end{align*}
	 Combining the above bounds, we eventually obtain $\tchl(\hat{\pi}) - \tchl(\pi_{\blambda}^*) \leq  \tilde\cO(T^{-1/2})$ with high probability.

\subsection{Proof Sketch of Theorem \ref{thm:pre-free}}
 
We can apply a similar decomposition as above for the proof of Theorem \ref{thm:tch}. By optimism and policy updating rules in Algorithm \ref{alg:exploit-tch} for $\tilde\nu_i$ and $\pi^k$, we directly obtain
\begin{align}
	\begin{aligned}\label{eq:pf-sk-pftch}
&\tchl(\hat{\pi}) - \tchl(\pi_{\blambda}^*)\\[-3pt]
&\qquad \lesssim \underbrace{\sumi\lambda_i \Big(\tilde V_{i,1}(s_1)  -  V_{i,1}^{\tilde \nu_i}(s_1)\Big)}_{\text{Err(IV)}} + \underbrace{\frac{1}{K}\sumk\sum_{i=1}^m \lambda_i \Big(V_{i,1}^k(s_1) - V_{i,1}^{\pi^k}(s_1)\Big)}_{\text{Err(V)}} \\[-3pt]
&\qquad\quad+ \underbrace{\max_{\bw \in\Delta_m} \frac{1}{K}\sumk \sum_{i=1}^m (w_i-w_i^k)\lambda_i \Big(\tilde V_{i,1}(s_1) + \iota - V_{i,1}^k(s_1)\Big)}_{\text{Err(VI)}}
\end{aligned}
\end{align}
with high probability. Furthermore, we can show that Err(IV) and Err(V) connect to the exploration phase in Algorithm \ref{alg:pure-explore}, and thus the following relations hold with high probability
\begin{align*}
&\text{Err(IV)} + \text{Err(V)} \\
&\qquad \leq \cO\bigg(\frac{1}{T}\sumt \overline{V}_1^t(s_1)\bigg) \leq \tilde\cO\bigg(\frac{1}{\sqrt{T}}  + \frac{1}{T} \sumt\sumh [\overline r_h^t(s_h^t, a_h^t) +\Phi_h^t(s_h^t, a_h^t)]\bigg) \leq \tilde\cO\bigg(\frac{1}{\sqrt{T}}\bigg),
\end{align*}
where we can see that Err(IV) and Err(V) are bounded by the average of the exploration reward and the transition bonus defined on the data collected in Algorithm \ref{alg:pure-explore}. Thus, we obtain that both terms are bounded by $\tilde\cO(T^{-1/2})$ by the design of the exploration reward using the bonus terms of reward and transition estimation. Moreover, different options in Algorithm \ref{alg:pure-explore} will further lead to different dependence on $|\cS|$, $m$, and $H$, consequently resulting in two convergence guarantees in Theorem \ref{thm:pre-free}. On the other hand, we have
\begin{align*}
\text{Err(VI)}\leq \tilde\cO\bigg(\frac{1}{\sqrt{K}}\bigg),	
\end{align*}
due to the mirror ascent step in the planning phase. Therefore, combining the above results, with high probability, we have $\tchl(\hat{\pi}) - \tchl(\pi_{\blambda}^*)\leq \tilde\cO(T^{-1/2} + K^{-1/2})$.

\subsection{Proof of Sketch of Theorem \ref{thm:stch}}
Defining $\pi_{\mu,\blambda}^*:=\minimize_{\pi\in\Pi}\stchlm(\pi)$, we can decompose $\stchlm(\hat{\pi}) - \stchlm(\pi_{\mu,\blambda}^*)$ into three error terms as 
\begin{align}
	\begin{aligned}\label{eq:pf-sk-onlinestch}
	&\stchlm(\hat{\pi}) - \stchlm(\pi_{\mu,\blambda}^*) \\[-3pt]
	&\qquad \leq \underbrace{\max_{\bw\in\Delta_m}   \frac{1}{T}\sumt \sumi w_i\lambda_i \Big(V_{i,1}^*(s_1)  -  V_{i,1}^{\tilde\nu_i^t}(s_1)\Big)}_{\text{Err(VII)}} \\[-3pt]
	&\quad\qquad + \underbrace{\max_{\bw \in\Delta_m} \frac{1}{T} \sum_{t=1}^T \sum_{i=1}^m \Big[(w_i-w_i^t)\lambda_i \Big(V_{i,1}^{\tilde \nu_i^t}(s_1) + \iota - V_{i,1}^{\pi^t}(s_1)\Big) -w_i\log w_i +w_i^t\log w_i^t\Big]}_{\text{Err(VIII)}}	\\[-3pt]
	&\quad \qquad + \underbrace{\frac{1}{T}\sum_{t=1}^T\sum_{i=1}^m w_i^t\lambda_i \Big(V_{i,1}^{\pi_{\blambda}^*}(s_1) - V_{i,1}^{\pi^t}(s_1)\Big)}_{\text{Err(IX)}}.
\end{aligned}
\end{align}
As we can observe, the decomposition \eqref{eq:pf-sk-onlinestch} is similar to \eqref{eq:pf-sk-onlinetch} except for the term Err(VIII). Comparing Err(VIII) with Err(II) in \eqref{eq:pf-sk-onlinetch}, Err(VIII) has extra regularization terms $w_i\log w_i$, $w_i^t\log w_i^t$, which will lead to an $\tilde\cO(1/(\mu T)+\mu/T)$ rate associated with the mirror ascent step for $\bw$ instead of $\tilde\cO(1/\sqrt{T})$. For other terms in \eqref{eq:pf-sk-onlinestch}, they admit the same $\tilde\cO(1/\sqrt{T})$ rate as in \eqref{eq:pf-sk-onlinetch}. Thus, combining these results gives the $\tilde\cO(1/\sqrt{T} + 1/(\mu T) +\mu/T)$ rate as in Theorem \ref{thm:stch}.

\subsection{Proof of Sketch of Theorem \ref{thm:pf-stch}}
For Theorem \ref{thm:pf-stch}, we can apply a similar decomposition as above, and we obtain
\begin{align}
	\begin{aligned}\label{eq:pf-sk-pfstch}
	&\tchl(\hat{\pi}) - \tchl(\pi_{\blambda}^*)\\[-3pt]
	&\qquad \lesssim \underbrace{\sumi\lambda_i \Big(\tilde V_{i,1}(s_1)  -  V_{i,1}^{\tilde \nu_i}(s_1)\Big)}_{\text{Err(X)}} + \underbrace{\frac{1}{K}\sumk\sum_{i=1}^m \lambda_i \Big(V_{i,1}^k(s_1) - V_{i,1}^{\pi^k}(s_1)\Big)}_{\text{Err(XI)}} \\[-3pt]
	&\qquad\quad+ \underbrace{\max_{\bw \in\Delta_m} \frac{1}{K}\sumk \sum_{i=1}^m \Big[ (w_i-w_i^k)\lambda_i \Big(\tilde V_{i,1}(s_1) + \iota - V_{i,1}^k(s_1)\Big)-w_i\log w_i +w_i^k\log w_i^k\Big]}_{\text{Err(XII)}}.
	\end{aligned}
\end{align}
The decomposition \eqref{eq:pf-sk-pfstch} is similar to \eqref{eq:pf-sk-pftch} except for the term Err(XII). Comparing Err(XII) with Err(VI) in \eqref{eq:pf-sk-pftch}, Err(XII) also has extra regularization terms $w_i\log w_i$, $w_i^t\log w_i^t$, which thus leads to an $\tilde\cO(1/(\mu K)+\mu/K)$ rate associated with the mirror ascent step for $\bw$ instead of $\tilde\cO(1/\sqrt{K})$ in Theorem \ref{thm:pre-free}. For other terms above, they admit the same $\tilde\cO(1/\sqrt{T})$ rate as in \eqref{eq:pf-sk-pftch}. Thus, combining these results gives the $\tilde\cO(1/\sqrt{T} + 1/(\mu K) +\mu/K)$ rate in Theorem \ref{thm:pf-stch}.

\section{Conclusion}
This paper investigates MORL, which focuses on learning a Pareto optimal policy in the presence of multiple reward functions. Our work first systematically analyzes several multi-objective optimization targets and identifies Tchebycheff scalarization as a favorable scalarization method for MORL. We then reformulate its minimization problem into a new min-max-max optimization problem and propose an online UCB-based algorithm and a preference-free framework to learn all Pareto optimal policies with provable sample efficiency. Finally, we analyze an extension of Tchebycheff scalarization named smooth Tchebycheff scalarization and extend our algorithms and theoretical analysis to this optimization target.

\bibliography{example_paper}

\begin{thebibliography}{95}
\providecommand{\natexlab}[1]{#1}
\providecommand{\url}[1]{\texttt{#1}}
\expandafter\ifx\csname urlstyle\endcsname\relax
  \providecommand{\doi}[1]{doi: #1}\else
  \providecommand{\doi}{doi: \begingroup \urlstyle{rm}\Url}\fi

\bibitem[Agarwal et~al.(2022)Agarwal, Aggarwal, and Lan]{agarwal2022multi}
M.~Agarwal, V.~Aggarwal, and T.~Lan.
\newblock Multi-objective reinforcement learning with non-linear scalarization.
\newblock In \emph{Proceedings of the 21st International Conference on
  Autonomous Agents and Multiagent Systems}, pages 9--17, 2022.

\bibitem[Altman(2021)]{altman2021constrained}
E.~Altman.
\newblock \emph{Constrained Markov decision processes}.
\newblock Routledge, 2021.

\bibitem[Andrychowicz et~al.(2017)Andrychowicz, Wolski, Ray, Schneider, Fong,
  Welinder, McGrew, Tobin, Pieter~Abbeel, and
  Zaremba]{andrychowicz2017hindsight}
M.~Andrychowicz, F.~Wolski, A.~Ray, J.~Schneider, R.~Fong, P.~Welinder,
  B.~McGrew, J.~Tobin, O.~Pieter~Abbeel, and W.~Zaremba.
\newblock Hindsight experience replay.
\newblock \emph{Advances in neural information processing systems}, 30, 2017.

\bibitem[Azar et~al.(2017)Azar, Osband, and Munos]{azar2017minimax}
M.~G. Azar, I.~Osband, and R.~Munos.
\newblock Minimax regret bounds for reinforcement learning.
\newblock In \emph{International conference on machine learning}, pages
  263--272. PMLR, 2017.

\bibitem[Barrett and Narayanan(2008)]{barrett2008learning}
L.~Barrett and S.~Narayanan.
\newblock Learning all optimal policies with multiple criteria.
\newblock In \emph{Proceedings of the 25th international conference on Machine
  learning}, pages 41--47, 2008.

\bibitem[Beck and Teboulle(2012)]{beck2012smoothing}
A.~Beck and M.~Teboulle.
\newblock Smoothing and first order methods: A unified framework.
\newblock \emph{SIAM Journal on Optimization}, 22\penalty0 (2):\penalty0
  557--580, 2012.

\bibitem[Bowman~Jr(1976)]{bowman1976relationship}
V.~J. Bowman~Jr.
\newblock On the relationship of the tchebycheff norm and the efficient
  frontier of multiple-criteria objectives.
\newblock In \emph{Multiple Criteria Decision Making: Proceedings of a
  Conference Jouy-en-Josas, France May 21--23, 1975}, pages 76--86. Springer,
  1976.

\bibitem[Boyd and Vandenberghe(2004)]{boyd2004convex}
S.~P. Boyd and L.~Vandenberghe.
\newblock \emph{Convex optimization}.
\newblock Cambridge university press, 2004.

\bibitem[Busa-Fekete et~al.(2017)Busa-Fekete, Sz{\"o}r{\'e}nyi, Weng, and
  Mannor]{busa2017multi}
R.~Busa-Fekete, B.~Sz{\"o}r{\'e}nyi, P.~Weng, and S.~Mannor.
\newblock Multi-objective bandits: Optimizing the generalized gini index.
\newblock In \emph{International Conference on Machine Learning}, pages
  625--634. PMLR, 2017.

\bibitem[Caramia et~al.(2020)Caramia, Dell’Olmo, Caramia, and
  Dell’Olmo]{caramia2020multi}
M.~Caramia, P.~Dell’Olmo, M.~Caramia, and P.~Dell’Olmo.
\newblock Multi-objective optimization.
\newblock \emph{Multi-objective Management in Freight Logistics: Increasing
  Capacity, Service Level, Sustainability, and Safety with Optimization
  Algorithms}, pages 21--51, 2020.

\bibitem[Chen et~al.(2022)Chen, Modi, Krishnamurthy, Jiang, and
  Agarwal]{chen2022statistical}
J.~Chen, A.~Modi, A.~Krishnamurthy, N.~Jiang, and A.~Agarwal.
\newblock On the statistical efficiency of reward-free exploration in
  non-linear rl.
\newblock \emph{Advances in Neural Information Processing Systems},
  35:\penalty0 20960--20973, 2022.

\bibitem[Chen et~al.(2024)Chen, Fernando, Ying, and Chen]{chen2024three}
L.~Chen, H.~Fernando, Y.~Ying, and T.~Chen.
\newblock Three-way trade-off in multi-objective learning: Optimization,
  generalization and conflict-avoidance.
\newblock \emph{Advances in Neural Information Processing Systems}, 36, 2024.

\bibitem[Chen et~al.(2023)Chen, Tian, Fan, Li, He, and Jin]{chen2023preference}
W.~Chen, J.~Tian, C.~Fan, Y.~Li, H.~He, and Y.~Jin.
\newblock Preference-controlled multi-objective reinforcement learning for
  conditional text generation.
\newblock In \emph{Proceedings of the AAAI Conference on Artificial
  Intelligence}, volume~37, pages 12662--12672, 2023.

\bibitem[Chen(2012)]{chen2012smoothing}
X.~Chen.
\newblock Smoothing methods for nonsmooth, nonconvex minimization.
\newblock \emph{Mathematical programming}, 134:\penalty0 71--99, 2012.

\bibitem[Chen et~al.(2019)Chen, Ghadirzadeh, Bj{\"o}rkman, and
  Jensfelt]{chen2019meta}
X.~Chen, A.~Ghadirzadeh, M.~Bj{\"o}rkman, and P.~Jensfelt.
\newblock Meta-learning for multi-objective reinforcement learning.
\newblock In \emph{2019 IEEE/RSJ International Conference on Intelligent Robots
  and Systems (IROS)}, pages 977--983. IEEE, 2019.

\bibitem[Chen et~al.(2020)Chen, Ngiam, Huang, Luong, Kretzschmar, Chai, and
  Anguelov]{chen2020just}
Z.~Chen, J.~Ngiam, Y.~Huang, T.~Luong, H.~Kretzschmar, Y.~Chai, and
  D.~Anguelov.
\newblock Just pick a sign: Optimizing deep multitask models with gradient sign
  dropout.
\newblock \emph{Advances in Neural Information Processing Systems},
  33:\penalty0 2039--2050, 2020.

\bibitem[Cheng et~al.(2023)Cheng, Huang, Yang, and Liang]{cheng2023improved}
Y.~Cheng, R.~Huang, J.~Yang, and Y.~Liang.
\newblock Improved sample complexity for reward-free reinforcement learning
  under low-rank mdps.
\newblock \emph{arXiv preprint arXiv:2303.10859}, 2023.

\bibitem[Choo and Atkins(1983)]{choo1983proper}
E.~U. Choo and D.~R. Atkins.
\newblock Proper efficiency in nonconvex multicriteria programming.
\newblock \emph{Mathematics of Operations Research}, 8\penalty0 (3):\penalty0
  467--470, 1983.

\bibitem[Das and Dennis(1997)]{das1997closer}
I.~Das and J.~E. Dennis.
\newblock A closer look at drawbacks of minimizing weighted sums of objectives
  for pareto set generation in multicriteria optimization problems.
\newblock \emph{Structural optimization}, 14:\penalty0 63--69, 1997.

\bibitem[Deb et~al.(2016)Deb, Sindhya, and Hakanen]{deb2016multi}
K.~Deb, K.~Sindhya, and J.~Hakanen.
\newblock Multi-objective optimization.
\newblock In \emph{Decision sciences}, pages 161--200. CRC Press, 2016.

\bibitem[Drugan and Nowe(2013)]{drugan2013designing}
M.~M. Drugan and A.~Nowe.
\newblock Designing multi-objective multi-armed bandits algorithms: A study.
\newblock In \emph{The 2013 international joint conference on neural networks
  (IJCNN)}, pages 1--8. IEEE, 2013.

\bibitem[Ehrgott(2005)]{ehrgott2005multicriteria}
M.~Ehrgott.
\newblock \emph{Multicriteria optimization}, volume 491.
\newblock Springer Science \& Business Media, 2005.

\bibitem[Ehrgott and Wiecek(2005)]{ehrgott2005multiobjective}
M.~Ehrgott and M.~M. Wiecek.
\newblock Multiobjective programming.
\newblock \emph{Multiple criteria decision analysis: State of the art surveys},
  78:\penalty0 667--708, 2005.

\bibitem[Fernando et~al.(2022)Fernando, Shen, Liu, Chaudhury, Murugesan, and
  Chen]{fernando2022mitigating}
H.~D. Fernando, H.~Shen, M.~Liu, S.~Chaudhury, K.~Murugesan, and T.~Chen.
\newblock Mitigating gradient bias in multi-objective learning: A provably
  convergent approach.
\newblock In \emph{The Eleventh International Conference on Learning
  Representations}, 2022.

\bibitem[Geoffrion(1968)]{geoffrion1968proper}
A.~M. Geoffrion.
\newblock Proper efficiency and the theory of vector maximization.
\newblock \emph{Journal of mathematical analysis and applications}, 22\penalty0
  (3):\penalty0 618--630, 1968.

\bibitem[Giagkiozis and Fleming(2015)]{giagkiozis2015methods}
I.~Giagkiozis and P.~J. Fleming.
\newblock Methods for multi-objective optimization: An analysis.
\newblock \emph{Information Sciences}, 293:\penalty0 338--350, 2015.

\bibitem[Gunantara(2018)]{gunantara2018review}
N.~Gunantara.
\newblock A review of multi-objective optimization: Methods and its
  applications.
\newblock \emph{Cogent Engineering}, 5\penalty0 (1):\penalty0 1502242, 2018.

\bibitem[Guo et~al.(2024)Guo, Cui, Yuan, Ding, Wang, Chen, Sun, Xie, Zhou, Lin,
  et~al.]{guo2024controllable}
Y.~Guo, G.~Cui, L.~Yuan, N.~Ding, J.~Wang, H.~Chen, B.~Sun, R.~Xie, J.~Zhou,
  Y.~Lin, et~al.
\newblock Controllable preference optimization: Toward controllable
  multi-objective alignment.
\newblock \emph{arXiv preprint arXiv:2402.19085}, 2024.

\bibitem[Hayes et~al.(2022)Hayes, R{\u{a}}dulescu, Bargiacchi,
  K{\"a}llstr{\"o}m, Macfarlane, Reymond, Verstraeten, Zintgraf, Dazeley,
  Heintz, et~al.]{hayes2022practical}
C.~F. Hayes, R.~R{\u{a}}dulescu, E.~Bargiacchi, J.~K{\"a}llstr{\"o}m,
  M.~Macfarlane, M.~Reymond, T.~Verstraeten, L.~M. Zintgraf, R.~Dazeley,
  F.~Heintz, et~al.
\newblock A practical guide to multi-objective reinforcement learning and
  planning.
\newblock \emph{Autonomous Agents and Multi-Agent Systems}, 36\penalty0
  (1):\penalty0 26, 2022.

\bibitem[Hu et~al.(2024)Hu, Xian, Wu, Fan, Yin, and Zhao]{hu2024revisiting}
Y.~Hu, R.~Xian, Q.~Wu, Q.~Fan, L.~Yin, and H.~Zhao.
\newblock Revisiting scalarization in multi-task learning: A theoretical
  perspective.
\newblock \emph{Advances in Neural Information Processing Systems}, 36, 2024.

\bibitem[Hwang et~al.(2023)Hwang, Weihs, Park, Lee, Kembhavi, and
  Ehsani]{hwang2023promptable}
M.~Hwang, L.~Weihs, C.~Park, K.~Lee, A.~Kembhavi, and K.~Ehsani.
\newblock Promptable behaviors: Personalizing multi-objective rewards from
  human preferences.
\newblock \emph{arXiv preprint arXiv:2312.09337}, 2023.

\bibitem[Jaksch et~al.(2010)Jaksch, Ortner, and Auer]{jaksch2010near}
T.~Jaksch, R.~Ortner, and P.~Auer.
\newblock Near-optimal regret bounds for reinforcement learning.
\newblock \emph{Journal of Machine Learning Research}, 11:\penalty0 1563--1600,
  2010.

\bibitem[Jiang et~al.(2023)Jiang, Zhang, Zhou, Gu, Zeng, and
  Zhu]{jiang2023multiobjective}
J.~Jiang, W.~Zhang, S.~Zhou, L.~Gu, X.~Zeng, and W.~Zhu.
\newblock Multi-objective online learning.
\newblock In \emph{The Eleventh International Conference on Learning
  Representations}, 2023.
\newblock URL \url{https://openreview.net/forum?id=dKkMnCWfVmm}.

\bibitem[Jin et~al.(2020{\natexlab{a}})Jin, Krishnamurthy, Simchowitz, and
  Yu]{jin2020reward}
C.~Jin, A.~Krishnamurthy, M.~Simchowitz, and T.~Yu.
\newblock Reward-free exploration for reinforcement learning.
\newblock In \emph{International Conference on Machine Learning}, pages
  4870--4879. PMLR, 2020{\natexlab{a}}.

\bibitem[Jin et~al.(2020{\natexlab{b}})Jin, Yang, Wang, and
  Jordan]{jin2020provably}
C.~Jin, Z.~Yang, Z.~Wang, and M.~I. Jordan.
\newblock Provably efficient reinforcement learning with linear function
  approximation.
\newblock In \emph{Conference on learning theory}, pages 2137--2143. PMLR,
  2020{\natexlab{b}}.

\bibitem[Kasimbeyli et~al.(2019)Kasimbeyli, Ozturk, Kasimbeyli, Yalcin, and
  Erdem]{kasimbeyli2019comparison}
R.~Kasimbeyli, Z.~K. Ozturk, N.~Kasimbeyli, G.~D. Yalcin, and B.~I. Erdem.
\newblock Comparison of some scalarization methods in multiobjective
  optimization: comparison of scalarization methods.
\newblock \emph{Bulletin of the Malaysian Mathematical Sciences Society},
  42:\penalty0 1875--1905, 2019.

\bibitem[Klamroth and J{\o}rgen(2007)]{klamroth2007constrained}
K.~Klamroth and T.~J{\o}rgen.
\newblock Constrained optimization using multiple objective programming.
\newblock \emph{Journal of Global Optimization}, 37:\penalty0 325--355, 2007.

\bibitem[Li et~al.(2020)Li, Zhang, and Wang]{li2020deep}
K.~Li, T.~Zhang, and R.~Wang.
\newblock Deep reinforcement learning for multiobjective optimization.
\newblock \emph{IEEE transactions on cybernetics}, 51\penalty0 (6):\penalty0
  3103--3114, 2020.

\bibitem[Lin et~al.(2024)Lin, Zhang, Yang, Liu, Wang, and Zhang]{lin2024smooth}
X.~Lin, X.~Zhang, Z.~Yang, F.~Liu, Z.~Wang, and Q.~Zhang.
\newblock Smooth tchebycheff scalarization for multi-objective optimization.
\newblock \emph{arXiv preprint arXiv:2402.19078}, 2024.

\bibitem[Liu et~al.(2021{\natexlab{a}})Liu, Liu, Jin, Stone, and
  Liu]{liu2021conflict}
B.~Liu, X.~Liu, X.~Jin, P.~Stone, and Q.~Liu.
\newblock Conflict-averse gradient descent for multi-task learning.
\newblock \emph{Advances in Neural Information Processing Systems},
  34:\penalty0 18878--18890, 2021{\natexlab{a}}.

\bibitem[Liu et~al.(2021{\natexlab{b}})Liu, Tong, and Liu]{liu2021profiling}
X.~Liu, X.~Tong, and Q.~Liu.
\newblock Profiling {Pareto} front with multi-objective stein variational
  gradient descent.
\newblock \emph{Advances in Neural Information Processing Systems},
  34:\penalty0 14721--14733, 2021{\natexlab{b}}.

\bibitem[Lu et~al.(2022)Lu, Herman, and Yu]{lu2022multi}
H.~Lu, D.~Herman, and Y.~Yu.
\newblock Multi-objective reinforcement learning: Convexity, stationarity and
  pareto optimality.
\newblock In \emph{The Eleventh International Conference on Learning
  Representations}, 2022.

\bibitem[Lu et~al.(2019)Lu, Wang, Hu, and Zhang]{lu2019multi}
S.~Lu, G.~Wang, Y.~Hu, and L.~Zhang.
\newblock Multi-objective generalized linear bandits.
\newblock \emph{arXiv preprint arXiv:1905.12879}, 2019.

\bibitem[Mahapatra and Rajan(2020)]{mahapatra2020multi}
D.~Mahapatra and V.~Rajan.
\newblock Multi-task learning with user preferences: Gradient descent with
  controlled ascent in {Pareto} optimization.
\newblock In \emph{International Conference on Machine Learning}, pages
  6597--6607. PMLR, 2020.

\bibitem[Mahapatra et~al.(2023)Mahapatra, Dong, Chen, and
  Momma]{mahapatra2023multi}
D.~Mahapatra, C.~Dong, Y.~Chen, and M.~Momma.
\newblock Multi-label learning to rank through multi-objective optimization.
\newblock In \emph{Proceedings of the 29th ACM SIGKDD Conference on Knowledge
  Discovery and Data Mining}, pages 4605--4616, 2023.

\bibitem[Miettinen(1999)]{miettinen1999nonlinear}
K.~Miettinen.
\newblock \emph{Nonlinear multiobjective optimization}, volume~12.
\newblock Springer Science \& Business Media, 1999.

\bibitem[Miryoosefi and Jin(2022)]{miryoosefi2022simple}
S.~Miryoosefi and C.~Jin.
\newblock A simple reward-free approach to constrained reinforcement learning.
\newblock In \emph{International Conference on Machine Learning}, pages
  15666--15698. PMLR, 2022.

\bibitem[Modi et~al.(2024)Modi, Chen, Krishnamurthy, Jiang, and
  Agarwal]{modi2024model}
A.~Modi, J.~Chen, A.~Krishnamurthy, N.~Jiang, and A.~Agarwal.
\newblock Model-free representation learning and exploration in low-rank mdps.
\newblock \emph{Journal of Machine Learning Research}, 25\penalty0
  (6):\penalty0 1--76, 2024.

\bibitem[Natarajan and Tadepalli(2005)]{natarajan2005dynamic}
S.~Natarajan and P.~Tadepalli.
\newblock Dynamic preferences in multi-criteria reinforcement learning.
\newblock In \emph{Proceedings of the 22nd international conference on Machine
  learning}, pages 601--608, 2005.

\bibitem[Nemirovski et~al.(2009)Nemirovski, Juditsky, Lan, and
  Shapiro]{nemirovski2009robust}
A.~Nemirovski, A.~Juditsky, G.~Lan, and A.~Shapiro.
\newblock Robust stochastic approximation approach to stochastic programming.
\newblock \emph{SIAM Journal on optimization}, 19\penalty0 (4):\penalty0
  1574--1609, 2009.

\bibitem[Nesterov(2005)]{nesterov2005smooth}
Y.~Nesterov.
\newblock Smooth minimization of non-smooth functions.
\newblock \emph{Mathematical programming}, 103:\penalty0 127--152, 2005.

\bibitem[Perez et~al.(2010)Perez, Germain-Renaud, K{\'e}gl, and
  Loomis]{perez2010multi}
J.~Perez, C.~Germain-Renaud, B.~K{\'e}gl, and C.~Loomis.
\newblock Multi-objective reinforcement learning for responsive grids.
\newblock \emph{Journal of Grid Computing}, 8:\penalty0 473--492, 2010.

\bibitem[Pirotta et~al.(2015)Pirotta, Parisi, and Restelli]{pirotta2015multi}
M.~Pirotta, S.~Parisi, and M.~Restelli.
\newblock Multi-objective reinforcement learning with continuous pareto
  frontier approximation.
\newblock In \emph{Proceedings of the AAAI conference on artificial
  intelligence}, volume~29, 2015.

\bibitem[Puterman(1990)]{puterman1990markov}
M.~L. Puterman.
\newblock Markov decision processes.
\newblock \emph{Handbooks in operations research and management science},
  2:\penalty0 331--434, 1990.

\bibitem[Qiao and Wang(2022)]{qiao2022near}
D.~Qiao and Y.-X. Wang.
\newblock Near-optimal deployment efficiency in reward-free reinforcement
  learning with linear function approximation.
\newblock \emph{arXiv preprint arXiv:2210.00701}, 2022.

\bibitem[Qiu et~al.(2020)Qiu, Wei, Yang, Ye, and Wang]{qiu2020upper}
S.~Qiu, X.~Wei, Z.~Yang, J.~Ye, and Z.~Wang.
\newblock Upper confidence primal-dual reinforcement learning for cmdp with
  adversarial loss.
\newblock \emph{Advances in Neural Information Processing Systems},
  33:\penalty0 15277--15287, 2020.

\bibitem[Qiu et~al.(2021{\natexlab{a}})Qiu, Wei, Ye, Wang, and
  Yang]{qiu2021provably}
S.~Qiu, X.~Wei, J.~Ye, Z.~Wang, and Z.~Yang.
\newblock Provably efficient fictitious play policy optimization for zero-sum
  markov games with structured transitions.
\newblock In \emph{International Conference on Machine Learning}, pages
  8715--8725. PMLR, 2021{\natexlab{a}}.

\bibitem[Qiu et~al.(2021{\natexlab{b}})Qiu, Ye, Wang, and Yang]{qiu2021reward}
S.~Qiu, J.~Ye, Z.~Wang, and Z.~Yang.
\newblock On reward-free rl with kernel and neural function approximations:
  Single-agent mdp and markov game.
\newblock In \emph{International Conference on Machine Learning}, pages
  8737--8747. PMLR, 2021{\natexlab{b}}.

\bibitem[Qiu et~al.(2023)Qiu, Wei, and Kolar]{qiu2023gradient}
S.~Qiu, X.~Wei, and M.~Kolar.
\newblock Gradient-variation bound for online convex optimization with
  constraints.
\newblock In \emph{Proceedings of the AAAI Conference on Artificial
  Intelligence}, volume~37, pages 9534--9542, 2023.

\bibitem[Riquelme et~al.(2015)Riquelme, Von~L{\"u}cken, and
  Baran]{riquelme2015performance}
N.~Riquelme, C.~Von~L{\"u}cken, and B.~Baran.
\newblock Performance metrics in multi-objective optimization.
\newblock In \emph{2015 Latin American computing conference (CLEI)}, pages
  1--11. IEEE, 2015.

\bibitem[Roijers et~al.(2013)Roijers, Vamplew, Whiteson, and
  Dazeley]{roijers2013survey}
D.~M. Roijers, P.~Vamplew, S.~Whiteson, and R.~Dazeley.
\newblock A survey of multi-objective sequential decision-making.
\newblock \emph{Journal of Artificial Intelligence Research}, 48:\penalty0
  67--113, 2013.

\bibitem[Sener and Koltun(2018)]{sener2018multi}
O.~Sener and V.~Koltun.
\newblock Multi-task learning as multi-objective optimization.
\newblock \emph{Advances in neural information processing systems}, 31, 2018.

\bibitem[Stamenkovic et~al.(2022)Stamenkovic, Karatzoglou, Arapakis, Xin, and
  Katevas]{stamenkovic2022choosing}
D.~Stamenkovic, A.~Karatzoglou, I.~Arapakis, X.~Xin, and K.~Katevas.
\newblock Choosing the best of both worlds: Diverse and novel recommendations
  through multi-objective reinforcement learning.
\newblock In \emph{Proceedings of the Fifteenth ACM International Conference on
  Web Search and Data Mining}, pages 957--965, 2022.

\bibitem[Steuer(1986)]{steuer1986multiple}
R.~E. Steuer.
\newblock Multiple criteria optimization.
\newblock \emph{Theory, computation, and application}, 1986.

\bibitem[Strehl and Littman(2008)]{strehl2008analysis}
A.~L. Strehl and M.~L. Littman.
\newblock An analysis of model-based interval estimation for markov decision
  processes.
\newblock \emph{Journal of Computer and System Sciences}, 74\penalty0
  (8):\penalty0 1309--1331, 2008.

\bibitem[Tekin and Tur{\u{g}}ay(2018)]{tekin2018multi}
C.~Tekin and E.~Tur{\u{g}}ay.
\newblock Multi-objective contextual multi-armed bandit with a dominant
  objective.
\newblock \emph{IEEE Transactions on Signal Processing}, 66\penalty0
  (14):\penalty0 3799--3813, 2018.

\bibitem[Tseng(2008)]{tseng2008accelerated}
P.~Tseng.
\newblock On accelerated proximal gradient methods for convex-concave
  optimization.
\newblock \emph{submitted to SIAM Journal on Optimization}, 1, 2008.

\bibitem[Turgay et~al.(2018)Turgay, Oner, and Tekin]{turgay2018multi}
E.~Turgay, D.~Oner, and C.~Tekin.
\newblock Multi-objective contextual bandit problem with similarity
  information.
\newblock In \emph{International Conference on Artificial Intelligence and
  Statistics}, pages 1673--1681. PMLR, 2018.

\bibitem[v.~Neumann(1928)]{v1928theorie}
J.~v.~Neumann.
\newblock Zur theorie der gesellschaftsspiele.
\newblock \emph{Mathematische annalen}, 100\penalty0 (1):\penalty0 295--320,
  1928.

\bibitem[Van~Moffaert and Now{\'e}(2014)]{van2014multi}
K.~Van~Moffaert and A.~Now{\'e}.
\newblock Multi-objective reinforcement learning using sets of pareto
  dominating policies.
\newblock \emph{The Journal of Machine Learning Research}, 15\penalty0
  (1):\penalty0 3483--3512, 2014.

\bibitem[Van~Moffaert et~al.(2013{\natexlab{a}})Van~Moffaert, Drugan, and
  Now{\'e}]{van2013hypervolume}
K.~Van~Moffaert, M.~M. Drugan, and A.~Now{\'e}.
\newblock Hypervolume-based multi-objective reinforcement learning.
\newblock In \emph{Evolutionary Multi-Criterion Optimization: 7th International
  Conference, EMO 2013, Sheffield, UK, March 19-22, 2013. Proceedings 7}, pages
  352--366. Springer, 2013{\natexlab{a}}.

\bibitem[Van~Moffaert et~al.(2013{\natexlab{b}})Van~Moffaert, Drugan, and
  Now{\'e}]{van2013scalarized}
K.~Van~Moffaert, M.~M. Drugan, and A.~Now{\'e}.
\newblock Scalarized multi-objective reinforcement learning: Novel design
  techniques.
\newblock In \emph{2013 IEEE symposium on adaptive dynamic programming and
  reinforcement learning (ADPRL)}, pages 191--199. IEEE, 2013{\natexlab{b}}.

\bibitem[Wang et~al.(2024)Wang, Lin, Xiong, Yang, Diao, Qiu, Zhao, and
  Zhang]{wang2024arithmetic}
H.~Wang, Y.~Lin, W.~Xiong, R.~Yang, S.~Diao, S.~Qiu, H.~Zhao, and T.~Zhang.
\newblock Arithmetic control of llms for diverse user preferences: Directional
  preference alignment with multi-objective rewards.
\newblock \emph{arXiv preprint arXiv:2402.18571}, 2024.

\bibitem[Wang et~al.(2020)Wang, Du, Yang, and Salakhutdinov]{wang2020reward}
R.~Wang, S.~S. Du, L.~Yang, and R.~R. Salakhutdinov.
\newblock On reward-free reinforcement learning with linear function
  approximation.
\newblock \emph{Advances in neural information processing systems},
  33:\penalty0 17816--17826, 2020.

\bibitem[Wang and Sebag(2013)]{wang2013hypervolume}
W.~Wang and M.~Sebag.
\newblock Hypervolume indicator and dominance reward based multi-objective
  monte-carlo tree search.
\newblock \emph{Machine learning}, 92:\penalty0 403--429, 2013.

\bibitem[Wei et~al.(2019)Wei, Yu, and Neely]{wei2019online}
X.~Wei, H.~Yu, and M.~J. Neely.
\newblock Online primal-dual mirror descent under stochastic constraints.
\newblock \emph{arXiv preprint arXiv:1908.00305}, 2019.

\bibitem[Wiering et~al.(2014)Wiering, Withagen, and Drugan]{wiering2014model}
M.~A. Wiering, M.~Withagen, and M.~M. Drugan.
\newblock Model-based multi-objective reinforcement learning.
\newblock In \emph{2014 IEEE symposium on adaptive dynamic programming and
  reinforcement learning (ADPRL)}, pages 1--6. IEEE, 2014.

\bibitem[Wu et~al.(2021{\natexlab{a}})Wu, Braverman, and
  Yang]{wu2021accommodating}
J.~Wu, V.~Braverman, and L.~Yang.
\newblock Accommodating picky customers: Regret bound and exploration
  complexity for multi-objective reinforcement learning.
\newblock \emph{Advances in Neural Information Processing Systems},
  34:\penalty0 13112--13124, 2021{\natexlab{a}}.

\bibitem[Wu et~al.(2021{\natexlab{b}})Wu, Zhang, Yang, and Wang]{wu2021offline}
R.~Wu, Y.~Zhang, Z.~Yang, and Z.~Wang.
\newblock Offline constrained multi-objective reinforcement learning via
  pessimistic dual value iteration.
\newblock \emph{Advances in Neural Information Processing Systems},
  34:\penalty0 25439--25451, 2021{\natexlab{b}}.

\bibitem[Xiao et~al.(2024)Xiao, Ban, and Ji]{xiao2024direction}
P.~Xiao, H.~Ban, and K.~Ji.
\newblock Direction-oriented multi-objective learning: Simple and provable
  stochastic algorithms.
\newblock \emph{Advances in Neural Information Processing Systems}, 36, 2024.

\bibitem[Xu et~al.(2020)Xu, Tian, Ma, Rus, Sueda, and
  Matusik]{xu2020prediction}
J.~Xu, Y.~Tian, P.~Ma, D.~Rus, S.~Sueda, and W.~Matusik.
\newblock Prediction-guided multi-objective reinforcement learning for
  continuous robot control.
\newblock In \emph{International conference on machine learning}, pages
  10607--10616. PMLR, 2020.

\bibitem[Yahyaa et~al.(2014{\natexlab{a}})Yahyaa, Drugan, and
  Manderick]{yahyaa2014annealing}
S.~Q. Yahyaa, M.~M. Drugan, and B.~Manderick.
\newblock Annealing-pareto multi-objective multi-armed bandit algorithm.
\newblock In \emph{2014 IEEE Symposium on Adaptive Dynamic Programming and
  Reinforcement Learning (ADPRL)}, pages 1--8. IEEE, 2014{\natexlab{a}}.

\bibitem[Yahyaa et~al.(2014{\natexlab{b}})Yahyaa, Drugan, and
  Manderick]{yahyaa2014scalarized}
S.~Q. Yahyaa, M.~M. Drugan, and B.~Manderick.
\newblock The scalarized multi-objective multi-armed bandit problem: An
  empirical study of its exploration vs. exploitation tradeoff.
\newblock In \emph{2014 International Joint Conference on Neural Networks
  (IJCNN)}, pages 2290--2297. IEEE, 2014{\natexlab{b}}.

\bibitem[Yala et~al.(2022)Yala, Mikhael, Lehman, Lin, Strand, Wan, Hughes,
  Satuluru, Kim, Banerjee, et~al.]{yala2022optimizing}
A.~Yala, P.~G. Mikhael, C.~Lehman, G.~Lin, F.~Strand, Y.-L. Wan, K.~Hughes,
  S.~Satuluru, T.~Kim, I.~Banerjee, et~al.
\newblock Optimizing risk-based breast cancer screening policies with
  reinforcement learning.
\newblock \emph{Nature medicine}, 28\penalty0 (1):\penalty0 136--143, 2022.

\bibitem[Yang et~al.(2019)Yang, Sun, and Narasimhan]{yang2019generalized}
R.~Yang, X.~Sun, and K.~Narasimhan.
\newblock A generalized algorithm for multi-objective reinforcement learning
  and policy adaptation.
\newblock \emph{Advances in neural information processing systems}, 32, 2019.

\bibitem[Yang et~al.(2024)Yang, Pan, Luo, Qiu, Zhong, Yu, and
  Chen]{yang2024rewards}
R.~Yang, X.~Pan, F.~Luo, S.~Qiu, H.~Zhong, D.~Yu, and J.~Chen.
\newblock Rewards-in-context: Multi-objective alignment of foundation models
  with dynamic preference adjustment.
\newblock \emph{arXiv preprint arXiv:2402.10207}, 2024.

\bibitem[Yang et~al.(2020)Yang, Jin, Wang, Wang, and Jordan]{yang2020provably}
Z.~Yang, C.~Jin, Z.~Wang, M.~Wang, and M.~Jordan.
\newblock Provably efficient reinforcement learning with kernel and neural
  function approximations.
\newblock \emph{Advances in Neural Information Processing Systems},
  33:\penalty0 13903--13916, 2020.

\bibitem[Yu et~al.(2021{\natexlab{a}})Yu, Kumar, Chebotar, Hausman, Levine, and
  Finn]{yu2021conservative}
T.~Yu, A.~Kumar, Y.~Chebotar, K.~Hausman, S.~Levine, and C.~Finn.
\newblock Conservative data sharing for multi-task offline reinforcement
  learning.
\newblock \emph{Advances in Neural Information Processing Systems},
  34:\penalty0 11501--11516, 2021{\natexlab{a}}.

\bibitem[Yu et~al.(2021{\natexlab{b}})Yu, Tian, Zhang, and Sra]{yu2021provably}
T.~Yu, Y.~Tian, J.~Zhang, and S.~Sra.
\newblock Provably efficient algorithms for multi-objective competitive {RL}.
\newblock In \emph{International Conference on Machine Learning}, pages
  12167--12176. PMLR, 2021{\natexlab{b}}.

\bibitem[Zhang et~al.(2023)Zhang, Zhang, and Gu]{zhang2023optimal}
J.~Zhang, W.~Zhang, and Q.~Gu.
\newblock Optimal horizon-free reward-free exploration for linear mixture mdps.
\newblock In \emph{International Conference on Machine Learning}, pages
  41902--41930. PMLR, 2023.

\bibitem[Zhang et~al.(2021)Zhang, Du, and Ji]{zhang2021near}
Z.~Zhang, S.~Du, and X.~Ji.
\newblock Near optimal reward-free reinforcement learning.
\newblock In \emph{International Conference on Machine Learning}, pages
  12402--12412. PMLR, 2021.

\bibitem[Zhong et~al.(2024)Zhong, Ma, Zhang, Yang, Zhang, Qi, and
  Yang]{zhong2024panacea}
Y.~Zhong, C.~Ma, X.~Zhang, Z.~Yang, Q.~Zhang, S.~Qi, and Y.~Yang.
\newblock Panacea: Pareto alignment via preference adaptation for llms.
\newblock \emph{arXiv preprint arXiv:2402.02030}, 2024.

\bibitem[Zhou et~al.(2022)Zhou, Liu, Kalathil, Kumar, and Tian]{zhou2022anchor}
R.~Zhou, T.~Liu, D.~Kalathil, P.~Kumar, and C.~Tian.
\newblock Anchor-changing regularized natural policy gradient for
  multi-objective reinforcement learning.
\newblock \emph{Advances in Neural Information Processing Systems},
  35:\penalty0 13584--13596, 2022.

\bibitem[Zhou et~al.(2023)Zhou, Liu, Yang, Shao, Liu, Yue, Ouyang, and
  Qiao]{zhou2023beyond}
Z.~Zhou, J.~Liu, C.~Yang, J.~Shao, Y.~Liu, X.~Yue, W.~Ouyang, and Y.~Qiao.
\newblock Beyond one-preference-for-all: Multi-objective direct preference
  optimization.
\newblock \emph{arXiv preprint arXiv:2310.03708}, 2023.

\bibitem[Zhu et~al.(2023)Zhu, Dang, and Grover]{zhu2023scaling}
B.~Zhu, M.~Dang, and A.~Grover.
\newblock Scaling pareto-efficient decision making via offline multi-objective
  rl.
\newblock \emph{arXiv preprint arXiv:2305.00567}, 2023.

\end{thebibliography}
\bibliographystyle{abbrvnat}

\newpage
\appendix

{\centering
	{\LARGE Appendix}
	\par }
\vspace{0.6cm}
{
	\hypersetup{linkcolor=black}
	\tableofcontents
}


\addtocontents{toc}{\protect\setcounter{tocdepth}{2}}

\newpage

 	\onecolumn
 \vspace{1em}
 \renewcommand{\thesection}{\Alph{section}}

 \vspace{0.5em}
 
\section{Proofs for Section \ref{sec:pareto}}\label{supp:subgap}

\subsection{Proof of Property \ref{pro:property}}

\begin{proof} The proof of this proposition is divided into two parts: the proof of property (a) and the proof of property (b). We first develop a technique from the perspective of analysis to prove the property (a) of this proposition. Then, based on property (a), we further develop the proof of property (b).
	
\vspace{3pt}
\noindent\textbf{Part 1)} Prove property (a). As show in Definition \ref{def:pareto}, $\pi'$ dominates $\pi$ if and only if 
\begin{align*}
&\forall i\in[m], ~~V_{i,1}^{\pi'}(s_1) \geq V_{i,1}^{\pi}(s_1),\\
\text{and}~~~~&\exists j\in[m], ~~V_{j,1}^{\pi'}(s_1) > V_{j,1}^\pi(s_1).	
\end{align*}
For a clear presentation, we let $\pi^0$ be a policy not in $\Pi_{\mathrm{P}}^*$. Through our proof, we need to find one $\pi^* \in \Pi_{\mathrm{P}}^*$ dominating $\pi^0$. 

One might think of a straightforward argument that if $\pi^0$ is dominated by some $\pi$ not in $\Pi_{\mathrm{P}}^*$, then we can always find another policy $\pi'$ dominating $\pi$ such that a policy $\pi^*\in\Pi_{\mathrm{P}}^*$ dominating $\pi^0$ will be eventually found until this process stops, using the fact that the dominance relation is transitive. However, we note that only applying this argument is not sufficient to prove property (a) of Property \ref{pro:property} as 1) we have no guarantee whether the sequence generated by the above procedure will converge, and 2) even if the sequence converges to some point, we have no guarantee this point is in Pareto set $\Pi_{\mathrm{P}}^*$. Thus, the proof of property (a) is non-trivial and challenging. 

In what follows, we manage to construct such a sequence of policies with controllable and quantifiable value difference between two consecutive policies starting from $\pi^0$. We further prove that the sequence of policies will converge, and the point it converges to is a Pareto optimal policy.

Inspired by the above definition of dominance, we first define a set $\cB(\pi,\epsilon)$ for $\epsilon>0$ as follows,
\begin{align*}
	\cB(\pi,\epsilon) := \bigg\{\pi'~\bigg|~ \pi' \text{ dominates } \pi \text{, and } \sum_{i=1}^m V_{i,1}^{\pi'}(s_1)\ge \sum_{i=1}^m V_{i,1}^{\pi}(s_1)+\epsilon\bigg\}.	
\end{align*}
Policies in this set can dominate $\pi$, and we have a difference of at least $\epsilon$ between summations of values under $\pi$ and any policy in this set. Therefore, any policy in $\cB(\pi,\epsilon)$ is strictly better than $\pi$ with at least an $\epsilon$ improvement in their value summations over $m$ objectives. Then, we can use the set $\cB(\pi,\epsilon)$ to control and characterize the value differences.

Starting with any $\pi^0$, we can create a sequence $\pi^0_{(\epsilon)},\pi^1_{(\epsilon)},\pi^2_{(\epsilon)},\cdots$ with $\pi^k_{(\epsilon)}\in \cB(\pi^{k-1}_{(\epsilon)},\epsilon)$ and $\pi^0_{(\epsilon)} = \pi^0$. If $\cB(\pi^k_{(\epsilon)},\epsilon) = \emptyset$, then the process stops. Overall, following this idea, we construct of the sequence that is proved to converge to a Pareto optimal by the following scheme: 1) starting from $\pi^0$, set $\epsilon=1$ to generate a sequence by the above process until it stops; 2) shrink $\epsilon$ to $1/2$ and starting from the stopping point of the last sequence, generate a subsequent sequence until it stops; 3) iteratively shrink $\epsilon$ in a rate of $1/n$ to generate subsequent sequences, where $n$ is the number of the current sequence; 4) we further construct a subsequence based on the above sequences and show that we can find the Pareto optimal point that dominating $\pi^0$ with $n\rightarrow\infty$. Next, we present our formal proof of the aforementioned scheme. 

Note that we have $0\leq V_{i,1}^{\pi}(s_1)\leq H$ for any $i$ and any policy $\pi$. Then, we have $\sum_{i=1}^m V_{i,1}^{\pi}(s_1) \le mH$ for any $\pi$ and $\sum_{i=1}^m V_{i,1}^{\pi^0}(s_1) \geq 0$. As we require at least an $\epsilon$ increment of value summations, the constructing process should stop within finite steps, and the last point is some $\pi^k_{(\epsilon)}$ with some $k\leq mH/\epsilon$.

We use $S(\pi,\epsilon)$ to denote such a finite sequence starting from a point $\pi$. Then, we can construct our sequence starting with $\pi^0$ in the following way. We first set $\epsilon = 1$ and obtain a finite sequence $S(\pi^0_{(1)},1) = \{\pi^0_{(1)},\pi^1_{(1)},\pi^2_{(1)},\cdots, \pi^{k_1}_{(1)}\}$ with $\pi^0_{(1)}= \pi^0$. We note that $k_1$ can be zero since $\epsilon=1$ might be too large for the construction. 
Then, by setting $\epsilon=1/2$, starting with $\pi^{0}_{(1/2)} = \pi^{k_1}_{(1)}$, we construct a sequence $S(\pi^{0}_{(1/2)}, 1/2) = \{\pi^{0}_{(1/2)},\pi^1_{(1/2)},\pi^2_{(1/2)},\cdots, \pi^{k_2}_{(1/2)}\}$. After that, we start with $\pi^{0}_{(1/3)}=\pi^{k_2}_{(1/2)}$ and $\epsilon=1/3$ to create another sequence. In other words, we construct $S(\pi^{0}_{(1/n)}, 1/n)$ with the starting point $\pi^{0}_{(1/n)} = \pi^{k_{n-1}}_{(1/(n-1))}$ for $n=1,2,3,\cdots$.
From the construction, we can see that a policy $\pi^i_{(1/n)}$ (if $i\neq k_n$) is dominated by $\pi^{i'}_{(1/n)}$ for $i'>i$ and $\pi^{j}_{(1/n')}$ for any $j$ and $n'>n$.

We construct another sequence $\{\overline{\pi}^n\}_{n\geq 1}$ with $\overline{\pi}^n:= \pi^{k_n}_{(1/n)}$ (if $k_n=0$, $y_n = y_{n-1}$), i.e., using all the last points in $S(\pi^{0}_{(1/n)}, 1/n)$ for $n \geq 1$. This sequence is well defined as each $S(\pi,\epsilon)$ is non-empty. By Bolzano-Weierstrass Theorem, a convergent subsequence of $\{\overline{\pi}^n\}_{n\geq 1}$ can exist, which is denoted as $\{\overline{\pi}^{b_n}\}_{n\geq 1}$. We are aiming to show that $\overline{\pi}^*=\lim_{n\to\infty} \overline{\pi}^{b_n}\in \Pi_{\mathrm{P}}^*$.

Firstly, the compactness of the set of all policies, which is denoted as $\Pi$, guarantees $\overline{\pi}^*\in\Pi$. Then, we prove by contradiction that $\overline{\pi}^*$ is Pareto optimal. 

If $\overline{\pi}^*$ is not Pareto optimal, there exists some policy $\tilde{\pi}$ dominating $\overline{\pi}^*$. Then, by the definition of dominance, we have $C := \sum_{i=1}^m V_{i,1}^{\tilde{\pi}}(s_1) -\sum_{i=1}^m V_{i,1}^{\overline{\pi}^*}(s_1) >0$. From the construction of $\overline{\pi}^{b_n}$, we have $\overline{\pi}^{b_n} = \overline{\pi}^{b_{n-1}}$ or $\overline{\pi}^{b_n}$ dominates $\overline{\pi}^{b_{n-1}}$. Hence, we always have
\begin{align*}
&\sum_{i=1}^m V_{i,1}^{\overline{\pi}^{b_n}}(s_1)\ge \sum_{i=1}^m V_{i,1}^{\overline{\pi}^{b_{n-1}}}(s_1)~~\text{ for any } n.	
\\
&V_{i,1}^{\overline{\pi}^{b_n}}(s_1)\ge  V_{i,1}^{\overline{\pi}^{b_{n-1}}}(s_1)~~\text{ for any } n, i.	
\end{align*}
Moreover, the above inequality also implies that
\begin{align}
	V_{i,1}^{\overline{\pi}^*}(s_1)\ge  V_{i,1}^{\overline{\pi}^{b_n}}(s_1)~~\text{ for any } n, i. \label{eq:contr0}
\end{align}
As $V_{i,1}^{\pi}(s_1)$ is continuous in $\pi$, we further have 
\begin{align}\label{eq:contr1}
\sum_{i=1}^m V_{i,1}^{\tilde{\pi}}(s_1)-C=\sum_{i=1}^m V_{i,1}^{\overline{\pi}^*}(s_1)\ge \sum_{i=1}^m V_{i,1}^{\overline{\pi}^{b_{n}}}(s_1) ~~\text{ for any }	n.
\end{align}
Combining \eqref{eq:contr0} and \eqref{eq:contr1}, we know that $V_{i,1}^{\overline{\pi}^*}(s_1)\ge  V_{i,1}^{\overline{\pi}^{b_n}}(s_1)$ for all $i\in[m]$ and at least one $j\in[m]$ satisfies $V_{j,1}^{\overline{\pi}^*}(s_1)>  V_{j,1}^{\overline{\pi}^{b_n}}(s_1)$. This implies $\tilde{\pi}$ dominates any $\overline{\pi}^{b_n}$ if $\tilde{\pi}$ dominates $\overline{\pi}^*$.

From the definition of $\overline{\pi}^n:= \pi^{k_n}_{(1/n)}$, we know that $\cB(\overline{\pi}^n,1/n)=\emptyset$, which implies $\sum_{i=1}^m V_{i,1}^{\pi'}(s_1)\le \sum_{i=1}^m V_{i,1}^{\overline{\pi}^n}(s_1)+1/n$ if there exists any $\pi'$ dominating $\overline{\pi}^n$. By similar argument, if there exists any $\pi'$ dominating $\overline{\pi}^{b_n}$, then $\sum_{i=1}^m V_{i,1}^{\pi'}(s_1)\le \sum_{i=1}^m V_{i,1}^{\overline{\pi}^{b_n}}(s_1)+1/b_n$. Since $\tilde{\pi}$ dominates any $\overline{\pi}^{b_n}$, then $\sum_{i=1}^m V_{i,1}^{\tilde{\pi}}(s_1)\le \sum_{i=1}^m V_{i,1}^{\overline{\pi}^{b_n}}(s_1)+1/b_n$.
Then, for large enough $n$, we have $1/b_n \le C/2 $. And consequently, for any $\pi'$ dominating $\overline{\pi}^{b_n}$, we should have 
\begin{align*}
\sum_{i=1}^m V_{i,1}^{\tilde{\pi}}(s_1)\le \sum_{i=1}^m V_{i,1}^{\overline{\pi}^{b_n}}(s_1)+C/2.	
\end{align*}
However, according to \eqref{eq:contr1}, we have  $\sum_{i=1}^m V_{i,1}^{\tilde{\pi}}(s_1)=\sum_{i=1}^m V_{i,1}^{\overline{\pi}^*}(s_1)+C$, which contradicts the above result. Therefore, $\overline{\pi}^*\in\Pi_{\mathrm{P}}^*$.

Now we have shown that $\overline{\pi}^*=\lim_{n\to\infty} \overline{\pi}^{b_n}\in \Pi_{\mathrm{P}}^*$. Finally, to prove property (a), the remaining thing we need to prove is that if $\pi^0$ is not Pareto optimal, it is dominated by $\overline{\pi}^*$. 

From the construction of $\overline{\pi}^*$ and $\overline{\pi}^{b_n}$, we have $V_{i,1}^{\overline{\pi}^*}(s_1)\geq \cdots \geq V_{i,1}^{\overline{\pi}^{b_n}}(s_1)\geq ..\cdots\geq V_{i,1}^{\overline{\pi}^{b_1}}(s_1) \geq V_{i,1}^{\pi^0}(s_1)$ for any $i$. If $\pi^0$ is not dominated by $\overline{\pi}^*$, then by the definition of non-dominance, it is not difficult to show  $V_{i,1}^{\overline{\pi}^*}(s_1)= \cdots = V_{i,1}^{\overline{\pi}^{b_n}}(s_1)= ..\cdots = V_{i,1}^{\pi^0}(s_1)$ for any $n$ and thus show $\overline{\pi}^*= \cdots = \overline{\pi}^{b_n}= ..\cdots = \pi^0$ using the definition of $\cB$ and $\overline{\pi}^{b_n}$.
Therefore, we can obtain $\cB(\pi^0,1/b_n)=\emptyset$ for any $n$. However, as $\pi^0$ is not Pareto optimal, we can find an $\pi^{0*}$ dominating $\pi^0$ and define $C_0:= \sum_{i=1}^m V_{i,1}^{\pi^{0*}}(s_1)-\sum_{i=1}^m V_{i,1}^{\pi^0}(s_1) >0$. For $1/C_0 < b_n$, we have $\pi^{0*}\in \cB(\pi^0,1/b_n)$ which implies $\cB(\pi^0,1/b_n)\neq\emptyset$. We reach a contradiction. Therefore, $\overline{\pi}^*$ should dominate $\pi^0$ with $\overline{\pi}^*=\lim_{n\to\infty} \overline{\pi}^{b_n}\in \Pi_{\mathrm{P}}^*$. Combining all the above results, we prove property (a) of Property \ref{pro:property}.

\vspace{3pt}
\noindent\textbf{Part 2)} Prove property (b).  We first prove that $\pi\in\Pi_{\mathrm{P}}^*$ implies that $\pi$ is not dominated by any Pareto optimal policy $\pi^*\in\Pi_{\mathrm{P}}^*$. By the definition of Pareto optimality in Definition \ref{def:pareto}, $\pi\in\Pi_{\mathrm{P}}^*$ indicates that $\pi$ is not dominated by any other policies, which including the Pareto optimal policies. On the other hand, $\pi$ is also not dominated by itself according to the definition of non-dominance. Therefore, we know that $\pi\in\Pi_{\mathrm{P}}^*$ implies that $\pi$ is not dominated by any Pareto optimal policy.

Next, we prove that if $\pi$ is not dominated by any Pareto optimal policy $\pi^*\in\Pi_{\mathrm{P}}^*$, then $\pi\in\Pi_{\mathrm{P}}^*$. Assuming that $\pi\notin\Pi_{\mathrm{P}}^*$, then by property (a), there exists some Pareto optimal policy $\pi^*$ dominating $\pi$, which contradicts that $\pi$ is not dominated by any Pareto optimal policy $\pi^*\in\Pi_{\mathrm{P}}^*$. Thus, we have $\pi\in\Pi_{\mathrm{P}}^*$. This completes the proof of property (b) of this proposition.

The proof of Property \ref{pro:property} is completed by combining the results in \textbf{Part 1)} and \textbf{Part 2)}.
\end{proof}

 \subsection{Proof of Proposition \ref{cond:subopt-iff}}
 
%
 
 \begin{proof}
 	According to the definitions of Pareto optimality and weak Pareto optimality, we have 
 	\begin{itemize}
 		\item a policy $\pi$ is Pareto optimal if and only if for all $\pi'\in\Pi$, $\exists i\in[m], V_{i,1}^\pi(s_1) > V_{i,1}^{\pi'}(s_1)$ or $\forall i\in[m], V_{i,1}^\pi(s_1) = V_{i,1}^{\pi'}(s_1)$.
 		
 		\item a policy $\pi$ is weakly Pareto optimal if and only if for all $\pi'\in\Pi$, $\exists i\in[m], V_{i,1}^\pi(s_1) \geq V_{i,1}^{\pi'}(s_1)$.
 		
 	\end{itemize}
 	Therefore, it is not difficult to observe that $\Pi_{\mathrm{P}}^* \subset \Pi_{\mathrm{W}}^*$ without Condition \ref{cond:subopt-iff}.
 	
 	Next, we show that under the condition in Proposition \ref{cond:subopt-iff}, $\Pi_{\mathrm{W}}^* \subset  \Pi_{\mathrm{P}}^* $. Let $\pi\in \Pi_{\mathrm{W}}^*$. If $\pi$ is not Pareto optimal, by the condition, we obtain
 	\begin{align}
 		\exists\pi^*\in\Pi_{\mathrm{P}}^*, \forall i\in [m], V_{i,1}^{\pi}(s_1) <  V_{i,1}^{\pi^*}(s_1).\label{eq:contraditct-cond}
 	\end{align}
 	On the other hand, by the definition of weak Pareto optimality, we have
 	\begin{align*}
 		\forall \pi'\in\Pi, \exists i\in[m], V_{i,1}^\pi(s_1) \geq V_{i,1}^{\pi'}(s_1),
 	\end{align*}
 	implying that
 	\begin{align*}
 		\exists i\in[m], V_{i,1}^\pi(s_1) \geq V_{i,1}^{\pi^*}(s_1),
 	\end{align*}
 	which contradicts \eqref{eq:contraditct-cond}. Therefore, $\pi$ must be Pareto optimal. Then we prove $\Pi_{\mathrm{W}}^* \subset  \Pi_{\mathrm{P}}^*$ under the condition in Proposition \ref{cond:subopt-iff}. This completes the proof.
 \end{proof}

\section{Proofs for Section \ref{sec:metric}}

\subsection{Proof of Proposition \ref{prop:geometry}}
	\begin{proof}
		We use the notion of the occupancy measure. Specifically, denoting by $\breve\theta_h(s,a;\pi)$ the joint distribution of $(s,a)$ at step $h$ induced by $\pi$ and $\PP$, the value function can be written as
		\begin{align}\label{polytope:eqn3}
			V^{\pi}_{i,1}(s_1)=\sumh\sum_{s\in\cS}\sum_{a\in\cA} \breve\theta_h(s,a;\pi) r_{i,h}(s,a).
		\end{align} 
		The occupancy measure for the next step can be calculated by
		\begin{align}\label{polytope:eqn1}
			\breve\theta_{h+1}(s',a';\pi) = \sum_{s\in\cS}\sum_{a\in\cA}\breve\theta_h(s,a;\pi)\PP_h(s'|s,a)\pi_{h+1}(a'|s').
		\end{align}
        Let $\bm \theta := (\theta_h(s,a))_{s\in \cS, a\in \cA, h \le H}$ be a point in $\mathbb R^{|\cS|\times |\cA| \times H} $. Consider the constraints
        	\begin{align}\label{polytope:eqn2}
			\sum_{s\in\cS}\sum_{a\in\cA}\theta_{h}(s,a)=1,\quad \theta_h(s,a) \ge 0, \text{ and } \theta_1(s,a) =0 \text{ for $s\neq s_1$.}
		\end{align}
		and 
		\begin{align}\label{polytope:eqn4}
			\sum_{s\in\cS}\sum_{a\in\cA}\theta_h(s,a)\PP_h(s'|s,a) = \sum_{a'\in\cA} \theta_{h+1}(s',a').
		\end{align}
        and denote
		\begin{align*}
			\Theta = \{
			\bm \theta\in \mathbb R^{|\cS|\times |\cA| \times H} ~~|~~ \text{$\bm \theta$ satisfies \eqref{polytope:eqn2} and \eqref{polytope:eqn4}} \}
		\end{align*} and 
		\begin{align*}
		\mathbb V_{\Theta} = \left\{\sumh\sum_{s\in\cS}\sum_{a\in\cA} \theta_h(s,a) \bm r_{h}(s,a)~~\Big|~~\bm \theta\in \Theta\right\},
		\end{align*}
		for any joint distribution $\breve\theta_h(s,a;\pi)$ induced by policy $\pi$, we have $\breve\theta_h(s,a;\pi)\in \Theta$ and thus $\mathbb V(\Pi) \subseteq \mathbb V_\Theta$.
  
		On the other hand, for any $\bm \theta'\in \Theta$, we construct the policy $\pi'$ such that  $\pi_h'(a|s)=\frac{\theta_h'(s,a)}{\sum_{a\in\cA}\theta_h'(s,a)}$ when $\sum_{a\in\cA}\theta_h'(s,a)\neq 0$. Then, with \eqref{polytope:eqn2}, we have $\breve\theta_1(s,a;\pi') =\theta_1'(s,a)$. Next, we prove by induction. Assume that $\breve\theta_h(s,a;\pi')  =\theta_h'(s,a)$ holds for $h$, with $\pi'$ and \eqref{polytope:eqn4}, we have
		\begin{align}
			\begin{aligned}
				\breve\theta_{h+1}(s',a';\pi') 
				&=\sum_{s\in\cS}\sum_{a\in\cA}\breve\theta_{h}(s',a';\pi')\PP_h(s'|s,a)\pi'_{h+1}(a'|s') \\&=\sum_{s\in\cS}\sum_{a\in\cA}\theta_h'(s,a)\PP_h(s'|s,a)\pi'_{h+1}(a'|s') \\
				&=\sum_{a\in \cA} \theta_{h+1}'(s',a) \pi'_{h+1}(a'|s') =\theta_{h+1}'(s',a'). 
			\end{aligned}
		\end{align}
		By induction, we have $\breve \theta_h(s,a;\pi') = \theta'_h(s,a)$ holds for any $h$. Consequently, we have $\mathbb{V}(\Pi) \supseteq \mathbb V_{\Theta}$. Therefore, $\mathbb{V}(\Pi) = \mathbb V_{\Theta}$ and it remains to show that $\mathbb V_{\Theta}$ is a convex polytope. 
		
		From the definition of $\Theta$, it is clear that $\Theta$ is a convex polytope in $\mathbb R^{|\cS|\times |\cA| \times H}$ given by the constraints \eqref{polytope:eqn2} and \eqref{polytope:eqn4}. Let $\hat{\theta}_1,\dots, \hat{\theta}_D$ denote the extreme points of the convex polytope. In other words, $\Theta$ is the convex hull of $\hat{\theta}_1,\dots, \hat{\theta}_D$ and can be written as 
		\begin{align*}
			\Theta = \left\{ \sum_{d=1}^D\alpha_d\hat{\theta}_d ~~\Big|~~ \text{$\alpha_d\ge 0$ and $\sum_{d=1}^D \alpha_d =1$ }\right\}.
		\end{align*}
		Then, we have
		\begin{align*}
		V_\Theta = \left\{\sum_{d=1}^D \alpha_d \hat  V_d ~~\Big|~~\text{$\alpha_d\ge 0$ and $\sum_{i=d}^D \alpha_d =1$ }\right\},	
		\end{align*}
		where $\hat V_d = \sumh\sum_{s\in\cS}\sum_{a\in\cA} \hat \theta_{d,h}(s,a) \bm r_{h}(s,a)$ is the value function with $\hat \theta_d$. Therefore, $V_\Theta$ is the convex hull of $\hat V_1,\ldots,\hat V_d$, which is a convex polytope.
	\end{proof}

\subsection{Proof of Proposition \ref{prop:linear-comb}} \label{sec:proof-linear-comb}

\begin{proof}
	In this proof, we first prove that for any policy class $\Pi$ (either stochastic or deterministic), we have $\{\pi~|~\pi\in\argmax_{\pi\in\Pi} \linl(\pi), \forall \blambda \in\Delta_m\}\subseteq \Pi_{\mathrm{W}}^*$ and $\{\pi~|~\pi\in\argmax_{\pi\in\Pi} \linl(\pi), \forall \blambda \in\Delta_m^o\} \subseteq \Pi_{\mathrm{P}}^*$. Then, we show that for a stochastic policy class $\Pi$, we have $\Pi_{\mathrm{W}}^* \subseteq \{\pi~|~\pi\in\argmax_{\pi\in\Pi} \linl(\pi), \forall \blambda \in\Delta_m\}$ and $\Pi_{\mathrm{P}}^* \subseteq \{\pi~|~\pi\in\argmax_{\pi\in\Pi} \linl(\pi), \forall \blambda \in\Delta_m^o\}$ utilizing Proposition \ref{prop:geometry}.  Finally, we provide an example to show that a (weakly) Pareto optimal policy may not be the solution to $\max_{\pi\in\Pi} \linl(\pi)$ for any $\blambda \in\Delta_m$ when $\Pi$ is a deterministic policy class.

	\vspace{5pt}
	\noindent\textbf{Part 1)} Prove $\{\pi~|~\pi\in\argmax_{\pi\in\Pi} \linl(\pi), \forall \blambda \in\Delta_m\}\subseteq \Pi_{\mathrm{W}}^*$ and $\{\pi~|~\pi\in\argmax_{\pi\in\Pi} \linl(\pi), \forall \blambda \in\Delta_m^o\} \subseteq \Pi_{\mathrm{P}}^*$ for any policy class $\Pi$ (either stochastic or deterministic). We revisit the proofs in prior works \citep{geoffrion1968proper,ehrgott2005multicriteria,steuer1986multiple,ehrgott2005multiobjective} on linear scalarization for multi-objective optimization and adapt their proofs to the MORL setting.

	 We now prove that all the solutions to $\max_{\pi\in\Pi} \blambda^\top \bV_1^{\pi}(s_1), \forall \blambda \in\Delta_m$, are weakly Pareto optimal.
	This claim is proved by contradiction. Define $\pi_{\blambda}^\dag:=\max_{\pi\in\Pi} \blambda^\top \bV_1^{\pi}(s_1), \allowbreak \forall \blambda \in\Delta_m$. By this definition of the solution to this problem, we know that
	\begin{align}
		\blambda^\top \bV_1^{\pi_{\blambda}^\dag}(s_1)\geq \blambda^\top \bV_1^{\pi}(s_1), ~\forall \pi\in\Pi, \label{eq:linear-comb-contra}
	\end{align}
	Now assume that $\pi_{\blambda}^\dag$ is not a weakly Pareto optimal policy. Then by Definition \ref{def:weakpareto}, there must exists a policy $\breve\pi\in\Pi$ such that $\breve\pi$ satisfies
	\begin{align*}	
		&\forall i\in[m], \quad V_{i,1}^{\breve\pi}(s_1)> V_{i,1}^{\pi_{\blambda}^\dag}(s_1), 
	\end{align*}
	which further leads to
	\begin{align*}
		\blambda^\top \bV_1^{\pi_{\blambda}^\dag}(s_1)< \blambda^\top \bV_1^{\breve\pi}(s_1), \quad \forall \blambda\in\Delta_m.
	\end{align*}
	The above inequality contradicts \eqref{eq:linear-comb-contra}, which implies $\pi_{\blambda}^\dag$ is a weakly Pareto optimal policy. This completes the proof.
	
	Following the above proof, we can further prove that all the solutions to $\max_{\pi\in\Pi} \blambda^\top \bV_1^{\pi}(s_1), \forall \blambda \in\Delta_m^o,$ are Pareto optimal, where $\Delta_m^o$ is the relative interior of $\Delta_m$. And $\Delta_m^o$ is defined as $\Delta_m^o:=\{\bx: x_i>0, \sumi x_i=1\}$.
	We similarly prove this claim by contradiction. Let $\pi_{\blambda}^\ddag:=\max_{\pi\in\Pi} \blambda^\top \bV_1^{\pi}(s_1), \forall \blambda \in\Delta_m^o$. By this definition, we have
	\begin{align}
		\blambda^\top \bV_1^{\pi_{\blambda}^\ddag}(s_1)\geq \blambda^\top \bV_1^{\pi}(s_1), ~\forall \pi, \label{eq:linear-comb-contra2}
	\end{align}
	Assume that $\pi_{\blambda}^\ddag$ is not a Pareto optimal policy. Then by Definition \ref{def:pareto}, there must exist a policy $\breve\pi^o\in\Pi$ such that $\breve\pi^o$ satisfies
	\begin{align*}	
		&\forall i\in[m], \quad V_{i,1}^{\breve\pi^o}(s_1)\geq V_{i,1}^{\pi_{\blambda}^\ddag}, \\
		&\exists j\in[m], \quad V_{j,1}^{\breve\pi^o}(s_1)> V_{j,1}^{\pi_{\blambda}^\ddag}, 
	\end{align*}
	which further leads to
	\begin{align*}
		\blambda^\top \bV_1^{\pi_{\blambda}^\ddag}(s_1)< \blambda^\top \bV_1^{\breve\pi^o}(s_1), \quad \forall \blambda\in\Delta_m^o.
	\end{align*}
	since all $\lambda_i>0$ if $\blambda\in\Delta_m^o$. The above inequality contradicts \eqref{eq:linear-comb-contra2}, which thus implies $\pi_{\blambda}^\ddag$ is a Pareto optimal policy. This completes the proof of the claim.

	\setlength{\abovecaptionskip}{-3pt}
	\setlength{\belowcaptionskip}{-15pt}
	\begin{figure}[t!] 
		\centering
		\includegraphics[height=2.2in]{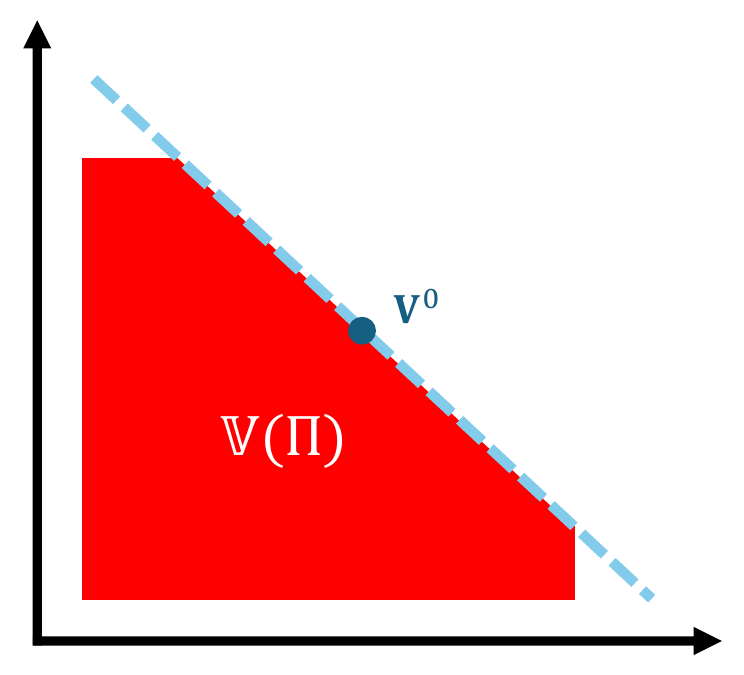}
		\caption{\small Illustration of \textbf{Part 2)}}
		\label{fig:proof1}
	\end{figure}
	
	\vspace{5pt}
	\noindent \textbf{Part 2)} Prove for a stochastic policy class $\Pi$, we have $\Pi_{\mathrm{P}}^* \subseteq \{\pi~|~\pi\in\argmax_{\pi\in\Pi} \linl(\pi), \forall \blambda \in\Delta_m^o\}$ and $\Pi_{\mathrm{W}}^* \subseteq \{\pi~|~\pi\in\argmax_{\pi\in\Pi} \linl(\pi), \forall \blambda \in\Delta_m\}$. According to Proposition \ref{prop:geometry}, we know that $\mathbb V(\Pi)$ is a convex polytope when $\Pi$ is a stochastic policy class.
	
	Now we prove that if $\mathbb V(\Pi)$ is a convex polytope, any Pareto optimal policy can be the solution of $\max_{\pi} \mathrm{LIN}_{\bm\lambda}(\pi)$ for some $\bm\lambda\in \Delta_m^o$.
	It is equivalent to show that for any vector of value functions $\bm V^0 = (V^0_1,\dots,V^0_m) \in \VV(\Pi)$ with $\bm V^0 = \bV_1^{\pi_0}(s_1)$ for a Pareto optimal policy $\pi_0\in\Pi_{\mathrm{P}}^*$, we have $\bm V^0 = \argmax_{\bm V \in \mathbb V(\Pi)} \blambda_0^\top \bV$ and $\max_{\bm V \in \mathbb V(\Pi)} \blambda_0^\top \bV = \max_{\pi\in\Pi} \sumi\blambda_0^\top \bV_1^\pi(s_1)$ for some $\blambda_0 \in\Delta_m^o$. If there exists a hyperplane $\{\bm V\in\RR^m~|~\sumi a_i V_i = C \}$ satisfying $C \geq 0$, $a_i>0$, $\sumi a_i V^0_i = C$ and $\mathbb V(\Pi)\subseteq \{\bm V\in\mathbb R^m ~|~ \sumi a_i V_i \le C \}$, then we have $\bm V^0 =\argmax_{\bm V\in V(\Pi)} \sumi a_i V_i$. Without loss of generality, we can assume that $\sumi a_i =1$. Then, setting $\lambda_i = a_i$ leads to the desired result directly. As $\VV(\Pi)$ is a convex polytope for the stochastic policy class $\Pi$, the boundary of it must be the combinations of some faces of this polytope. We denote these faces by $F_i,i=1,\dots,n_F$, where $n_F$ is the number of faces. Correspondingly, the boundary of $\VV(\Pi)$ is $\cup_{i=1}^{n_F} F_i$. As values for the Pareto front must be a subset of the boundary, for any $\bm V^0\in\VV(\Pi_\mathrm{P}^*) =\{\bV_1^\pi(s_1)~|~\pi\in\Pi_\mathrm{P}^*\}$, it must be on some face $F_i$. In other words, $\bm V^0 \in F_i^* := F_i\cap \VV(\Pi_\mathrm{P}^*)$ and $ F_i^*$ is a face with dimension at most $m-1$. Consider the case that $ F_i^*$ is an $(m-1)$-dimension face first. Suppose that the normal vector of the face is $(a_1,a_2,\dots, a_m)$. Without loss of generality, we assume that $a_1>0$. Then, we must have $a_i>0$ for any $i>0$, leading to $C = \sumi a_i V_i^0 \geq 0$. Otherwise, supposing that $a_2\le0$, we can find two points $\bm V^1$ and $\bm V^2$ on $F_i^*$ such that $\sumi a_iV_i^1 = \sumi a_iV_i^2$ with $ V^1_1\ge V^2_1$, $ V^1_2 >  V^2_2$ and $ V^1_i=  V^2_i$ for $i>2$, implying $\bV^2$ does not correspond to a Pareto optimal policy. It contradicts the fact that the face corresponds to the Pareto front. Therefore, when the dimension of $F_i^*$ is $m-1$, $F_i^*$ satisfies $a_i>0$ and the convexity of the polytope implies that $ \{\bm V\in\mathbb R^m ~|~ \sumi a_i V_i \le C \}$. Thus, the hyperplane we need is directly the one that $F_i^*$ lies on. If the dimension of $F_i^*$ is smaller than $m-1$, similarly, we could find a $m-1$-dimension hyperplane such that the $F_i^*$ is on that plane and $\mathbb V(\Pi)\subseteq \{\bm V\in\mathbb R^m|\sumi a_i V_i \le C\}$. The above construction is illustrated in Figure \ref{fig:proof1}. This completes the proof of $\Pi_{\mathrm{P}}^* \subseteq \{\pi~|~\pi\in\argmax_{\pi\in\Pi} \linl(\pi), \forall \blambda \in\Delta_m^o\}$.
 
 Following the same spirit, we could show that $\Pi_{\mathrm{W}}^* \subseteq \{\pi~|~\pi\in\argmax_{\pi\in\Pi} \linl(\pi), \forall \blambda \in\Delta_m\}$. In fact, it is sufficient to show that for any weak Pareto optimal policy $\pi_0\in \Pi_{\mathrm{W}}^*$ and the corresponding value function $\bV^0$, there exists a hyperplane $\{\bV\in\RR^m~|~\sumi a_i V_i = C \}$ satisfying $C \geq 0$, $a_i\ge 0$, $\sumi a_i V^0_i = C$ and $\mathbb V(\Pi)\subseteq \{\bm V\in\mathbb R^m ~|~ \sumi a_i V_i \le C \}$. In other words, we relax the condition $a_i > 0$ in the above proof to $a_i\ge 0$ for weak Pareto points. The proof is similar to that of $\Pi_{\mathrm{P}}^* \subseteq \{\pi~|~\pi\in\argmax_{\pi\in\Pi} \linl(\pi), \forall \blambda \in\Delta_m^o\}$. Specifically, for a weak Pareto point $\bm V^0\in \breve F_i^* := F_i\cap \mathbb V(\Pi_{\mathrm{W}})^*$, if $\breve F_i^*$ is a face of dimension $m-1$, the normal vector of it satisfies $a_i\ge 0$. If we assume that $a_i\geq 0$ for $i\in\cI^+$ and $a_j<0$ for $j\in[m]\backslash\cI^+$, we can find two points $\bV^1$ and $\bV^2$ on $\breve F_i^*$ such that $\sumi a_i V_i^1 = \sumi a_i V_i^2$ with $V_i^2 = V_i^1 + \epsilon_i$ for $i\in\cI^+$ and $V_j^2 = V_j^1 + \epsilon_j$ for $j\in[m]\backslash\cI^+$ where  $\epsilon_i \geq 0$ for all $i \in [m]$. This requires $\sum_{i\in\cI^+} a_i \epsilon_i = \sum_{j\in[m]\backslash\cI^+} (-a_j) \epsilon_j$. Supposing we have at least one $i\in\cI^+$ such that $a_i>0$ (we let $\sumi a_i = 1$ without loss of generality), we can find sufficiently small positive $\epsilon_i$ for all $i\in[m]$ that satisfy $\sum_{i\in\cI^+} a_i \epsilon_i = \sum_{j\in[m]\backslash\cI^+} (-a_j) \epsilon_j$, which implies $V_i^1 > V_i^2$ for all $i\in[m]$, or equivalently $\bV^2$ does not correspond to a weakly Pareto optimal policy according to Definition \ref{def:weakpareto}. This contradicts the fact that all the values on the face correspond to weak Pareto optimal policies. The remaining proof is exactly the same as that for $\Pi_{\mathrm{P}}^* \subseteq \{\pi~|~\pi\in\argmax_{\pi\in\Pi} \linl(\pi), \forall \blambda \in\Delta_m^o\}$.

	\vspace{5pt}
	\noindent \textbf{Part 3)}
	Now we use a concrete example to illustrate that for a deterministic policy class $\Pi$, a Pareto optimal policy may not be the solution of $\max_{\pi\in\Pi} \blambda^\top \bV_1^{\pi}(s_1)$ for any $\blambda \in\Delta_m$. We consider a multi-objective multi-arm bandit problem, a simple and special MOMDP whose state space size $|\cS|=1$, episode length $H=1$, with a deterministic policy. Here we assume this bandit problem has $m=2$ reward functions $r_1$ and $r_2$ and $|\cA|=4$ actions. We define
	\begin{align*}
		&r_1(a_1)=1, \quad r_1(a_2)=0.5,\quad r_1(a_3)=0.6,\quad r_1(a_4)=0.65,\\
		&r_2(a_1)=0, \quad r_2(a_2)=0.5,\quad r_2(a_3)=0.2,\quad r_2(a_4)=0.3,
	\end{align*}
	which are the reward values of the four actions for each reward function. Via the definition of the Pareto optimal policy, we can identify that $a_1$, $a_2$, and $a_4$ are the Pareto optimal policy, or Pareto optimal arms in a multi-arm bandit problem.
	
	Letting $0\leq \lambda \leq 1$, we generate a reward function as a linear combination of $r_1$ and $r_2$ via $\lambda$ and $1-\lambda$, i.e., $r_\lambda(a):=\lambda r_1(a) + (1-\lambda) r_2(a)$, such that
	\begin{align*}
		r_\lambda(a_1)=\lambda, \quad r_\lambda(a_2)=0.5,\quad r_\lambda(a_3)=0.2+0.4\lambda,\quad r_\lambda(a_4)=0.3+0.35\lambda.
	\end{align*}
	Then we have:
	\begin{itemize}
		\item when $\lambda< 0.5$,  we have $a_2=\argmax_a r_\lambda(a)$,
		\item when $\lambda= 0.5$,  we have $a_1,a_2=\argmax_a r_\lambda(a)$,	
		\item when $\lambda> 0.5$,  we have $a_1=\argmax_a r_\lambda(a)$.
	\end{itemize}
	Therefore, we only find $a_1$ and $a_2$ as the solutions for any $\lambda$, which are also Pareto optimal arms. However, the Pareto optimal arm $a_3$ is not identified by solving $\max_a r_\lambda(a)$, for all $0\leq \lambda\leq 1$, which completes the example for the second claim. The proof of this proposition is completed.
\end{proof}

\subsection{Proof of Proposition \ref{prop:eq-psg}}

Following the proof in \citep{jiang2023multiobjective} that focuses on a general online convex minimization problem, we present our proof specializing in the multi-objective RL maximization problem. 
\begin{proof} By Definition \ref{def:subopt}, we can rewrite the Pareto suboptimality gap $\psg$ as follows:
	\begin{align}
		\begin{aligned}\label{eq:subopt-eqform}
			&\psg(\pi):= \sup_{\pi^*\in \Pi_{\mathrm{P}}^*}\newinf_{\epsilon(\pi^*) \geq 0} \epsilon(\pi^*), \\
			&\text{s.t.}~~~~~~ \exists i\in[m], ~V_{i,1}^{\pi}(s_1) + \epsilon(\pi^*)   >  V_{i,1}^{\pi^*}(s_1), \\
			& ~~ ~~~\textbf{or}~~ \forall i\in[m], ~V_{i,1}^{\pi}(s_1) + \epsilon(\pi^*)  =  V_{i,1}^{\pi^*}(s_1).
		\end{aligned}
	\end{align} 
	If $\exists i\in[m]$, $V_{i,1}^{\pi}(s_1) \geq  V_{i,1}^{\pi^*}(s_1)$ for a policy $\pi^*\in\Pi_{\mathrm{P}}^*$, then the policy $\pi$ satisfies either of the following cases:
	\begin{itemize}[itemsep = 0pt, topsep=0pt]
		\item[(a)] $\exists i\in[m]$, $V_{i,1}^{\pi}(s_1) =  V_{i,1}^{\pi^*}(s_1)$,
		\item[(b)] $\forall i\in\cI \subset [m]$, $V_{i,1}^{\pi}(s_1) >  V_{i,1}^{\pi^*}(s_1)$ and $\forall j\in [m]\backslash \cI$, $V_{j,1}^{\pi}(s_1) <  V_{i,1}^{\pi^*}(s_1)$.
	\end{itemize}
	Case (a) indicates that $\epsilon(\pi^*)>0$ satisfies the above constraint in \eqref{eq:subopt-eqform}, implying that $\inf_{\epsilon(\pi^*) > 0} \epsilon(\pi^*)= 0$. On the other hand, case (b) indicates that $\epsilon(\pi^*)=0$ satisfies the constraint in \eqref{eq:subopt-eqform}, thus implying that $\inf_{\epsilon(\pi^*) = 0} \epsilon(\pi^*) = 0$. 
	
	If $\forall i\in[m]$, $V_{i,1}^{\pi}(s_1) <  V_{i,1}^{\pi^*}(s_1)$ for a policy $\pi^*\in\Pi_{\mathrm{P}}^*$, then we have $\epsilon(\pi^*) = \min_i   (V_{i,1}^{\pi^*}(s_1)-V_{i,1}^{\pi}(s_1))$ such that  $\inf_{\epsilon(\pi^*) \geq 0} \epsilon(\pi^*) = \min_i   (V_{i,1}^{\pi^*}(s_1)-V_{i,1}^{\pi}(s_1)) =\inf_{\blambda^*\in \Delta_m} \blambda^*{}^\top(\bV_1^{\pi^*}(s_1)-\bV_1^{\pi}(s_1))>0$, where $\Delta_m:=\{\blambda = (\lambda_i)_{i=1}^m \in\RR^m ~|~ \sum_{i=1}^m \lambda_i = 1, \lambda_i\geq 0 \}$ is a simplex in $\RR^m$.
	
	Overall, we know that for any policy $\pi\in\Pi$, the Pareto suboptimality gap is equivalent to
	\begin{align*}
		\psg(\pi):= \sup_{\pi^*\in \Pi_{\mathrm{P}}^*}\newinf_{\epsilon(\pi^*) \geq 0} \epsilon(\pi^*) = \sup_{\pi^*\in \Pi_{\mathrm{P}}^*}\newinf_{\blambda^*\in \Delta_m} \blambda^*{}^\top(\bV_1^{\pi^*}(s_1)-\bV_1^{\pi}(s_1)).
	\end{align*}
	This completes the proof of Proposition \ref{prop:eq-psg}
\end{proof}

\subsection{Proof of Proposition \ref{prop:subopt}}

\begin{proof} The proof of this proposition is divided into two parts: we first prove that $\pi\in \Pi_{\mathrm{W}}^*$ implies $\psg(\pi) = 0$; then we show that $\psg(\pi) = 0$ implies $\pi$ is weakly Pareto optimal.
Recall that the definition of $\psg$ is
\begin{align}
	\begin{aligned}\label{eq:so-recall}
		&\psg(\pi):= \inf_{\epsilon \geq 0} \epsilon, \\
		&\text{s.t.}~~\forall \pi^*\in \Pi_{\mathrm{P}}^*, \exists i\in[m], V_{i,1}^{\pi}(s_1) + \epsilon   >  V_{i,1}^{\pi^*}(s_1),\\
		&\qquad\qquad~~~ \text{or}~~ \forall i\in[m], V_{i,1}^{\pi}(s_1) + \epsilon  =  V_{i,1}^{\pi^*}(s_1).
	\end{aligned}
\end{align}
And we say a policy $\pi$ is not dominated by $\pi'$ if and only if 
\begin{align}
	&\exists i \in [m],~~   V_{i,1}^{\pi}(s_1) > V_{i,1}^{\pi'}(s_1), \label{eq:nd_cond1} \\
	\text{or} ~~&\forall i \in [m], ~~ V_{i,1}^{\pi}(s_1)=V_{i,1}^{\pi'}(s_1). \label{eq:nd_cond2}
\end{align}
	
%
%

\vspace{5pt}
\noindent\textbf{Part 1)} {Prove $\pi\in \Pi_{\mathrm{W}}^* \Rightarrow \psg(\pi) = 0$.} By the definition of weak Pareto optimality, for $\pi\in\Pi_{\mathrm{W}}^*$, we have
 for all $\pi'$,
\begin{align*}	
	&\exists i\in[m], \quad V_{i,1}^{\pi}(s_1)\geq   V_{i,1}^{\pi'}(s_1).
\end{align*}
Letting $\pi'=\pi^*\in\Pi_{\mathrm{P}}^*$, we have
for all $\pi^*\in\Pi_{\mathrm{P}}^*$, $\exists i\in[m], V_{i,1}^{\pi}(s_1)\geq V_{i,1}^{\pi^*}(s_1)$. Therefore, we have for any $\epsilon>0$, $\forall \pi^*\in\Pi_{\mathrm{P}}^*$, $\exists i\in[m], V_{i,1}^{\pi}(s_1) + \epsilon > V_{i,1}^{\pi^*}(s_1)$, which satisfies the constraint in \eqref{eq:so-recall}. Then, for $\pi\in\Pi_{\mathrm{W}}^*$, $\psg(\pi)=\inf_{\epsilon>0}\epsilon = 0$, which completes the proof of \textbf{Part 1)}.

\vspace{5pt}
\noindent
\textbf{Part 2)} {Prove $\psg(\pi) = 0  \Rightarrow \pi$ is weakly Pareto optimal.} $\psg(\pi) = 0$ indicates that the zero value is achieved at either $\epsilon = 0$ or arbitrarily small $\epsilon\rightarrow 0^+$ with $\epsilon > 0$ (due to $\inf_{\epsilon > 0}\epsilon = 0$). 

When it is achieved at $\epsilon = 0$, by \eqref{eq:so-recall}, we have that $\pi$ satisfies  
\begin{align*}
\forall \pi^*\in \Pi_{\mathrm{P}}^*, ~~\exists i\in[m], V_{i,1}^{\pi}(s_1) >  V_{i,1}^{\pi^*}(s_1),\quad \text{or}\quad \forall i\in[m], V_{i,1}^{\pi}(s_1)  =  V_{i,1}^{\pi^*}(s_1). 	
\end{align*}
By \eqref{eq:nd_cond1} and \eqref{eq:nd_cond2}, it implies that $\pi$ is not dominated by any $\pi^*\in \Pi_{\mathrm{P}}^*$. Further employing the property of the Pareto set that $\pi\in\Pi_{\mathrm{P}}^*$ if and only if $\pi$ is not dominated by any $\pi^*\in\Pi_{\mathrm{P}}^*$ as in Property \ref{pro:property}, we know that the policy $\pi$ is a Pareto optimal policy in $\Pi_{\mathrm{P}}^*$ for when $\epsilon = 0$. 

On the other hand, when $\psg(\pi) = 0$ is achieved at arbitrarily small $\epsilon$ with $\epsilon > 0$, the constraint in \eqref{eq:so-recall} is equivalent to 
\begin{align}
	\forall \pi^*\in \Pi_{\mathrm{P}}^*, \exists i\in[m], V_{i,1}^{\pi}(s_1) \geq  V_{i,1}^{\pi^*}(s_1), \label{eq:issue-recall}	
\end{align}
since no policy $\pi$ satisfies the second constraint in \eqref{eq:so-recall} that $\forall i\in[m], V_{i,1}^{\pi}(s_1) + \epsilon  =  V_{i,1}^{\pi^*}(s_1)$ for an arbitrarily small $\epsilon$. In addition, according to \textbf{(a)} of Property \ref{pro:property}, we have that 
\begin{align*}
\forall \pi'\notin \Pi_{\mathrm{P}}^*, \exists \pi^*\in \Pi_{\mathrm{P}}^*, \forall i\in[m],  V_{i,1}^{\pi^*}(s_1) \geq  V_{i,1}^{\pi'}(s_1),		
\end{align*}
according to the definition of dominance in Definition \ref{def:pareto}. Therefore, we construct a minimal subset $\tilde{\Pi}^*\subset\Pi_{\mathrm{P}}^*$ that all policies in $\tilde{\Pi}^*$ can dominate at least one policy outside $\Pi_{\mathrm{P}}^*$ and the rest of policies in $\Pi_{\mathrm{P}}^*\backslash\tilde{\Pi}^*$ dominates no policy, which gives
\begin{align*} 
\forall \pi'\notin \Pi_{\mathrm{P}}^*, \exists \pi^*\in \tilde{\Pi}^*, \forall i\in[m],  V_{i,1}^{\pi^*}(s_1) \geq  V_{i,1}^{\pi'}(s_1).	
\end{align*}
Thus, combining it with \eqref{eq:issue-recall}, if $\psg(\pi) = 0$ is achieved at arbitrarily small $\epsilon$, we have
\begin{align*}
	\forall \pi'\notin \Pi_{\mathrm{P}}^*, \exists i\in[m],  V_{i,1}^{\pi}(s_1) \geq  V_{i,1}^{\pi'}(s_1).
\end{align*}
Furthermore, combining \eqref{eq:issue-recall} with the above equation, we have
\begin{align*}
	\forall \pi', \exists i\in[m],  V_{i,1}^{\pi}(s_1) \geq  V_{i,1}^{\pi'}(s_1),
\end{align*}
which indicates that $\pi$ satisfies
\begin{align*}
	\nexists \pi', \forall i\in[m],  V_{i,1}^{\pi}(s_1) <  V_{i,1}^{\pi'}(s_1).
\end{align*}	
Further by Definition \ref{def:weakpareto}, $\pi$ is weakly Pareto optimal.

Jointly considering the cases of $\epsilon=0$ and $\epsilon$ being arbitrarily small,  we eventually prove that $\psg(\pi) = 0$ implies $\pi$ is weakly Pareto optimal. The proof of Proposition \ref{prop:subopt} is completed by combining the results in \textbf{Part 1)} and \textbf{Part 2)}.
\end{proof}

\subsection{Proof of Proposition \ref{prop:tch}}

\begin{proof} We revisit the proofs in prior works on Tchebycheff scalarization for multi-objective optimization, e.g., \citet{choo1983proper,ehrgott2005multicriteria}. We adapt their proofs to the MORL setting and present them here for completeness.
	 
The proof of the first claim in this proposition is divided into two parts: we first prove $\pi_{\blambda}^*\in\argmin_{\pi\in\Pi} \tchl(\pi)$ for a $\blambda\in\Delta_m$ $\Rightarrow \pi_{\blambda}^*$ is weakly Pareto optimal; then we prove $\breve{\pi}$ is weakly Pareto optimal $\Rightarrow \breve{\pi}\in\argmin_{\pi\in\Pi} \tchl(\pi)$ for some $\blambda\in\Delta_m$.

	\vspace{5pt}
\noindent\textbf{Part 1)} Prove $\pi_{\blambda}^*\in\argmin_{\pi\in\Pi} \tchl(\pi)$ for a $\blambda\in\Delta_m^o$ $\Rightarrow \pi_{\blambda}^*$ is weakly Pareto optimal. This is equivalent to showing that the solutions to $\minimize_{\pi\in\Pi} \tchl(\pi)$ for all $\blambda$ are weakly Pareto optimal. 
	
	Since for a given $\blambda$ we have $\pi_{\blambda}^*\in\argmin_{\pi\in\Pi} \tchl(\pi)$, then by the definition of $\tchl(\pi)$ in Definition \ref{def:tch}, we know that
	\begin{align*}
		\max_i \{\lambda_i (V_{i,1}^*(s_1) + \iota - V_{i,1}^{\pi_{\blambda}^*}(s_1))\}= \min_{\pi\in\Pi} \max_i \{\lambda_i (V_{i,1}^*(s_1) + \iota - V_{i,1}^{\pi}(s_1))\}, 
	\end{align*}
	which indicates that 
	\begin{align}\label{eq:prop-tch-1}
		\max_i \{\lambda_i (V_{i,1}^*(s_1) + \iota - V_{i,1}^{\pi_{\blambda}^*}(s_1))\}\leq  \max_i \{\lambda_i (V_{i,1}^*(s_1) + \iota - V_{i,1}^{\pi}(s_1))\}, \quad \forall \pi\in\Pi.
	\end{align}	
	Now we prove by contradiction that $\pi_{\blambda}^*$ is weakly Pareto optimal. If $\pi_{\blambda}^*$ is not weakly Pareto optimal, then by Definition \ref{def:weakpareto}, we know that there exists another policy $\pi'\in\Pi$ satisfying $V_{i,1}^{\pi_{\blambda}^*}(s_1) < V_{i,1}^{\pi'}(s_1)$ for all $i\in[m]$, which means that 
	\begin{align*}
		\exists \pi',  \lambda_i (V_{i,1}^*(s_1) + \iota - V_{i,1}^{\pi_{\blambda}^*}(s_1)) > \lambda_i (V_{i,1}^*(s_1) + \iota - V_{i,1}^{\pi'}(s_1)).
	\end{align*} 
	Given that $\blambda\in \Delta_m^o$, 
	we further have
	\begin{align*}
		&\max_{i\in[m]} \{\lambda_i (V_{i,1}^*(s_1) + \iota - V_{i,1}^{\pi_{\blambda}^*}(s_1))\} 
		> \max_{i\in [m]} \{\lambda_i (V_{i,1}^*(s_1) + \iota - V_{i,1}^{\pi'}(s_1))\},
	\end{align*}
	which contradicts \eqref{eq:prop-tch-1} that $\max_i \{\lambda_i (V_{i,1}^*(s_1) + \iota - V_{i,1}^{\pi_{\blambda}^*}(s_1))\} \leq \max_i \{\lambda_i (V_{i,1}^*(s_1) + \iota - V_{i,1}^{\pi'}(s_1))\}$. Therefore, $\pi_{\blambda}^*$ is a weakly Pareto optimal policy.
	
\vspace{5pt}
\noindent	
\textbf{Part 2)} Prove $\breve{\pi}$ is weakly Pareto optimal $\Rightarrow \breve{\pi}=\argmin_{\pi\in\Pi} \tchl(\pi)$ for some $\blambda\in\Delta_m^o$. This is equivalent to showing that if $\breve{\pi}$ is weakly Pareto optimal, then we can find a $\blambda$ such that $\breve{\pi}$ is a minimizer of the problem $\minimize_{\pi\in\Pi} \tchl(\pi)$.

For the policy $\breve{\pi}$, we let $\blambda$ be
\begin{align*}
	\lambda_i = \frac{(V_{i,1}^*(s_1) + \iota - V_{i,1}^{\breve{\pi}}(s_1))^{-1}}{\sum_{j=1}^m (V_{j,1}^*(s_1) + \iota - V_{j,1}^{\breve{\pi}}(s_1))^{-1}}>0,
\end{align*}
where $V_{j,1}^*(s_1) + \iota - V_{j,1}^{\breve{\pi}}(s_1) > 0$ always holds, and we have $\lambda_i (V_{i,1}^*(s_1) + \iota - V_{i,1}^{\breve{\pi}}(s_1))  = \lambda_{i'} (V_{i',1}^*(s_1) + \iota - V_{i',1}^{\breve{\pi}}(s_1)) = [\sum_{j=1}^m (V_{j,1}^*(s_1) + \iota - V_{j,1}^{\breve{\pi}}(s_1))^{-1}]^{-1}$ for all $i\neq i'$ such that 
\begin{align*}
\tchl(\breve{\pi})  =\lambda_i (V_{i,1}^*(s_1) + \iota - V_{i,1}^{\breve{\pi}}(s_1)), \quad \forall i\in[m].	
\end{align*} 
Next, we show that if $\breve{\pi}$ is weakly Pareto optimal, then $\breve{\pi}$ minimizes the function $\tchl(\pi)$ with $\blambda$ defined above. We prove this claim by contradiction. If there exists another policy $\pi'$ satisfying $\tchl(\pi') < \tchl(\breve{\pi})$, we have
\begin{align*}
\max_{j\in[m]}\{\lambda_j (V_{j,1}^*(s_1) + \iota - V_{j,1}^{\pi'}(s_1))\} < \tchl(\breve{\pi}) = \lambda_i (V_{i,1}^*(s_1) + \iota - V_{i,1}^{\breve{\pi}}(s_1)), \quad \forall i\in[m],
\end{align*}
which further implies that
\begin{align*}
	\lambda_i (V_{i,1}^*(s_1) + \iota - V_{i,1}^{\pi'}(s_1)) < \lambda_i (V_{i,1}^*(s_1) + \iota - V_{i,1}^{\breve{\pi}}(s_1)), \quad \forall i\in[m].
\end{align*}
The above inequality means that $\breve{\pi}$ is not a weakly Pareto optimal according to Definition \ref{def:weakpareto}, which contradicts the premise that $\breve{\pi}$ is a weakly Pareto optimal. Therefore, we can prove Part 2). This completes the first claim of this proposition.

Now we prove the second claim of this proposition by contradiction. Let $\breve{\pi}:=\argmin_{\pi\in\Pi}\tchl(\pi)$ for a $\blambda$. Assume that $\breve{\pi}$ is not Pareto optimal. According to the definition of the Pareto optimal policy in Definition \ref{def:pareto}, we have that there exists a policy $\pi\neq \breve{\pi}$ such that
\begin{align*}
	V_{i,1}^{\breve{\pi}}(s_1)\leq V_{i,1}^{\pi}(s_1), ~\forall i\in[m].
\end{align*}
According to the definition of $\tchl$, the above inequality leads to
\begin{align*}
	\tchl(\pi) \leq \tchl(\breve{\pi}).
\end{align*}
However, since $\breve{\pi}$ is the only solution to $\tchl(\pi)$, we have $\tchl(\pi) > \tchl(\breve{\pi})$, which contradicts the above inequality that $\tchl(\pi) \leq \tchl(\breve{\pi})$. This completes our proof of Proposition \ref{prop:tch}.
\end{proof}

\subsection{Proof of Proposition \ref{prop:tch-unique}}

\begin{proposition} There always exists a subset $\Lambda \subseteq \Delta_m$ such that $\{\pi~|~\pi \in \min_{\pi\in\Pi} \tchl(\pi), \forall \blambda \in\Lambda\subseteq\Delta_m^o\}= \Pi_{\mathrm{P}}^*$. Supposing that $\pi_{\blambda}^*$ is one arbitrary solution to $\min_{\pi\in\Pi} \tchl(\pi)$, for each $\blambda\in\Lambda$, all the solutions to $\min_{\pi\in\Pi} \tchl(\pi)$ is Pareto optimal with the same value $V_{i,1}^{\pi_{\blambda}^*}(s_1)$ for all $i\in[m]$. 
\end{proposition}

\begin{proof}
For each $\pi\in \Pi_\mathrm{P}^*$, we set the associated $\blambda$ as 
\begin{align*}
	\lambda_i = \frac{(V_{i,1}^*(s_1) + \iota - V_{i,1}^{\pi}(s_1))^{-1}}{\sum_{j=1}^m (V_{j,1}^*(s_1) + \iota - V_{j,1}^{\pi}(s_1))^{-1}}>0.
\end{align*}
Then we can define the set 
\begin{align*}
\Lambda:=\left\{\blambda ~:~ \lambda_i = \frac{(V_{i,1}^*(s_1) + \iota - V_{i,1}^{\pi}(s_1))^{-1}}{\sum_{j=1}^m (V_{j,1}^*(s_1) + \iota - V_{j,1}^{\pi}(s_1))^{-1}},~~ \forall \pi\in\Pi_{\mathrm{P}}^* \right\}.	
\end{align*}
Now we first show $\Pi_{\mathrm{P}}^*\subseteq \{\pi~|~\pi \in \min_{\pi\in\Pi} \tchl(\pi), \forall \blambda \in\Lambda\subseteq\Delta_m^o\}$. According to \textbf{Part 2)} of the proof of Proposition \ref{prop:tch} in the subsection above, we can show for any $\breve\pi\in\Pi_{\mathrm{P}}^*\subseteq \Pi_{\mathrm{W}}^*$, it is the solution to the problem $\min_{\pi\in\Pi}\tchl(\pi)$ with $\lambda_i = \frac{(V_{i,1}^*(s_1) + \iota - V_{i,1}^{\breve{\pi}}(s_1))^{-1}}{\sum_{j=1}^m (V_{j,1}^*(s_1) + \iota - V_{j,1}^{\breve{\pi}}(s_1))^{-1}}$. Thus, we conclude $\Pi_{\mathrm{P}}^*\subseteq \{\pi~|~\pi \in \min_{\pi\in\Pi} \tchl(\pi), \forall \blambda \in\Lambda\subseteq\Delta_m^o\}$.

Next, we prove $ \{\pi~|~\pi \in \min_{\pi\in\Pi} \tchl(\pi), \forall \blambda \in\Lambda\subseteq\Delta_m^o\}\subseteq \Pi_{\mathrm{P}}^*$ and all solutions of $\min_{\pi\in\Pi} \tchl(\pi)$ for each $\blambda\in\Lambda$ has the same value. Since the uniqueness of the solution to $\min_{\pi\in\Pi} \tchl(\pi)$ leads to Pareto optimality according to Proposition \ref{prop:tch}, we now assume $\min_{\pi\in\Pi} \tchl(\pi)$ has more than one solution. We assume $\tilde\pi$ is one solution to $\min_{\pi\in\Pi} \tchl(\pi)$ under some $\blambda$ with $\lambda_i =  \frac{(V_{i,1}^*(s_1) + \iota - V_{i,1}^{\breve{\pi}}(s_1))^{-1}}{\sum_{j=1}^m (V_{j,1}^*(s_1) + \iota - V_{j,1}^{\breve{\pi}}(s_1))^{-1}}$ for some  $\breve{\pi}\in\Pi_{\mathrm{P}}^*$. According to the above proof, we know that $\breve{\pi}$ is also the solution to the problem $\min_{\pi\in\Pi} \tchl(\pi)$. Now we further assume that $\tilde\pi \neq \breve\pi$. Since $\breve{\pi}\in \Pi_{\mathrm{P}}^*$, if we can show $V_{i,1}^{\tilde\pi}(s_1) = V_{i,1}^{\breve\pi}(s_1), \forall i\in[m]$, then we obtain $\tilde\pi\in\Pi_{\mathrm{P}}$, which will complete the proof. We prove it by contradiction. Since $\breve\pi$ is a solution, we have
\begin{align*}
	\min_\pi \tchl(\pi) = \tchl(\breve\pi) = \frac{1}{\sum_{j=1}^m (V_{j,1}^*(s_1) + \iota - V_{j,1}^{\breve{\pi}}(s_1))^{-1}}.
\end{align*}
Then, we discuss different cases:
\begin{itemize}
	\item If there exists $i'$ such that $V_{i',1}^{\tilde\pi}(s_1) < V_{i',1}^{\breve\pi}(s_1)$, then $\tchl(\tilde\pi)>\min_\pi \tchl(\pi)$, contradicting the assumption that $\tilde\pi\in \argmin_\pi \tchl(\pi)$. Thus, we must have $V_{i,1}^{\tilde\pi}(s_1) \geq V_{i,1}^{\breve\pi}(s_1), ~\forall i\in[m]$.
	\item If there exists $i'$ such that $V_{i',1}^{\tilde\pi}(s_1) > V_{i',1}^{\breve\pi}(s_1)$ with $V_{i,1}^{\tilde\pi}(s_1) = V_{i,1}^{\breve\pi}(s_1)$ for other $i\neq i'$, then by Definition \ref{def:pareto}, we know $\tilde\pi$ dominates $\breve\pi$, implying that $\breve\pi$ is not Pareto optimal, which contradicts $\breve\pi\in\Pi_{\mathrm{P}}$. 
\end{itemize}
By the above analysis, we obtain $V_{i,1}^{\tilde\pi}(s_1) = V_{i,1}^{\breve\pi}(s_1), \forall i\in[m]$ and $\tilde\pi\in\Pi_{\mathrm{P}}$. The proof is completed.	
\end{proof}

\section{Proofs for Section \ref{sec:tchs}}

This section provides detailed proofs of Proposition \ref{prop:eq-tch} and Theorem \ref{thm:tch} along with several important lemmas for this theorem.

\subsection{Proof of Proposition \ref{prop:eq-tch}}

\begin{proof} According to \eqref{eq:tch}, we have 
	\begin{align*}
		\tch_{\blambda}(\pi):= \max_{i\in[m]} \{ \lambda_i (V_{i,1}^*(s_1) + \iota - V_{i,1}^\pi(s_1)) \}.
	\end{align*}
	For any $x_i$, the inequality $\max_{i\in[m]}\{x_i\} = \max_{(w_i)_{i=1}^m\in\Delta_m} \sum_{i=1}^m w_i x_i$ always holds. Thus, we obtain 
	\begin{align*}
		\tch_{\blambda}(\pi)= \max_{(w_i)_{i=1}^m\in\Delta_m} \sum_{i=1}^m w_i \lambda_i (V_{i,1}^*(s_1) + \iota - V_{i,1}^\pi(s_1)).
	\end{align*}
	Moreover, we have $V_{i,1}^*(s_1):=\max_{\pi\in\Pi} V_{i,1}^{\pi}(s_1)$ by its definition. Therefore, we eventually obtain that 
	\begin{align*}
		\tch_{\blambda}(\pi)&= \max_{(w_i)_{i=1}^m \in\Delta_m}  \sum_{i=1}^m w_i\lambda_i ( \max_{\{\nu_i\in\Pi\}_{i=1}^m} V_{i,1}^{\nu_i}(s_1) + \iota - V_{i,1}^\pi(s_1))\\
		&= \max_{(w_i)_{i=1}^m \in\Delta_m} \max_{\{\nu_i\in\Pi\}_{i=1}^m} \sum_{i=1}^m w_i\lambda_i (V_{i,1}^{\nu_i}(s_1) + \iota - V_{i,1}^\pi(s_1)).
	\end{align*}
	This completes the proof.
\end{proof}

\subsection{Lemmas for Theorem \ref{thm:tch}} \label{sec:lemma-online}

\begin{lemma} \label{lem:v-decomp-2} Algorithm \ref{alg:online-tch} ensures that for any policy $\pi$ and policy $\nu_i$
	\begin{align*}
		&V_{i,1}^{\pi}(s_1) - V_{i,1}^t(s_1) = \sum_{h=1}^H \EE_{\pi, \PP} [ \varsigma_{i,h}^t(s_h, a_h) \given s_1 ] + \sum_{h=1}^H \EE_{\pi, \PP} [ \langle \pi_h(\cdot | s_h)-\pi_h^t(\cdot | s_h), Q_{i,h}^t(s_h,\cdot) \rangle_{\cA}  \given s_1 ],\\
		&V_{i,1}^{\nu_i}(s_1) - \tilde{V}_{i,1}^t(s_1) = \sum_{h=1}^H \EE_{\nu_i, \PP} [ \tilde{\varsigma}_{i,h}^t(s_h, a_h) \given s_1 ] + \sum_{h=1}^H \EE_{\nu_i, \PP} [ \langle \nu_{i,h}(\cdot | s_h)-\tilde \nu_{i,h}^t(\cdot | s_h), \tilde Q_{i,h}^t(s_h,\cdot) \rangle_{\cA}  \given s_1 ].
	\end{align*}
	where $s_h, a_h$ are random variables for the state and action, and we define the model prediction errors of the $Q$-functions as $\varsigma_{i,h}^t(s,a) = r_{i,h}(s,a) +  \PP_h V_{i,h+1}^t(s, a) - Q_{i,h}^t(s,a)$ and $\tilde\varsigma_{i,h}^t(s,a) = r_{i,h}(s,a) +  \PP_h \tilde V_{i,h+1}^t(s, a) - \tilde Q_{i,h}^t(s,a)$.

\end{lemma}
\begin{proof} 
For any $h$ and $s$, we have the following decomposition for $V_{i,1}^{\pi}(s_1) - V_{i,1}^t(s_1)$,
\begin{align}
	\begin{aligned} \label{eq:V_diff_1_init-1}
		V_{i,h}^{\pi}(s) - V_{i,h}^t(s)&=  \langle\pi_h(\cdot | s), Q_{i,h}^{\pi}(s, \cdot )\rangle_\cA  -  \langle \pi_h^t(\cdot | s), Q_h^t(s, \cdot )\rangle_\cA \\
		&=  \langle\pi_h(\cdot | s), Q_{i,h}^{\pi}(s, \cdot )\rangle_\cA  -  \langle\pi_h(\cdot | s), Q_{i,h}^t(s, \cdot )\rangle_\cA  \\
		&\quad +   \langle\pi_h(\cdot | s), Q_{i,h}^t(s, \cdot )\rangle_\cA  -  \langle\pi_h^t(\cdot | s), Q_{i,h}^t(s, \cdot )\rangle_\cA  \\
		&=\langle\pi_h(\cdot | s),Q_{i,h}^\pi(s, \cdot ) - Q_{i,h}^t(s, \cdot )\rangle_\cA  +   \langle \pi_h(\cdot | s)-\pi_h^t(\cdot | s), Q_{i,h}^t(s, \cdot)\rangle_\cA,
	\end{aligned}
\end{align}
where the first inequality is by Bellman equation in \eqref{eq:bellman} and the definition of $V_{i,h}^k$ in Algorithm \ref{alg:online-tch}. Furthermore, by the definition of the model prediction error $\varsigma_{i,h}^t$ defined in this lemma, we have
\begin{align*}
	&\langle\pi_h(\cdot | s),Q_{i,h}^\pi(s, \cdot ) - Q_{i,h}^t(s, \cdot )\rangle_\cA  \\
	& \qquad =  \sum_{a \in \cA}  \pi_h(a | s) \bigg[ \sum_{s'\in \cS}  \PP_h(s'|s, a) \big[V_{i,h+1}^{\pi}(s') -  V_{i,h+1}^k(s')\big] + \varsigma_{i,h}^t(s,a) \bigg] \\
	&\qquad =  \sum_{a \in \cA} \sum_{s'\in \cS} \pi_h(a | s) \PP_h(s'|s, a) \big[V_{i,h+1}^{\pi}(s') -  V_{i,h+1}^k(s')\big]  +  \sum_{a \in \cA}   \pi_h(a | s) \varsigma_{i,h}^t(s,a).
\end{align*}
Combining this result with \eqref{eq:V_diff_1_init-1}, we obtain
\begin{align}
	\begin{aligned} \label{eq:V_diff_1_rec-1}
		V_{i,h}^{\pi}(s) - V_{i,h}^t(s)  & =  \sum_{a \in \cA} \sum_{s'\in \cS} \pi_h(a | s) \PP_h(s'|s, a) \big[V_{i,h+1}^{\pi}(s') -  V_{i,h+1}^k(s')\big]\\
		&\quad   +  \sum_{a \in \cA}   \pi_h(a | s) \varsigma_{i,h}^t(s,a) + \langle \pi_h(\cdot | s)-\pi_h^t(\cdot | s), Q_{i,h}^t(s, \cdot)\rangle_\cA.
	\end{aligned}
\end{align}
Note that the above equality constructs a recursion of the value function difference $V_h^{\pi}(s) - V_h^k(s)$. As we define $V_{H+1}^{\pi}(s) = V_{H+1}^k(s) = 0$, recursively applying \eqref{eq:V_diff_1_rec-1} gives
\begin{align*}  
	V_{i,1}^{\pi}(s_1) - V_{i,1}^t(s_1)  & =   \sum_{h=1}^H \EE_{\pi, \PP} \{  \varsigma_{i,h}^t(s_h, a_h) \given s_1 \}   + \sum_{h=1}^H \EE_{\pi, \PP} \big\{ \langle\pi_h(\cdot | s_h)-\pi_h^t(\cdot | s_h), Q_{i,h}^t(s_h, \cdot)  \rangle_\cA \biggiven  s_1 \big\},
\end{align*}
where $(s_h,a_h)$ are random variables denoting the state and action at the $h$-th step following a distribution jointly determined by $\pi, \PP$. 

For the second inequality in this lemma, we have the same analysis by replacing the policies and value functions in the above derivation with $\nu_i, \tilde{\nu}_i^t$ and $\tilde V_{i,h}^t, \tilde Q_{i,h}^t$. Thus, we have
\begin{align*}
	V_{i,1}^{\nu_i}(s_1) - \tilde{V}_{i,1}^t(s_1) &= \sum_{h=1}^H \EE_{\nu_i, \PP} [ \tilde{\varsigma}_{i,h}^t(s_h, a_h) \given s_1 ]  + \sum_{h=1}^H \EE_{\nu_i, \PP} [ \langle \nu_{i,h}(\cdot | s_h)-\tilde \nu_{i,h}^t(\cdot | s_h), \tilde Q_{i,h}^t(s_h,\cdot) \rangle_{\cA}  \given s_1 ].
\end{align*}
This completes our proof.
\end{proof}

\begin{lemma}[Optimism]\label{lem:v-opt-2} 	
Let $\hat{\PP}_h^t$ and $\hat{r}_{i,h}^t, \forall i\in[m]$ be the estimated transition and reward functions via some estimation procedure using the trajectories generated by $\pi^t$. Suppose that there exist $\Phi_h^t$ and $\Psi_{i,h}^t$ such that $\hat{\PP}_h^t$ and $\hat{r}_{i,h}^t$ satisfy for any $V:\cS\mapsto [0,H]$,
\begin{align*}
	|(\hat{\PP}_h^t - \PP_h)V(s,a)| \leq \Phi_h^t(s,a),\quad \text{and} \quad |\hat{r}_{i,h}^t(s,a) - r_{i,h}(s,a)| \leq \Psi_{i,h}^t(s,a), \forall i\in[m], 
\end{align*}
and the corresponding optimistic Q-function is defined as $Q_{i,h}^t(\cdot,\cdot) = \{ (\hat{r}_{i,h}^t + \hat{\PP}_h^t V_{i,h+1}^t + \Phi_h^t+\Psi_{i,h}^t)(\cdot,\cdot)\}_{[0,H-h+1]}$. Furthermore, we let $\tilde{\PP}_{i,h}^t$ and $\tilde{r}_{i,h}^t, \forall i\in[m]$ be the estimated transition and reward functions using the trajectories generated by $\tilde\nu_i^t$. Suppose that there exist $\Phi_{i,h}^t$ and $\Psi_{i,h}^t$ such that $\tilde{\PP}_h^t$ and $\tilde{r}_{i,h}^t$ satisfy for any $V:\cS\mapsto [0,H]$,
\begin{align*}
	|(\tilde{\PP}_h^t - \PP_h)V(s,a)| \leq \tilde{\Phi}_{i,h}^t(s,a),\quad \text{and} \quad |\tilde{r}_{i,h}^t(s,a) - r_{i,h}(s,a)| \leq \tilde{\Psi}_{i,h}^t(s,a), \forall i\in[m], 
\end{align*}
and the corresponding optimistic Q-function is defined as  $\tilde{Q}_{i,h}^t(\cdot,\cdot) = \{ (\tilde{r}_{i,h}^t + \tilde{\PP}_{i,h}^t \tilde{V}_{i,h+1}^t + \tilde{\Phi}_{i,h}^t+\tilde{\Psi}_{i,h}^t)(\cdot,\cdot)\}_{[0,H-h+1]}$. Then Algorithm \ref{alg:online-tch} ensures that for any policies $\pi$ and $\nu_i$
	\begin{align*}
		&\sum_{h=1}^H \EE_{\pi, \PP} \big[ \varsigma_{i,h}^t(s_h, a_h)   \biggiven s_1 \big] \leq 0 \quad \text{and} \quad  \sum_{h=1}^H \EE_{\nu_i, \PP} \big[ \tilde{\varsigma}_{i,h}^t(s_h, a_h)   \biggiven s_1 \big] \leq 0, 
	\end{align*}
	where $\varsigma_{i,h}^t$ and $\tilde{\varsigma}_{i,h}^t$ is the model prediction errors defined as in Lemma \ref{lem:v-decomp-2}.
\end{lemma}

\begin{proof} 
By plugging in the definition of $Q_{i,h}^t$ in Algorithm \ref{alg:online-tch} as above, we decompose the prediction error as follows
\begin{align}
	\begin{aligned}\label{eq:pred_err_re1}
		\hspace{-0.2cm}\varsigma_{i,h}^t(s, a)&=r_{i,h}(s,a) + \big\langle \PP_h(\cdot \given s, a), V_{i,h+1}^t(\cdot) \big\rangle_\cS - Q_{i,h}^t(s,a) \\
		&\leq r_{i,h}(s,a) + \PP_h V_{i,h+1}^t(s,a) \\
		&\quad- \Big\{ (\hat{r}_{i,h}^t + \hat{\PP}_h^t V_{i,h+1}^t+ \Phi_h^t+\Psi_{i,h}^t)(s,a), H-h+1 \Big\}_{[0,H-h+1]}  \\
		&\leq  \max \Big\{(r_{i,h} - \hat{r}_{i,h}^t + (\PP_h- \hat{\PP}_h^t) V_{i,h+1}^t -  \Phi_h^t-\Psi_{i,h}^t)(s,a), 0 \Big\}, 
	\end{aligned}
\end{align}
where the first inequality holds because 
\begin{align}
	&r_{i,h}(s,a) + \big\langle \PP_h(\cdot \given s_h, a_h), V_{i,h+1}^t(\cdot) \big\rangle_\cS \nonumber \\
	&\qquad \leq r_{i,h}(s,a) + \big\|\PP_h(\cdot \given s_h, a_h)\big\|_1 \|V_{i,h+1}^t(\cdot) \|_\infty  \leq 1 +  \max_{s' \in \cS} \big|V_{i,h+1}^t(s') \big| \leq 1 + H-h, \label{eq:h-bound}
\end{align}
due to $\big\|\PP_h(\cdot \given s_h, a_h)\big\|_1 = 1$ and also the definition of $Q_{i,h+1}^t$ such that for any $s' \in \cS$
\begin{align}
	\begin{aligned}\label{eq:bound_V_up}
		\big|V_{i,h+1}^t(s') \big| &=  \Big| \langle\pi_{h+1}^t(\cdot | s'),  Q_{i,h+1}^t(s', \cdot) \rangle \Big|\\
		&\leq \big\| \pi_{h+1}^t(\cdot | s')\big\|_1 \big\|Q_{i,h+1}^t(s', \cdot) \big\|_\infty \\
		&\leq \max_{a} \big|Q_{i,h+1}^t(s', a)\big|\leq H-h.
	\end{aligned}
\end{align}
Note that we have the condition that for any $V:\cS\mapsto [0,H]$,
\begin{align*}
	&|\hat{r}_{i,h}^t(s,a) - r_{i,h}(s,a)| \leq \Psi_{i,h}^t(s,a), \forall i\in[m], \qquad |(\hat{\PP}_h^t - \PP_h)V(s,a)| \leq \Phi_h^t(s,a).
\end{align*}
Then, we obtain
\begin{align*}
	&r_{i,h}(s,a) - \hat{r}_{i,h}^t(s,a) - \Psi_{i,h}^t(s,a)  \leq \big|r_{i,h}(s,a) - \hat{r}_{i,h}^t(s,a)\big| - \beta_h^{r,k}(s,a) \leq 0.
\end{align*}
In addition, we can obtain
\begin{align*}
	&\big\langle \PP_h(\cdot \given s, a) - \hat{\PP}_h^t(\cdot |s, a), V_{i,h+1}^t(\cdot) \big\rangle_\cS  - \Phi_h^t(s,a)\leq \Phi_h^t(s,a)   - \Phi_h^t(s,a)  = 0.
\end{align*}
Thus, we have
\begin{align*}
	r_{i,h}(s,a) - \hat{r}_{i,h}^t(s,a) + \big\langle \PP_h(\cdot \given s, a) - \hat{\PP}_h^t(\cdot |s, a), V_{i,h+1}^t(\cdot) \big\rangle_\cS  - \Phi_h^t(s,a)-\Psi_{i,h}^t(s,a) \leq 0.
\end{align*}
Combining the above inequality with \eqref{eq:pred_err_re1}, we have 
\begin{align*}
	&\varsigma_{i,h}^t(s, a) \leq 0,
\end{align*}
which leads to
\begin{align*}
	\sum_{h=1}^H \EE_{\pi, \PP} \big[ \varsigma_{i,h}^t(s_h, a_h)   \biggiven s_1 \big] \leq 0.
\end{align*}
Following the above proof, we can similarly prove
\begin{align*}
	\sum_{h=1}^H \EE_{\nu_i, \PP} \big[ \tilde \varsigma_{i,h}^t(s_h, a_h)   \biggiven s_1 \big] \leq 0.
\end{align*}
This completes the proof.
\end{proof}

\begin{lemma} \label{lem:online-err-2} 
	Under the same conditions as Lemma \ref{lem:v-opt-2},
	then with probability at least $1-\delta$, Algorithm \ref{alg:online-tch} ensures that for all $i\in [m]$,
	\begin{align*}
		&\sum_{t=1}^T \big[V_{i,1}^t(s_1)  - V_{i,1}^{\pi^t}(s_1)\big]  \leq \cO\bigg(\sqrt{H^3 T \log \frac{m}{\delta}} \bigg) + 2 \sum_{t=1}^T\sum_{h=1}^H [\Psi_{i,h}^t(s_h^t, a_h^t)  +  \Phi_h^t(s_h^t, a_h^t)],\\
		&\sum_{t=1}^T \big[\tilde{V}_{i,1}^t(s_1)  - V_{i,1}^{\tilde\nu_i^t}(s_1)\big]  \leq \cO\bigg(\sqrt{H^3 T \log \frac{m}{\delta}} \bigg) + 2 \sum_{t=1}^T\sum_{h=1}^H [\tilde \Psi_{i,h}^t(\tilde s_{i,h}^t, \tilde a_{i,h}^t)  +   \tilde \Phi_{i,h}^t(\tilde s_{i,h}^t, \tilde a_{i,h}^t)],
	\end{align*}
	where $\{ (s_h^t, a_h^t)\}_{h=1}^H, \forall t\in [T],$ is generated via the policy $\pi^t$ and $\{ (\tilde s_{i,h}^t, \tilde a_{i,h}^t)\}_{h=1}^H, \forall t\in [T],$ is generated via the policy $\tilde\nu_i^t$.
\end{lemma}

\begin{proof} 
	
	We start with proving the first inequality of this lemma. Assume that the trajectory $\{ (s_h^t, a_h^t)\}_{h=1}^H, \forall t\in [T],$ is generated following the policy $\pi^t$. Thus, we expand the bias term at the $h$-th step of the $t$-th episode, which is
	\begin{align}
		\begin{aligned} \label{eq:bia_1_init-1}
			V_{i,h}^t(s_h^t)  - V_{i,h}^{\pi^t}(s_h^t) & = \langle \pi_h^t(\cdot | s_h^t), Q_{i,h}^t(s_h^t, \cdot )  -  Q_{i,h}^{\pi^t}(s_h^t, \cdot )\rangle_\cA\\
			&=  Q_{i,h}^t(s_h^t, a_h^t)  -  Q_{i,h}^{\pi^t}(s_h^t, a_h^t)\\
			&= \big\langle \PP_h(\cdot \given s_h^t, a_h^t), V_{i,h+1}^t(\cdot) - V_{h+1}^{\pi^t} (\cdot) \big\rangle_\cS - \varsigma_{i,h}^t(s_h^t, a_h^t) \\
			& = \xi_{i,h}^t + V_{i,h+1}^t(s_{h+1}^t) - V_{h+1}^{\pi^t} (s_{h+1}^t) - \varsigma_{i,h}^t(s_h^t, a_h^t),
		\end{aligned}
	\end{align}
	where the first equality by Algorithm \ref{alg:online-tch} and \eqref{eq:bellman}, the second inequality is due to that we take a greedy policy as shown in the algorithm such that $\pi_h^t(a_h^t|s_h^t) = 1$, and the third equality is by the definition of the model prediction error $\varsigma_{i,h}^t$ in Lemma \ref{lem:v-opt-2}. 
	 Here we introduce the martingale difference sequence $\{\xi_{i,h}^t\}_{h>0, t>0}$, which is defined as
	\begin{align*}
		& \xi_{i,h}^t := \big\langle \PP_h(\cdot \given s_h^t, a_h^t), V_{i,h+1}^t(\cdot) - V_{h+1}^{\pi^t} (\cdot) \big\rangle_\cS - \big[ V_{i,h+1}^t(s_{h+1}^t) - V_{h+1}^{\pi^t} (s_{h+1}^t)\big],
	\end{align*}
	such that 
	\begin{align*} \EE_{s_{h+1}^t \sim \PP_h(\cdot \given s_h^t, a_h^t)} \big[\xi_{i,h}^t \biggiven \cF_h^t\big] = 0,
	\end{align*}
	where  $\cF_h^t$ is the filtration of all randomness up to $(h-1)$-th step of the $t$-th episode plus $s_h^t, a_h^t$.
	
	Note that \eqref{eq:bia_1_init-1} constructs a recursion for $V_{i,h}^t(s_h^t)  - V_{i,h}^{\pi^t}(s_h^t)$. Since $V_{H+1}^t(\cdot) = \boldsymbol{0}$ and $V_{H+1}^{\pi^t} (\cdot)= \boldsymbol{0}$,  recursively applying \eqref{eq:bia_1_init-1} gives
	\begin{align} \label{eq:bias_diff_init-1}
		V_{i,1}^t(s_1)  - V_{i,1}^{\pi^t}(s_1) =  \sum_{h=1}^H \zeta_{i,h}^t + \sum_{h=1}^H \xi_{i,h}^t - \sum_{h=1}^H \varsigma_{i,h}^t(s_h^t, a_h^t).
	\end{align}
	Moreover, by the updating rule of $Q_{i,h}^t$ in Algorithm \ref{alg:online-tch}, we have
	\begin{align*}
		-\varsigma_{i,h}^t(s_h^t, a_h^t) &= \{ \hat{r}_h^t(s_h^t, a_h^t)  +  \big\langle \hat{\PP}_h^t(\cdot |s_h^t, a_h^t), V_{i,h+1}^t(\cdot) \big\rangle_\cS  + \Psi_{i,h}^t(s_h^t, a_h^t) +\Phi_h^t(s_h^t, a_h^t) \}_{[0,H-h+1]}\\
		&\quad -  r_{i,h}(s_h^t, a_h^t) -  \big\langle \PP_h(\cdot \given s_h, a_h), V_{i,h+1}^t(\cdot) \big\rangle_\cS.
	\end{align*}
	Then, we can bound $-\varsigma_{i,h}^t(s_h^t, a_h^t)$ as
	\begin{align*}
		-\varsigma_{i,h}^t(s_h^t, a_h^t)&\leq -  r_{i,h}(s_h^t, a_h^t) - \big\langle \PP_h(\cdot \given s_h^t, a_h^t), V_{i,h+1}^t(\cdot)\big\rangle_\cS + \hat{r}_h^t(s_h^t, a_h^t) \\
		&\quad + \big\langle  \hat{\PP}_h^t(\cdot |s_h^t, a_h^t), V_{i,h+1}^t(\cdot) \big\rangle_\cS  + \Psi_{i,h}^t(s_h^t, a_h^t) +\Phi_h^t(s_h^t, a_h^t)  \\
		&\leq \big|\hat{r}_h^t(s_h^t, a_h^t) -  r_{i,h}(s_h^t, a_h^t)\big| \\
		&\quad  + \Big| \big\langle \PP_h(\cdot \given s_h^t, a_h^t) - \hat{\PP}_h^t(\cdot \given s_h^t, a_h^t), V_{i,h+1}^t(\cdot) \big\rangle_\cS \Big| + \Psi_{i,h}^t(s_h^t, a_h^t) +\Phi_h^t(s_h^t, a_h^t) .
	\end{align*}
	where the first inequality is due to  $\{x\}_{[0,H-h+1]} \leq x$ if $x\geq 0$. 
	
	As we have the condition that for any $V:\cS\mapsto [0,H]$, 
	\begin{align*}
		&|\hat{r}_{i,h}^t(s,a) - r_{i,h}(s,a)| \leq \Psi_{i,h}^t(s,a), \forall i\in[m],\\
		&|(\hat{\PP}_h^t - \PP_h)V(s,a)| \leq \Phi_h^t(s,a).
	\end{align*}
	Putting the above together yields
	\begin{align*}
		-\varsigma_{i,h}^t(s_h^t, a_h^t)&\leq  \big|\hat{r}_h^t(s_h^t, a_h^t) -  r_{i,h}(s_h^t, a_h^t)\big| + \Big| \big\langle \PP_h(\cdot \given s_h^t, a_h^t) - \hat{\PP}_h^t(\cdot \given s_h^t, a_h^t), V_{i,h+1}^t(\cdot) \big\rangle_\cS \Big| \\
		&\quad + \Psi_{i,h}^t(s_h^t, a_h^t) +\Phi_h^t(s_h^t, a_h^t) \leq  2\Psi_{i,h}^t(s_h^t, a_h^t) +2\Phi_h^t(s_h^t, a_h^t).
	\end{align*}
	Therefore, by \eqref{eq:bias_diff_init-1}, we have
	\begin{align*}
		&\sum_{t=1}^T \big[V_{i,1}^t(s_1)  - V_{i,1}^{\pi^t}(s_1)\big]  \leq  \sum_{t=1}^T\sum_{h=1}^H \xi_{i,h}^t + 2 \sum_{t=1}^T\sum_{h=1}^H \Psi_{i,h}^t(s_h^t, a_h^t)  + 2\sum_{t=1}^T \sum_{h=1}^H \Phi_h^t(s_h^t, a_h^t). 
	\end{align*}
	By Azuma-Hoeffding inequality and the union bound, with probability at least $1-\delta$, for all $i\in[m]$, the following inequality hold 
	\begin{align*}
		&\sum_{t=1}^T \sum_{h=1}^H \xi_{i,h}^t \leq \cO\left(\sqrt{H^3 T \log \frac{m}{\delta}} \right), 
	\end{align*} 
	where we use the facts that $ | Q_{i,h}^t(s_h^t, a_h^t)  -  Q_{i,h}^{\pi^t}(s_h^t, a_h^t) | \leq 2H$ and $| V_{i,h+1}^t(s_{h+1}^t) - V_{h+1}^{\pi^t} (s_{h+1}^t) |\leq 2H$. Combining the above together, we obtain that with probability at least $1-\delta$, for all $i \in  [m]$, 
	\begin{align*}
		&\sum_{t=1}^T \big[V_{i,1}^t(s_1)  - V_{i,1}^{\pi^t}(s_1)\big]  \leq \cO\left(\sqrt{H^3 T \log \frac{m}{\delta}} \right) + 2 \sum_{t=1}^T\sum_{h=1}^H [\Psi_{i,h}^t(s_h^t, a_h^t)  +  \Phi_h^t(s_h^t, a_h^t)]. 
	\end{align*}
	Following the above proof, we similarly obtain
	\begin{align*}
	&\sum_{t=1}^T \big[\tilde{V}_{i,1}^t(s_1)  - V_{i,1}^{\tilde\nu_i^t}(s_1)\big]  \leq \cO\bigg(\sqrt{H^3 T \log \frac{m}{\delta}} \bigg) + 2 \sum_{t=1}^T\sum_{h=1}^H [\tilde \Psi_{i,h}^t(\tilde s_{i,h}^t, \tilde a_{i,h}^t)  +   \tilde \Phi_{i,h}^t(\tilde s_{i,h}^t, \tilde a_{i,h}^t)].	
	\end{align*}
	Taking the union bound eventually finishes the proof.
\end{proof}

\begin{lemma} \label{lem:sum-bonus-2} 
	Suppose that $\Phi_h^t(s,a) = \sqrt{\frac{2 H^2 |\cS|\log (8m|\cS||\cA|HT/\delta)}{ N_h^{t-1}(s,a)\vee 1}} \wedge H$ and $\Psi_{i,h}^t(s,a) = \sqrt{\frac{2 \log (8m|\cS||\cA|HT/\delta)}{  N_h^{t-1}(s,a) \vee 1}} \allowbreak \wedge 1$ are the instantiation of the bonus terms following \texttt{OptQ} in Line \ref{line:tch-mainQ} of Algorithm \ref{alg:online-tch} where $N_h^{t-1}(s,a) = \sum_{\tau=1}^{t-1}\mathbf{1}_{\{(s,a)=(s_h^\tau,a_h^\tau)\}}$.
	Suppose that $\tilde{\Phi}_{i,h}^t(s,a) = \sqrt{\frac{2H^2|\cS| \log (8m|\cS||\cA|HT/\delta)}{ \tilde{N}_{i,h}^{t-1}(s,a) \vee 1}} \wedge H $ and $\tilde{\Psi}_{i,h}^t(s,a)= \sqrt{\frac{2 \log (8m|\cS||\cA|HT/\delta)}{ \tilde N_{i,h}^{t-1}(s,a) \vee 1}} \wedge 1$ are the instantiation of the bonus terms following \texttt{OptQ} in Line \ref{line:tch-auxQ} of Algorithm \ref{alg:online-tch} where $\tilde N_{i,h}^{t-1}(s,a) = \sum_{\tau=1}^{t-1}\mathbf{1}_{\{(s,a)=(\tilde s_{i,h}^\tau,\tilde a_{i,h}^\tau)\}}$.
	The updating rules in Algorithm \ref{alg:online-tch} ensure 
\begin{align*}
	&\sum_{t=1}^T\sum_{h=1}^H \Psi_{i,h}^t(s_h^t, a_h^t) \leq \cO \left(H \sqrt{ T |\cS| |\cA| \log \frac{m|\cS| |\cA|  HT}{\delta}} \right), \\
	&\sum_{t=1}^T \sum_{h=1}^H \Phi_h^t(s_h^t, a_h^t) \leq \cO \left(H^2 \sqrt{ T |\cS|^2 |\cA| \log \frac{m|\cS| |\cA|  HT}{\delta}} \right),
\end{align*}
and 
\begin{align*}
	&\sum_{t=1}^T\sum_{h=1}^H \tilde{\Psi}_{i,h}^t(\tilde{s}_{i,h}^t, \tilde{a}_{i,h}^t) \leq \cO \left(H \sqrt{ T |\cS| |\cA| \log \frac{m|\cS| |\cA|  HT}{\delta}} \right), \\
	&\sum_{t=1}^T \sum_{h=1}^H \tilde{\Phi}_{i,h}^t(\tilde{s}_{i,h}^t, \tilde{a}_{i,h}^t) \leq \cO \left(H^2 \sqrt{ T |\cS|^2 |\cA| \log \frac{m|\cS| |\cA|  HT}{\delta}} \right),
\end{align*}
	where the trajectory $\{ (s_h^t, a_h^t)\}_{h=1}^H, \forall t\in [T],$ is generated following the policy $\pi^t$, and the trajectory $\{ (\tilde{s}_{i,h}^t, \tilde{a}_h^t)\}_{h=1}^H, \forall t\in [T],$ is generated following the policies $\tilde\nu_i^t$ for all $i\in[m]$.
\end{lemma} 

\begin{proof}  
	
 We show that 
\begin{align*}
	\sum_{t=1}^T\sum_{h=1}^H \Phi_h^t(s_h^t, a_h^t) & = \sum_{t=1}^T\sum_{h=1}^H   \sqrt{\frac{ 2H^2|\cS| \log (8m|\cS| |\cA|  HT/\delta)}{\max\{N^{t-1}_h(s_h^t, a_h^t), 1\}}} \\
	& \leq \sum_{t=1}^T\sum_{h=1}^H   \sqrt{\frac{ 4H^2|\cS|\log (8m|\cS| |\cA|  HT/\delta)}{N^t_h(s_h^t, a_h^t)}} \\
	&\leq \sum_{h=1}^H ~ \sum_{\substack{(s, a)\in \cS\times\cA\\ N^T_h(s, a) > 0}}\sum_{n=1}^{N^T_h(s, a)}  \sqrt{\frac{ 2H^2|\cS|\log (8m|\cS| |\cA|  HT/\delta)}{n}}.
\end{align*}
Moreover, we  have
\begin{align*}
	&\sum_{h=1}^H ~ \sum_{\substack{(s, a)\in \cS\times\cA\\ N^T_h(s, a) > 0}}\sum_{n=1}^{N^T_h(s, a)}  \sqrt{\frac{2H^2|\cS| \log (8m|\cS| |\cA|  HT/\delta)}{n}} \\
	&\qquad\leq \sum_{h=1}^H ~ \sum_{(s, a)\in \cS\times\cA}  \cO \left(\sqrt{ 2H^2|\cS| N^T_h(s, a) \log \frac{m|\cS| |\cA|  HT}{\delta}} \right) \\
	&\qquad \leq \cO \left(H^2 \sqrt{ T |\cS|^2 |\cA| \log \frac{m|\cS| |\cA|  HT}{\delta}} \right),
\end{align*}
where the last inequality is based on the consideration that $\sum_{(s, a)\in \cS\times\cA} N_h^T(s,a) = T$ such that $\sum_{(s, a)\in \cS\times\cA} \sqrt{ N^T_h(s, a)} \leq \cO\left(\sqrt{ T |\cS||\cA|}\right) $ when $T$ is sufficiently large. Combining the above results, we obtain
\begin{align*}
	\sum_{t=1}^T\sum_{h=1}^H \Phi_h^t(s_h^t, a_h^t)  \leq \cO \left(H^2 \sqrt{  |\cS|^2|\cA|T \log \frac{8m|\cS| |\cA|  HT}{\delta}} \right).
\end{align*}
For $\Psi_{i,h}^t$, we similarly have
\begin{align*}
	\sum_{t=1}^T \sum_{h=1}^H \Psi_{i,h}^t(s_h^t, a_h^t) & = \sum_{t=1}^T\sum_{h=1}^H   \sqrt{\frac{2 \log (8m|\cS| |\cA| HT/\delta)}{\max\{N^{t-1}_h(s_h^t, a_h^t), 1\}}} \\
	&\leq \sum_{h=1}^H ~ \sum_{(s, a )\in \cS\times\cA }   \sqrt{ 2N^T_h(s, a)   \log \frac{8m|\cS| |\cA| HT}{\delta}}  \\
	&\leq \cO\left(H \sqrt{ T |\cS| |\cA|  \log \frac{m|\cS| |\cA| HT}{\delta}}\right) ,
\end{align*}
when $T$ is sufficiently large. 
Finally, following the above proof, we can similarly prove
\begin{align*}
	&\sum_{t=1}^T\sum_{h=1}^H \tilde{\Psi}_{i,h}^t(\tilde{s}_{i,h}^t, \tilde{a}_{i,h}^t) \leq \cO \left(H \sqrt{ T |\cS| |\cA| \log \frac{m|\cS| |\cA|  HT}{\delta}} \right), \\
	&\sum_{t=1}^T \sum_{h=1}^H \tilde{\Phi}_{i,h}^t(\tilde{s}_{i,h}^t, \tilde{a}_{i,h}^t) \leq \cO \left(H^2 \sqrt{ T |\cS|^2 |\cA| \log \frac{m|\cS| |\cA|  HT}{\delta}} \right).
\end{align*}
This completes the proof.
\end{proof}

\begin{lemma}\label{lem:oco} Setting $\eta = \log^{\frac{1}{2}} m/(H\sqrt{T})$, the updating rule of $\bw$ in Algorithm \ref{alg:online-tch} ensures
\begin{align*}
\max_{\bw\in\Delta_m} \frac{1}{T}\sumt (\bw  - \bw^t)^\top [ \blambda \odot (\tilde{\bV}_1^t(s_1)+\biota -\bV_1^t(s_1))] \leq \frac{6H\log^{\frac{1}{2}} m}{\sqrt{T}}.
\end{align*}
\end{lemma}

\begin{proof} 
	The mirror ascent step at the $(t+1)$-th episode is equivalent to solving the following maximization problem
	\begin{align*}
		\maximize_{\bw\in\Delta_m} \quad &\eta\langle \bw, \blambda \odot (\tilde{\bV}_1^t(s_1)+\biota -\bV_1^t(s_1)) \rangle -  D_{\mathrm{KL}}\big( \bw, \bw^t \big).
	\end{align*}
	We can equivalently reformulate this maximization problem to a minimization problem as
	\begin{align*}
		\minimize_{\bw\in\Delta_m} \quad &-\eta\langle \bw, \blambda \odot (\tilde{\bV}_1^t(s_1)+\biota -\bV_1^t(s_1)) \rangle +  D_{\mathrm{KL}}\big( \bw, \bw^t \big).
	\end{align*}
	And we let $\bw^{t+1}$ be the solution of the above optimization problem.
	Note that $\bw^{t+1}$ is guaranteed to stay in the relative interior of a probability simplex if we initialize $w_i^0 = 1 / m$. Thus, applying Lemma \ref{lem:pushback} gives
	\begin{align*}
		&-\eta \langle \bw^{t+1}, \blambda \odot (\tilde{\bV}_1^t(s_1)+\biota -\bV_1^t(s_1))  \rangle  + \eta \langle \bw, \blambda \odot (\tilde{\bV}_1^t(s_1)+\biota -\bV_1^t(s_1))  \rangle \\
		&\qquad \leq D_{\mathrm{KL}}\big( \bw, \bw^t \big) -  D_{\mathrm{KL}}\big( \bw, \bw^{t+1}\big) -  D_{\mathrm{KL}}\big( \bw^{t+1}, \bw^t \big),
	\end{align*}
	where $\bw$ is an arbitrary variable in $\Delta_m$. 
	Rearranging the terms leads to
	\begin{align}
		\begin{aligned} \label{eq:oco-1}
			&\eta \langle  \bw- \bw^t,  \blambda \odot (\tilde{\bV}_1^t(s_1)+\biota -\bV_1^t(s_1)) \rangle  \\
			&\qquad \leq D_{\mathrm{KL}}\big( \bw, \bw^t\big) -  D_{\mathrm{KL}}\big( \bw, \bw^{t+1} \big) -  D_{\mathrm{KL}}\big( \bw^{t+1}, \bw^t \big) \\
			&\qquad \quad + \eta \langle  \bw^{t+1} -  \bw^t,  \blambda \odot (\tilde{\bV}_1^t(s_1)+\biota -\bV_1^t(s_1))  \rangle.
		\end{aligned}
	\end{align}
	By Pinsker's inequality, we have
	\begin{align*}
		&-D_{\mathrm{KL}}\big( \bw^{t+1}, \bw^t \big) \leq -\frac{1}{2} \big\|\bw^{t+1} - \bw^t\big\|^2_1.
	\end{align*}
	Further by Cauchy-Schwarz inequality, we have
	\begin{align*}
		&\eta \langle  \bw^{t+1} -  \bw^t,  \blambda \odot (\tilde{\bV}_1^t(s_1)+\biota -\bV_1^t(s_1)) \rangle\\ 
		&\qquad \leq \eta \|\bw^{t+1} -  \bw^t\|_1  \|\blambda \odot (\tilde{\bV}_1^t(s_1)+\biota -\bV_1^t(s_1)) \|_\infty\\
		&\qquad\leq \frac{1}{2} \big\|\bw^{t+1} - \bw^t \big\|_1^2 + \frac{\eta^2}{2}  \big\|\blambda \odot (\tilde{\bV}_1^t(s_1)+\biota -\bV_1^t(s_1)) \big\|_\infty^2 \\
		&\qquad \leq \frac{1}{2} \big\|\bw^{t+1} - \bw^t \big\|_1^2 + \frac{9H^2\eta^2}{2},
	\end{align*}
	where the last inequality is due to $\|\blambda \odot  (\tilde{\bV}_1^t(s_1)+\biota -\bV_1^t(s_1))\|_\infty^2 = (\max_i \lambda_i (\tilde{V}_{i,1}^t(s_1)+\iota -V_{i,1}^t(s_1)))^2 \leq 9H^2$ with $\iota,\lambda_i\leq 1$. Therefore, combining the above inequalities with \eqref{eq:oco-1} gives
	\begin{align*}
		& \langle  \bw- \bw^t,  \blambda \odot (\tilde{\bV}_1^t(s_1)+\biota -\bV_1^t(s_1)) \rangle  \\
		&\qquad \leq \frac{1}{\eta} D_{\mathrm{KL}}\big( \bw, \bw^t\big) -  \frac{1}{\eta} D_{\mathrm{KL}}\big( \bw, \bw^{t+1} \big) + \frac{9H^2\eta}{2}.
	\end{align*}
	Taking summation from $1$ to $T$ on both sides, we have
	\begin{align*}
		&\sumt \bw^\top \left( \blambda \odot (\tilde{\bV}_1^t(s_1)+\biota -\bV_1^t(s_1))\right) - \sumt (\bw^t)^\top \left( \blambda \odot (\tilde{\bV}_1^t(s_1)+\biota -\bV_1^t(s_1))\right) \\
		& \leq \sumt \left(\frac{1}{\eta} D_{\mathrm{KL}}\big( \bw, \bw^t\big)  - \frac{1}{\eta}D_{\mathrm{KL}}\big( \bw, \bw^{t+1}\big)\right) + 5H^2\sumt \eta \\
		& = \frac{1}{\eta}D_{\mathrm{KL}}\big( \bw, \bw^1\big)  - \frac{1}{\eta}D_{\mathrm{KL}}\big( \bw, \bw^{T+1}\big)  + 5H^2 \sumt \eta \\
		& \leq \frac{1}{\eta}D_{\mathrm{KL}}\big( \bw, \bw^1\big)    + 5H^2 \sumt \eta \leq \frac{\log m}{\eta}    + 5H^2 T \eta,
	\end{align*}
	where the second inequality is due to $ \frac{1}{\eta}D_{\mathrm{KL}}\big( \bw, \bw^{T+1}\big)> 0$ and the last inequality is due to that our initialization of this algorithm ensures that $w_i^1=1/m$ such that $D_{\mathrm{KL}}\big( \bw, \bw^1\big)= \sumi w_i \log (w_im) \leq \log m$. Setting $\eta = \log^{\frac{1}{2}} m/(H\sqrt{T})$, dividing both sides by $T$, we have
	\begin{align*}
		\max_{\bw\in\Delta_m} \frac{1}{T}\sumt (\bw  - \bw^t)^\top [ \blambda \odot (\tilde{\bV}_1^t(s_1)+\biota -\bV_1^t(s_1))] \leq \frac{6H\log^{\frac{1}{2}} m}{\sqrt{T}}.
	\end{align*}
	This completes the proof.	
\end{proof}

\begin{lemma}[Concentration]  \label{lem:concentrate}
	Suppose that $\hat{r}_{i,h}^t = \frac{\sum_{\tau=1}^{t-1}\mathbf{1}_{\{(s,a)=(s_h^\tau,a_h^\tau)\}}r_{i,h}^\tau}{ N_h^{t-1}(s,a)\vee 1}$, $\hat \PP_h^t = \frac{N_h^{t-1}(s,a,s')}{ N_h^{t-1}(s,a)\vee 1}$,  $\Phi_h^t(s,a) = \sqrt{\frac{2 H^2 |\cS|\log (8m|\cS||\cA|HT/\delta)}{ N_h^{t-1}(s,a)\vee 1}} \wedge H$, and $\Psi_{i,h}^t(s,a) = \sqrt{\frac{2 \log (8m|\cS||\cA|HT/\delta)}{  N_h^{t-1}(s,a) \vee 1}} \wedge 1$ are the instantiation of reward and transition estimates and their associated bonus terms following \texttt{OptQ} in Line \ref{line:tch-mainQ} of Algorithm \ref{alg:online-tch} where $N_h^{t-1}(s,a) = \sum_{\tau=1}^{t-1}\mathbf{1}_{\{(s,a)=(s_h^\tau,a_h^\tau)\}}$.
	Suppose that $\tilde{r}_{i,h}^t = \frac{\sum_{\tau=1}^{t-1}\mathbf{1}_{\{(s,a)=(s_h^\tau,a_h^\tau)\}}\tilde r_{i,h}^\tau}{ \tilde N_h^{t-1}(s,a)\vee 1}$, $\tilde \PP_h^t = \frac{\tilde N_h^{t-1}(s,a,s')}{ \tilde N_h^{t-1}(s,a)\vee 1}$, $\tilde{\Phi}_{i,h}^t(s,a) = \sqrt{\frac{2H^2|\cS| \log (8m|\cS||\cA|HT/\delta)}{ \tilde{N}_{i,h}^{t-1}(s,a) \vee 1}} \wedge H$, and $\tilde{\Psi}_{i,h}^t(s,a)= \sqrt{\frac{2 \log (8m|\cS||\cA|HT/\delta)}{ \tilde N_{i,h}^{t-1}(s,a) \vee 1}} \wedge 1$ are the instantiation of the estimates and the bonus terms following \texttt{OptQ} for Line \ref{line:tch-auxQ} in Algorithm \ref{alg:online-tch} where $\tilde N_{i,h}^{t-1}(s,a) = \sum_{\tau=1}^{t-1}\mathbf{1}_{\{(s,a)=(\tilde s_{i,h}^\tau,\tilde a_{i,h}^\tau)\}}$. Then we have with probability at least $1-\delta$, for any $V:\cS\mapsto [0,H]$,
	\begin{align*}
			|(\hat{\PP}_h^t  - \PP_h)V(s,a)| \leq \Phi_h^t(s,a), \quad |\hat{r}_{i,h}^t(s,a) - r_{i,h}(s,a)| \leq \Psi_{i,h}^t(s,a), ~\forall i\in[m],
	\end{align*}
	and
\begin{align*} 
			|(\tilde{\PP}_h^t - \PP_h)V(s,a)| \leq \tilde{\Phi}_{i,h}^t(s,a), \quad  |\tilde{r}_{i,h}^t(s,a) - r_{i,h}(s,a)| \leq \tilde{\Psi}_{i,h}^t(s,a), ~\forall i\in[m]. 		
	\end{align*}
\end{lemma}
\begin{proof}According to Lemma \ref{lem:P_bound}, we obtain that with probability at least $1-\delta'$, for any $V:\cS\mapsto [0,H]$,
\begin{align*}
&|\hat{\PP}_h^t V(s,a) - \PP_hV(s,a)| \leq \|\hat{\PP}_h^t(\cdot|s,a)- \PP_h(\cdot|s,a)\|_1\|V(\cdot)\|_\infty  \leq H\sqrt{\frac{2 |\cS|\log (2|\cS||\cA|H/\delta)}{ N_h^{t-1}(s,a)\vee 1}} 
\end{align*}
where the first inequality is by Cauchy-Schwarz inequality. Moreover, we further have $0\leq \hat{\PP}_h^t V(s,a)\leq \|\hat{\PP}_h^t(\cdot|s,a)\|_1 \|V(\cdot)\|_\infty \leq H$ and similarly $0\leq \PP_h V(s,a)\leq H$. Therefore, 
\begin{align*}
	|\hat{\PP}_h^t V(s,a) - \PP_hV(s,a)|\leq H.
\end{align*}
Combining the above result, we have
\begin{align*}
	|\hat{\PP}_h^t V(s,a) - \PP_hV(s,a)|\leq \sqrt{\frac{2 H^2 |\cS|\log (2|\cS||\cA|H/\delta')}{ N_h^{t-1}(s,a)\vee 1}}  \wedge H.
\end{align*}
Furthermore, according to Lemma \ref{lem:r_bound}, we have
\begin{align*}
	|\hat{r}_{i,h}^t(s,a) - r_{i,h}(s,a)| \leq  \sqrt{\frac{2 \log (2|\cS||\cA|H/\delta)}{ N_h^{t-1}(s,a)\vee 1}}.
\end{align*} 
Moreover, since $0\leq \hat{r}_{i,h}^t(s,a)\leq 1$ and $0\leq r_{i,h}(s,a)\leq 1$ such that $|\hat{r}_{i,h}^t(s,a) - r_{i,h}(s,a)| \leq 1$, thus we obtain
\begin{align*}
	|\hat{r}_{i,h}^t(s,a) - r_{i,h}(s,a)| \leq  \sqrt{\frac{2 \log (2|\cS||\cA|H/\delta)}{ N_h^{t-1}(s,a)\vee 1}}\wedge 1.
\end{align*}
Similar to the proof above, we also have 
\begin{align*}
	&|(\tilde{\PP}_h^t - \PP_h)(s,a)| \leq \sqrt{\frac{2 H^2 |\cS|\log (2|\cS||\cA|H/\delta')}{ N_h^{t-1}(s,a)\vee 1}}  \wedge H,\quad |\tilde{r}_{i,h}^t(s,a) - r_{i,h}(s,a)| \leq \sqrt{\frac{2 \log (2|\cS||\cA|H/\delta')}{ N_h^{t-1}(s,a)\vee 1}}\wedge 1, 		
\end{align*}
with probability at least $1-\delta'$ for each of the above inequality. By union bound, we have with probability at least $1-\delta$, for all $t\in[T]$ and $i\in[m]$,
\vspace{-0.1cm}
\begin{align*}
	|(\hat{\PP}_h^t  - \PP_h)V(s,a)| \leq \Phi_h^t(s,a), \quad |\hat{r}_{i,h}^t(s,a) - r_{i,h}(s,a)| \leq \Psi_{i,h}^t(s,a), ~\forall i\in[m],
\end{align*}
and
\vspace{-0.1cm}
\begin{align*} 
	|(\tilde{\PP}_h^t - \PP_h)(s,a)| \leq \tilde{\Phi}_{i,h}^t(s,a), \quad  |\tilde{r}_{i,h}^t(s,a) - r_{i,h}(s,a)| \leq \tilde{\Psi}_{i,h}^t(s,a), ~\forall i\in[m], 		
\end{align*}
where $\Phi_h^t(s,a) = \sqrt{\frac{2 H^2 |\cS|\log (8m|\cS||\cA|HT/\delta)}{ N_h^{t-1}(s,a)\vee 1}} \wedge H$, $\Psi_{i,h}^t(s,a) = \sqrt{\frac{2 \log (8m|\cS||\cA|HT/\delta)}{  N_h^{t-1}(s,a) \vee 1}} \wedge 1$, $\tilde{\Phi}_{i,h}^t(s,a) = \sqrt{\frac{2H^2|\cS| \log (8m|\cS||\cA|HT/\delta)}{ \tilde{N}_{i,h}^{t-1}(s,a) \vee 1}} \wedge H$, and $\tilde{\Psi}_{i,h}^t(s,a)= \sqrt{\frac{2 \log (8m|\cS||\cA|HT/\delta)}{ \tilde N_{i,h}^{t-1}(s,a) \vee 1}} \wedge 1$. This completes the proof.
\end{proof}

	 \vspace{-0.5cm}
\subsection{Proof of Theorem \ref{thm:tch}} \label{subsec:proof-tch}

\begin{proof} 
Defining $\pi_{\blambda}^*$ as the minimizer of the minimization problem $\minimize_{\pi\in\Pi}\tchl(\pi)$, we have 
	 \vspace{-0.1cm}
\begin{align}
	&\tchl(\hat{\pi}) - \tchl(\pi_{\blambda}^*) \nonumber\\
	& = \max_{\bw \in\Delta_m}  \sum_{i=1}^m w_i\lambda_i \left(V_{i,1}^*(s_1) + \iota - \frac{1}{T}\sum_{t=1}^T V_{i,1}^{\pi^t}(s_1)\right) - \max_{\bw \in\Delta_m}  \sum_{i=1}^m w_i\lambda_i (V_{i,1}^*(s_1) + \iota - V_{i,1}^{\pi_{\blambda}^*}(s_1))  \nonumber\\
	&= \max_{\bw \in\Delta_m}  \sum_{i=1}^m w_i\lambda_i \left(V_{i,1}^*(s_1) + \iota - \frac{1}{T}\sum_{t=1}^T V_{i,1}^{\pi^t}(s_1)\right) -  \frac{1}{T}\sum_{t=1}^T \sum_{i=1}^m w_i^t\lambda_i (V_{i,1}^{\tilde\nu_i^t}(s_1) + \iota - V_{i,1}^{\pi^t}(s_1)) \nonumber\\
&\quad + \frac{1}{T}\sum_{t=1}^T \sum_{i=1}^m w_i^t\lambda_i (V_{i,1}^{\tilde\nu_i^t}(s_1) + \iota - V_{i,1}^{\pi^t}(s_1))  - \max_{\bw \in\Delta_m}\sum_{i=1}^m w_i\lambda_i (V_{i,1}^*(s_1) + \iota - V_{i,1}^{\pi_{\blambda}^*}(s_1)),  \label{eq:tch-decomp}
\end{align}
where the first equality is by the output of the solution $\hat{\pi}$ such that $V_{i,1}^{\hat{\pi}}=\frac{1}{T}\sum_{t=1}^T V_{i,1}^{\pi^t}$ and the second equality is by adding and subtracting the term $\frac{1}{T}\sum_{t=1}^T \sum_{i=1}^m w_i^t\lambda_i (V_{i,1}^{\tilde\nu_i^t}(s_1) + \iota - V_{i,1}^{\pi^t}(s_1))$. 

We further decompose and bound the term $\max_{\bw \in\Delta_m}  \sum_{i=1}^m w_i\lambda_i \left(V_{i,1}^*(s_1) + \iota - \frac{1}{T}\sum_{t=1}^T V_{i,1}^{\pi^t}(s_1)\right) -  \frac{1}{T}\sum_{t=1}^T \sum_{i=1}^m w_i^t\lambda_i (V_{i,1}^{\tilde\nu_i^t}(s_1) + \iota - V_{i,1}^{\pi^t}(s_1))$ in RHS of \eqref{eq:tch-decomp} as follows
\begin{align*}
	&\max_{\bw \in\Delta_m}  \sum_{i=1}^m w_i\lambda_i \left(V_{i,1}^*(s_1) + \iota - \frac{1}{T}\sum_{t=1}^T V_{i,1}^{\pi^t}(s_1)\right) -  \frac{1}{T}\sum_{t=1}^T \sum_{i=1}^m w_i^t\lambda_i (V_{i,1}^{\tilde\nu_i^t}(s_1) + \iota - V_{i,1}^{\pi^t}(s_1)) \\
	& = \max_{\bw \in\Delta_m}  \sum_{i=1}^m w_i\lambda_i \left(V_{i,1}^*(s_1) + \iota - \frac{1}{T}\sum_{t=1}^T V_{i,1}^{\pi^t}(s_1)\right) - \max_{\bw \in\Delta_m}  \sum_{i=1}^m w_i\lambda_i \frac{1}{T}\sum_{t=1}^T (V_{i,1}^{\tilde\nu_i^t}(s_1) + \iota - V_{i,1}^{\pi^t}(s_1)) \\
	&\quad + \max_{\bw \in\Delta_m} \sum_{i=1}^m w_i\lambda_i \frac{1}{T}\sum_{t=1}^T (V_{i,1}^{\tilde\nu_i^t}(s_1) + \iota - V_{i,1}^{\pi^t}(s_1)) - \frac{1}{T}\sum_{t=1}^T \sum_{i=1}^m w_i^t\lambda_i (V_{i,1}^{\tilde\nu_i^t}(s_1) + \iota - V_{i,1}^{\pi^t}(s_1))\\
	&\leq  \underbrace{\max_{\bw\in\Delta_m}   \frac{1}{T}\sumt \sumi w_i\lambda_i (V_{i,1}^*(s_1)  -  V_{i,1}^{\tilde\nu_i^t}(s_1))}_{\text{Term(I)}}  \\
	&\quad + \underbrace{\max_{\bw \in\Delta_m} \frac{1}{T}\sum_{t=1}^T \sum_{i=1}^m (w_i-w_i^t)\lambda_i  (V_{i,1}^{\tilde\nu_i^t}(s_1) + \iota - V_{i,1}^{\pi^t}(s_1))}_{\text{Term(II)}},
\end{align*}
where the first equality is by adding and subtracting the term $\max_{\bw \in\Delta_m} \sum_{i=1}^m w_i\lambda_i  \frac{1}{T}\cdot \allowbreak \sum_{t=1}^T (V_{i,1}^{\tilde\nu_i^t}(s_1) + \iota - V_{i,1}^{\pi^t}(s_1))$, the first inequality is by the inequality $\max_x f(x) - \max_x g(x) \leq \max_x |f(x)-g(x)|$ and $V_{i,1}^*(s_1)  \geq  V_{i,1}^{\tilde\nu_i^t}(s_1)$ since $V_{i,1}^*(s_1):=\max_{\nu_i\in\Pi} V_{i,1}^{\nu_i}(s_1)$. Now we can see that Term(I) is the learning error for the optimal value w.r.t. each individual objective. And Term(II) is associated with learning toward the optimal $\bw$.  

We next decompose and bound the term $\frac{1}{T}\sum_{t=1}^T \sum_{i=1}^m w_i^t\lambda_i (V_{i,1}^{\tilde\nu_i^t}(s_1) + \iota - V_{i,1}^{\pi^t}(s_1))  - \max_{\bw \in\Delta_m}\allowbreak\sum_{i=1}^m w_i\lambda_i (V_{i,1}^*(s_1) + \iota - V_{i,1}^{\pi_{\blambda}^*}(s_1))$ in RHS of \eqref{eq:tch-decomp} as follows
\begin{align*}
	&\frac{1}{T}\sum_{t=1}^T \sum_{i=1}^m w_i^t\lambda_i (V_{i,1}^{\tilde\nu_i^t}(s_1) + \iota - V_{i,1}^{\pi^t}(s_1))  - \max_{\bw \in\Delta_m}\sum_{i=1}^m w_i\lambda_i (V_{i,1}^*(s_1) + \iota - V_{i,1}^{\pi_{\blambda}^*}(s_1))\\
	&\qquad \leq \frac{1}{T}\sum_{t=1}^T\sum_{i=1}^m w_i^t\lambda_i (V_{i,1}^{\tilde\nu_i^t}(s_1) + \iota - V_{i,1}^{\pi^t}(s_1))  - \frac{1}{T}\sum_{t=1}^T\max_{\bw \in\Delta_m}\sum_{i=1}^m w_i\lambda_i (V_{i,1}^{\tilde\nu_i^t}(s_1) + \iota - V_{i,1}^{\pi_{\blambda}^*}(s_1))\\
	&\qquad \leq \frac{1}{T}\sum_{t=1}^T\sum_{i=1}^m w_i^t\lambda_i (V_{i,1}^{\tilde\nu_i^t} (s_1) + \iota - V_{i,1}^{\pi^t}(s_1))  - \frac{1}{T}\sum_{t=1}^T\sum_{i=1}^m w_i^t\lambda_i (V_{i,1}^{\tilde\nu_i^t}(s_1) + \iota - V_{i,1}^{\pi_{\blambda}^*}(s_1))\\
	&\qquad=\underbrace{\frac{1}{T}\sum_{t=1}^T\sum_{i=1}^m w_i^t\lambda_i (V_{i,1}^{\pi_{\blambda}^*}(s_1) - V_{i,1}^{\pi^t}(s_1))}_{\text{Term(III)}},
\end{align*}
where the first inequality is due to $V_{i,1}^*(s_1)\geq V_{i,1}^{\tilde\nu_i^t}(s_1)$ and
\begin{align*}
\max_{\bw \in\Delta_m}\sum_{i=1}^m w_i\lambda_i (V_{i,1}^*(s_1) + \iota - V_{i,1}^{\pi_{\blambda}^*}(s_1)) \geq  \frac{1}{T}\sumt\max_{\bw \in\Delta_m}\sum_{i=1}^m w_i\lambda_i (V_{i,1}^{\tilde\nu_i^t}(s_1) + \iota - V_{i,1}^{\pi_{\blambda}^*}(s_1)).
\end{align*} 
Term(III) depicts the learning error toward the (weakly) Pareto optimal solution $\pi_{\blambda}^*$ under the given preference $\blambda$. 
Therefore, combining the above results with \eqref{eq:tch-decomp}, we obtain
\begin{align}
	\tchl(\hat{\pi}) - \tchl(\pi_{\blambda}^*) \leq \text{Term(I)+Term(II)+Term(III)}. \label{eq:tch-decomp-sum}
\end{align}
Then, we turn to bounding the above three terms respectively based on the updating rule in Algorithm \ref{alg:online-tch}. To prove the upper bounds, we first assume that the transition and reward functions can be estimated by certain procedures such that the true transition and reward functions satisfy the following conditions of bounded estimation errors. At the end of this proof, we show that the estimation method in this work satisfies these conditions with high probability. This also shows the generality of our theoretical proof for any reward and transition estimation satisfying the following conditions. Under the same conditions as in Lemma \ref{lem:v-opt-2}, suppose that for any $V:\cS\mapsto [0,H]$,
\begin{align}
	\begin{aligned} \label{eq:tch-cond-1}
	&|(\hat{\PP}_h^t- \PP_h)V (s,a)| \leq \Phi_h^t(s,a), \\
	&|\hat{r}_{i,h}^t(s,a) - r_{i,h}(s,a)| \leq \Psi_{i,h}^t(s,a), ~\forall i\in[m],
	\end{aligned}
\end{align}
and  $ Q_{i,h}^t(\cdot,\cdot) = \{ (\hat{r}_{i,h}^t + \hat{\PP}_h^t V_{i,h+1}^t + \Phi_h^t+\Psi_{i,h}^t)(\cdot,\cdot)\}_{[0,H-h+1]}$ is the associated optimistic Q-function is.
In addition, we suppose that 
\begin{align} 
	\begin{aligned}	\label{eq:tch-cond-2}
		&|(\tilde{\PP}_h^t- \PP_h)V(s,a)| \leq \tilde{\Phi}_{i,h}^t(s,a), \\
		&|\tilde{r}_{i,h}^t(s,a) - r_{i,h}(s,a)| \leq \tilde{\Psi}_{i,h}^t(s,a), ~\forall i\in[m], 		
	\end{aligned}
\end{align}
and  
$\tilde{Q}_{i,h}^t(\cdot,\cdot) = \{ (\tilde{r}_{i,h}^t + \tilde{\PP}_{i,h}^t \tilde{V}_{i,h+1}^t + \tilde{\Phi}_{i,h}^t+\tilde{\Psi}_{i,h}^t)(\cdot,\cdot)\}_{[0,H-h+1]}$ is the corresponding optimistic Q-function.

Next, we give the upper bounds of Term(I), Term(II), and Term(III). For Term(I), we have
\vspace{-0.1cm}
\begin{align*}
	\text{Term(I)}&=\max_{\bw\in\Delta_m} \frac{1}{T}\sumt \sumi w_i\lambda_i (V_{i,1}^*(s_1)  -  V_{i,1}^{\tilde\nu_i^t}(s_1)) \\
	&\leq  \frac{1}{T}\sumt \max_{\bw\in\Delta_m} \sumi w_i\lambda_i (V_{i,1}^* (s_1) -  V_{i,1}^{\tilde\nu_i^t}(s_1))
	\\
	&=\frac{1}{T}\sumt \max_{i\in[m]} \lambda_i (V_{i,1}^* (s_1)-  V_{i,1}^{\tilde\nu_i^t}(s_1))\\
	&\leq\frac{1}{T}\sumt \sumi \lambda_i (V_{i,1}^*(s_1) -  V_{i,1}^{\tilde\nu_i^t}(s_1)),
\end{align*}
where the first inequality is due to the fact that $\max$ is a convex function as well as $V_{i,1}^*(s_1) \geq  V_{i,1}^{\tilde\nu_i^t}(s_1)$  and the first equality is due to $\max_{\bw\in\Delta_m} \bw^\top \bx = \max x_i$. Furthermore, we have
\vspace{-0.1cm}
\begin{align*}
&V_{i,1}^*(s_1) -  V_{i,1}^{\tilde\nu_i^t}(s_1) = V_{i,1}^* (s_1) - \tilde{V}_{i,1}^t(s_1)   + \tilde{V}_{i,1}^t(s_1)  -  V_{i,1}^{\tilde\nu_i^t}(s_1).
\end{align*}
By Lemma \ref{lem:v-decomp-2}, we obtain 
\vspace{-0.1cm}
\begin{align} 
	\begin{aligned}
	\label{eq:opt-opt-1}
	&V_{i,1}^*(s_1) - \tilde{V}_{i,1}^t(s_1)  \\
	&\qquad \leq \sum_{h=1}^H \EE_{\nu_i^*, \PP} [ \tilde{\varsigma}_{i,h}^t(s_h, a_h) \given s_1 ] + \sum_{h=1}^H \EE_{\nu_i^*, \PP} [ \langle \nu_{i,h}^*(\cdot | s_h)-\tilde \nu_{i,h}^t(\cdot | s_h), \tilde Q_{i,h}^t(s_h,\cdot) \rangle_{\cA}  \given s_1 ]\leq 0,
\end{aligned}
\end{align}
where $\nu_i^* := \argmax_\nu V_{i,1}^\nu(s_1)$ is the individual optimal policy, $\tilde\varsigma_{i,h}^t(s,a) = r_{i,h}(s,a) +  \PP_h \tilde V_{i,h+1}^t(s, a) - \tilde Q_{i,h}^t(s,a)$, and second inequality is by optimism as in Lemma \ref{lem:v-opt-2} and the greedy updating rule for $\tilde \nu_{i,h}^t$ in Algorithm \ref{alg:online-tch} such that $\sum_{h=1}^H \EE_{\nu_i^*, \PP} [ \tilde{\varsigma}_{i,h}^t(s_h, a_h) \given s_1 ] \leq 0$ and $\sum_{h=1}^H \EE_{\nu_i^*, \PP} [ \langle \nu_{i,h}^*(\cdot | s_h)-\tilde \nu_{i,h}^t(\cdot | s_h), \tilde Q_{i,h}^t(s_h,\cdot) \rangle_{\cA}  \given s_1 ] \leq 0$.
Thus, we have with probability at least $1-\delta'$,
\begin{align}
	\begin{aligned}\label{eq:I-last}
		\text{Term(I)}&\leq \frac{1}{T}\sumt \sumi \lambda_i (\tilde{V}_{i,1}^t(s_1)  -  V_{i,1}^{\tilde\nu_i^t}(s_1))=  \sumi \lambda_i \frac{1}{T}\sumt (\tilde{V}_{i,1}^t(s_1)  -  V_{i,1}^{\tilde\nu_i^t}(s_1))\\
		&\leq  \cO\bigg(\sqrt{\frac{H^3  \log (m/\delta')}{T}} \bigg) + \frac{2}{T} \sumi \sum_{t=1}^T\sum_{h=1}^H \lambda_i[\tilde \Psi_{i,h}^t(\tilde s_{i,h}^t, \tilde a_{i,h}^t)  +   \tilde \Phi_{i,h}^t(\tilde s_{i,h}^t, \tilde a_{i,h}^t)],
	\end{aligned} 
\end{align}
where the inequality is due to Lemma \ref{lem:online-err-2}.

Next, we bound Term(II). Specifically, we have
\begin{align*}
		\text{Term(II)}&= \max_{\bw \in\Delta_m} \frac{1}{T}\sum_{t=1}^T \sum_{i=1}^m (w_i-w_i^t)\lambda_i  (V_{i,1}^{\tilde\nu_i^t} - \tilde V_{i,1}^t + \tilde V_{i,1}^t+ \iota - V_{i,1}^{\pi^t} + V_{i,1}^t - V_{i,1}^t) \\
		&\leq  \max_{\bw\in\Delta_m}\frac{1}{T}\sumt (\bw-\bw^t)^\top [ \blambda \odot (\tilde{\bV}_1^t(s_1)+\biota -\bV_1^t(s_1))]\\
		&\quad  + \max_{\bw \in\Delta_m} \frac{1}{T}\sum_{t=1}^T \sum_{i=1}^m w_i\lambda_i  (V_{i,1}^{\tilde\nu_i^t}(s_1) - \tilde V_{i,1}^t(s_1)  - V_{i,1}^{\pi^t}(s_1) + V_{i,1}^t(s_1))  \\
		&\quad - \frac{1}{T}\sum_{t=1}^T \sum_{i=1}^m w_i^t\lambda_i (V_{i,1}^{\tilde\nu_i^t}(s_1) - \tilde V_{i,1}^t (s_1) - V_{i,1}^{\pi^t}(s_1) + V_{i,1}^t(s_1)),
\end{align*}
where the inequality is due to the fact that $\max$ is a convex function and $\bx \odot \by:=(x_1y_1,x_2y_2,\cdots,x_my_m)$. By Lemma \ref{lem:oco} which analyzes the updating rule of $\bw$ in Algorithm \ref{alg:online-tch}, we have
\begin{align*}
	\max_{\bw\in\Delta_m}\frac{1}{T}\sumt (\bw-\bw^t)^\top [ \blambda \odot (\tilde{\bV}_1^t(s_1)+\biota -\bV_1^t(s_1))]\leq \frac{6H\log^{\frac{1}{2}} m}{\sqrt{T}}.
\end{align*}
Moreover, by Lemma \ref{lem:online-err-2}, we further obtain that with probability at least $1-\delta'$,
\begin{align*}
&\max_{\bw \in\Delta_m} \frac{1}{T}\sum_{t=1}^T \sum_{i=1}^m w_i\lambda_i  (V_{i,1}^{\tilde\nu_i^t} - \tilde V_{i,1}^t  - V_{i,1}^{\pi^t} + V_{i,1}^t) - \frac{1}{T}\sum_{t=1}^T \sum_{i=1}^m w_i^t\lambda_i (V_{i,1}^{\tilde\nu_i^t} - \tilde V_{i,1}^t  - V_{i,1}^{\pi^t} + V_{i,1}^t)\\
&\qquad \leq \max_{\bw \in\Delta_m} \frac{1}{T}\sum_{t=1}^T \sum_{i=1}^m w_i\lambda_i  ( V_{i,1}^t(s_1) - V_{i,1}^{\pi^t}(s_1)) + \frac{1}{T}\sum_{t=1}^T \sum_{i=1}^m w_i^t\lambda_i (\tilde V_{i,1}^t(s_1)-V_{i,1}^{\tilde\nu_i^t}(s_1))\\
&\qquad \leq \frac{1}{T}\sum_{t=1}^T \sum_{i=1}^m \lambda_i  ( V_{i,1}^t(s_1) - V_{i,1}^{\pi^t}(s_1)) + \frac{1}{T}\sum_{t=1}^T \sum_{i=1}^m \lambda_i (\tilde V_{i,1}^t(s_1)-V_{i,1}^{\tilde\nu_i^t}(s_1))\\
&\qquad \leq \frac{2}{T} \sumi \sum_{t=1}^T\sum_{h=1}^H \lambda_i [\tilde \Psi_{i,h}^t(\tilde s_{i,h}^t, \tilde a_{i,h}^t)  +   \tilde \Phi_{i,h}^t(\tilde s_{i,h}^t, \tilde a_{i,h}^t) + \Psi_{i,h}^t(s_h^t, a_h^t)  +  \Phi_h^t(s_h^t, a_h^t)] \\
&\qquad \quad + \cO\bigg(\sqrt{\frac{H^3  \log (m/\delta')}{T}} \bigg),
\end{align*}
where the first and the second inequalities are by \eqref{eq:opt-opt-1} such that $V_{i,1}^{\tilde\nu_i^t}(s_1) - \tilde V_{i,1}^t (s_1) \leq V_{i,1}^*(s_1) - \tilde V_{i,1}^t(s_1) \leq 0$ and also due to $0\leq w_i, w_i^t\leq 1$ and Lemmas \ref{lem:v-decomp-2} and \ref{lem:v-opt-2} with $\varsigma_{i,h}^t(s,a) = r_{i,h}(s,a) +  \PP_h V_{i,h+1}^t(s, a) - Q_{i,h}^t(s,a)$ such that 
\begin{align}
\begin{aligned} \label{eq:inv-opt-pi-2}
&V_{i,1}^{\pi^t}(s_1) - V_{i,1}^t(s_1) \\
&\qquad =  \sum_{h=1}^H \EE_{\pi^t, \PP} [ \varsigma_{i,h}^t(s_h, a_h) \given s_1 ] + \sum_{h=1}^H \EE_{\pi^t, \PP} [ \langle \pi_h^t(\cdot | s_h)-\pi_h^t(\cdot | s_h), Q_{i,h}^t(s_h,\cdot) \rangle_{\cA}  \given s_1 ] \leq 0,	
\end{aligned}
\end{align}
and the third inequality is by Lemma \ref{lem:online-err-2}. Therefore, combining the above results for Term(II), we obtain with probability at least $1-\delta'$,
\begin{align}
\text{Term(II)}& \leq \frac{2}{T} \sumi \sum_{t=1}^T\sum_{h=1}^H \lambda_i [\tilde \Psi_{i,h}^t(\tilde s_{i,h}^t, \tilde a_{i,h}^t)  +   \tilde \Phi_{i,h}^t(\tilde s_{i,h}^t, \tilde a_{i,h}^t) + \Psi_{i,h}^t(s_h^t, a_h^t)  +  \Phi_h^t(s_h^t, a_h^t)] \nonumber \\
& \quad + \cO\bigg(\sqrt{\frac{H^3  \log (m/\delta')}{T}} \bigg), \label{eq:II-last}
\end{align}
where the trajectory $\{ (s_h^t, a_h^t)\}_{h=1}^H, \forall t\in [T],$ is generated via the policy $\pi^t$, and $\{ (\tilde s_{i,h}^t, \tilde a_{i,h}^t)\}_{h=1}^H, \forall t\in [T],$ is generated via the policy $\tilde\nu_i^t$. 

Finally, we turn to Term(III). We can decompose this term as follows,
\begin{align*}
	\text{Term(III)}&=\frac{1}{T}\sum_{t=1}^T\sum_{i=1}^m w_i^t\lambda_i (V_{i,1}^{\pi_{\blambda}^*}(s_1) - V_{i,1}^{\pi^t}(s_1))\\
	&= \frac{1}{T}\sum_{t=1}^T\sum_{i=1}^m w_i^t\lambda_i (V_{i,1}^{\pi_{\blambda}^*}(s_1) - V_{i,1}^t(s_1)) + \frac{1}{T}\sum_{t=1}^T\sum_{i=1}^m w_i^t\lambda_i (V_{i,1}^t(s_1) - V_{i,1}^{\pi^t}(s_1)).
\end{align*}
For the term $\frac{1}{T}\sum_{t=1}^T\sum_{i=1}^m w_i^t\lambda_i (V_{i,1}^t(s_1) - V_{i,1}^{\pi^t}(s_1))$ above, we have with probability at least $1-\delta'$,
\begin{align*}
&\frac{1}{T}\sum_{t=1}^T\sum_{i=1}^m w_i^t\lambda_i (V_{i,1}^t(s_1) - V_{i,1}^{\pi^t}(s_1)) \\
&\qquad \leq \frac{1}{T}\sum_{t=1}^T\sum_{i=1}^m \lambda_i (V_{i,1}^t(s_1) - V_{i,1}^{\pi^t}(s_1))\\
&\qquad \leq \frac{2}{T} \sumi \sum_{t=1}^T\sum_{h=1}^H \lambda_i [\Psi_{i,h}^t(s_h^t, a_h^t)  +  \Phi_h^t(s_h^t, a_h^t)] + \cO\bigg(\sqrt{\frac{H^3  \log (m/\delta')}{T}} \bigg),
\end{align*}
where the first inequality is due to \eqref{eq:inv-opt-pi-2} and $0\leq w_i^t\leq 1$, and the second inequality is by Lemma \ref{lem:online-err-2}. For the term $\frac{1}{T}\sum_{t=1}^T\sum_{i=1}^m w_i^t\lambda_i (V_{i,1}^{\pi_{\blambda}^*} - V_{i,1}^t)$, by Lemma \ref{lem:v-decomp-2} and optimism in Lemma \ref{lem:v-opt-2}, we have
\begin{align*}
&\frac{1}{T}\sum_{t=1}^T\sum_{i=1}^m w_i^t\lambda_i (V_{i,1}^{\pi_{\blambda}^*}(s_1) - V_{i,1}^t(s_1)) \\
&\qquad\leq 	\frac{1}{T}\sum_{t=1}^T\sum_{i=1}^m w_i^t\lambda_i \sum_{h=1}^H \EE_{\pi_{\blambda}^*, \PP} [ \varsigma_{i,h}^t(s_h, a_h) \given s_1 ] \\
&\qquad\quad + \frac{1}{T}\sum_{t=1}^T\sum_{i=1}^m w_i^t\lambda_i\sum_{h=1}^H \EE_{\pi_{\blambda}^*, \PP} [ \langle \pi_{\blambda}^*(\cdot | s_h)-\pi_h^t(\cdot | s_h), Q_{i,h}^t(s_h,\cdot) \rangle_{\cA}  \given s_1 ]\\
&\qquad \leq \frac{1}{T}\sum_{t=1}^T\sum_{h=1}^H \EE_{\pi_{\blambda}^*, \PP} [ \langle \pi_{\blambda}^*(\cdot | s_h)-\pi_h^t(\cdot | s_h), (\bw^t \odot \blambda)^\top\bQ_h^t(s_h,\cdot) \rangle_{\cA}  \given s_1 ]\leq 0,
\end{align*}
where the last inequality is due to the updating rule of $\pi^t$ in Algorithm \ref{alg:online-tch},i.e.,  $\pi_h^t=\argmax_{\pi_h}\langle (\bw^t\odot\blambda)^\top \bQ_h^t(\cdot, \cdot), \pi_h(\cdot|\cdot) \rangle_\cA$. Therefore, combining the above results corresponding to Term(III), we obtain with probability at least $1-\delta'$, 
\begin{align}
	\begin{aligned}\label{eq:III-last}
	\text{Term(III)} \leq \frac{2}{T} \sumi \sum_{t=1}^T\sum_{h=1}^H \lambda_i [\Psi_{i,h}^t(s_h^t, a_h^t)  +  \Phi_h^t(s_h^t, a_h^t)] + \cO\bigg(\sqrt{\frac{H^3  \log (m/\delta')}{T}} \bigg).
	\end{aligned}
\end{align}
Further combining \eqref{eq:I-last},\eqref{eq:II-last}, and \eqref{eq:III-last}, we have with probability at least $1-\delta'$,
\begin{align}
&\tchl(\hat{\pi}) - \tchl(\pi_{\blambda}^*) \nonumber\\
&\qquad \leq \frac{4}{T} \sumi \sum_{t=1}^T\sum_{h=1}^H \lambda_i [\tilde \Psi_{i,h}^t(\tilde s_{i,h}^t, \tilde a_{i,h}^t)  +   \tilde \Phi_{i,h}^t(\tilde s_{i,h}^t, \tilde a_{i,h}^t) + \Psi_{i,h}^t(s_h^t, a_h^t)  +  \Phi_h^t(s_h^t, a_h^t)] \nonumber\\
&\qquad\quad + \cO\bigg(\sqrt{\frac{H^3  \log (m/\delta')}{T}} \bigg). \label{eq:all-last-1}
\end{align}
Now according to Lemma \ref{lem:concentrate}, we immediately know that the conditions \eqref{eq:tch-cond-1} and \eqref{eq:tch-cond-2} hold with probability at least $1-\delta'$, if $\hat{\PP}_h^t$, $\hat{r}_{i,h}^t$, $\Phi_h^t$, and $\Psi_{i,h}^t$ are instantiated via \texttt{OptQ} in Line \ref{line:tch-mainQ} of Algorithm \ref{alg:online-tch}, and $\tilde{\PP}_h^t$, $\tilde{r}_{i,h}^t$, $\tilde\Phi_h^t$, and $\tilde\Psi_{i,h}^t$ are the instantiated via \texttt{OptQ} for Line \ref{line:tch-auxQ} in Algorithm \ref{alg:online-tch}. Finally, by Lemma \ref{lem:sum-bonus-2}, we obtain that 
\begin{align*}
	&\sum_{t=1}^T\sum_{h=1}^H \Psi_{i,h}^t(s_h^t, a_h^t) \leq \cO \left(H \sqrt{ T |\cS| |\cA| \log \frac{m|\cS| |\cA|  HT}{\delta'}} \right), \\
	&\sum_{t=1}^T \sum_{h=1}^H \Phi_h^t(s_h^t, a_h^t) \leq \cO \left(H^2 \sqrt{ T |\cS|^2 |\cA| \log \frac{m|\cS| |\cA|  HT}{\delta'}} \right),
\end{align*} 
and 
\begin{align*}
	&\sum_{t=1}^T\sum_{h=1}^H \tilde{\Psi}_{i,h}^t(\tilde{s}_{i,h}^t, \tilde{a}_{i,h}^t) \leq \cO \left(H \sqrt{ T |\cS| |\cA| \log \frac{m|\cS| |\cA|  HT}{\delta'}} \right), \\
	&\sum_{t=1}^T \sum_{h=1}^H \tilde{\Phi}_{i,h}^t(\tilde{s}_{i,h}^t, \tilde{a}_{i,h}^t) \leq \cO \left(H^2 \sqrt{ T |\cS|^2 |\cA| \log \frac{m|\cS| |\cA|  HT}{\delta'}} \right),
\end{align*}
Combining the above inequalities with \eqref{eq:all-last-1}, by union bound, we eventually obtain that with probability at least $1-\delta$, 
\begin{align*}
	&\tchl(\hat{\pi}) - \tchl(\pi_{\blambda}^*) \leq \cO\bigg(\sqrt{\frac{H^4 |\cS|^2|\cA| \log (m|\cS||\cA|HT/\delta)}{T}} \bigg).
\end{align*}
This completes the proof.
\end{proof}

\section{Proofs for Section \ref{sec:pre-free}}

This section provides a detailed proof of Theorem \ref{thm:pre-free}. Before presenting the main proof of the theorem, we first provide several important lemmas.

\subsection{Lemmas for Theorem \ref{thm:pre-free}}\label{sec:lemma-pre-free}

\begin{lemma} \label{lem:v-decomp-3} Algorithm \ref{alg:exploit-tch} ensures that for any policy $\pi$ and policy $\nu_i$
	\begin{align*}
		&V_{i,1}^{\pi}(s_1) - V_{i,1}^k(s_1) = \sum_{h=1}^H \EE_{\pi, \PP} [ \varsigma_{i,h}^k(s_h, a_h) \given s_1 ] + \sum_{h=1}^H \EE_{\pi, \PP} [ \langle \pi_h(\cdot | s_h)-\pi_h^k(\cdot | s_h), Q_{i,h}^k(s_h,\cdot) \rangle_{\cA}  \given s_1 ],\\
		&V_{i,1}^{\nu_i}(s_1) - \tilde{V}_{i,1}(s_1) = \sum_{h=1}^H \EE_{\nu_i, \PP} [ \tilde{\varsigma}_{i,h}(s_h, a_h) \given s_1 ] + \sum_{h=1}^H \EE_{\nu_i, \PP} [ \langle \nu_{i,h}(\cdot | s_h)-\tilde \nu_{i,h}(\cdot | s_h), \tilde Q_{i,h}(s_h,\cdot) \rangle_{\cA}  \given s_1 ].
	\end{align*}
	where $s_h, a_h$ are random variables for the state and action, and we define the model prediction error of the $Q$-function as $\varsigma_{i,h}^k(s,a) = r_{i,h}(s,a) +  \PP_h V_{i,h+1}^k(s, a) - Q_{i,h}^k(s,a)$ and $\tilde\varsigma_{i,h}(s,a) = r_{i,h}(s,a) +  \PP_h \tilde V_{i,h+1}(s, a) - \tilde Q_{i,h}(s,a)$.

\end{lemma}
\begin{proof} This lemma can be proved by exactly following the proof of Lemma \ref{lem:v-decomp-2} and substituting the value functions and policies with the ones defined in this lemma.		
\end{proof}

\begin{lemma}[Optimism]\label{lem:v-opt-3} 	
	Let $\hat{\PP}_h$ and $\hat{r}_{i,h}, \forall i\in[m],$ be the estimated transition and reward functions via some estimation procedure in Algorithm \ref{alg:exploit-tch}. Suppose that there exist $\Phi_h$ and $\Psi_{i,h}$ such that $\hat{\PP}_h$ and $\hat{r}_{i,h}$ satisfy for any $V:\mapsto [0,H]$,
	\begin{align*}
		|(\hat{\PP}_h- \PP_h)V(s,a)| \leq \Phi_h(s,a),\quad \text{and} \quad |\hat{r}_{i,h}(s,a) - r_{i,h}(s,a)| \leq \Psi_{i,h}(s,a), \forall i\in[m], 
	\end{align*}
	and the corresponding optimistic Q-functions in Algorithm \ref{alg:exploit-tch} are defined as $Q_{i,h}^k(\cdot,\cdot) = \{ (\hat{r}_{i,h} + \hat{\PP}_h V_{i,h+1}^k + \Phi_h+\Psi_{i,h})(\cdot,\cdot)\}_{[0,H-h+1]}$ and   $\tilde{Q}_{i,h}(\cdot,\cdot) = \{ (\hat{r}_{i,h} + \hat{\PP}_{i,h} \tilde{V}_{i,h+1} + \Phi_h+\Psi_{i,h})(\cdot,\cdot)\}_{[0,H-h+1]}$. Then Algorithm \ref{alg:exploit-tch} ensures that for any policies $\pi$ and $\nu_i$
	\begin{align*}
		&\sum_{h=1}^H \EE_{\pi, \PP} \big[ \varsigma_{i,h}^k(s_h, a_h) \biggiven s_1 \big] \leq 0 \quad \text{and} \quad  \sum_{h=1}^H \EE_{\nu_i, \PP} \big[ \tilde{\varsigma}_{i,h}(s_h, a_h)   \biggiven s_1 \big] \leq 0, 
	\end{align*}
	where $\varsigma_{i,h}^k$ and $\tilde{\varsigma}_{i,h}$ is the model prediction errors defined as in Lemma \ref{lem:v-decomp-3}.
\end{lemma}

\begin{lemma}\label{lem:pred-err-bound-3} 	
Under the conditions of Lemma \ref{lem:v-opt-3}, we have 
	\begin{align*}
		&-\varsigma_{i,h}^k(s, a)  \leq 2\Phi_h(s,a)+2\Psi_{i,h}(s,a) \quad \text{and} \quad  -\tilde{\varsigma}_{i,h}(s, a)    \leq 2\Phi_h(s,a)+2\Psi_{i,h}(s,a), 
	\end{align*}
	where $\varsigma_{i,h}^k$ and $\tilde{\varsigma}_{i,h}$ is the model prediction errors defined as in Lemma \ref{lem:v-decomp-3}.
\end{lemma}
\begin{proof}By plugging in the definition of $Q_{i,h}^k$ in Algorithm \ref{alg:exploit-tch} as above, we decompose the prediction error as follows
	\begin{align*}
			-\varsigma_{i,h}^k(s, a)&= Q_{i,h}^k(s,a) -r_{i,h}(s,a) - \big\langle \PP_h(\cdot \given s, a), V_{i,h+1}^k(\cdot) \big\rangle_\cS   \\
			&\leq \Big\{ (\hat{r}_{i,h} + \hat{\PP}_h V_{i,h+1}^k+ \Phi_h+\Psi_{i,h})(s,a), H-h+1 \Big\}_{[0,H-h+1]}  \\
			&\quad- r_{i,h}(s,a) - \PP_h V_{i,h+1}^t(s,a)  \\
			&\leq  \max \Big\{(\hat{r}_{i,h}-r_{i,h}  + (\hat{\PP}_h - \PP_h) V_{i,h+1}^k +  \Phi_h+\Psi_{i,h})(s,a), 0 \Big\}\\
			&\leq  |\hat{r}_{i,h}(s,a)-r_{i,h}(s,a)|  + |(\hat{\PP}_h - \PP_h) V_{i,h+1}^k(s,a)| +  \Phi_h(s,a)+\Psi_{i,h}(s,a), 
	\end{align*}
	where the first inequality holds because 
	\begin{align*}
		&r_{i,h}(s,a) + \big\langle \PP_h(\cdot \given s_h, a_h), V_{i,h+1}^k(\cdot) \big\rangle_\cS \nonumber \\
		&\qquad \leq r_{i,h}(s,a) + \big\|\PP_h(\cdot \given s_h, a_h)\big\|_1 \|V_{i,h+1}^k(\cdot) \|_\infty  \leq 1 +  \max_{s' \in \cS} \big|V_{i,h+1}^k(s') \big| \leq 1 + H-h, 
	\end{align*}
	due to $\big\|\PP_h(\cdot \given s_h, a_h)\big\|_1 = 1$ and also the definition of $Q_{i,h+1}^k$ such that for any $s' \in \cS$
	\begin{align*}
			\big|V_{i,h+1}^k(s') \big| &=  \Big| \langle\pi_{h+1}^k(\cdot | s'),  Q_{i,h+1}^t(s', \cdot) \rangle \Big|\\
			&\leq \big\| \pi_{h+1}^k(\cdot | s')\big\|_1 \big\|Q_{i,h+1}^k(s', \cdot) \big\|_\infty \\
			&\leq \max_{a} \big|Q_{i,h+1}^k(s', a)\big|\leq H-h.
	\end{align*}
	By the condition that 
	\begin{align*}
		|(\hat{\PP}_h- \PP_h)V(s,a)| \leq \Phi_h(s,a),\quad \text{and} \quad |\hat{r}_{i,h}(s,a) - r_{i,h}(s,a)| \leq \Psi_{i,h}(s,a), \forall i\in[m], 
	\end{align*}
	we have
	\begin{align*}
		&-\varsigma_{i,h}^k(s, a)\\
		&\qquad \leq |\hat{r}_{i,h}(s,a)-r_{i,h}(s,a)|  + |(\hat{\PP}_h - \PP_h) V_{i,h+1}^k(s,a)| +  \Phi_h(s,a)+\Psi_{i,h}(s,a)\\
		&\qquad\leq  2\Phi_h(s,a)+2\Psi_{i,h}(s,a).
	\end{align*}
	For the term  $-\tilde{\varsigma}_{i,h}(s, a)$, we can derive the upper bound following similar analysis to $\varsigma_{i,h}^k(s, a)$, and thus we obtain 
	\begin{align*}
		-\tilde{\varsigma}_{i,h}(s, a)    \leq 2\Phi_h(s,a)+2\Psi_{i,h}(s,a).
	\end{align*}
	 This completes the proof.
\end{proof}

\begin{proof} This lemma can be proved by exactly following the proof of Lemma \ref{lem:v-opt-2} and substituting the definitions of value functions and bonus terms with the ones defined in this lemma.
\end{proof}

\begin{lemma} \label{lem:bonus-cmp-3} According to the definitions of $\Phi_h(s,a) $ and $\Psi_{i,h}(s,a)$ in Algorithm \ref{alg:exploit-tch}, and $\Phi_h^t(s,a)$ and $\Psi_{i,h}^t(s,a)$ in Algorithm \ref{alg:pure-explore}, we have
\begin{align*}
&\Phi_h(s,a) \leq \frac{1}{T}\sum_{t=1}^T \Phi_h^t(s,a), \quad \text{and}\quad \Psi_{i,h}(s,a) \leq \frac{1}{T}\sum_{t=1}^T \Psi_{i,h}^t(s,a).
\end{align*}
\end{lemma}
\begin{proof} We have
	\begin{align*}
		\Phi_h(s,a) = \sqrt{\frac{2 \log (8m|\cS||\cA|HT/\delta)}{ N_h(s,a)\vee 1}} \wedge H \leq \Phi_h^t(s,a) = \sqrt{\frac{2 \log (8m|\cS||\cA|HT/\delta)}{ N_h^{t-1}(s,a)\vee 1}}\wedge H,
	\end{align*}
	since we always have $N_h^{t-1}(s,a)\leq N_h(s,a)$. Therefore, we obtain
	\begin{align*}
		\Phi_h(s,a) \leq \frac{1}{T}\sumt \Phi_h^t(s,a).
	\end{align*}
	Similarly, we obtain $\Psi_{i,h}(s,a) \leq \frac{1}{T}\sumt \Psi_{i,h}^t(s,a)$. This completes the proof.
\end{proof}

\begin{lemma} \label{lem:online-err-3} Under the condition 
	\begin{align*}
			|(\hat \PP_h^t - \PP_h)V(s,a)| \leq \Phi_h^t(s,a)
	\end{align*}
	for any $V:\cS\mapsto [0,H]$, we have 
	\begin{align*}
		&\sumt \max_{\pi\in\Pi} \EE_{\pi, \PP} \left[\sumh \Phi_h^t(s_h,a_h)\right] \leq H \sumt \overline{V}_1^t(s_1), \\
		&\sumt \max_{\nu_i\in\Pi} \EE_{\nu_i, \PP} \left[\sumh \Psi_{i,h}^t(s_h,a_h)\right]\leq \sumt \overline{V}_1^t(s_1),\\
		&\sumt \overline{V}_1^t(s_1)\leq\cO\bigg(\sqrt{H^3 T \log \frac{1}{\delta}} \bigg)  + \sumt\sumh [\overline r_h^t(s_h^t, a_h^t) +\Phi_h^t(s_h^t, a_h^t)],
	\end{align*}	
	where the last inequality holds with probability at least $1-\delta$.
\end{lemma}

\begin{proof} First we show $\max_{\pi\in\Pi} \EE_{\pi, \PP} \left[\sumh \Phi_h^t(s,a)\right] \leq H \overline{V}_{i,1}^t$. Note that $0\leq \Phi_h^t(s,a)/H \leq 1$. Then we define $V_1^\pi(s_1;\Phi^t/H)$ as a value function based on the reward function and $V_1^\pi(s_1;\Phi^t/H) = \EE_{\pi, \PP} \left[\sumh \Phi_h^t(s_h,a_h)\right]$. Moreover, according to the definition of $\overline{r}_h^t$ in the algorithm, we have
\begin{align*}
\Phi_h^t(\cdot,\cdot)/H \leq \overline{r}_h^t(\cdot,\cdot)= \max\{\Phi_h^t(\cdot,\cdot)/H, \Psi_{1,h}^t(\cdot,\cdot),\cdot, \Psi_{m,h}^t(\cdot,\cdot)\},
\end{align*}
which further yields
\begin{align*}
V_1^\pi(s_1;\Phi^t/H) \leq V_1^\pi(s_1;\overline{r}^t).	
\end{align*}
By simply adapting the proof of Lemma \ref{lem:v-decomp-2} to this lemma, we have
\begin{align*}
	V_1^\pi(s_1;\overline{r}^t) - \overline V_1^t(s_1) = \sum_{h=1}^H \EE_{\pi, \PP} [ \overline \varsigma_h^t(s_h, a_h) \given s_1 ] + \sum_{h=1}^H \EE_{\pi, \PP} [ \langle \pi_h(\cdot | s_h)-\overline \pi_h^t(\cdot | s_h), Q_{i,h}^t(s_h,\cdot) \rangle_{\cA}  \given s_1],
\end{align*}
where $\overline \varsigma_h^t(s,a) := \overline r_h^t(s,a) + \PP_h \overline V_{h+1}^t(s,a) - \overline Q_h^t(s,a)$. Moreover, we can further adapt the proof of Lemma \ref{lem:v-opt-2} to this lemma and show that 
\begin{align*}
	\sum_{h=1}^H \EE_{\pi, \PP} [ \overline \varsigma_h^t(s_h, a_h) \given s_1 ] \leq 0,
\end{align*}
under the condition that $|(\hat \PP_h^t - \PP_h)V(s,a)| \leq \Phi_h^t(s,a)$. On the other hand, the greedy policy updating rule $\overline{\pi}_h^t=\argmax_{\pi_h}\big\langle \overline{Q}_h^t(\cdot, \cdot), \pi_h(\cdot|\cdot) \big\rangle_\cA$ ensures that
\begin{align*}
\sum_{h=1}^H \EE_{\pi, \PP} [ \langle \pi_h(\cdot | s_h)-\overline \pi_h^t(\cdot | s_h), Q_h^t(s_h,\cdot) \rangle_{\cA}  \given s_1] \leq 0.	
\end{align*}
Therefore, we can show that $V_1^\pi(s_1;\overline{r}^t) -\overline V_1^t(s_1) \leq 0$ holds for any policy $\pi$ and thus
\begin{align*}
	 \max_{\pi\in\Pi} V_1^\pi(s_1;\overline{r}^t) \leq \overline V_1^t(s_1).
\end{align*}
Combining the above results leads to
\begin{align*}
		\sumt \max_{\pi\in\Pi} \EE_{\pi, \PP} \left[\sumh \Phi_h^t(s_h,a_h)\right] \leq H \sumt \overline{V}_1^t(s_1).
\end{align*}
Using a similar technique gives 
\begin{align*}
\sumt \max_{\pi\in\Pi} \EE_{\pi, \PP} \left[\sumh \Psi_{i,h}^t(s_h,a_h)\right] \leq \sumt \overline{V}_1^t(s_1).	
\end{align*}
Next, we prove the upper bound of $\sumt \overline{V}_1^t(s_1)$. Specifically, we have
assume that the trajectory $\{ (s_h^t, a_h^t)\}_{h=1}^H, \forall t\in [T],$ is generated following the policy $\overline{\pi}^t$. Thus, we expand the bias term at the $h$-th step of the $t$-th episode, which is
\begin{align}
	\begin{aligned} \label{eq:bia_1_init-2}
		\overline{V}_h^t(s_h^t)  & = \langle \pi_h^t(\cdot | s_h^t), \overline Q_h^t(s_h^t, \cdot) \rangle_\cA=  \overline Q_h^t(s_h^t, a_h^t)\\
		&=  \big\langle \hat \PP_h^t(\cdot \given s_h^t, a_h^t), \overline V_{h+1}^t(\cdot) \big\rangle_\cS + \overline r_h^t(s_h^t, a_h^t) \\
		&=  \big\langle (\hat \PP_h^t-\PP_h)(\cdot \given s_h^t, a_h^t), \overline V_{h+1}^t(\cdot) \big\rangle_\cS+ \big\langle \PP_h(\cdot \given s_h^t, a_h^t), \overline V_{h+1}^t(\cdot) \big\rangle_\cS + \overline r_h^t(s_h^t, a_h^t) \\
		&\leq  \Phi_h^t(s_h^t,a_h^t)+ \big\langle \PP_h(\cdot \given s_h^t, a_h^t), \overline V_{h+1}^t(\cdot) \big\rangle_\cS + \overline r_h^t(s_h^t, a_h^t) \\
		& = \overline \xi_h^t + \overline V_{h+1}^t(s_{h+1}^t) + \overline r_h^t(s_h^t, a_h^t),
	\end{aligned}
\end{align}
where the inequality is by $\big\langle (\hat \PP_h^t-\PP_h)(\cdot \given s_h^t, a_h^t), \overline V_{h+1}^t(\cdot) \big\rangle_\cS\leq  \Phi_h^t(s_h^t, a_h^t)$ and 
we introduce the martingale difference sequence $\{\overline \xi_h^t\}_{h>0, t>0}$ defined as
\begin{align*}
	& \overline \xi_h^t := \big\langle \PP_h(\cdot \given s_h^t, a_h^t), V_{i,h+1}^t(\cdot) - V_{h+1}^{\pi^t} (\cdot) \big\rangle_\cS - \big[ V_{i,h+1}^t(s_{h+1}^t) - V_{h+1}^{\pi^t} (s_{h+1}^t)\big],
\end{align*}
such that 
\begin{align*}
\EE_{s_{h+1}^t \sim \PP_h(\cdot \given s_h^t, a_h^t)} \big[\overline \xi_h^t \biggiven \cF_h^t\big] = 0,
\end{align*}
where  $\cF_h^t$ is the filtration of all randomness up to $(h-1)$-th step of the $t$-th episode plus $s_h^t, a_h^t$. Here \eqref{eq:bia_1_init-2} constructs a recursion for $\overline{V}_h^t(s_h^t)$. Since $\overline{V}_{H+1}^k(\cdot) = 0$,  recursively applying \eqref{eq:bia_1_init-2} gives
\begin{align*} 
	\sumt \overline V_1^t(s_1) &=  \sum_{h=1}^H \overline \xi_h^t + \sumt\sum_{h=1}^H \overline r_h^t(s_h^t, a_h^t) +  \sumt\sumh\Phi_h^t(s_h^t, a_h^t)\\
	&\leq   \cO\bigg(\sqrt{H^3 T \log \frac{1}{\delta}} \bigg)  + \sumt\sumh [\overline r_h^t(s_h^t, a_h^t) +\Phi_h^t(s_h^t, a_h^t)]
\end{align*}
with probability at least $1-\delta$, where the last inequality is by Azuma-Hoeffding inequality. This completes the proof.
\end{proof}

\begin{lemma} \label{lem:sum-bonus-3} 
	Suppose that $\Phi_h^t(s,a) = \sqrt{\frac{2 H^2 |\cS|\log (8m|\cS||\cA|HT/\delta)}{ N_h^{t-1}(s,a)\vee 1}} \wedge H$ and $\Psi_{i,h}^t(s,a) = \sqrt{\frac{2 \log (8m|\cS||\cA|HT/\delta)}{  N_h^{t-1}(s,a) \vee 1}} \wedge 1$ are the instantiation of the bonus terms in Line \ref{line:pf-sample} of Algorithm \ref{alg:pure-explore} where $N_h^{t-1}(s,a) = \sum_{\tau=1}^{t-1}\mathbf{1}_{\{(s,a)=(s_h^\tau,a_h^\tau)\}}$.
	The updating rules in Algorithm \ref{alg:pure-explore} ensure
	\begin{align*}
		&\sum_{t=1}^T\sum_{h=1}^H \Psi_{i,h}^t(s_h^t, a_h^t) \leq \cO \left(H \sqrt{ T |\cS| |\cA| \log \frac{m|\cS| |\cA|  HT}{\delta}} \right), \\
		&\sum_{t=1}^T \sum_{h=1}^H \Phi_h^t(s_h^t, a_h^t) \leq \cO \left(H^2 \sqrt{ T |\cS|^2 |\cA| \log \frac{m|\cS| |\cA|  HT}{\delta}} \right),
	\end{align*}
	where the trajectory $\{ (s_h^t, a_h^t)\}_{h=1}^H, \forall t\in [T],$ is generated following the policy $\overline \pi^t$.
\end{lemma} 

\begin{proof}  This lemma can be proved by exactly following the proof of Lemma \ref{lem:sum-bonus-2} and substituting the definitions of bonuses and trajectories with the ones defined in this lemma.	
\end{proof}

\begin{lemma}\label{lem:oco-2} Setting, $\eta_k = \sqrt{\log m /(H^2K)}$ if $k>0$ and $0$ otherwise, the updating rule of $\bw$ in Algorithm \ref{alg:exploit-tch} ensures
	\begin{align*}
		\max_{\bw\in\Delta_m}\frac{1}{K}\sumk (\bw-\bw^k)^\top [ \blambda \odot (\tilde{\bV}_1(s_1)+\biota -\bV_1^k(s_1))] \leq \frac{6H\log^{\frac{1}{2}}m}{\sqrt{K}}.
	\end{align*}
\end{lemma}
\begin{proof}  This lemma can be proved by exactly following the proof of Lemma \ref{lem:oco}.	
\end{proof}

\begin{lemma}[Concentration]  \label{lem:concentrate-2}
	Suppose that $\hat{r}_{i,h} = \frac{\sum_{\tau=1}^T\mathbf{1}_{\{(s,a)=(s_h^\tau,a_h^\tau)\}}r_{i,h}^\tau}{ N_h(s,a)\vee 1}$,  $\hat \PP_h = \frac{N_h(s,a,s')}{ N_h(s,a)\vee 1}$,  $\Phi_h(s,a) = \sqrt{\frac{2 H^2 |\cS|\log (8m|\cS||\cA|HT/\delta)}{ N_h(s,a)\vee 1}} \wedge H$ are the instantiation of reward and transition estimates and the bonus terms following Line \ref{line:plan-est} of Algorithm \ref{alg:exploit-tch} where $N_h(s,a) = \sum_{\tau=1}^T\mathbf{1}_{\{(s,a)=(s_h^\tau,a_h^\tau)\}}$.
	Suppose that $\hat \PP_h^t = \frac{ N_h^{t-1}(s,a,s')}{  N_h^{t-1}(s,a)\vee 1}$, $\Phi_{i,h}^t(s,a) = \sqrt{\frac{2H^2|\cS| \log (8m|\cS||\cA|HT/\delta)}{ N_{i,h}^{t-1}(s,a) \vee 1}} \wedge H$, and $\Psi_{i,h}^t(s,a)= \sqrt{\frac{2 \log (8m|\cS||\cA|HT/\delta)}{ N_{i,h}^{t-1}(s,a) \vee 1}} \wedge 1$ are the instantiation of the transition estimate and the bonus terms following Line \ref{line:pf-estimate} in Algorithm \ref{alg:pure-explore} where $N_{i,h}^{t-1}(s,a) = \sum_{\tau=1}^{t-1}\mathbf{1}_{\{(s,a)=( s_h^\tau,a_h^\tau)\}}$. Then with probability at least $1-\delta$, for any $V:\cS\mapsto [0,H]$,
	\begin{align*}
		|(\hat{\PP}_h  - \PP_h)V(s,a)| \leq \Phi_h(s,a), \quad |\hat{r}_{i,h}(s,a) - r_{i,h}(s,a)| \leq \Psi_{i,h}(s,a), ~\forall i\in[m],
	\end{align*}
	and
	\vspace{-0.2cm}
	\begin{align*} 
		|(\hat{\PP}_h^t  - \PP_h)V(s,a)| \leq \Phi_h^t(s,a).
	\end{align*}
\end{lemma}

\begin{proof} The proof of this lemma can be obtained following the one for Lemma \ref{lem:concentrate}.
\end{proof}

\subsection{Proof of Theorem \ref{thm:pre-free}} \label{subsec:proof-tch-pf}

\begin{proof} We first decompose the term $\tchl(\hat{\pi}) - \min_{\pi\in\Pi}\tchl(\pi)$ as
	\begin{align*}
		&\tchl(\hat{\pi}) - \tchl(\pi_{\blambda}^*) \nonumber\\
		&\qquad = \max_{\bw \in\Delta_m}  \sum_{i=1}^m w_i\lambda_i \left(V_{i,1}^*(s_1) + \iota - \frac{1}{K}\sum_{k=1}^K V_{i,1}^{\pi^k}(s_1)\right) - \max_{\bw \in\Delta_m}  \sum_{i=1}^m w_i\lambda_i (V_{i,1}^*(s_1) + \iota - V_{i,1}^{\pi_{\blambda}^*}(s_1))  \nonumber\\
		&\qquad = \max_{\bw \in\Delta_m}  \sum_{i=1}^m w_i\lambda_i \left(V_{i,1}^* + \iota - \frac{1}{K}\sum_{k=1}^K V_{i,1}^{\pi^k}(s_1)\right) -  \frac{1}{K}\sum_{k=1}^K \sum_{i=1}^m w_i^t\lambda_i (V_{i,1}^{\tilde{\nu}_i}(s_1) + \iota - V_{i,1}^{\pi^k}(s_1)) \nonumber\\
		&\qquad\quad + \frac{1}{K}\sum_{k=1}^K \sum_{i=1}^m w_i^t\lambda_i (V_{i,1}^{\tilde{\nu}_i}(s_1) + \iota - V_{i,1}^{\pi^k}(s_1))  - \max_{\bw \in\Delta_m}\sum_{i=1}^m w_i\lambda_i (V_{i,1}^*(s_1) + \iota - V_{i,1}^{\pi_{\blambda}^*}(s_1)),  
	\end{align*}
	where the first equality is by the definition of the output $\hat{\pi}$ such that $V_{i,1}^{\hat{\pi}}=\frac{1}{K}\sum_{k=1}^K V_{i,1}^{\pi^k}$. Following the analysis from \eqref{eq:tch-decomp} to \eqref{eq:tch-decomp-sum} in Section \ref{subsec:proof-tch}, we similarly obtain
	\begin{align}
		\begin{aligned}\label{eq:tch-decomp-1}
			&\tchl(\hat{\pi}) - \tchl(\pi_{\blambda}^*)\\
			&\quad  \leq \underbrace{\max_{\bw\in\Delta_m}  \sumi w_i\lambda_i (V_{i,1}^*(s_1)  -  V_{i,1}^{\tilde{\nu}_i}(s_1))}_{\text{Term(I)}} + \underbrace{\max_{\bw \in\Delta_m} \frac{1}{K}\sum_{k=1}^K \sum_{i=1}^m (w_i-w_i^k)\lambda_i  (V_{i,1}^{\tilde{\nu}_i}(s_1) + \iota - V_{i,1}^{\pi^k}(s_1))}_{\text{Term(II)}}\\
			&\qquad  \quad + \underbrace{\frac{1}{K}\sum_{k=1}^K\sum_{i=1}^m w_i^k\lambda_i (V_{i,1}^{\pi_{\blambda}^*}(s_1) - V_{i,1}^{\pi^k}(s_1))}_{\text{Term(III)}}.
		\end{aligned}
	\end{align}
	Then, we turn to bounding the above three terms respectively based on the updating rules in Algorithm \ref{alg:pure-explore} and Algorithm \ref{alg:exploit-tch}. We first assume that the transition and reward functions can be estimated by certain procedures such that the true transition and reward functions satisfy the following conditions of bounded estimation errors.

	Let $\hat\PP_h^t$, $\hat r_{i,h}^t$, and $\hat{\PP}_h$ be the estimated reward and transitions via some estimation procedures in Algorithm \ref{alg:pure-explore} and Algorithm \ref{alg:exploit-tch} respectively. Suppose that there exist $\Phi_h^t$, $\Psi_{i,h}^t$, and $\Phi_h$ such that for any $V:\cS\mapsto[0,H]$,
	\vspace{-0.2cm}
	\begin{align} 
			&|(\hat \PP_h^t-  \PP_h)V(s,a)| \leq \Phi_h^t(s,a), ~\forall i\in[m]\label{eq:tch-cond-11}\\
			&\|(\hat{\PP}_h - \PP_h)V(s,a)| \leq \Phi_h(s,a), \quad | \hat{r}_{i,h}(s,a) - r_{i,h}(s,a)| \leq \Psi_{i,h}(s,a). \label{eq:tch-cond-22}	
	\end{align}
	Next, we give the upper bounds of Term(I), Term(II), and Term(III). For Term(I), we have
	\begin{align*}
		\text{Term(I)}&=\max_{\bw\in\Delta_m}  \sumi w_i\lambda_i (V_{i,1}^*(s_1)  -  V_{i,1}^{\tilde{\nu}_i}(s_1)) \\
		&\leq\sumi \lambda_i (V_{i,1}^*(s_1) -  V_{i,1}^{\tilde{\nu}_i}(s_1))\\
		&= \sumi \lambda_i(V_{i,1}^*(s_1) - \tilde{V}_{i,1}(s_1)  + \tilde{V}_{i,1}(s_1)  -  V_{i,1}^{\tilde{\nu}_i}(s_1)),
	\end{align*}
	where we use $1\geq w_i\geq 0$ and $V_{i,1}^*(s_1)  \geq  V_{i,1}^{\tilde{\nu}_i}(s_1)$.
	For $V_{i,1}^*(s_1) - \tilde{V}_{i,1}(s_1)$, by Lemma \ref{lem:v-decomp-3}, we have
	\begin{align}
		\begin{aligned}\label{eq:opt-opt-2}
			&V_{i,1}^*(s_1) - \tilde{V}_{i,1}(s_1) \\
			&\qquad \leq \sum_{h=1}^H \EE_{\nu_i^*, \PP} [ \tilde{\varsigma}_{i,h}(s_h, a_h) \given s_1 ] + \sum_{h=1}^H \EE_{\nu_i^*, \PP} [ \langle \nu_{i,h}^*(\cdot | s_h)-\tilde{\nu}_{i,h}(\cdot | s_h), \tilde Q_{i,h}(s_h,\cdot) \rangle_{\cA}  \given s_1 ] \leq 0,
		\end{aligned}
	\end{align}
	where $\tilde\varsigma_{i,h}(s,a) = r_{i,h}(s,a) +  \PP_h \tilde V_{i,h+1}(s, a) - \tilde Q_{i,h}(s,a)$, and the second inequality is by Lemma \ref{lem:v-opt-3} and $\tilde{\nu}_{i,h}=\argmax_{\nu_{i,h}} \langle \nu_{i,h}(\cdot|s), \tilde Q_{i,h}(s,\cdot) \rangle_\cA$ in Algorithm \ref{alg:exploit-tch} such that we have $\sum_{h=1}^H \EE_{\nu_i^*, \PP} [ \tilde{\varsigma}_{i,h}(s_h, a_h) \given s_1 ] \leq 0$, and $\sum_{h=1}^H \EE_{\nu_i^*, \PP} [ \langle \nu_{i,h}^*(\cdot | s_h)-\tilde{\nu}_{i,h}(\cdot | s_h), \tilde Q_{i,h}(s_h,\cdot) \rangle_{\cA}  \given s_1 ] \leq 0$.
	Thus we obtain
	\begin{align}
		\begin{aligned}\label{eq:I-last-2}
			\text{Term(I)}&\leq \sumi \lambda_i (\tilde{V}_{i,1}(s_1)  -  V_{i,1}^{\tilde{\nu}_i}(s_1)).
		\end{aligned} 
	\end{align}
	Next, we bound Term(II) as follows
	\begin{align*}
		&\text{Term(II)}\\
		& = \max_{\bw \in\Delta_m} \frac{1}{K}\sum_{k=1}^K \sum_{i=1}^m (w_i-w_i^k)\lambda_i  (V_{i,1}^{\tilde{\nu}_i}(s_1) - \tilde V_{i,1}(s_1) + \tilde V_{i,1}(s_1)+ \iota - V_{i,1}^{\pi^k}(s_1) + V_{i,1}^k(s_1) - V_{i,1}^k(s_1)) \\
		&\leq  \max_{\bw\in\Delta_m}\frac{1}{K}\sumk (\bw-\bw^k)^\top [ \blambda \odot (\tilde{\bV}_1(s_1)+\biota -\bV_1^k(s_1))]\\
		&\quad  + \max_{\bw \in\Delta_m} \frac{1}{K}\sum_{k=1}^K \sum_{i=1}^m w_i\lambda_i  (V_{i,1}^{\tilde{\nu}_i}(s_1) - \tilde V_{i,1}(s_1)  - V_{i,1}^{\pi^k}(s_1) + V_{i,1}^k(s_1))  \\
		&\quad - \frac{1}{K}\sum_{k=1}^K \sum_{i=1}^m w_i^k\lambda_i (V_{i,1}^{\tilde{\nu}_i}(s_1) - \tilde V_{i,1}(s_1)  - V_{i,1}^{\pi^k}(s_1) + V_{i,1}^k(s_1)),
	\end{align*}
	where the inequality is due to the fact that $\max$ is a convex function. By Lemma \ref{lem:oco-2}, we have
	\begin{align*}
		\max_{\bw\in\Delta_m}\frac{1}{K}\sumk (\bw-\bw^k)^\top [ \blambda \odot (\tilde{\bV}_1(s_1)+\biota -\bV_1^k(s_1))]\leq \frac{6H\log^{\frac{1}{2}}m}{\sqrt{K}}.
	\end{align*}
	Moreover, by Lemma \ref{lem:online-err-3}, we further obtain that
	\begin{align*}
		&\max_{\bw \in\Delta_m} \frac{1}{K}\sum_{k=1}^K \sum_{i=1}^m w_i\lambda_i  (V_{i,1}^{\tilde{\nu}_i} - \tilde V_{i,1}  - V_{i,1}^{\pi^k} + V_{i,1}^k) - \frac{1}{K}\sum_{k=1}^K \sum_{i=1}^m w_i^k\lambda_i (V_{i,1}^{\tilde{\nu}_i} - \tilde V_{i,1}  - V_{i,1}^{\pi^k} + V_{i,1}^k)\\
		&\qquad \leq \max_{\bw \in\Delta_m} \frac{1}{K}\sum_{k=1}^K \sum_{i=1}^m w_i\lambda_i  ( V_{i,1}^k - V_{i,1}^{\pi^k}) + \frac{1}{K}\sum_{k=1}^K \sum_{i=1}^m w_i^k\lambda_i (\tilde V_{i,1}-V_{i,1}^{\tilde{\nu}_i})\\
		&\qquad \leq \frac{1}{K}\sum_{k=1}^K \sum_{i=1}^m \lambda_i  ( V_{i,1}^k - V_{i,1}^{\pi^k}) +  \sum_{i=1}^m \lambda_i (\tilde V_{i,1}-V_{i,1}^{\tilde{\nu}_i}),
	\end{align*}
	where the first and the second inequalities are by \eqref{eq:opt-opt-2} such that $V_{i,1}^{\tilde{\nu}_i} - \tilde V_{i,1} \leq V_{i,1}^* - \tilde V_{i,1} \leq 0$ and also due to $0\leq w_i, w_i^k\leq 1$ and Lemmas \ref{lem:v-decomp-3} and \ref{lem:v-opt-3} such that
	\begin{align}
		\begin{aligned} \label{eq:inv-opt-pi-3}
			&V_{i,1}^{\pi^k}(s_1) - V_{i,1}^k(s_1) \\
			&\qquad =  \sum_{h=1}^H \EE_{\pi^k, \PP} [ \varsigma_{i,h}^k(s_h, a_h) \given s_1 ] + \sum_{h=1}^H \EE_{\pi^k, \PP} [ \langle \pi_h^k(\cdot | s_h)-\pi_h^k(\cdot | s_h), Q_{i,h}^k(s_h,\cdot) \rangle_{\cA}  \given s_1 ] \leq 0,
		\end{aligned}
	\end{align}	
	where $\varsigma_{i,h}^k(s,a) = r_{i,h}(s,a) +  \PP_h V_{i,h+1}^k(s, a) - Q_{i,h}^k(s,a)$.
	Combining the above results gives
	\begin{align}
		\text{Term(II)}& \leq \frac{1}{K}\sum_{k=1}^K \sum_{i=1}^m \lambda_i  ( V_{i,1}^k(s_1) - V_{i,1}^{\pi^k}(s_1)) +  \sum_{i=1}^m \lambda_i (\tilde V_{i,1}(s_1)-V_{i,1}^{\tilde{\nu}_i}(s_1)) + \frac{6H\log^{\frac{1}{2}}m}{\sqrt{K}}. \label{eq:II-last-2}
	\end{align}
	Finally, we will bound Term(III) as follows,
	\begin{align*}
		\text{Term(III)}
		&= \frac{1}{K}\sum_{k=1}^K\sum_{i=1}^m w_i^k\lambda_i (V_{i,1}^{\pi_{\blambda}^*}(s_1) - V_{i,1}^k(s_1) + V_{i,1}^k(s_1) - V_{i,1}^{\pi^k}(s_1))\\
		&= \frac{1}{K}\sum_{k=1}^K\sum_{i=1}^m w_i^k\lambda_i (V_{i,1}^{\pi_{\blambda}^*}(s_1) - V_{i,1}^k(s_1)) + \frac{1}{K}\sum_{k=1}^K\sum_{i=1}^m w_i^k\lambda_i (V_{i,1}^k(s_1) - V_{i,1}^{\pi^k}(s_1))\\
		&\leq \frac{1}{K}\sum_{k=1}^K\sum_{i=1}^m w_i^k\lambda_i (V_{i,1}^{\pi_{\blambda}^*}(s_1) - V_{i,1}^k(s_1)) + \frac{1}{K}\sum_{k=1}^K\sum_{i=1}^m \lambda_i (V_{i,1}^k(s_1) - V_{i,1}^{\pi^k}(s_1)),
	\end{align*}
	where the last inequality is by \eqref{eq:inv-opt-pi-3} and $w_i^k\in[0,1]$. For $\frac{1}{K}\sum_{k=1}^K\sum_{i=1}^m w_i^k\lambda_i (V_{i,1}^{\pi_{\blambda}^*}(s_1) - V_{i,1}^k(s_1))$, using Lemma \ref{lem:v-decomp-3} and optimism as in Lemma \ref{lem:v-opt-3}, we have
	\vspace{-0.1cm}
	\begin{align*}
		&\frac{1}{K}\sum_{k=1}^K\sum_{i=1}^m w_i^k\lambda_i (V_{i,1}^{\pi_{\blambda}^*}(s_1) - V_{i,1}^k(s_1)) \\
		&\qquad\leq 	\frac{1}{K}\sum_{k=1}^K\sum_{i=1}^m w_i^k\lambda_i \sum_{h=1}^H \EE_{\pi_{\blambda}^*, \PP} [ \varsigma_{i,h}^k(s_h, a_h) \given s_1 ] \\
		&\qquad\quad + \frac{1}{K}\sum_{k=1}^K\sum_{i=1}^m w_i^k\lambda_i\sum_{h=1}^H \EE_{\pi_{\blambda}^*, \PP} [ \langle \pi_{\blambda}^*(\cdot | s_h)-\pi_h^k(\cdot | s_h), Q_{i,h}^k(s_h,\cdot) \rangle_{\cA}  \given s_1 ]\\
		&\qquad\leq \frac{1}{K}\sum_{k=1}^K\sum_{h=1}^H \EE_{\pi_{\blambda}^*, \PP} [ \langle \pi_{\blambda}^*(\cdot | s_h)-\pi_h^k(\cdot | s_h), (\bw^k \odot \blambda)^\top\bQ_h^k(s_h,\cdot) \rangle_{\cA}  \given s_1 ]\leq 0,
	\end{align*}
	where the last inequality is due to the updating rule of $\pi^k$ in Algorithm \ref{alg:exploit-tch},i.e.,  $\pi_h^k=\argmax_{\pi_h}\langle (\bw^k\odot\blambda)^\top \bQ_h^k(\cdot, \cdot), \pi_h(\cdot|\cdot) \rangle_\cA$. Combining the above results for Term(III), we have
	\vspace{-0.1cm}
	\begin{align}
		\begin{aligned}\label{eq:III-last-2}
			\text{Term(III)} \leq \frac{1}{K}\sum_{k=1}^K\sum_{i=1}^m \lambda_i (V_{i,1}^k(s_1) - V_{i,1}^{\pi^k}(s_1)).
		\end{aligned}
	\end{align}
	Further combining \eqref{eq:I-last-2},\eqref{eq:II-last-2},  \eqref{eq:III-last-2}, and \eqref{eq:tch-decomp-1}, we have obtain
	\vspace{-0.1cm}
	\begin{align}
		\begin{aligned}
		\label{eq:tch-decomp-final-2}
		\tchl(\hat{\pi}) - \tchl(\pi_{\blambda}^*) &\leq  \frac{2}{K}\sum_{i=1}^m \sum_{k=1}^K  \lambda_i  ( V_{i,1}^k(s_1) - V_{i,1}^{\pi^k}(s_1)) \\
		&\quad +  2\sum_{i=1}^m \lambda_i (\tilde V_{i,1}(s_1)-V_{i,1}^{\tilde{\nu}_i}(s_1)) + \frac{6H\log^{\frac{1}{2}}m}{\sqrt{K}}.
	\end{aligned}
	\end{align}
	Next, we will bound the terms $1/K\cdot \sum_{i=1}^m \sum_{k=1}^K  \lambda_i  ( V_{i,1}^k - V_{i,1}^{\pi^k})$ and $\sum_{i=1}^m \lambda_i (\tilde V_{i,1}-V_{i,1}^{\tilde{\nu}_i})$ following the updating rules of the exploration phase in Algorithm \ref{alg:pure-explore}. Similar to \eqref{eq:inv-opt-pi-3}, by Lemma \ref{lem:v-decomp-3}, we have
	\vspace{-0.1cm}
	\begin{align*}
		\frac{1}{K}\sum_{i=1}^m \sum_{k=1}^K  \lambda_i  ( V_{i,1}^k(s_1) - V_{i,1}^{\pi^k}(s_1)) &=\frac{1}{K}\sumi \lambda_i \sumk\sum_{h=1}^H \EE_{\pi^k, \PP} [ -\varsigma_{i,h}^k(s_h, a_h) \given s_1 ] \\
		&\leq  2\sumi \lambda_i\max_{\pi\in\Pi} \EE_{\pi, \PP} \sumh[ \Phi_h(s,a)+\Psi_{i,h}(s,a)],
	\end{align*}
	where the last inequality is by Lemma \ref{lem:pred-err-bound-3}. In addition, by Lemma \ref{lem:v-decomp-3} and \ref{lem:pred-err-bound-3}, we also have
	\begin{align*}
		\sumi   \lambda_i  ( \tilde V_{i,1} - V_{i,1}^{\tilde\nu_i}) &=\sumi \lambda_i \sum_{h=1}^H \EE_{\tilde\nu_i, \PP} [ -\tilde\varsigma_{i,h}(s_h, a_h) \given s_1 ] \\
		&\leq  2\sumi \lambda_i\EE_{\tilde\nu_i, \PP} \left[\sum_{h=1}^H\big( \Phi_h(s_h,a_h)+\Psi_{i,h}(s_h,a_h)\big)\right]\\
		&\leq  2\sumi \lambda_i\max_{\pi\in\Pi} \EE_{\pi, \PP} \left[\sum_{h=1}^H\big( \Phi_h(s_h,a_h)+\Psi_{i,h}(s_h,a_h)\big)\right].
	\end{align*}
	Therefore, further with \eqref{eq:tch-decomp-final-2}, we have
	\vspace{-0.1cm}
	\begin{align*}
		\tchl(\hat{\pi}) - \tchl(\pi_{\blambda}^*) \leq  8\sumi \lambda_i\max_{\pi\in\Pi} \EE_{\pi, \PP} \left[\sum_{h=1}^H\big( \Phi_h(s_h,a_h)+\Psi_{i,h}(s_h,a_h)\big)\right] + \frac{6H\log^{\frac{1}{2}}m}{\sqrt{K}}.
	\end{align*}
	Employing Lemma \ref{lem:bonus-cmp-3} and Lemma \ref{lem:online-err-3}, we have with probability at least $1-\delta$,
	\vspace{-0.1cm}
	\begin{align*}
		&\max_{\pi\in\Pi} \EE_{\pi, \PP} \left[\sum_{h=1}^H\big( \Phi_h(s_h,a_h)+\Psi_{i,h}(s_h,a_h)\big)\right] \\
		&\qquad \leq \frac{1}{T} \sumt \max_{\pi\in\Pi} \EE_{\pi, \PP}\max_{\pi\in\Pi} \left[\sum_{h=1}^H\big( \Phi_h^t(s_h,a_h)+\Psi_{i,h}^t(s_h,a_h)\big)\right]\\
		&\qquad \leq \frac{2H}{T} \sumt \overline V_1^t(s_1) \leq\cO\bigg(\sqrt{\frac{H^5  \log (1/\delta')}{T}} \bigg)  + \frac{2H}{T} \sumt\sumh [\overline r_h^t(s_h^t, a_h^t) +\Phi_h^t(s_h^t, a_h^t)],
	\end{align*}
	which further gives
	\vspace{-0.1cm}
	\begin{align*}
		\tchl(\hat{\pi}) - \tchl(\pi_{\blambda}^*) \leq  \cO\bigg(\sqrt{\frac{H^5  \log (1/\delta')}{T}} + \sqrt{\frac{H^2\log m }{K}} \bigg)  + \frac{16H}{T} \sumt\sumh [\overline r_h^t(s_h^t, a_h^t) +\Phi_h^t(s_h^t, a_h^t)].
	\end{align*}
	According to Algorithm \ref{alg:pure-explore}, we know that $\overline r^t_h(\cdot,\cdot)=\max\{\Phi_h^t(\cdot,\cdot)/H, \Psi_{1,h}^t(\cdot,\cdot), \Psi_{2,h}^t(\cdot,\cdot), \cdot,\Psi_{m,h}^t(\cdot,\cdot)\}$ in option [I]  for the exploration reward construction. Moreover, by the definition of the bonus terms as in Line \ref{line:pf-estimate} of Algorithm \ref{alg:pure-explore}, we know that $r^t_h \leq \Phi_h^t(\cdot,\cdot)$. Thus, we further have
	\vspace{-0.1cm}
	\begin{align*}
		\tchl(\hat{\pi}) - \tchl(\pi_{\blambda}^*) \leq  \cO\bigg(\sqrt{\frac{H^5  \log (1/\delta')}{T}} + \sqrt{\frac{H^2\log m}{K}} \bigg)  + \frac{32H}{T} \sumt\sumh [\Phi_h^t(s_h^t, a_h^t)].
	\end{align*}
	By Lemma \ref{lem:concentrate-2}, we notice that the conditions \eqref{eq:tch-cond-11} and \eqref{eq:tch-cond-22} hold for the estimation methods applied in Algorithm \ref{alg:pure-explore} and Algorithm \ref{alg:exploit-tch} with probability at least $1-\delta'$. Therefore, further applying Lemma \ref{lem:sum-bonus-3} and by union bound, we have with probability at least $1-\delta$,
	\begin{align*}
			\tchl(\hat{\pi}) - \tchl(\pi_{\blambda}^*) &\leq  \cO\bigg( \sqrt{\frac{H^2\log m}{K}} \bigg)  + \cO \bigg( \sqrt{\frac{ H^6 |\cS|^2 |\cA| \log (m|\cS| |\cA|  HT/\delta)}{T}} \bigg).
	\end{align*}
	We now analyze the case that we adopt $\overline r^t_h(\cdot,\cdot)=\Phi_h^t(\cdot,\cdot)/H+ \sumi \Psi_{i,h}^t(\cdot,\cdot)$ in option [II]. Applying Lemmas \ref{lem:sum-bonus-3} and \ref{lem:concentrate-2}, following the above result, we have with probability at least $1-\delta$,
	\begin{align*}
		&\tchl(\hat{\pi}) - \tchl(\pi_{\blambda}^*) \\
		&\qquad\leq  \cO\bigg(\sqrt{\frac{H^5  \log (1/\delta')}{T}} + \sqrt{\frac{H^2\log m}{K}} \bigg)  + \frac{32H}{T} \sumt\sumh [2\Phi_h^t(s_h^t, a_h^t) + \sumi \Psi_{i,h}^t(s_h^t, a_h^t)]\\
		&\qquad\leq  \cO\bigg( \sqrt{\frac{H^2\log m}{K}} \bigg)  + \cO \bigg( \sqrt{\frac{ (H^2|\cS|+m) H^4 |\cS| |\cA| \log (m|\cS| |\cA|  HT/\delta)}{T}} \bigg).
	\end{align*}
	This completes the proof of this theorem.
\end{proof}

\section{Proofs for Section \ref{sec:stch}}

\subsection{Proof of Proposition \ref{prop:stch}} \label{sec:pr-prop-stch}

\begin{proof} 
	In this proof, we first show that for any policy class $\Pi$ (either stochastic or deterministic), we have $\{\pi~|~\pi\in\argmin_{\pi\in\Pi} \stchlm(\pi), \forall \blambda \in\Delta_m\}\subseteq \Pi_{\mathrm{W}}^*$ and $\{\pi~|~\pi\in\argmin_{\pi\in\Pi} \stchlm(\pi), \forall \blambda \in\Delta_m^o\} \subseteq \Pi_{\mathrm{P}}^*$ for $\mu>0$. In addition, we prove that for a deterministic policy class $\Pi$, there exists $\mu^*>0$ such that for $0<\mu<\mu^*$, $\Pi_{\mathrm{P}}^* \subseteq \{\pi~|~\pi\in\argmin_{\pi\in\Pi} \stchlm(\pi), \forall \blambda \in\Delta_m^o\}$.
	Then, we show that for a stochastic policy class $\Pi$, we have $\Pi_{\mathrm{W}}^* \subseteq \{\pi~|~\pi\in\argmin_{\pi\in\Pi} \stchlm(\pi), \forall \blambda \in\Delta_m\}$ and $\Pi_{\mathrm{P}}^* \subseteq \{\pi~|~\pi\in\argmin_{\pi\in\Pi} \stchlm(\pi), \forall \blambda \in\Delta_m^o\}$ for any $\mu>0$.  Finally, we provide an example to show that a Pareto optimal policy may not be the solution to $\min_{\pi\in\Pi} \stchlm(\pi)$ for any $\blambda \in\Delta_m$ and $\mu>0$ when $\Pi$ is a deterministic policy class.

\vspace{5pt}
\noindent\textbf{Part 1)} Prove $\{\pi~|~\pi\in\argmin_{\pi\in\Pi} \stchlm(\pi), \forall \blambda \in\Delta_m\}\subseteq \Pi_{\mathrm{W}}^*$ and $\{\pi~|~\pi\in\argmin_{\pi\in\Pi} \stchlm(\pi), \allowbreak \forall \blambda \in\Delta_m^o\} \subseteq \Pi_{\mathrm{P}}^*$ with $\mu>0$ for any policy class $\Pi$ (either stochastic or deterministic). The proof of this claim can be immediately obtained by the prior work \citet{lin2024smooth}.

\vspace{5pt}
\noindent \textbf{Part 2)} Prove that for a deterministic policy class $\Pi$, there exists $\mu^*>0$ such that for $0<\mu<\mu^*$, $\Pi_{\mathrm{P}}^* \subseteq \{\pi~|~\pi\in\argmin_{\pi\in\Pi} \stchlm(\pi), \forall \blambda \in\Delta_m^o\}$. We need to show that if a policy is Pareto optimal, there always exists a $\blambda$ such that this policy is a solution to $\pi\in\min_{\pi\in\Pi} \stchlm(\pi)$.

For any $\pi\in\Pi_{\mathrm{P}}^*$, we define its associated $\blambda$ as $\lambda_i:=\frac{(V_{i,1}^*(s_1)+\tau - V_{i,1}^\pi(s_1))^{-1}}{\sum_{j=1}^m (V_{j,1}^*(s_1)+\tau - V_{j,1}^\pi(s_1))^{-1}}>0$. Now we let $\breve\pi$ be an arbitrary policy in $\Pi_{\mathrm{P}}^*$. We need to find a problem-dependent $\mu^*$ such that when $\mu< \mu^*$, $\breve\pi$ is a solution to $\min_{\pi\in\Pi}\stchlm(\pi)$ with $\blambda$ defined above.

We have $\stchlm(\breve\pi)=\mu\log [m\exp(1/(\mu\sum_{j=1}^m (V_{j,1}^*(s_1)+\tau - V_{j,1}^{\breve\pi}(s_1))^{-1}))]$. For any other $\tilde\pi\neq\breve\pi$, to ensure $\stchlm(\breve\pi)\leq\stchlm(\tilde\pi)$, according to \eqref{eq:stch-err} that $\stchlm(\pi) - \mu \log m \leq  \tchl(\pi) \leq  \stchlm(\pi)$,
we require
\begin{align*}
	\stchlm(\breve\pi)\leq\stchlm(\tilde\pi),
\end{align*}
which can be guaranteed by 
\begin{align*}
	 \tchl(\breve \pi) + \mu \log m \leq \tchl(\tilde \pi).
\end{align*}
According to Proposition \ref{prop:tch-unique}, we have that $\breve\pi$ is the optimal solution to $\min_{\pi\in\Pi} \tchl(\pi)$ with $\blambda$ defined as above. We let $\tilde\pi$ be a policy such that $V_{i,1}^{\tilde{\pi}}(s_1) \neq V_{i,1}^{\tilde{\pi}}(s_1)$ for some $i\in[m]$ and $\stchlm(\tilde\pi)$ is the minimal value for policies satisfying the above condition. (If such a policy does not exist, then all policies lead to the same function value as $\stchlm(\breve\pi)$.) Thus, we can set $\mu^*=(\tchl(\tilde \pi)-\tchl(\breve \pi))/\log m$. This completes the proof of \textbf{Part 2)}.

\setlength{\abovecaptionskip}{-3pt}
\setlength{\belowcaptionskip}{-15pt}
\begin{figure}[t!] 
		\centering
		\includegraphics[height=2.2in]{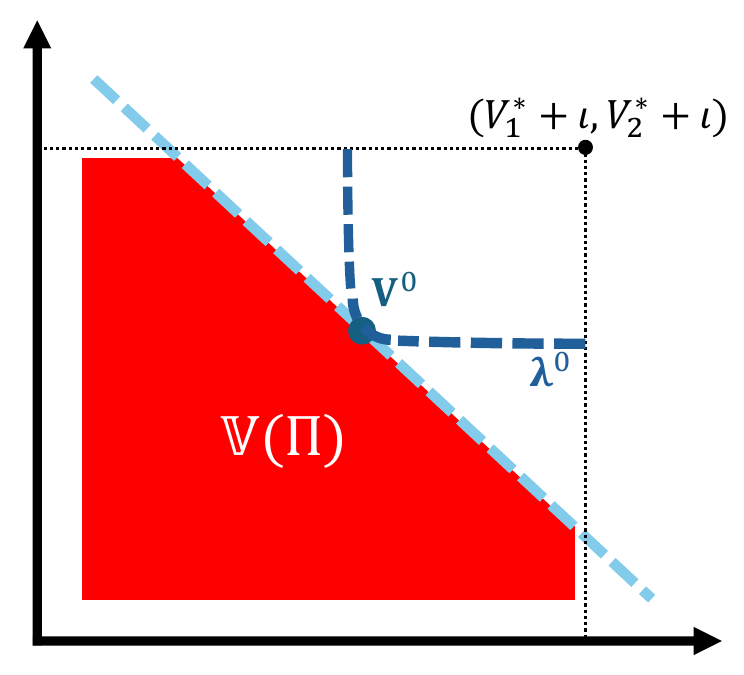}
	\caption{\small Illustration of \textbf{Part 3)}}
	\label{fig:proof2}
\end{figure}

\vspace{5pt}
\noindent \textbf{Part 3)} Prove that for a stochastic policy class $\Pi$, we have $\Pi_{\mathrm{W}}^* \subseteq \{\pi~|~\pi\in\argmin_{\pi\in\Pi} \stchlm(\pi), \forall \blambda \in\Delta_m\}$ and $\Pi_{\mathrm{P}}^* \subseteq \{\pi~|~\pi\in\argmin_{\pi\in\Pi} \stchlm(\pi), \forall \blambda \in\Delta_m^o\}$ for any $\mu>0$. According to Proposition \ref{prop:geometry}, we know that $\mathbb V(\Pi)$ is a convex polytope when $\Pi$ is a stochastic policy class.

As shown in the \textbf{Part 2)} of the proof of Proposition \ref{prop:linear-comb}, when $\mathbb V(\Pi)$ is a convex polytope, for any $\bm V^0$ on the Pareto front, one can find a hyperplane $\{\bm V\in\mathbb R^m|\sumi a_i V_i = C\}$ with $a_i>0, C > 0$ satisfying $\sumi a_i V^0_i = C$ and $\mathbb V(\Pi)\subseteq \{\bm V\in\mathbb R^m|\sumi a_i V_i \le C\}$. With a slight abuse of notation, we view $\stchlm$ as a function of $\bm V$ and denote $\stchlm(\bm V):=\mu \log(\sum_{i=1}^m \exp(\lambda_i(V_i^*+\iota-V_i)/\mu))$. For any $\mu>0$, if we can find positive $\lambda_i$'s such that $\bm V^0 \in \argmin_{\bV: \sumi a_i V_i =C}\stchlm(\bm V)$, then 
\begin{align*}
	\min_{\bm V\in \mathbb V(\Pi)}\stchlm(\bm V)\le \stchlm(\bm V^0) &= \min_{\bm V:\sumi a_i V_i = C}\stchlm(\bm V)\\
	&= \min_{\bm V:\sumi a_i V_i \le C}\stchlm(\bm V)\le \min_{\bm V\in \mathbb V(\Pi)}\stchlm(\bm V),
\end{align*}
which implies that $\bm V^0\in\argmin_{\bm V\in \mathbb V(\Pi)}\mathrm{STCH}_{\blambda}^\mu(\bm V^0)$. This leads to the desired result.
	
Consider the problem $\min_{\bm V:\sumi a_i V_i =C}\mathrm{STCH}_{\blambda}^\mu(\bm V)$. It is equivalent to  
\begin{align*}
\min_{\bm V:\sumi a_i V_i =C}\sumi \exp(\lambda_i(V_i^*+\iota-V_i)/\mu),	
\end{align*}
since $\log$ is a strictly increasing function. The method of Lagrange multipliers leads to the conditions 
\begin{align}\label{lemma:lambda:eqn1}
	\lambda_i \exp(\lambda_i(V_i^*+\iota-V_i)/\mu) = \mu w a_i,
\end{align}
where $w$ is the Lagrange multiplier. The LHS of \eqref{lemma:lambda:eqn1} is an increasing function of $\lambda_i$ which goes to $0$ when $\lambda_i \to 0+$ and tends to $\infty$ when $\lambda_i\to \infty$. Therefore, for $w>0$, as $a_i>0$, we can always find $\lambda_i^0$'s such that $ \lambda_i^0 \exp(\lambda_i^0(V_i^*+\iota-V_i^0)/\mu) = \mu w a_i$. Then, for $\blambda^0 = (\lambda_1^0,\dots,\lambda_m^0)$, we have $\bm V^0=\argmin_{\bm V\in \mathbb V(\Pi)}\mathrm{STCH}_{\blambda^0}^\mu(\bm V^0)$. One can adjust $w$ to make sure that $\sum_i\lambda_i^0=1$. The above construction is illustrated in Figure \ref{fig:proof2}. This completes the proof for $\Pi_{\mathrm{P}}^* \subseteq \{\pi~|~\pi\in\argmin_{\pi\in\Pi} \stchlm(\pi), \forall \blambda \in\Delta_m^o\}$. 
	
For any weak Pareto optimal policy, the only difference is that some $a_i$ might be 0. Set the corresponding $\lambda_i$ to be 0 completes the proof for $\Pi_{\mathrm{W}}^* \subseteq \{\pi~|~\pi\in\argmin_{\pi\in\Pi} \stchlm(\pi), \forall \blambda \in\Delta_m\}$.

\vspace{5pt}
\noindent \textbf{Part 4)}
Now we use a concrete example to illustrate that for a deterministic policy class $\Pi$, a weakly Pareto optimal policy may not be the solution of $\min_{\pi\in\Pi} \stchlm(\pi)$ for any $\blambda \in\Delta_m$. We consider a multi-objective multi-arm bandit problem, a simple and special MOMDP whose state space size $|\cS|=1$, episode length $H=1$, with a deterministic policy. Here we assume this bandit problem has $m=2$ reward functions $r_1$ and $r_2$ and $|\cA|=3$ actions. We define
\begin{align*}
	&r_1(a_1)=1, \quad r_1(a_2)=0.5,\quad r_1(a_3)=0.5,\\
	&r_2(a_1)=0.5, \quad r_2(a_2)=1,\quad r_2(a_3)=0.5,
\end{align*}
which are the reward values of the three actions for each reward function. Via the definition of the (weakly) Pareto optimal policy, we can find that $a_1$, $a_2$ are the Pareto optimal arms and that $a_3$ is a weakly Pareto optimal arm but not a Pareto optimal arm. Then, we have
\begin{align*}
	&\stchlm(a_1) = \mu\log (e^{\iota\lambda_1/\mu}+e^{(0.5+\iota)\lambda_2/\mu}),\\
	&\stchlm(a_2) = \mu\log (e^{(0.5+\iota)\lambda_1/\mu}+e^{\iota\lambda_2/\mu}),\\
	&\stchlm(a_3) = \mu\log (e^{(0.5+\iota)\lambda_1/\mu}+e^{(0.5+\iota)\lambda_2/\mu}),
\end{align*}
which leads to \textbf{(1)} $\stchlm(a_3)>\stchlm(a_1)$ and $\stchlm(a_3)>\stchlm(a_2)$ when $\lambda_1>0$ and $\lambda_2>0$, \textbf{(2)} $\stchlm(a_3)>\stchlm(a_2)$ when $\lambda_1=0$ and $\lambda_2=1$, and \textbf{(3)} $\stchlm(a_3)>\stchlm(a_1)$ when $\lambda_1=1$ and $\lambda_2=0$. This implies that $a_3$ is not a minimizer of  $\stchlm(\pi)$ for any $\blambda \in\Delta_m$ and $\mu>0$.
The proof of this proposition is completed.	
\end{proof}

\subsection{Proof of Proposition \ref{prop:stch-unique}}
\begin{proof}
	According to Proposition \ref{prop:stch}, when $\lambda_i>0$, the solution to $\min_{\pi\in\Pi}\stchlm(\pi)$ is Pareto optimal. 
	We now prove that all Pareto optimal policies associated with the same $\blambda$ have the same values on all objectives. 
	
	For the case where $\Pi$ is a stochastic policy class, assume that $\tilde\pi$ and $\breve\pi$ are two solutions to $\min_{\pi\in\Pi}\stchlm(\pi)$. Assume that there exists $i$ such that $V_{i,1}^{\tilde\pi}(s_1) \neq V_{i,1}^{\breve\pi}(s_1)$. We let $\mathbb{V}(\Pi)$ be the set of $(V_{1,1}^{\pi}(s_1),\cdots,V_{m,1}^{\pi}(s_1))$	for all $\pi\in\Pi$. Then we have $\frac{1}{2} [(V_{1,1}^{\breve\pi}(s_1),\cdots,V_{m,1}^{\breve\pi}(s_1))+ (V_{1,1}^{\tilde\pi}(s_1),\cdots,V_{m,1}^{\tilde\pi}(s_1))]\in\mathbb{V}(\Pi)$ since $\mathbb{V}(\Pi)$ is convex, which implies that there exists a policy $\pi'$ such that 
	\begin{align*}
	V_{i,1}^{\pi'}(s_1) = \frac{V_{1,1}^{\breve\pi}(s_1) + V_{1,1}^{\tilde\pi}(s_1)}{2}.	
	\end{align*}
	On the other hand, the log-sum-exp function $\log \sum_{i=1}^m \exp x_i$ is strictly convex along any direction except for the direction of $(1,1,\cdots,1)$. We note that if $\left(\frac{\lambda_1(V_{1,1}^*(s_1)+\tau-V_{1,1}^{\breve\pi}(s_1))}{\mu} ,\cdots,\frac{\lambda_m(V_{m,1}^*(s_1)+\tau-V_{m,1}^{\breve\pi}(s_1))}{\mu}\right ) \allowbreak - \left(\frac{\lambda_1(V_{1,1}^*(s_1)+\tau-V_{1,1}^{\tilde\pi}(s_1))}{\mu} ,\cdots,\frac{\lambda_m(V_{m,1}^*(s_1)+\tau-V_{m,1}^{\tilde\pi}(s_1))}{\mu}\right ) = \alpha \cdot  (1,1,\cdots,1)$ for some $\alpha \neq 0$, it contradicts that both $\breve\pi$ and $\tilde\pi$ are Pareto optimal. Therefore, for policies $\tilde\pi$ and $\breve\pi$ that do not satisfy the above condition, we can apply the strict convexity and obtain that
	\begin{align*}
	\stchlm(\pi') < \frac{ \stchlm(\breve\pi) + \stchlm(\tilde\pi)}{2} = \stchlm(\breve\pi),
	\end{align*}
	which additionally contradicts that $\breve\pi$ is a minimizer of $\stchlm(\pi)$. Thus, we must have $V_{i,1}^{\tilde\pi}(s_1) = V_{i,1}^{\breve\pi}(s_1)$ for any $i\in[m]$. This completes the first case.

	If $\Pi$ is a deterministic policy class, for any $0<\mu<\mu^*$, the Pareto optimal policies that are the solutions to $\min_{\pi\in\Pi}\stchlm(\pi)$ for all  $\blambda$ with $\lambda_i>0$ have the same values on all objectives.

	Next, we prove this proposition for the deterministic policy class. Following the proof in Section \ref{sec:pr-prop-stch}, when $\mu<\mu^*$, the optimal solution $\breve\pi$ to $\min_{\pi\in\Pi}\stchlm(\pi)$ with $\lambda_i:=\frac{(V_{i,1}^*(s_1)+\tau - V_{i,1}^{\breve\pi}(s_1))^{-1}}{\sum_{j=1}^m (V_{j,1}^*(s_1)+\tau - V_{j,1}^{\breve\pi}(s_1))^{-1}}$ has the same values $V_{i,1}^{\breve\pi}(s_1)$ for all $i\in[m]$ as any other optimal solutions to $\min_{\pi\in\Pi}\stchlm(\pi)$.		
	This completes the proof.
\end{proof}

\subsection{Proof of Proposition \ref{prop:eq-stch}}
\begin{proof} The equivalent form in this proposition is obtained by applying the Fenchel conjugate of the log-sum-exp function \citep{boyd2004convex}. Then, we will consider to solve the following problem for any $\btheta\in \RR^m$,
\begin{align}
    \max_{\bw \in\Delta_m} \langle \bw, \btheta\rangle-\sumi w_i\log w_i. \label{eq:eq-stch-init}
\end{align}
We note that $\langle \bw, \theta\rangle-\sumi w_i\log w_i$ is a concave function. The Lagrangian function of the above problem is 
\begin{align*}
    L(\bw;\upsilon,\bxi):=\langle \bw, \theta\rangle-\sumi w_i\log w_i + \upsilon \left(\sumi w_i - 1\right) + \sumi \xi_i w_i.
\end{align*}
Then, by KKT condition, we have
\begin{align*}
&\frac{\partial L(\bw^*;\upsilon^*,\bxi^*)}{\partial w_i} =  \theta_i - \log w_i^* - 1 +\upsilon^* + \xi_i^* = 0,\\
&\sumi w_i^* - 1 =0,\\
&\xi_i^* w_i^* =0, ~~~\xi_i^* \geq 0,
\end{align*}
where $(\bw^*,\upsilon^*,\bx^*)$ is the solution to $\min_{\upsilon,\bxi}\max_{\bw\in\Delta_m} L(\bw;\upsilon,\bxi)$. By the first two equations above, we obtain that $w_i^* = \frac{\exp(\theta_i + \xi_i^*)}{\sumi \exp(\theta_i + \xi_i^*)} > 0$, which, combined with the third equation $\xi_i^* w_i^* =0$, further implies $\xi_i^*=0$ such that 
\begin{align*}
    w_i^* = \frac{\exp(\theta_i)}{\sumi \exp(\theta_i)}.
\end{align*}
Therefore, plugging $\bw^*$ into \eqref{eq:eq-stch-init}, we have
\begin{align*}
    \max_{\bw \in\Delta_m} \langle \bw, \btheta\rangle-\sumi w_i\log w_i = \log\left[\sumi \exp(\theta_i)\right].
\end{align*}
Now we let $\theta_i = \frac{\lambda_i (V_{i,1}^*(s_1) + \iota - V_{i,1}^\pi(s_1))}{\mu}$. Then, multiplying both sides of the above equation by $\mu$ yields
\begin{align*}
   \stchlm(\pi)&=\mu\log\left[\sumi \exp\left(\frac{\lambda_i (V_{i,1}^*(s_1) + \iota - V_{i,1}^\pi(s_1))}{\mu}\right) \right] \\
   & =\max_{\bw \in\Delta_m} \sumi w_i \lambda_i (V_{i,1}^*(s_1) + \iota - V_{i,1}^\pi(s_1)) -\mu \sumi w_i\log w_i.
\end{align*}
Therefore, we eventually obtain that
\begin{align*}
    \min_{\pi\in\Pi} \stchlm(\pi)= \min_{\pi\in\Pi} \max_{\bw \in\Delta_m} \max_{\nu_i\in\Pi} \sumi w_i \lambda_i (V_{i,1}^{\nu_i}(s_1) + \iota - V_{i,1}^\pi(s_1)) -\mu \sumi w_i\log w_i.
\end{align*}
This completes the proof.
\end{proof}

\subsection{Lemmas for Theorem \ref{thm:stch} and Theorem \ref{thm:pf-stch}}

We note that the only differences between algorithms for Tchebycheff scalarization and algorithms for smooth Tchebycheff scalarization are the updating steps of  $\bw$. Therefore, to prove Theorem \ref{thm:stch} and Theorem \ref{thm:pf-stch}, we need to derive new lemmas for the updating steps of $\bw$ in Algorithm \ref{alg:online-stch} and Algorithm \ref{alg:exploit-stch}. We prove such lemmas in this section. Other lemmas for the proofs of Theorem \ref{thm:stch} and Theorem \ref{thm:pf-stch} remain the same as the ones in Section \ref{sec:lemma-online} and Section \ref{sec:lemma-pre-free}.

\begin{lemma}\label{lem:oco-stch} Setting $\eta_t = \frac{1}{\mu t}$ and $\alpha_t = \frac{1}{t^2}$ for $t\geq 1$, the updating rule of $\bw$ in Algorithm \ref{alg:online-stch} ensures
\begin{align*}
   & \max_{\bw \in\Delta_m} \frac{1}{T}\sumt  \left[\langle \bw-\bw^t, \blambda\odot (\tilde{\bV}_1^t(s_1) + \biota - \bV_1^t(s_1))\rangle -\mu  \langle \bw,\log \bw\rangle + \mu  \langle \bw^t,\log \bw^t\rangle  \right]\\
   &\qquad \leq \frac{18H^2\log T}{\mu T} + \frac{ 11\mu\log^3 ( mT)}{T},
\end{align*}
where we slightly abuse the logarithmic operator and let $\log \bw :=[\log w_1,\log w_2,\cdots,\log w_m]$ be an element-wise logarithmic operation for the vector $\bw$
\end{lemma}

\begin{proof} 
	The mirror ascent step at the $(t+1)$-th episode is equivalent to solving the following minimization problem
	\begin{align*}
		\minimize_{\bw\in\Delta_m} \quad &-\eta_t\langle \bw, \blambda \odot (\tilde{\bV}_1^t(s_1)+\biota -\bV_1^t(s_1))-\mu \log \tilde\bw^t \rangle +  D_{\mathrm{KL}}\big( \bw, \tilde\bw^t \big).
\end{align*}
	We let $\bw^{t+1}$ be the solution to the above optimization problem.
	Note that $\bw^{t+1}$ is guaranteed to stay in the relative interior of a probability simplex if we initialize $w_i^0 = 1 / m$. Thus, applying Lemma \ref{lem:pushback} gives
	\begin{align*}
		&\eta_t \langle \bw-\bw^{t+1}, \blambda \odot (\tilde{\bV}_1^t(s_1)+\biota -\bV_1^t(s_1)) -\mu \log \tilde \bw^t \rangle \\
		&\qquad \leq D_{\mathrm{KL}}\big( \bw, \tilde\bw^t \big) -  D_{\mathrm{KL}}\big( \bw, \bw^{t+1}\big) -  D_{\mathrm{KL}}\big( \bw^{t+1}, \tilde\bw^t \big),
	\end{align*}
	where $\bw$ is an arbitrary variable in $\Delta_m$. 
	Rearranging the terms leads to
	\begin{align}
		\begin{aligned} \label{eq:oco-11}
			&\eta_t \langle  \bw- \bw^t,  \blambda \odot (\tilde{\bV}_1^t(s_1)+\biota -\bV_1^t(s_1)) \rangle - \eta_t\mu\langle \bw, \log\bw\rangle +\eta_t \mu \langle \bw^t,  \log\bw^t \rangle   \\
			&\qquad \leq D_{\mathrm{KL}}\big( \bw, \tilde\bw^t\big) -  D_{\mathrm{KL}}\big( \bw, \bw^{t+1} \big) -  D_{\mathrm{KL}}\big( \bw^{t+1}, \tilde\bw^t \big) \\
			&\qquad \quad + \eta_t \langle  \bw^{t+1} -  \bw^t,  \blambda \odot (\tilde{\bV}_1^t(s_1)+\biota -\bV_1^t(s_1)) -\mu \log \tilde \bw^t  \rangle\\
            &\qquad \quad  - \eta_t\mu\langle \bw, \log\bw\rangle +\eta_t \mu \langle \bw^t,  \log\bw^t \rangle + \eta_t\mu \langle \bw - \bw^t,  \log\tilde \bw^t \rangle.
		\end{aligned}
	\end{align}
    For terms in RHS of \eqref{eq:oco-11}, by Lemma \ref{lem:mix} and the definition of KL divergence, we have
    \begin{align*}
    	&-D_{\mathrm{KL}}\big( \bw, \bw^{t+1}\big) \leq -D_{\mathrm{KL}}\big( \bw, \tilde \bw^{t+1}\big) + \alpha_t \log m,\\
    	&- \eta_t\mu\langle \bw, \log\bw\rangle +\eta_t \mu \langle \bw^t,  \log\bw^t \rangle + \eta_t\mu \langle \bw - \bw^t,  \log\tilde \bw^t \rangle \\
    	&\qquad = - \eta_t\mu D_{\mathrm{KL}}\big( \bw, \tilde\bw^t\big) + \eta_t\mu D_{\mathrm{KL}}\big( \bw^t, \tilde\bw^t\big) \leq -  \eta_t\mu D_{\mathrm{KL}}\big( \bw, \tilde\bw^t\big) + \mu \eta_t \alpha_t \log m,
    \end{align*}
    which thus leads to
    \begin{align}
		\begin{aligned} \label{eq:oco-22}
			&\eta_t \langle  \bw- \bw^t,  \blambda \odot (\tilde{\bV}_1^t(s_1)+\biota -\bV_1^t(s_1)) \rangle - \eta_t\mu\langle \bw, \log\bw\rangle +\eta_t \mu \langle \bw^t,  \log\bw^t \rangle   \\
			&\qquad \leq (1-\eta_t\mu )D_{\mathrm{KL}}\big( \bw, \tilde\bw^t\big) -  D_{\mathrm{KL}}\big( \bw, \tilde\bw^{t+1} \big) -  D_{\mathrm{KL}}\big( \bw^{t+1}, \tilde \bw^t \big) + 2 \alpha_t \log m\\
			&\qquad \quad + \eta_t \langle  \bw^{t+1} -  \bw^t,  \blambda \odot (\tilde{\bV}_1^t(s_1)+\biota -\bV_1^t(s_1)) -\mu \log \tilde \bw^t  \rangle,
		\end{aligned}
	\end{align}
	where we use the setting that $\mu \eta_t\leq 1$. In addition, by Pinsker's inequality, we have
	\begin{align*}
		&-D_{\mathrm{KL}}\big( \bw^{t+1}, \tilde \bw^t \big) \leq -\frac{1}{2} \big\|\bw^{t+1} - \tilde \bw^t\big\|^2_1.
	\end{align*}
	By Cauchy-Schwarz inequality, the last term in \eqref{eq:oco-22} is bounded as
	\begin{align*}
		&\eta_t \langle  \bw^{t+1} -  \bw^t,  \blambda \odot (\tilde{\bV}_1^t(s_1)+\biota -\bV_1^t(s_1)) -\mu \log \tilde \bw^t \rangle\\ 
		&\qquad \leq \eta_t \|\bw^{t+1} -  \bw^t\|_1  \|\blambda \odot (\tilde{\bV}_1^t(s_1)+\biota -\bV_1^t(s_1))-\mu \log \tilde \bw^t \|_\infty\\
		&\qquad \leq \eta_t \left(\|\bw^{t+1} -  \tilde \bw^t\|_1 + \|\tilde\bw^t -  \bw^t\|_1 \right )  \|\blambda \odot (\tilde{\bV}_1^t(s_1)+\biota -\bV_1^t(s_1))-\mu \log \tilde \bw^t \|_\infty\\
		&\qquad \leq \frac{1}{2} \big\|\bw^{t+1} - \bw^t \big\|_1^2 + \left(3H + \mu\log (m/\alpha_t)\right)^2\eta_t^2 + 2\alpha_t^2,
	\end{align*}
	where the last inequality uses Lemma \ref{lem:mix} and $\|\blambda \odot  (\tilde{\bV}_1^t(s_1)+\biota -\bV_1^t(s_1))-\mu \log \tilde \bw^t\|_\infty = \max_i \lambda_i (\tilde{V}_{i,1}^t(s_1)+\iota -V_{i,1}^t(s_1))) - \mu \log \tilde w_i^t \leq 3H + \mu\log (m/\alpha_t)$ with $\iota,\lambda_i\leq 1$.
	Plugging the above results in \eqref{eq:oco-22}, we have
	\begin{align*}
			&\eta_t \langle  \bw- \bw^t,  \blambda \odot (\tilde{\bV}_1^t(s_1)+\biota -\bV_1^t(s_1)) \rangle - \eta_t\mu\langle \bw, \log\bw\rangle +\eta_t \mu \langle \bw^t,  \log\bw^t \rangle   \\
			&\qquad \leq (1-\eta_t\mu )D_{\mathrm{KL}}\big( \bw, \tilde\bw^t\big) -  D_{\mathrm{KL}}\big( \bw, \tilde\bw^{t+1} \big)+ 2 \alpha_t \log m +  \left(3H + \mu\log (m/\alpha_t)\right)^2\eta_t^2 + 2\alpha_t^2,
	\end{align*}
	Now setting $\eta_t = 1/(\mu t)$ and $\alpha_t = 1/t^2$, dividing both sides by $T \eta_t$, and taking summation from $1$ to $T$, we obtain
	\begin{align*}
		&\frac{1}{T}\sumt\left[ \langle  \bw- \bw^t,  \blambda \odot (\tilde{\bV}_1^t(s_1)+\biota -\bV_1^t(s_1)) \rangle - \mu\langle \bw, \log\bw\rangle +\mu \langle \bw^t,  \log\bw^t \rangle  \right] \\
		&\qquad \leq \frac{1}{T}\sumt \left[\mu (t-1)D_{\mathrm{KL}}\big( \bw, \tilde \bw^t\big) -  \mu t D_{\mathrm{KL}}\big( \bw, \tilde \bw^{t+1} \big)  + \frac{2 \mu\log m}{t} +  \frac{(3H + \mu\log (t^2 m))^2}{\mu t} + \frac{2\mu}{t^3} \right]\\
		&\qquad \leq \frac{1}{T}\sumt \left[\frac{2 \mu\log m}{t} +  \frac{(3H + 2\mu\log ( mt))^2}{\mu t} + \frac{2\mu}{t^3} \right]\\
		&\qquad \leq \frac{3 \mu\log (mT)}{T} +  \frac{(3H + 2\mu\log ( mT))^2\log T}{\mu T}  \leq \frac{18H^2\log T}{\mu T} + \frac{ 11\mu\log^3 ( mT)}{T}.
	\end{align*}
	This completes the proof.	
\end{proof}

\begin{lemma}\label{lem:oco-pf-stch} Setting $\eta_k = \frac{1}{\mu k}$ and $\alpha_k = \frac{1}{k^2}$ for $k\geq 1$ and $\eta_0 = \alpha_0 = 0$, the updating rule of $\bw$ in Algorithm \ref{alg:exploit-stch} ensures
	\begin{align*}
		& \max_{\bw \in\Delta_m} \frac{1}{K}\sumk  \left[\langle \bw-\bw^k, \blambda\odot (\tilde{\bV}_1(s_1) + \biota - \bV_1^k)\rangle -\mu  \langle \bw,\log \bw\rangle + \mu  \langle \bw^k,\log \bw^k\rangle  \right]\\
		&\qquad \leq \frac{18H^2\log K}{\mu K} + \frac{ 11\mu\log^3 ( mK)}{K},
	\end{align*}
	where we slightly abuse the logarithmic operator and let $\log \bw :=[\log w_1,\log w_2,\cdots,\log w_m]$ be an element-wise logarithmic operation for the vector $\bw$
\end{lemma}

\begin{proof} The proof of this lemma exactly follows the proof of Theorem \ref{lem:oco-stch}.
\end{proof}

\subsection{Proof of Theorem \ref{thm:stch}}
\begin{proof}
We begin the proof of this theorem by decomposing $\stchlm(\hat{\pi})-\stchlm(\pi_{\mu,\blambda}^*)$ where $\hat{\pi} = \Mix(\pi^1,\pi^2,\cdots,\pi^T)$. Specifically, following the decomposition analysis in Section \ref{subsec:proof-tch}, we have
\begin{align}
	\begin{aligned}\label{eq:tch-decomp-11}
		&\stchlm(\hat{\pi}) - \stchlm(\pi_{\mu,\blambda}^*)\\
		&\qquad \leq \underbrace{\max_{\bw\in\Delta_m}   \frac{1}{T}\sumt \sumi w_i\lambda_i (V_{i,1}^*(s_1)  -  V_{i,1}^{\tilde\nu_i^t}(s_1))}_{\text{Term(I)}} \\
		&\qquad\quad + \underbrace{\max_{\bw \in\Delta_m} \frac{1}{T}\sum_{t=1}^T \sum_{i=1}^m [(w_i-w_i^t)\lambda_i  (V_{i,1}^{\tilde\nu_i^t}(s_1) + \iota - V_{i,1}^{\pi^t}(s_1)) - w_i\log w_i + w_i^t\log w_i^t]}_{\text{Term(II)}}\\
		&\qquad\quad +\underbrace{\frac{1}{T}\sum_{t=1}^T\sum_{i=1}^m w_i^t\lambda_i (V_{i,1}^{\pi_{\mu,\blambda}^*}(s_1) - V_{i,1}^{\pi^t}(s_1))}_{\text{Term(III)}}
	\end{aligned}
\end{align}
Then, we bound the above three terms respectively based on the updating rules in Algorithm \ref{alg:online-stch}. The only difference between Algorithm \ref{alg:online-tch} for Tchebycheff scalarization and Algorithm \ref{alg:online-stch} for smooth Tchebycheff scalarization are their updating steps of $\bw$. As we can observe, only the upper bound of Term(II) corresponds to the updating step of $\bw$. Thus, to prove Theorem \ref{thm:stch}, we can adopt the proof of Theorem \ref{thm:tch} to provide the upper bounds of Term(I) and Term(III). Then, we have
\begin{align*}
	\text{Term(I)} + \text{Term(III)} \leq \cO\bigg(\sqrt{\frac{H^4 |\cS|^2|\cA| \log (m|\cS||\cA|HT/\delta')}{T}} \bigg),
\end{align*}
with probability at least $1-\delta'$. In addition, to bound Term(II), following the proof of Theorem \ref{thm:tch}, we further apply Lemma \ref{lem:oco-stch} and obtain
\begin{align*}
	\text{Term(II)} &\leq \max_{\bw \in\Delta_m} \frac{1}{T}\sumt  \left[\langle \bw-\bw^t, \blambda\odot (\tilde{\bV}_1^t(s_1) + \biota - \bV_1^t(s_1))\rangle -\mu  \langle \bw,\log \bw\rangle + \mu  \langle \bw^t,\log \bw^t\rangle  \right] \\
	&\quad +\frac{1}{T}\sum_{t=1}^T \sum_{i=1}^m \lambda_i  ( V_{i,1}^t(s_1) - V_{i,1}^{\pi^t}(s_1)) + \frac{1}{T}\sum_{t=1}^T \sum_{i=1}^m \lambda_i (\tilde V_{i,1}^t(s_1)-V_{i,1}^{\tilde\nu_i^t}(s_1)) \\
	&\leq \cO\bigg(\sqrt{\frac{H^4 |\cS|^2|\cA| \log (m|\cS||\cA|HT/\delta')}{T}} \bigg) + \frac{18H^2\log T}{\mu T} + \frac{ 11\mu\log^3 ( mT)}{T}
\end{align*}
with probability at least $1-\delta'$. Further by \eqref{eq:tch-decomp-11} and using the union bound, we obtain 
\begin{align*}
\stchlm(\hat{\pi}) - \stchlm(\pi_{\mu,\blambda}^*) \leq 	\cO\bigg(\sqrt{\frac{H^4 |\cS|^2|\cA| \log (m|\cS||\cA|HT/\delta)}{T}} + \frac{H^2\log T}{\mu T} + \frac{ \mu\log^3 ( mT)}{T}\bigg)
\end{align*}
with probability at least $1-\delta$. This completes the proof of this theorem.
\end{proof}

\subsection{Proof of Theorem \ref{thm:pf-stch}}  
\begin{proof}
	We begin the proof of this theorem by decomposing $\stchlm(\hat{\pi})-\stchlm(\pi_{\mu,\blambda}^*)$ where $\hat{\pi} = \Mix(\pi^1,\pi^2,\cdots,\pi^K)$. Specifically, following the decomposition analysis in Section \ref{subsec:proof-tch-pf}, we have
	\begin{align}
		\begin{aligned}\label{eq:tch-decomp-22}
			&\stchlm(\hat{\pi}) - \stchlm(\pi_{\mu,\blambda}^*)\\
			&\qquad \leq \underbrace{\max_{\bw\in\Delta_m}   \frac{1}{T}\sumt \sumi w_i\lambda_i (V_{i,1}^*(s_1)  -  V_{i,1}^{\tilde\nu_i}(s_1))}_{\text{Term(I)}} \\
			&\qquad\quad + \underbrace{\max_{\bw \in\Delta_m} \frac{1}{T}\sum_{t=1}^T \sum_{i=1}^m [(w_i-w_i^k)\lambda_i  (V_{i,1}^{\tilde\nu_i}(s_1) + \iota - V_{i,1}^{\pi^k}(s_1)) - w_i\log w_i + w_i^k\log w_i^k]}_{\text{Term(II)}}\\
			&\qquad\quad +\underbrace{\frac{1}{T}\sum_{t=1}^T\sum_{i=1}^m w_i^k\lambda_i (V_{i,1}^{\pi_{\mu,\blambda}^*}(s_1) - V_{i,1}^{\pi^k}(s_1))}_{\text{Term(III)}}
		\end{aligned}
	\end{align}	
	Then, we need to bound the above three terms respectively based on the updating rules in Algorithm \ref{alg:pure-explore} and Algorithm \ref{alg:exploit-stch}. The only difference between Algorithm \ref{alg:exploit-tch} for Tchebycheff scalarization and Algorithm \ref{alg:exploit-stch} for smooth Tchebycheff scalarization are their updating steps of $\bw$. And only the upper bound of Term(II) corresponds to the updating step of $\bw$. Thus, to prove Theorem \ref{thm:pf-stch}, we can adapt the proof of Theorem \ref{thm:pre-free} to provide the upper bounds of Term(I) and Term(III), which gives
	\begin{align*}
		\text{Term(I)} + \text{Term(III)} \leq \cO\bigg(\sqrt{\frac{H^6 |\cS|^2|\cA| \log (m|\cS||\cA|HT/\delta')}{T}} \bigg),
	\end{align*}
	with probability at least $1-\delta'$ for option [I] in Algorithm \ref{alg:pure-explore}, and 
	\begin{align*}
		\text{Term(I)} + \text{Term(III)} \leq \cO\bigg(\sqrt{\frac{(H^2|\cS| + m) H^4 |\cS|^2|\cA| \log (m|\cS||\cA|HT/\delta')}{T}} \bigg),
	\end{align*}
	with probability at least $1-\delta'$ for option [II] in Algorithm \ref{alg:pure-explore}.

	 In addition, to bound Term(II), following the proof of Theorem \ref{thm:pre-free}, we obtain
	\begin{align*}
		\text{Term(II)} &\leq \max_{\bw \in\Delta_m} \frac{1}{T}\sumt  \left[\langle \bw-\bw^t, \blambda\odot (\tilde{\bV}_1^t(s_1) + \biota - \bV_1^t(s_1))\rangle -\mu  \langle \bw,\log \bw\rangle + \mu  \langle \bw^t,\log \bw^t\rangle  \right] \\
		&\quad +\frac{1}{K}\sum_{k=1}^K \sum_{i=1}^m \lambda_i  ( V_{i,1}^k(s_1) - V_{i,1}^{\pi^k}(s_1)) +  \sum_{i=1}^m \lambda_i (\tilde V_{i,1}(s_1)-V_{i,1}^{\tilde{\nu}_i}(s_1)).
	\end{align*}
	Further applying Lemma \ref{lem:oco-pf-stch} leads to
	\begin{align*}
		\text{Term(II)} \leq \cO\bigg(\sqrt{\frac{H^6 |\cS|^2|\cA| \log (m|\cS||\cA|HT/\delta')}{T}} \bigg) + \frac{18H^2\log K}{\mu K} + \frac{ 11\mu\log^3 ( mK)}{K}
	\end{align*}
	with probability at least $1-\delta'$ for option [I] in Algorithm \ref{alg:pure-explore}, and
	\begin{align*}
		\text{Term(II)} \leq \cO\bigg(\sqrt{\frac{(H^2|\cS| + m)H^4 |\cS||\cA| \log (m|\cS||\cA|HT/\delta')}{T}} \bigg) + \frac{18H^2\log K}{\mu K} + \frac{ 11\mu\log^3 ( mK)}{K}
	\end{align*}
	with probability at least $1-\delta'$ option [II] in Algorithm \ref{alg:pure-explore}. Combining the above results with \eqref{eq:tch-decomp-22} and using the union bound, we obtain 
	\begin{align*}
		\stchlm(\hat{\pi}) - \stchlm(\pi_{\mu,\blambda}^*) \leq 	\cO\bigg(\sqrt{\frac{H^6 |\cS|^2|\cA| \log (m|\cS||\cA|HT/\delta)}{T}} + \frac{H^2\log K}{\mu K} + \frac{ \mu\log^3 ( mK)}{K}\bigg)
	\end{align*}
	with probability at least $1-\delta$ for option [I], and
	\begin{small}
		\begin{align*}
			\stchlm(\hat{\pi}) - \stchlm(\pi_{\mu,\blambda}^*) \leq 	\cO\bigg(\sqrt{\frac{(H^2|\cS| + m)H^4 |\cS||\cA| \log (m|\cS||\cA|HT/\delta)}{T}} + \frac{H^2\log K}{\mu K} + \frac{ \mu\log^3 ( mK)}{K}\bigg)
		\end{align*}
	\end{small}
	with probability at least $1-\delta$ for option [II].
	The proof of this theorem is completed.
\end{proof}

\subsection{Additional Discussion for Algorithms} \label{sec:additional}

When $\mu$ is too small with $\mu\rightarrow 0^+$, we note that the results in Theorem \ref{thm:stch} and Theorem \ref{thm:pf-stch} can be worse as they have a dependence on $1/\mu$. Here, we can further show that the term $\tilde\cO(1/(\mu T)+ \mu/T)$ in Theorem \ref{thm:stch} (or $\tilde\cO(1/(\mu K)+ \mu/K)$ in Theorem \ref{thm:pf-stch}) can be replaced by $\tilde\cO(1/\sqrt{T})$ (or $\tilde\cO(1/\sqrt{K})$) without the dependence on $1/\mu$ under different settings of the step size $\eta_t$ (or $\eta_k$). 

Such results are associated with the learning guarantees for the update of $\bw$, i.e., Lemmas \ref{lem:oco-stch} and $\ref{lem:oco-pf-stch}$. Next, We will derive the different guarantees under different step sizes for the update of $\bw$ in Algorithm \ref{alg:online-stch} and Algorithm \ref{alg:exploit-stch}. We remark that under this setting, the weighted average steps for $\bw$ in both algorithms can be avoided. The proofs of the following lemmas will not use the weighted average steps, i.e., we substitute Line \ref{line:stch-mix} and Line \ref{line:stch-w} in Algorithm \ref{alg:online-stch} by 
\begin{align}
w_i^t \propto (w_i^{t-1})^{1-\mu\eta_{t-1}} \cdot \exp\{ \eta_{t-1} \cdot \lambda_i  (\tilde{V}_{i,1}^{t-1}(s_1)+\iota -V_{i,1}^{t-1}(s_1))\},~\forall i\in[m], \label{eq:subtitute-w1}
\end{align}
and substitute Line \ref{line:pf-stch-mix} and Line \ref{line:pf-stch-w} in Algorithm \ref{alg:exploit-stch} by 
\begin{align}
	w_i^k \propto (w_i^{k-1})^{1-\mu\eta_{k-1}} \cdot \exp\{ \eta_{k-1} \cdot \lambda_i  (\tilde{V}_{i,1}(s_1)+\iota -V_{i,1}^{k-1}(s_1))\},~\forall i\in[m], \label{eq:subtitute-w2}
\end{align}

\begin{lemma}\label{lem:oco-stch-2} When $\mu\leq 1$, Setting $\eta_t = \eta = 1/(2H\sqrt{T})$, the updating rule of $\bw$ in \eqref{eq:subtitute-w1} for Algorithm \ref{alg:online-stch} ensures
	\begin{align*}
		&  \max_{\bw \in\Delta_m} \frac{1}{T}\sumt  \left[\langle \bw-\bw^t, \blambda\odot (\tilde{\bV}_1^t(s_1) + \biota - \bV_1^t(s_1))\rangle -\mu  \langle \bw,\log \bw\rangle + \mu  \langle \bw^t,\log \bw^t\rangle  \right] \leq  \frac{2H\log m + 5H}{\sqrt{T}}.
	\end{align*}
\end{lemma}

\begin{proof} 
	The mirror ascent step at round $t+1$ via \eqref{eq:subtitute-w1} is equivalent to solving the following minimization problem 
	\begin{align*}
		\minimize_{\bw\in\Delta_m} \quad &-\eta_t\langle \bw, \blambda \odot (\tilde{\bV}_1^t(s_1)+\biota -\bV_1^t(s_1))-\mu \log \bw^t \rangle +  D_{\mathrm{KL}}\big( \bw, \bw^t \big).
	\end{align*}
	Let $\bw^{t+1}$ be the solution of the above optimization problem.
	Note that $\bw^{t+1}$ is guaranteed to stay in the relative interior of a probability simplex if we initialize $w_i^0 = 1 / m$. Thus, applying Lemma \ref{lem:pushback} gives
	\begin{align*}
		&-\eta_t \langle \bw^{t+1}, \blambda \odot (\tilde{\bV}_1^t(s_1)+\biota -\bV_1^t(s_1)) -\mu\log \bw^t \rangle  + \eta_t \langle \bw, \blambda \odot (\tilde{\bV}_1^t(s_1)+\biota -\bV_1^t(s_1)) -\mu \log \bw^t \rangle \\
		&\qquad \leq D_{\mathrm{KL}}\big( \bw, \bw^t \big) -  D_{\mathrm{KL}}\big( \bw, \bw^{t+1}\big) -  D_{\mathrm{KL}}\big( \bw^{t+1}, \bw^t \big),
	\end{align*}
	where $\bw$ is an arbitrary variable in $\Delta_m$. 
	Rearranging the terms leads to
	\begin{align}
		\begin{aligned} \label{eq:oco-111}
			&\eta_t \langle  \bw- \bw^t,  \blambda \odot (\tilde{\bV}_1^t(s_1)+\biota -\bV_1^t(s_1)) \rangle - \eta_t\mu\langle \bw, \log\bw\rangle +\eta_t \mu \langle \bw^t,  \log\bw^t \rangle   \\
			&\qquad \leq D_{\mathrm{KL}}\big( \bw, \bw^t\big) -  D_{\mathrm{KL}}\big( \bw, \bw^{t+1} \big) -  D_{\mathrm{KL}}\big( \bw^{t+1}, \bw^t \big) \\
			&\qquad \quad + \eta_t \langle  \bw^{t+1} -  \bw^t,  \blambda \odot (\tilde{\bV}_1^t(s_1)+\biota -\bV_1^t(s_1))  \rangle\\
			&\qquad \quad + \eta_t \mu \langle  \bw^t- \bw^{t+1}, \log \bw^t  \rangle + \eta_t \mu\langle \bw, \log \bw^t - \log \bw  \rangle.
		\end{aligned}
	\end{align}
	We note that the last term of \eqref{eq:oco-111} leads to $\eta_t \mu\langle \bw, \log \bw^t - \log \bw  \rangle= -\eta_t \mu D_{\mathrm{KL}}(\bw,\bw^t)$ by the definition of KL divergence. Further by the definition of KL divergence, we have
	\begin{align*}
		\eta_t \mu \langle  \bw^t- \bw^{t+1}, \log \bw^t  \rangle = \eta_t \mu D_{\mathrm{KL}}(\bw^{t+1},\bw^t) + \eta_t \mu \langle \bw^t, \log \bw^t\rangle - \eta_t \mu \langle \bw^{t+1}, \log \bw^{t+1}\rangle.
	\end{align*}
	Therefore, \eqref{eq:oco-111} is equivalently written as
	\begin{align}
		\begin{aligned} \label{eq:oco-222}
			&\eta_t \langle  \bw- \bw^t,  \blambda \odot (\tilde{\bV}_1^t(s_1)+\biota -\bV_1^t(s_1)) \rangle - \eta_t\mu\langle \bw, \log\bw\rangle +\eta_t \mu \langle \bw^t,  \log\bw^t \rangle   \\
			&\qquad \leq (1-\eta_t\mu )D_{\mathrm{KL}}\big( \bw, \bw^t\big) -  D_{\mathrm{KL}}\big( \bw, \bw^{t+1} \big) -  (1-\eta_t\mu) D_{\mathrm{KL}}\big( \bw^{t+1}, \bw^t \big) \\
			&\qquad \quad  + \eta_t \mu \langle \bw^t, \log \bw^t\rangle - \eta_t \mu \langle \bw^{t+1}, \log \bw^{t+1}\rangle\\
			&\qquad \quad + \eta_t \langle  \bw^{t+1} -  \bw^t,  \blambda \odot (\tilde{\bV}_1^t(s_1)+\biota -\bV_1^t(s_1))  \rangle.
		\end{aligned}
	\end{align}
	In addition, by Pinsker's inequality, we have
	\begin{align*}
		&-D_{\mathrm{KL}}\big( \bw^{t+1}, \bw^t \big) \leq -\frac{1}{2} \big\|\bw^{t+1} - \bw^t\big\|^2_1.
	\end{align*}
	By Cauchy-Schwarz inequality, the last term in \eqref{eq:oco-222} is bounded as
	\begin{align*}
		&\eta_t \langle  \bw^{t+1} -  \bw^t,  \blambda \odot (\tilde{\bV}_1^t(s_1)+\biota -\bV_1^t(s_1)) \rangle\\ 
		&\qquad \leq \eta_t \|\bw^{t+1} -  \bw^t\|_1  \|\blambda \odot (\tilde{\bV}_1^t(s_1)+\biota -\bV_1^t(s_1)) \|_\infty\\
		&\qquad \leq \frac{1-\eta_t\mu}{2} \big\|\bw^{t+1} - \bw^t \big\|_1^2 + \frac{9H^2\eta_t^2}{2(1-\eta_t\mu)},
	\end{align*}
	where the last inequality is due to $\|\blambda \odot  (\tilde{\bV}_1^t(s_1)+\biota -\bV_1^t(s_1))\|_\infty^2 = (\max_i \lambda_i (\tilde{V}_{i,1}^t(s_1)+\iota -V_{i,1}^t(s_1)))^2 \leq 9H^2$ with $\iota,\lambda_i\leq 1$. Therefore, combining the above inequalities with \eqref{eq:oco-222}, and due to  $\eta_t\mu< 1$, we have
	\begin{align*}
		&\eta_t \langle  \bw- \bw^t,  \blambda \odot (\tilde{\bV}_1^t(s_1)+\biota -\bV_1^t(s_1)) \rangle - \eta_t\mu\langle \bw, \log\bw\rangle +\eta_t \mu \langle \bw^t,  \log\bw^t \rangle   \\
		&\qquad \leq (1-\eta_t\mu )D_{\mathrm{KL}}\big( \bw, \bw^t\big) -  D_{\mathrm{KL}}\big( \bw, \bw^{t+1} \big)   + \eta_t \mu \langle \bw^t, \log \bw^t\rangle - \eta_t \mu \langle \bw^{t+1}, \log \bw^{t+1}\rangle + \frac{9H^2\eta_t^2}{2(1-\eta_t\mu)}\\
        &\qquad \leq D_{\mathrm{KL}}\big( \bw, \bw^t\big) -  D_{\mathrm{KL}}\big( \bw, \bw^{t+1} \big)   + \eta_t \mu \langle \bw^t, \log \bw^t\rangle - \eta_t \mu \langle \bw^{t+1}, \log \bw^{t+1}\rangle + \frac{9H^2\eta_t^2}{2(1-\eta_t\mu)}.
	\end{align*}
	Dividing both sides by $\eta_t$, setting $\eta_t=\eta$, and taking summation from $1$ to $T$ on both sides, we have
	\begin{align*}
		&\sumt \left[\langle  \bw- \bw^t,  \blambda \odot (\tilde{\bV}_1^t(s_1)+\biota -\bV_1^t(s_1)) \rangle - \mu\langle \bw, \log\bw\rangle + \mu \langle \bw^t,  \log\bw^t \rangle \right] \\
		&\qquad  \leq \sumt \left(\frac{1}{\eta} D_{\mathrm{KL}}\big( \bw, \bw^t\big)  - \frac{1}{\eta}D_{\mathrm{KL}}\big( \bw, \bw^{t+1}\big)\right) + 5H^2\sumt \frac{\eta}{1-\mu\eta} \\
		&\qquad \leq \frac{1}{\eta}D_{\mathrm{KL}}\big( \bw, \bw^1\big)  + 5H^2 \sumt \frac{\eta}{1-\eta\mu}\leq \frac{\log m}{\eta}  + 5H^2 \frac{T\eta}{1-\eta\mu}.
	\end{align*}
    where the first inequality is due to the fact that our initialization of this algorithm ensures that $w_i^1=1/m$ such that $D_{\mathrm{KL}}\big( \bw, \bw^1\big)= \sumi w_i \log (w_im) \leq \log m$. When $\mu\leq 1$, setting $\eta = 1/(2H\sqrt{T})$, dividing both sides by $T$, we have 
	\begin{align*}
		\max_{\bw\in\Delta_m} \frac{1}{T}\sumt (\bw  - \bw^t)^\top [ \blambda \odot (\tilde{\bV}_1^t(s_1)+\biota -\bV_1^t(s_1))] \leq \frac{\log m}{T\eta}  + 10H^2 \eta\leq \frac{2H\log m + 5H}{\sqrt{T}},
	\end{align*}
 where we use $1-\mu\eta \geq 1/2$. This completes the proof.	
\end{proof}

Following the proof above, we similarly obtain the following lemma for Algorithm \ref{alg:exploit-stch}.
\begin{lemma}\label{lem:oco-pfstch-2} When $\mu\leq 1$, setting $\eta_k = 1/(2H\sqrt{T})$ for $k\geq 1$ and $0$ otherwise, the updating rule of $\bw$ in \eqref{eq:subtitute-w2} for Algorithm \ref{alg:exploit-stch} ensures
	\begin{align*}
		&  \max_{\bw \in\Delta_m} \frac{1}{K}\sumt  \left[\langle \bw-\bw^k, \blambda\odot (\tilde{\bV}_1(s_1) + \biota - \bV_1^k(s_1))\rangle -\mu  \langle \bw,\log \bw\rangle + \mu  \langle \bw^k,\log \bw^k\rangle  \right] \leq  \frac{2H\log m + 5H}{\sqrt{K}}.
	\end{align*}
\end{lemma}

\section{Other Supporting Lemmas}

In the tabular case, once we collected $\ell$ trajectories, i.e., $\{s_h^\tau, a_h^\tau, r_{i,h}^\tau(s_h^\tau, a_h^\tau)\}_{\tau=1}^\ell$, we can always estimate the reward function and the transition kernel as  
\begin{align}
	\hat{r}_{i,h}(s,a)=\frac{\sum_{\tau=1}^\ell\mathbf{1}_{\{(s,a)=(s_h^\tau,a_h^\tau)\}}r_{i,h}^\tau(s_h^\tau,a_h^\tau)}{\max\{ N_h(s,a),1\}}, \quad \text{and} \quad \hat{\PP}_h(s'|s,a) = \frac{N_h(s,a,s')}{\max\{ N_h(s,a),1\}}, \label{eq:def-est}
\end{align}
where we define 
\begin{align*}
	N_h(s,a) := \sum_{\tau=1}^\ell \mathbf{1}_{\{(s,a) = (s_h^\tau,a_h^\tau)\}}, \quad \text{and} \quad N_h(s,a,s') := \sum_{\tau=1}^\ell \mathbf{1}_{\{(s,a,s') = (s_h^\tau,a_h^\tau,s_{h+1}^\tau)\}}.
\end{align*}

In what follows, we will show the estimation uncertainties based on the concentration inequalities, which correspond to the bonus terms that we are using in this work. The analysis of the concentration for the reward and transition estimation has been studied in many prior works on tabular MDPs (e.g., \citet{strehl2008analysis,jaksch2010near,azar2017minimax,qiu2020upper,qiu2021provably}). For completeness, we present the proof in this section.

\begin{lemma} \label{lem:r_bound} Given a dataset $\{s_h^\tau, a_h^\tau, r_{i,h}^\tau(s_h^\tau, a_h^\tau)\}_{\tau=1}^\ell$ for an $\ell\geq 1$, for all $h \in [H]$ and all $(s,a) \in \cS \times \cA$, with probability at least $1-\delta$, we have
	\begin{align*}
		\big|\hat{r}_{i,h} (s, a) - r_{i,h} (s, a)\big| \leq \sqrt{\frac{2
				\log (2|\cS||\cA|  H/\delta)}{ N_h(s,a)\vee 1}},
	\end{align*}
	where $\hat{r}_{i,h} (s, a)$ is defined as in \eqref{eq:def-est}, $r_{i,h}(s,a)$ is the true reward function, and $N_h(s,a) := \sum_{\tau=1}^\ell \mathbf{1}_{\{(s,a) = (s_h^\tau,a_h^\tau)\}}$. 
\end{lemma}
\begin{proof} The proof for this lemma relies on an application of Hoeffding's inequality. The definition of $\hat{r}_{i,h}$ shows that $\hat{r}_{i,h}(s,a)$ is the average of $N_h(s,a)$ samples of the observed rewards at $(s,a)$ when $N^{t-1}_h(s,a) > 0$. Then, for fixed $\ell$, $h\in [H]$, and state-action tuple $(s,a)\in \cS\times \cA$, when $N^{t-1}_h(s,a) > 0$, according to Hoeffding's inequality, with probability at least $1-\delta'$ where $\delta'\in (0, 1]$, we have
	\begin{align*}
		\big|\hat{r}_{i,h} (s, a) - r_{i,h} (s, a)\big| \leq \sqrt{\frac{
				\log (2/\delta')}{ 2N_h(s,a)}}.
	\end{align*}
	When $N_h(s,a) = 0$, we know  $\hat{r}_{i,h}(s,a) = 0$ such that $|\hat{r}_{i,h} (s, a) - r_{i,h} (s, a)| = |r_{i,h} (s, a)|\leq 1$. In addition, we have $ \sqrt{2\log (2/\delta')} \geq 1$. Combining the above results, with probability at least $1-\delta'$, we have
	\begin{align*}
		\big|\hat{r}_{i,h} (s, a) - r_{i,h} (s, a)\big| \leq \sqrt{\frac{
				2\log (2/\delta')}{ N_h(s,a)\vee 1}}.
	\end{align*}
	Moreover, by the union bound, letting $\delta = |\cS| |\cA| H \delta'$, with probability at least $1-\delta$ where $\delta\in (0,1]$, for all $h\in [H]$ and all state-action tuple $(s,a)\in \cS\times \cA$, we have
	\begin{align*}
		\big|\hat{r}_{i,h} (s, a) - r_{i,h} (s, a)\big| \leq \sqrt{\frac{2
				\log (2|\cS| |\cA| H /\delta)}{ \max \{ N_h(s,a), 1\} }}.
	\end{align*}
	This completes the proof.
\end{proof}

\begin{lemma} \label{lem:P_bound} Given a dataset $\{s_h^\tau, a_h^\tau, r_{i,h}^\tau(s_h^\tau, a_h^\tau)\}_{\tau=1}^\ell$ for an $\ell >1$,  for all $h \in [H]$ and all $(s,a) \in \cS \times \cA$, with probability at least $1-\delta$, we have
	\begin{align*}
		\left\|\hat{\PP}_h (\cdot \given s, a) - \PP_h (\cdot \given s, a)\right\|_1 \leq \sqrt{\frac{2|\cS| \log (2|\cS||\cA|H/\delta)}{ N_h(s,a)\vee 1}},
	\end{align*}
	where $\hat{\PP}_h(s'|s,a)$ is defined as in \eqref{eq:def-est}, $\PP_h(s'|s,a)$ is the true transition model, and $N_h(s,a) := \sum_{\tau=1}^\ell \mathbf{1}_{\{(s,a) = (s_h^\tau,a_h^\tau)\}}$. 
\end{lemma}
\begin{proof} By the duality, we have $\|\hat{\PP}_h (\cdot \given s, a) - \PP_h (\cdot \given s, a)\|_1 = \sup_{\|\bz \|_\infty \leq  1}  ~\langle \hat{\PP}_h (\cdot \given s, a) - \PP_h (\cdot \given s, a), \bz\rangle_\cS$, where $\langle\cdot,\cdot \rangle_\cS$ denotes the inner product over the state space $\cS$. To bound this term, we construct an $\epsilon$-cover, denoted as $\cC_\infty(\epsilon)$, for the set $\{\bz\in \RR^{|\cS|}: \|\bz\|_\infty \leq 1\}$ w.r.t. the $\ell_\infty$ distance such that for any $\bz \in \RR^{|\cS|}$, there always exists $\bz'\in \cC_\infty(\epsilon)$ satisfying $\|\bz-\bz' \|_\infty \leq \epsilon$. The covering number is $\cN_\infty(\epsilon)=|\cC_\infty(\epsilon)|= 1/\epsilon^{|\cS|}$. Thus, we have for any $\bz$ with $\|\bz\|_\infty \leq 1$, there exists $\bz'\in \cC_\infty(\epsilon)$ such that $\|\bz'-\bz\|_\infty \leq \epsilon$ and 
	\begin{align*}
		&\big\langle \hat{\PP}_h (\cdot \given s, a) - \PP_h (\cdot \given s, a), \bz\big\rangle_\cS  \\
		&\qquad= \big\langle \hat{\PP}_h (\cdot \given s, a) - \PP_h (\cdot \given s, a), \bz' \big\rangle_\cS + \big\langle \hat{\PP}_h (\cdot \given s, a) - \PP_h (\cdot \given s, a), \bz- \bz'\big\rangle_\cS \\
		&\qquad\leq  \big\langle \hat{\PP}_h (\cdot \given s, a) - \PP_h (\cdot \given s, a), \bz' \big\rangle_\cS +  \|\bz-\bz'\|_\infty \left\|\hat{\PP}_h (\cdot \given s, a) - \PP_h (\cdot \given s, a)\right\|_1\\
		&\qquad\leq  \big\langle \hat{\PP}_h (\cdot \given s, a) - \PP_h (\cdot \given s, a), \bz' \big\rangle_\cS +  \epsilon \left\|\hat{\PP}_h (\cdot \given s, a) - \PP_h (\cdot \given s, a)\right\|_1, 
	\end{align*}
	which leads to
	\begin{align}
		\begin{aligned} \label{eq:net}
			& \left\|\hat{\PP}_h (\cdot \given s, a) - \PP_h (\cdot \given s, a)\right\|_1\\
			&\qquad = \sup_{\|\bz \|_\infty \leq  1}  ~\big\langle \hat{\PP}_h (\cdot \given s, a) - \PP_h (\cdot \given s, a), \bz\big\rangle_\cS \\
			&\qquad \leq  \sup_{\bz' \in \cC_\infty(\epsilon)}  ~\big\langle \hat{\PP}_h (\cdot \given s, a) - \PP_h (\cdot \given s, a), \bz' \big\rangle_\cS +  \epsilon \left\|\hat{\PP}_h (\cdot \given s, a) - \PP_h (\cdot \given s, a)\right\|_1.
		\end{aligned}
	\end{align}
	By Hoeffding's inequality and the union bound over all $\bz' \in \cC_\infty(\epsilon)$, when $N_h(s,a)>0$, with probability at least $1-\delta'$ where $\delta'\in (0, 1]$,
	\begin{align} \label{eq:net_bound}
		\max_{\bz' \in \cC_\infty(\epsilon)}  ~\big\langle \hat{\PP}_h (\cdot \given s, a) - \PP_h (\cdot \given s, a), \bz' \big\rangle_\cS \leq  	\sqrt{\frac{|\cS| \log (1/\epsilon) + \log (1/\delta')}{2N_h(s,a)}}.
	\end{align}
	Letting $\epsilon = 1/2$, by \eqref{eq:net} and \eqref{eq:net_bound}, with probability at least $1-\delta'$, we have
	\begin{align*}
		\left\|\hat{\PP}_h (\cdot \given s, a) - \PP_h (\cdot \given s, a)\right\|_1 \leq 2 \sqrt{\frac{|\cS| \log 2 + \log (1/\delta')}{2N_h(s,a)}}.
	\end{align*}
	When $N_h(s,a)=0$, we have $\big\|\hat{\PP}_h (\cdot \given s, a) - \PP_h (\cdot \given s, a)\big\|_1 = \|\PP_h (\cdot \given s, a)\|_1 = 1$. Moreover, $1 < 2 \sqrt{\frac{|\cS| \log 2 + \log (1/\delta')}{2}}$ holds. Thus, with probability at least $1-\delta'$, we always have
	\begin{align*}
		\left\|\hat{\PP}_h (\cdot \given s, a) - \PP_h (\cdot \given s, a)\right\|_1 \leq 2 \sqrt{\frac{|\cS| \log 2 + \log (1/\delta')}{2\max \{N_h(s,a), 1\}}} \leq \sqrt{\frac{2|\cS| \log (2/\delta')}{\max \{N_h(s,a), 1\}}}.
	\end{align*}
	Then, by the union bound, letting $\delta = |\cS| |\cA| H \delta'$, with probability at least $1-\delta$, for all $(s, a)\in \cS\times \cA$ and all $h\in [H]$, we have
	\begin{align*}
		\left\|\hat{\PP}_h (\cdot \given s, a) - \PP_h (\cdot \given s, a)\right\|_1 \leq \sqrt{\frac{2|\cS| \log (2|\cS| |\cA| H /\delta)}{\max \{N_h(s,a), 1\}}},
	\end{align*}
	which completes the proof.
\end{proof}

Next, we present a lemma that is commonly used in the proof of mirror descent algorithm.
\begin{lemma}\label{lem:pushback}  Let $f: \Delta_n \mapsto \RR$ be a convex function, where $\Delta_n$ is the probability simplex in $\RR^n$. For any $\alpha \geq 0$,  $\bz \in \Delta_n$,  and $\by \in \Delta_n^o$ where $\Delta_n^o \subset \Delta_n$ is the relative interior of $\Delta_n^o$, supposing $\bx^{\mathrm{opt}} = \argmin_{\bx \in \Delta_n} f(\bx) + \beta D_{\mathrm{KL}} (\bx, \by)$, then the following inequality holds
	\begin{align*}
		f(\bx^{\mathrm{opt}}) + \beta D_{\mathrm{KL}} (\bx^{\mathrm{opt}}, \by) \leq f(\bz) + \beta D_{\mathrm{KL}} (\bz, \by) - \beta D_{\mathrm{KL}} (\bz, \bx^{\mathrm{opt}}),
	\end{align*}
	where $D_{\mathrm{KL}}$ denotes the regular KL divergence, $D_{\mathrm{KL}}(\bx,\by):=\sum_{i=1}^n x_i \log\frac{x_i}{y_i}$.
\end{lemma}
This proof can be obtained by slight modification from prior works \citep{tseng2008accelerated,nemirovski2009robust,wei2019online}.

Next, we present a lemma that characterizes the difference between $\bx\in\Delta_n$ and the vector $\tilde \bx \in\Delta_n$ that is a convex combination of $\bx$ and an all-one vector.
\begin{lemma}[\citet{wei2019online,qiu2023gradient}]\label{lem:mix} Suppose that $\bx\in\Delta_n$ and $\tilde x_i = (1-\alpha) x_i + \alpha/n$
	\begin{align*}
		&D_{\mathrm{KL}}(\by,\tilde \bx) - D_{\mathrm{KL}}(\by,\bx) \leq \alpha \log n, ~~\forall \by \in\Delta_n,\\
		&\|\tilde \bx - \bx\|_1 \leq 2\alpha,
	\end{align*}
	where $D_{\mathrm{KL}}$ denotes the regular KL divergence.
\end{lemma}

\end{document}